\documentclass[11pt,a4paper]{article}

\usepackage[a4paper, margin=28mm]{geometry}
\usepackage{libertine}

\usepackage{authblk}
\usepackage{orcidlink}
\usepackage{multicol}
\usepackage{geometry}
\usepackage{graphicx}
\usepackage{makecell}
\usepackage{float}
\usepackage[utf8x]{inputenc}
\usepackage{amsmath,amssymb}
\usepackage{xcolor}
\usepackage{mathtools}
\usepackage{amsthm}
\usepackage{latexsym}
\usepackage[mathcal]{eucal}
\usepackage{amscd}
\usepackage{graphicx}
\usepackage{amsfonts}
\usepackage{MnSymbol}
\usepackage{setspace}
\usepackage[english]{babel}
\usepackage{listings}
\usepackage{calligra}
\usepackage{titlesec}
\usepackage{tabulary}
\usepackage{tabularx}
\usepackage[ruled,vlined]{algorithm2e}
\usepackage{algorithmic}
\usepackage{multirow}
\usepackage{caption}
\usepackage{subcaption}
\usepackage{booktabs} 
\usepackage{url}
\usepackage{colortbl} 
\usepackage{xcolor}
\usepackage{orcidlink}
\usepackage{hyperref}

\usepackage{longtable} 
\usepackage{lscape} 
\usepackage{wrapfig}    

\usepackage[T1]{fontenc} 

\usepackage{tikz}

\newtheorem{definition}{Definition}[section]

\newtheorem{theorem}{Theorem}[section]

\theoremstyle{remark}

\definecolor{cream}{RGB}{255,253,208}

\definecolor{verylightcream}{HTML}{FFFEF2}

\titleformat{\section}
{\normalfont\Large\bfseries}{\thesection}{1em}{}

\titleformat{\section}
  {\bfseries\Large}
  {\thesection.}
  {0.75em}
  {}

\titleformat{\subsection}
  {\bfseries\large}
  {\thesubsection.}
  {0.5em}
  {}

\titleformat{\subsubsection}
  {\bfseries\normalsize}
  {\thesubsubsection.}
  {0.3em}
  {}

\title{\huge\textbf{Dynamics and Computational Principles of Echo State Networks: A Mathematical Perspective}
}

\author[
]{%
    Pradeep Singh\,\orcidlink{0000-0002-5372-3355}\thanks{\texttt{pradeep.cs@sric.iitr.ac.in}}, 
     Ashutosh Kumar\,\orcidlink{0009-0005-8975-5032}\thanks{\texttt{a\_kumar1@cs.iitr.ac.in}},
     Sutirtha Ghosh\,\orcidlink{0009-0004-9780-0545}\thanks{\texttt{s\_ghosh@cs.iitr.ac.in}}\\
    Hrishit B P\,\orcidlink{0009-0006-7314-9553}\thanks{\texttt{h\textunderscore bp@cs.iitr.ac.in}} ,
    Balasubramanian Raman\,\orcidlink{0000-0001-6277-6267}\thanks{\texttt{bala@cs.iitr.ac.in}}
}

\affil[]{%
    \small Machine Intelligence Lab\\ 
    Department of Computer Science and Engineering\\
    
    Indian Institute of Technology Roorkee\\
    Roorkee-247667, India
}

\date{}

\begin{document}
\maketitle


\begin{abstract}
\noindent Reservoir computing (RC) represents a class of state-space models (SSMs) characterized by a fixed state transition mechanism (the reservoir) and a flexible readout layer that maps from the state space. It is a paradigm of computational dynamical systems that harnesses the transient dynamics of high-dimensional state spaces for efficient processing of temporal data. Rooted in concepts from recurrent neural networks, RC achieves exceptional computational power by decoupling the training of the dynamic reservoir from the linear readout layer, thereby circumventing the complexities of gradient-based optimization. This work presents a  systematic  exploration of RC, addressing its foundational properties such as the echo state property, fading memory, and reservoir capacity through the lens of dynamical systems theory. We formalize the interplay between input signals and reservoir states, demonstrating the conditions under which reservoirs exhibit stability and expressive power. Further, we delve into the computational trade-offs and robustness characteristics of RC architectures, extending the discussion to their applications in signal processing, time-series prediction, and control systems. The analysis is complemented by theoretical insights into optimization, training methodologies, and scalability, highlighting open challenges and potential directions for advancing the theoretical underpinnings of RC.

\end{abstract}

\tableofcontents

\section{Introduction}
Reservoir computing (RC) is a paradigm for efficient recurrent neural network (RNN) training that emerged in the early 2000s, offering a brain-inspired alternative to traditional backpropagation-through-time.
 It leverages the natural ability of a high-dimensional dynamical system, called the \textit{reservoir}, to project input signals into a rich feature space. The reservoir's internal dynamics encode a temporal representation of the input data, which can be used by a linear readout layer to solve a wide variety of tasks. This idea was independently introduced by Jaeger in 2001 as the Echo State Network (ESN) \cite{jaeger2001echo} 
 and by Maass et al. in 2002 as the Liquid State Machine (LSM) \cite{maass2002real}. Both approaches achieved surprisingly accurate results on complex temporal tasks (e.g. chaotic time series prediction and speech processing) without training the recurrent weights \cite{yan2024emerging}. The term “reservoir computing” was later coined to unify these concepts by Verstraeten et al. (2007) \cite{Verstraeten2007}.
Both models exploit the inherent properties of the reservoir to achieve computational efficiency and simplicity, avoiding the complexities associated with training traditional RNNs.

\begin{figure}[!ht]
    \centering
    \includegraphics[width=\linewidth]{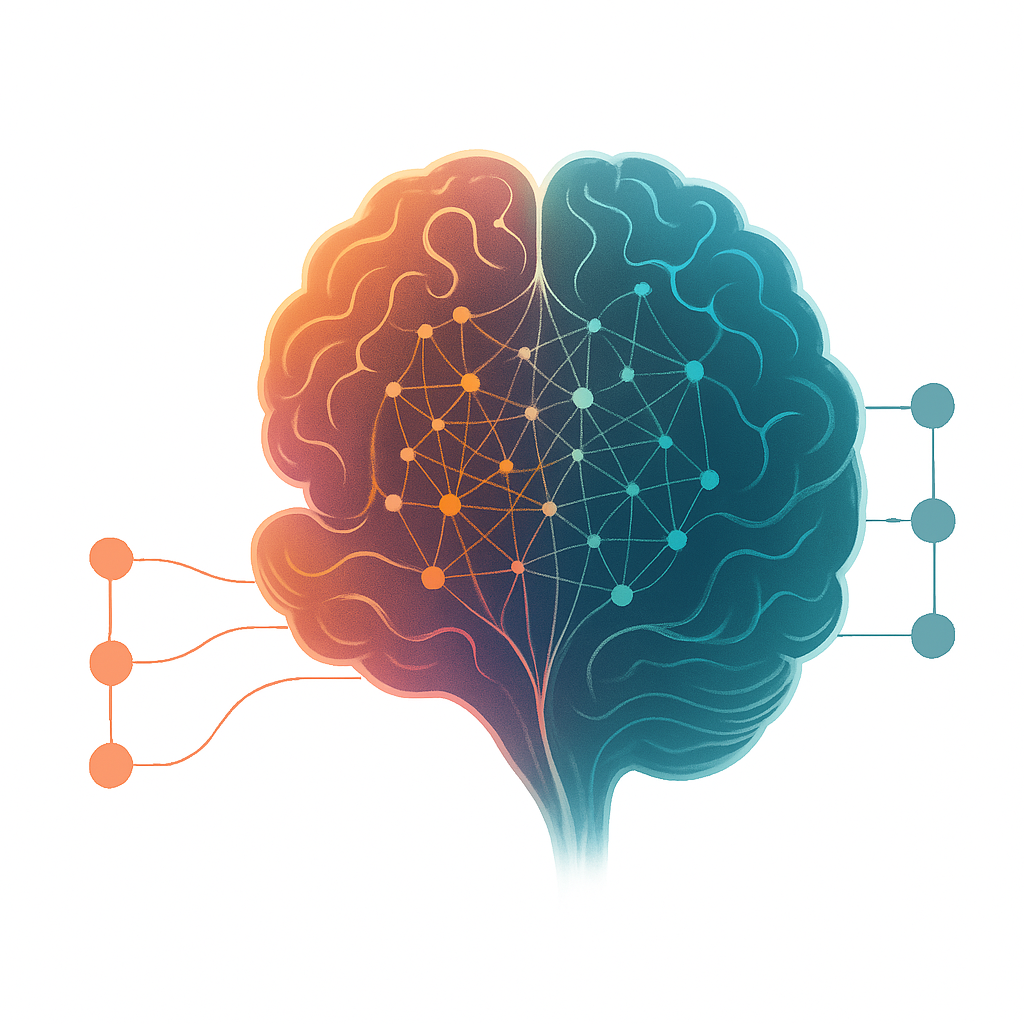}
    \caption{A brain-inspired Echo State Network, showing input nodes (left), a recurrent nonlinear reservoir (center), and output nodes (right). The architecture reflects cortical microcircuitry, enabling high-dimensional transient dynamics for temporal processing.}
    \label{fig:brain_esn}
\end{figure}

The motivation for this paper lies in the theoretical richness of reservoir computing and its wide range of applications, from time-series prediction \cite{lukovsevivcius2009reservoir} to robotics and neuroscience \cite{tanaka2019recent}. Jaeger et al. demonstrated its capability to predict chaotic systems and optimize energy usage in communication networks \cite{jaeger2004harnessing}. Similarly, Maass et al. highlighted its potential for real-time computation by exploring its neural basis \cite{maass2002real}. Over the years, RC has been successfully applied to time-series prediction \cite{lukovsevivcius2009reservoir}, signal processing \cite{Verstraeten2007}, and even physical systems modeling \cite{tanaka2019recent}. While the practical benefits of RC have been extensively documented, there remains a gap in the literature concerning the  mathematical foundations underpinning its success. These applications illustrate the versatility of RC, yet its theoretical principles—such as the Echo State Property (ESP) \cite{yildiz2012re} and fading memory \cite{schrauwen2007overview}—demand deeper exploration to fully understand the conditions under which RC systems perform optimally. For instance, the conditions that ensure the ESP or the trade-offs between memory capacity and computational power are not fully resolved \cite{gallicchio2017deep, lukovsevivcius2009reservoir}. Additionally, the interplay between the reservoir's dynamical properties and task-specific requirements is an area of active research.

The objective of this paper is threefold: first, to establish a  mathematical framework for analyzing reservoir dynamics; second, to explore the key properties and capacities of RC models—such as memory and computational power—from a theoretical perspective; and third, to review the advances and open problems in the field. In pursuit of these goals, we also organize and benchmark a comprehensive set of baseline reservoir computing methods, bringing together popular variants (including cycle reservoirs, random orthogonal ESNs, and minimum-complexity interaction ESNs) in a unified experimental platform.

This paper is structured as follows: we begin by tracing the historical evolution of reservoir computing in \S~\ref{sec:history}, followed by a detailed exposition of its mathematical foundations, including dynamical systems theory and the formal modeling of reservoirs. We then discuss the architectural variations and properties of RC models in \S~\ref{sec:archi}, leading to an analysis of their optimization in \S~\ref{sec:analysis}. Next, we present training methodologies alongside a comprehensive set of baseline experiments in \S~\ref{sec:training} to compare and benchmark key reservoir variants. Applications of RC are explored in \S~\ref{sec:applications}, following which we address challenges and open questions—including potential future directions—and conclude in \S~\ref{sec:conclusion}.

\section{Historical Context, Evolution and Background}\label{sec:history}

The key insight behind reservoir computing is that a randomly connected nonlinear dynamical system can serve as a rich “computational substrate,” whose driven responses to inputs can be linearly combined to perform useful computations \cite{yan2024emerging}. Early precursors of this idea appeared in the 1990s in neural modeling studies. For example, Dominey et al. described cortico-striatal circuits where fixed recurrent networks provided temporal context for learning sequence tasks in language \cite{dominey1995corticostriatal}. However, RC as a unified framework began with two landmark papers: Jaeger's 2001 technical report introducing the Echo State Network \cite{jaeger2001echo}, and Maass et al.’s 2002 paper introducing the Liquid State Machine \cite{maass2002real}. Jaeger demonstrated that a large random RNN, if it satisfies the Echo State Property, can be used to model dynamical systems by training only the output weights. Around the same time, Maass and colleagues, motivated by cortical microcircuits, showed that a network of random spiking neurons (an LSM) could serve as a “liquid filter” whose instantaneous spiking activity (the liquid state) could be read out to perform complex temporal computations \cite{yan2024emerging}. 
Both approaches share the fundamental idea of using a fixed, high-dimensional dynamical system—the reservoir—to transform input signals into a space where they can be linearly separable. This separation allows the downstream task to be solved efficiently with a simple linear readout layer, bypassing the complexities of training the recurrent connections within the network.  These works revealed that RNNs often “work well enough” even without training all weights– a radical departure from fully-trained recurrent networks. In the following years, researchers unified the theory, showing that any fading-memory dynamical system can in principle be approximated by a reservoir with a suitable readout, giving RC a firm theoretical footing \cite{lukovsevivcius2009reservoir}. By 2007, the term reservoir computing was adopted to encompass ESNs, LSMs, and similar approaches.

ESNs introduced the concept of the ``echo state property”, which ensures that the influence of the initial state of the reservoir fades over time, thereby stabilizing the reservoir dynamics. Jaeger demonstrated that by initializing the recurrent connections of the reservoir randomly and carefully scaling the spectral radius of the weight matrix, the system could exhibit rich temporal dynamics while maintaining stability \cite{jaeger2001echo}. This property allowed ESNs to perform tasks such as time-series prediction and chaotic system modeling with remarkable efficiency \cite{jaeger2004harnessing}. Concurrently, LSMs were proposed as biologically inspired models that employed spiking neural networks as reservoirs. LSMs captured the real-time dynamics of spiking activity in the reservoir, emphasizing their application to temporal processing tasks such as speech recognition and motor control \cite{maass2002real}.

The early 2000s saw a surge in interest in RC, fueled by its ability to perform computations using the inherent dynamics of the reservoir. Researchers began to systematically investigate the mathematical properties of reservoirs, including their memory capacity, expressiveness, and generalization abilities \cite{lukovsevivcius2009reservoir}. The work by Lukoševičius and Jaeger highlighted the computational efficiency of RC and its practical advantages over fully trainable RNNs. This period also marked the development of theoretical tools to better understand the ESP and its implications for stability and learning \cite{yildiz2012re}.
Over the years, RC has undergone significant evolution. The experimental unification of RC methods by Verstraeten et al. \cite{Verstraeten2007} provided a comparative analysis of various reservoir designs, establishing benchmarks for performance across diverse tasks. Advances in physical reservoir computing have further expanded the scope of RC by implementing reservoirs in hardware systems such as optical and memristive devices \cite{tanaka2019recent}. These developments have not only enhanced the computational efficiency of RC but have also opened new avenues for energy-efficient computing.
Key milestones in RC include the formalization of the fading memory property, the introduction of hybrid reservoir architectures, and the development of deep reservoir computing models. The fading memory property, analyzed by Schrauwen et al. \cite{schrauwen2007overview}, established a mathematical framework for understanding how reservoirs retain and process temporal information. Hybrid architectures that combine reservoirs with feedforward and convolutional layers have been proposed to enhance feature extraction and scalability \cite{gallicchio2017deep}. Moreover, deep reservoir computing has emerged as a promising approach to leveraging the hierarchical representation capabilities of deep networks within the RC framework.

Deep reservoir computing extends the principles of traditional reservoir computing into a hierarchical framework, enabling the extraction of progressively abstract and expressive features from input data. By stacking multiple reservoir layers, deep reservoir architectures aim to leverage the compositional nature of representations, similar to deep neural networks, while retaining the computational advantages of untrained reservoirs. Each layer in a deep reservoir receives inputs from the preceding layer, effectively forming a multi-scale temporal representation of the data. Gallicchio et al. \cite{gallicchio2017deep} provided a  analysis of deep reservoir computing, demonstrating its ability to model complex temporal dependencies and improve generalization across tasks. The introduction of multiple layers enhances the reservoir's capacity to capture both short- and long-term dependencies by combining the dynamics of individual reservoirs. Despite these advancements, the theoretical understanding of stability and expressiveness in deep reservoirs remains a subject of active research. Challenges such as the preservation of ESP across layers and the computational trade-offs associated with deeper architectures are critical areas for further investigation. The hierarchical structure of deep reservoirs positions them as a powerful tool for tasks requiring complex temporal feature extraction, bridging the gap between shallow reservoir models and fully trainable deep networks.

\subsection{Dynamical Systems Theory} \label{subsec:dynamicalsys}

Dynamical systems theory provides a rigorous framework for analyzing how systems evolve over time under deterministic rules \cite{Hirsch2013,Katok1995,Poincare1890}. Formally, a dynamical system is given by a pair \((\mathcal{M}, \phi_{t})\), where \(\mathcal{M}\) is a manifold representing the possible states of the system, and \(\phi_{t}: \mathcal{M} \to \mathcal{M}\) is a family of evolution maps parametrized by the continuous-time variable \(t \in \mathbb{R}\) or the discrete-time variable \(t \in \mathbb{Z}\). The evolution of a state \(\mathbf{x}_0 \in \mathcal{M}\) over time is described by \(\mathbf{x}(t) = \phi_{t}(\mathbf{x}_0)\). In continuous time, one often expresses this in the form \(\dot{\mathbf{x}}(t) = \mathbf{f}(\mathbf{x}(t))\), while in discrete time one writes \(\mathbf{x}(t+1) = \mathbf{f}(\mathbf{x}(t))\).

Among the simplest classes of dynamical systems are linear time-invariant (LTI) systems. A continuous-time LTI system on \(\mathbb{R}^n\) can be written as:
\begin{equation}
    \dot{\mathbf{x}}(t) = A\,\mathbf{x}(t),
\end{equation}
where \(A \in \mathbb{R}^{n \times n}\) is a constant matrix. 
The solution is given by \(\mathbf{x}(t) = e^{A\, t} \mathbf{x}_0\). In such linear systems, stability is closely tied to the eigenvalues of \(A\).
If all eigenvalues of \(A\) have negative real parts, trajectories converge to the origin; if at least one eigenvalue has a positive real part, trajectories diverge. Linear systems admit closed-form solutions, and their behavior is relatively straightforward to predict and classify using classical tools such as diagonalization and the Jordan canonical form \cite{coddington1955theory, kailath1980linear, zabczyk2008mathematical}. However, most real-world phenomena exhibit nonlinearity, leading to richer and often more complex dynamics. A canonical example is the Lorenz system \cite{lorenz1963deterministic, sparrow1982lorenz, tucker1999rigorous}:
\begin{equation}
\begin{aligned}
    \dot{x} &= \sigma (y - x), \\
    \dot{y} &= x (\rho - z) - y, \\
    \dot{z} &= x y - \beta z,
\end{aligned}
\end{equation}
where parameters \(\sigma\), \(\rho\), and \(\beta\) can induce chaotic behavior characterized by sensitivity to initial conditions and intricate strange attractors as depicted in Figure \ref{fig:lorenz}.
Further examples of chaotic nonlinear systems include the Rössler system \cite{rossler1976equation} (cf. Figure \ref{fig:rossler}), 
\begin{equation}
\begin{aligned}
    \dot{x} &= - (y + z), \\
    \dot{y} &= x + a\,y, \\
    \dot{z} &= b + z(x - c),
\end{aligned}
\end{equation}
the Chen system \cite{chen1999yet} (cf. Figure \ref{fig:chen}),
\begin{equation}
\begin{aligned}
    \dot{x} &= a(y - x), \\
    \dot{y} &= (c+a)x - xz + c\,y, \\
    \dot{z} &= x\,y - b\,z,
\end{aligned}
\end{equation}
and the Chua’s circuit \cite{chua1986chaos} (cf. Figure \ref{fig:chua}), often studied in the form
\begin{equation}
\begin{aligned}
    \dot{x} &= \alpha \bigl( y - x - f(x) \bigr), \\
    \dot{y} &= x - y + z, \\
    \dot{z} &= -\beta\, y,
\end{aligned}
\end{equation}
where \(f(x)\) is typically a piecewise-linear function defining the nonlinear resistor characteristic and is given by:
   
\begin{equation}
      f(x) = m1\cdot x + 0.5\ (m0 - m1)\ (|x+1| - |x-1|)
\end{equation}

Each of these systems can display periodic orbits, quasiperiodic motion, or chaos, depending on their parameter regimes.

\begin{figure}[!ht]
    \centering
    \begin{subfigure}[b]{0.49\textwidth}
        \centering
        \includegraphics[width=0.99\linewidth]{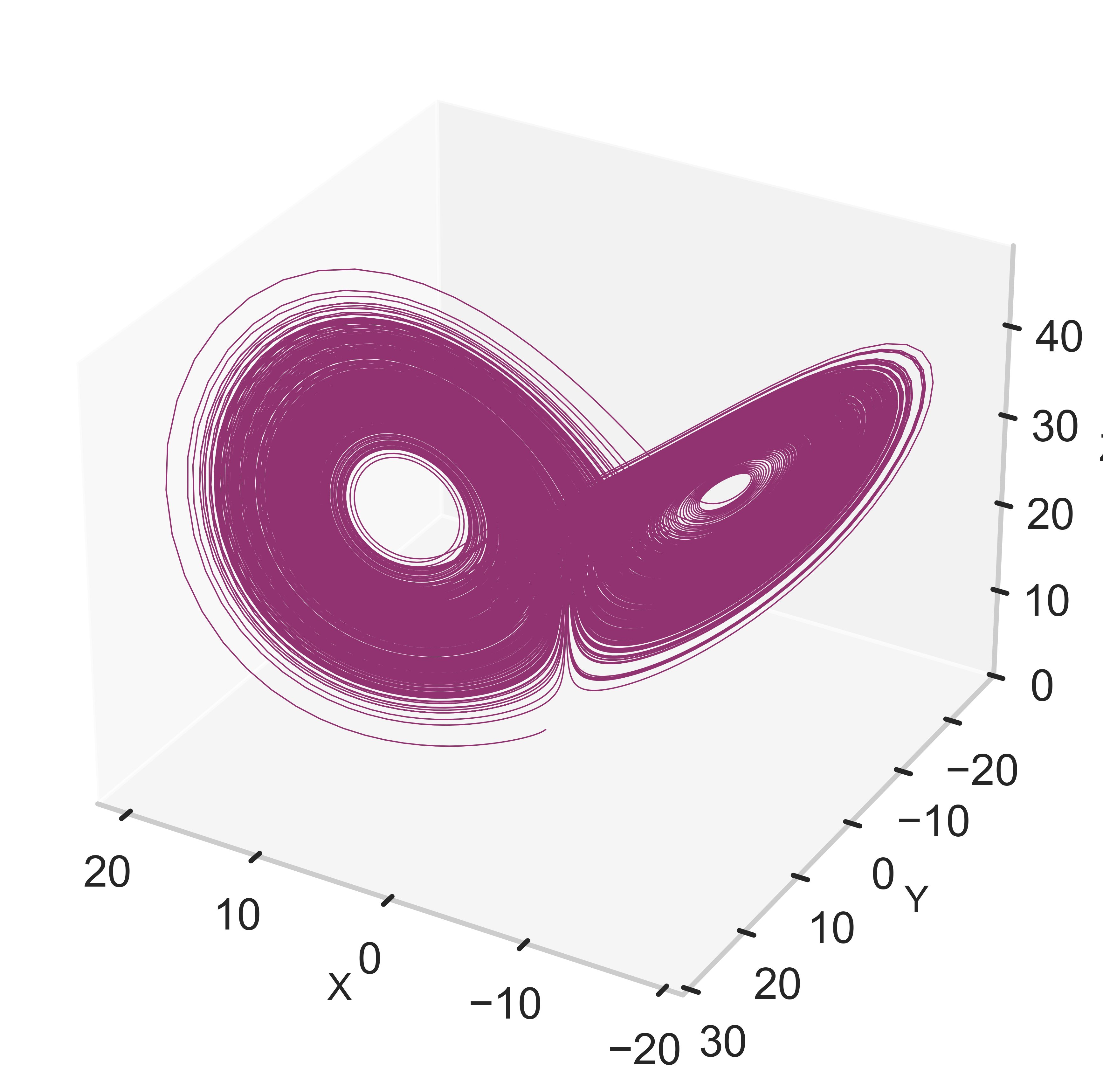}
        \caption{Lorenz attractor with butterfly-like geometry, highlighting sensitivity to initial conditions. Parameters: \(\sigma = 10\), \(\rho = 28\), \(\beta = 8/3\).}
        \label{fig:lorenz}
    \end{subfigure}
    \hfill
    \begin{subfigure}[b]{0.49\textwidth}
        \centering
        \includegraphics[width=0.99\linewidth]{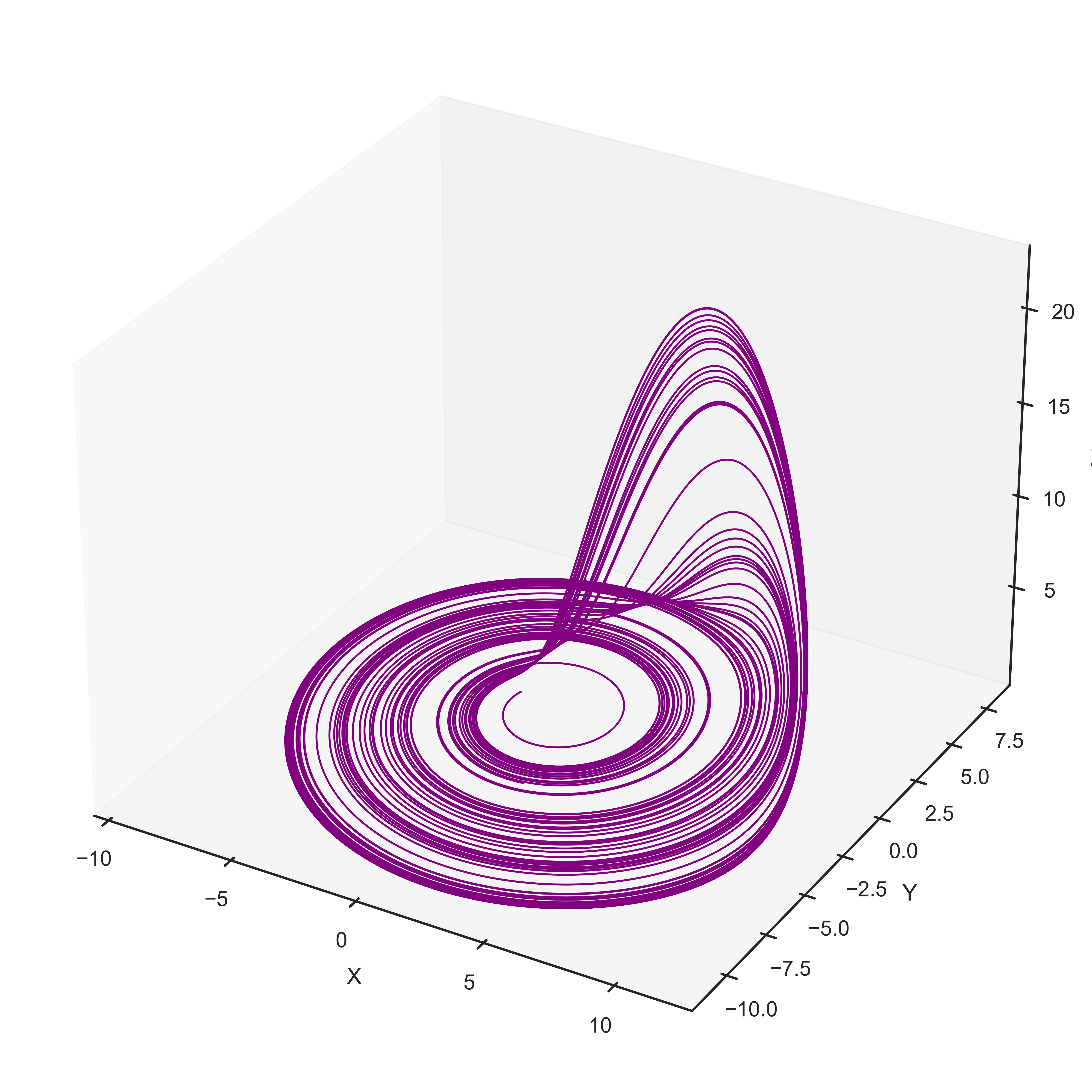}
        \caption{Rössler attractor with spiral dynamics and simpler topology. Parameters: \(a = 0.2\), \(b = 0.2\), \(c = 5.7\).}
        \label{fig:rossler}
    \end{subfigure}

    \vspace{0.5em} 

    \begin{subfigure}[b]{0.49\textwidth}
        \centering
        \includegraphics[width=0.99\linewidth]{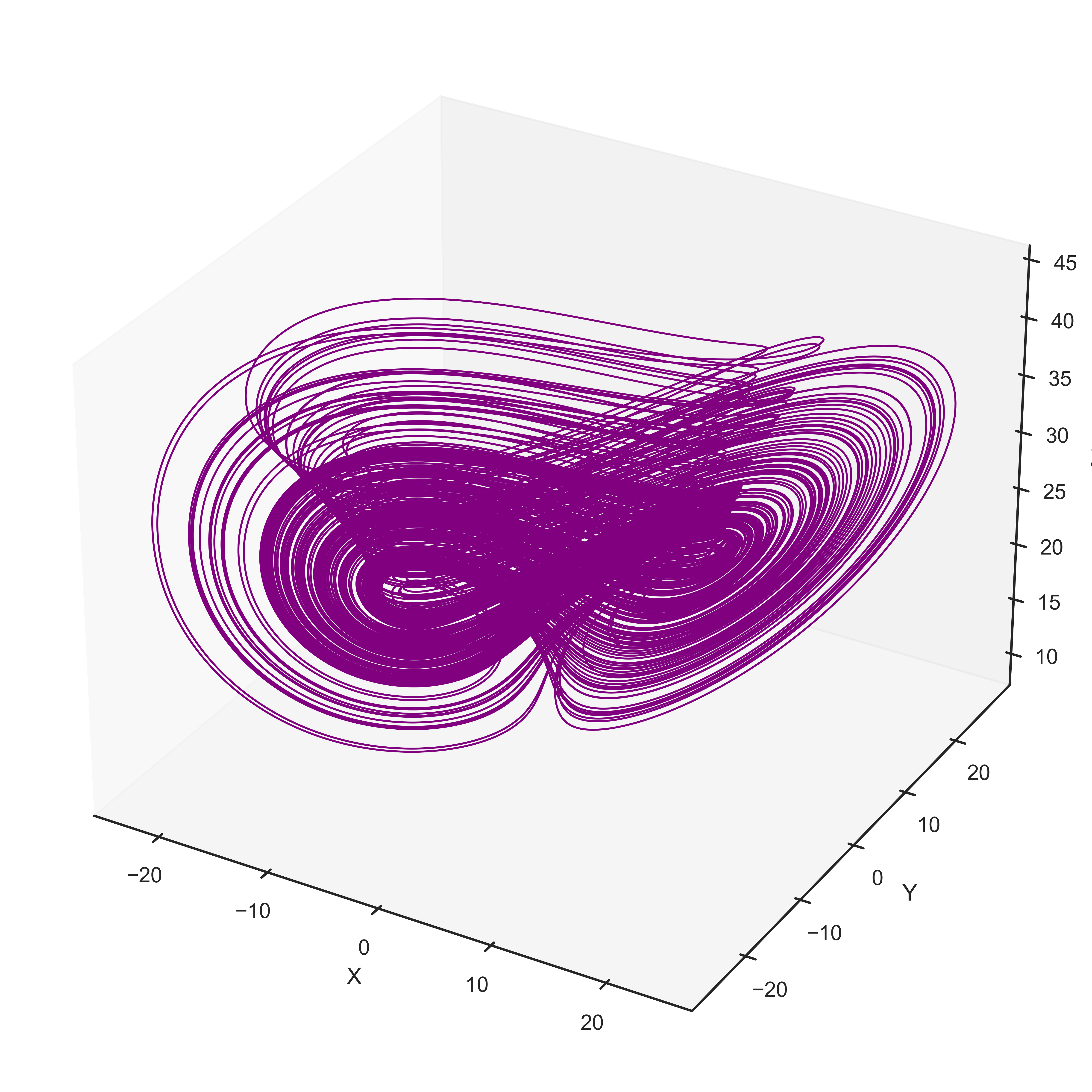}
        \caption{Chen attractor resembles Lorenz but shows distinct structure. Parameters: \(a = 35\), \(b = 3\), \(c = 28\).}
        \label{fig:chen}
    \end{subfigure}
    \hfill
    \begin{subfigure}[b]{0.49\textwidth}
        \centering
        \includegraphics[width=0.99\linewidth]{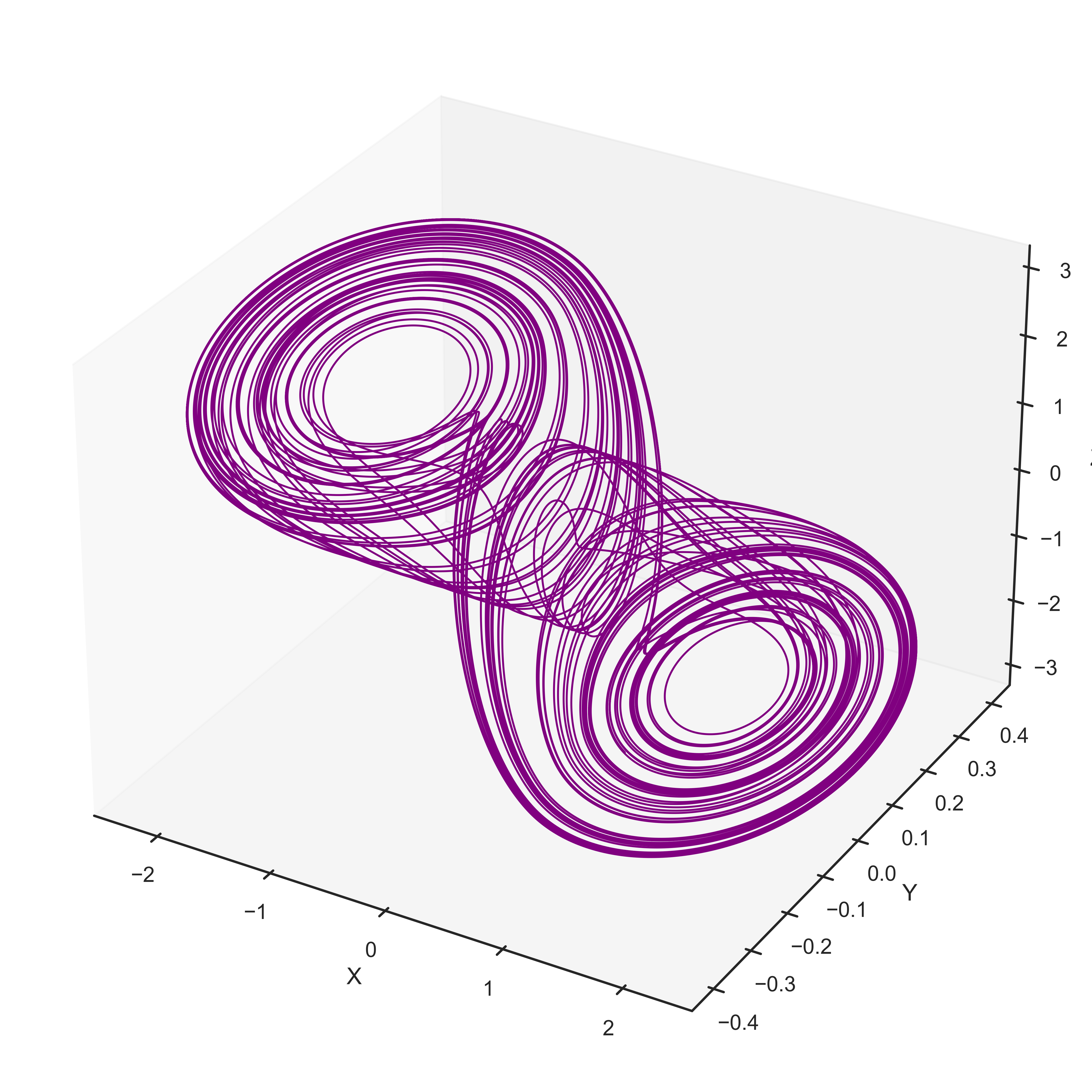}
        \caption{Chua attractor from an electronic circuit, with characteristic double-scroll shape. Parameters: \(\alpha = 15.6\), \(\beta = 28\).}
        \label{fig:chua}
    \end{subfigure}

    \vspace{0.5em} 


    \caption{Comparison of four canonical chaotic attractors—Lorenz, Rössler, Chen, and Chua—illustrating their distinct geometries and dynamics.}
    \label{fig:chaos_grid}
\end{figure}

Time-delay systems constitute yet another important class of nonlinear dynamical models. A classic example is the Mackey–Glass equation \cite{mackey1977oscillation},
\begin{equation}
    \frac{d\,x(t)}{d\,t} \;=\; \beta \,\frac{x\bigl(t-\tau\bigr)}{1 + x\bigl(t-\tau\bigr)^n} \;-\; \gamma \,x(t),
\end{equation}
where \(\tau\) is a delay parameter, and the interplay between \(\tau\), \(\beta\), \(\gamma\), and \(n\) can yield complex oscillatory and chaotic dynamics. Unlike finite-dimensional ODEs, delay-differential equations are inherently infinite-dimensional because their state at time \(t\) depends on the function \(x\) over an interval \([t-\tau, \, t]\).

In each of these nonlinear examples, small perturbations in initial conditions can lead to vastly different trajectories, a property commonly referred to as \emph{sensitive dependence on initial conditions}, which underlies chaotic behavior. Studying these systems involves classical tools such as Lyapunov exponents, bifurcation analysis, and Poincaré maps, which reveal transitions between stable equilibria, limit cycles, and chaotic attractors \cite{guckenheimer2013nonlinear, strogatz2018nonlinear}. These advanced methods serve to understand how orbits in the state space \(\mathcal{M}\) can cluster, diverge, or settle into attractors under parameter changes or external forcing.

In \emph{Figure \ref{fig:ftle_lorenz_xz}}, we examine the \emph{finite-time Lyapunov exponent} (FTLE) field of the Lorenz system on the two-dimensional slice \(\bigl\{(x,0,z)\,\mid\,x\in[-20,20],\,z\in[0,50]\bigr\}\).  For each point \(\mathbf{u}_{0}=(x_{0},\,0,\,z_{0})\), we introduce a small perturbation \(\delta\) along the \(x\)-axis, integrate both initial conditions forward over a short horizon \(T=2\), then measure
\begin{equation}
\text{FTLE}(\mathbf{u}_{0})
\;=\;
\frac{1}{T}\,\ln \Bigl(\tfrac{\lVert \mathbf{u}(T)\,-\,\mathbf{u}^{\prime}(T)\rVert}
{\lVert \mathbf{u}(0)\,-\,\mathbf{u}^{\prime}(0)\rVert}\Bigr),
\end{equation}
where \(\mathbf{u}(t)\) and \(\mathbf{u}^{\prime}(t)\) satisfy the Lorenz equations with initial conditions \(\mathbf{u}_{0}\) and \(\mathbf{u}_{0}+(\delta,0,0)\), respectively.  A positive FTLE value signals exponential divergence of nearby trajectories (local instability), whereas negative or smaller positive values indicate weaker expansion.  Such finite-time exponents offer a more nuanced view of chaos than classical Lyapunov exponents, as they capture transient expansions or contractions over a user-defined integration window.  In the broader context of echo state networks, analyzing FTLE fields of chaotic benchmarks like the Lorenz system provides insight into how localized perturbations may be amplified or damped, which is directly related to how reservoirs handle small signal variations and maintain echoes of their input in transient regimes.

\begin{figure}[!ht]
    \centering
    \includegraphics[width=0.98\linewidth]{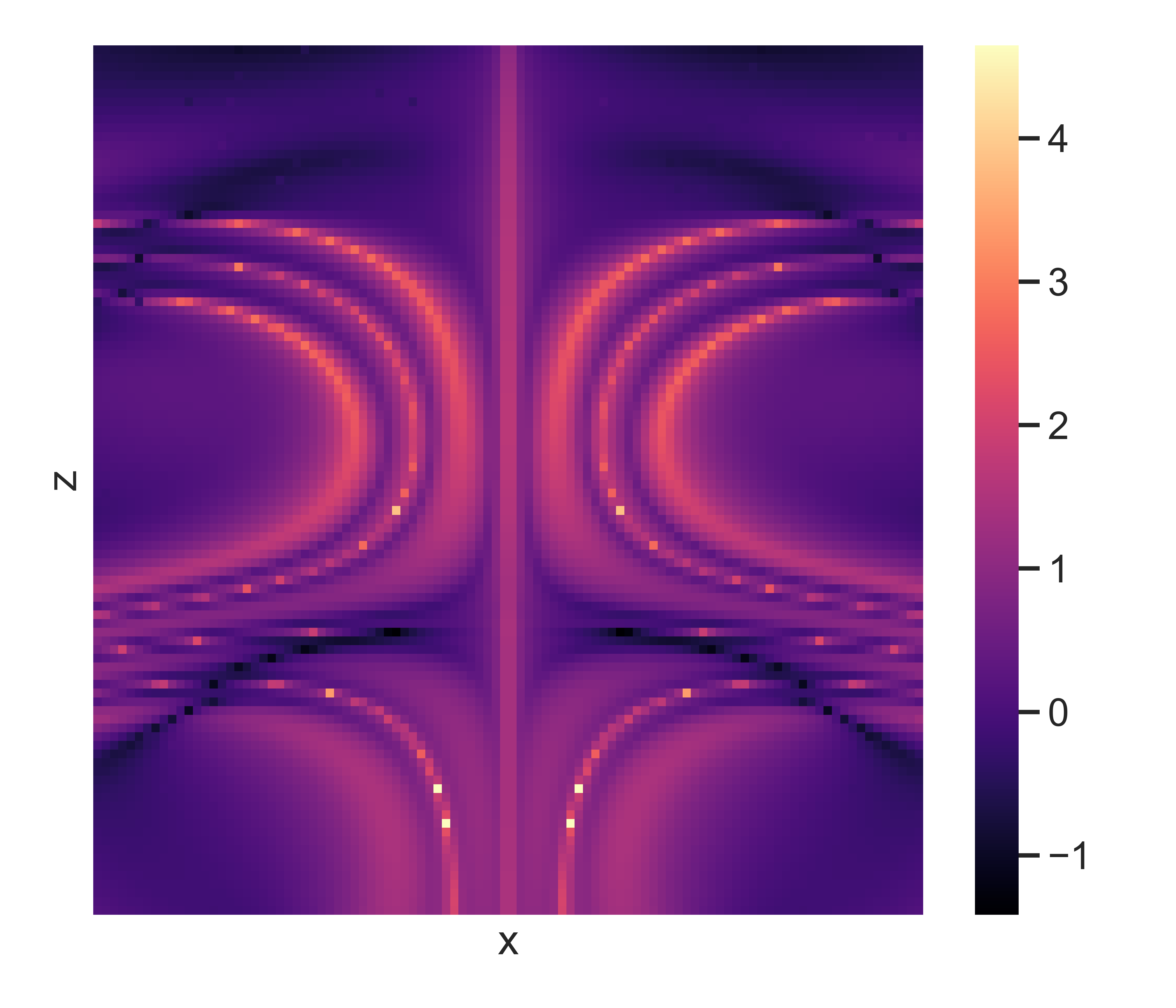}
    \caption{Finite-time Lyapunov exponent field of the Lorenz system, showing how a small perturbation in the \(x\)-direction evolves over a short horizon (\(T=2\)) from each \((x,z)\) initial condition with \(y=0\).  Warmer colors indicate stronger local divergence rates.
}
    \label{fig:ftle_lorenz_xz}
\end{figure}

\paragraph{Bifurcation Analysis.}
Bifurcation analysis is concerned with how the qualitative structure of a dynamical system’s trajectories changes as one or more parameters vary. Consider a continuous-time system of the form: 
\begin{equation}
  \dot{\mathbf{x}} = \mathbf{f}(\mathbf{x}, \boldsymbol{\mu}),
\end{equation}
where \(\mathbf{x} \in \mathbb{R}^n\) and \(\boldsymbol{\mu} \in \mathbb{R}^k\) is a parameter vector. A \emph{bifurcation} occurs when a small, continuous variation in \(\boldsymbol{\mu}\) induces a sudden, qualitative change in the set of equilibria or periodic orbits and their stability \cite{guckenheimer2013nonlinear, strogatz2018nonlinear, kuznetsov2013elements}.
A common starting point is to locate equilibrium (or fixed) points by solving: 
\begin{equation}
  \mathbf{f}(\mathbf{x}^*, \boldsymbol{\mu}) = \mathbf{0}.
\end{equation}
To assess local stability, one examines the Jacobian matrix evaluated at \(\mathbf{x}^*\):
\begin{equation}
  D_{\mathbf{x}} \mathbf{f}(\mathbf{x}^*, \boldsymbol{\mu}) \;=\; 
  \begin{pmatrix}
    \frac{\partial f_1}{\partial x_1} & \cdots & \frac{\partial f_1}{\partial x_n}\\
    \vdots & \ddots & \vdots \\
    \frac{\partial f_n}{\partial x_1} & \cdots & \frac{\partial f_n}{\partial x_n}
  \end{pmatrix}_{\mathbf{x} = \mathbf{x}^*}.
\end{equation}

If all eigenvalues of this Jacobian have negative real parts, the equilibrium \(\mathbf{x}^*\) is asymptotically stable; if at least one eigenvalue has a positive real part, \(\mathbf{x}^*\) is unstable. A \emph{bifurcation point} \(\boldsymbol{\mu}^*\) arises when one or more eigenvalues cross the imaginary axis in the complex plane.  

Different classes of bifurcations can occur depending on how eigenvalues and invariant sets (e.g., limit cycles) emerge or collide. For instance, a \emph{pitchfork bifurcation} often appears in symmetric systems, where a single stable equilibrium either splits into multiple equilibria (supercritical pitchfork) or merges into one (subcritical pitchfork). A \emph{Hopf bifurcation} arises when a complex conjugate pair of eigenvalues passes through the imaginary axis, giving rise to stable or unstable limit cycles \cite{marsden2013introduction, kuznetsov2013elements}. Formally, in a Hopf bifurcation at \(\boldsymbol{\mu} = \boldsymbol{\mu}^*\), one has a pair of eigenvalues of 
\(\displaystyle D_{\mathbf{x}} \mathbf{f}(\mathbf{x}^*, \boldsymbol{\mu}^*)\)
of the form \(\alpha(\boldsymbol{\mu}) \pm i\,\omega(\boldsymbol{\mu})\) that satisfies:
\begin{equation}
  \alpha(\boldsymbol{\mu}^*) = 0, 
  \quad 
  \omega(\boldsymbol{\mu}^*) \neq 0, 
\end{equation}
and \(\alpha(\boldsymbol{\mu})\) changes sign as \(\boldsymbol{\mu}\) crosses \(\boldsymbol{\mu}^*\). The birth of a small-amplitude periodic orbit (limit cycle) is then guaranteed near \(\boldsymbol{\mu}^*\), with the direction of the bifurcation (supercritical or subcritical) determining whether the cycle is stable or unstable.
\begin{figure}[!ht]
    \centering
   \includegraphics[width=0.99\linewidth]{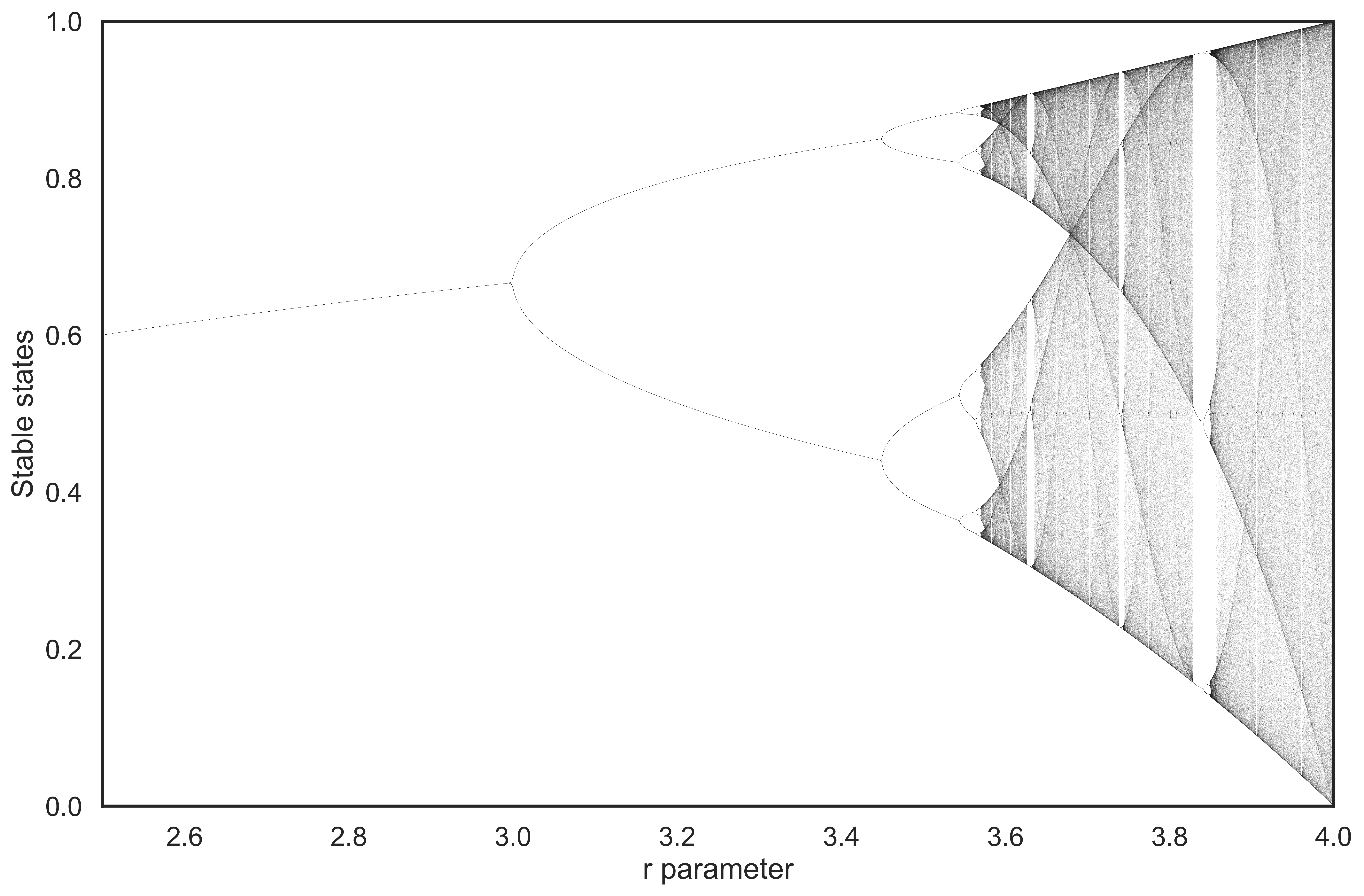}
     \caption{Bifurcation diagram of the logistic map defined by \( x_{n+1} = r x_n (1 - x_n) \). The diagram illustrates how varying the parameter \( r \) between 2.5 and 4.0 leads to transitions from stable fixed points to periodic oscillations and chaotic dynamics, highlighting the period-doubling bifurcation route to chaos.}
     \label{fig:bifurcation}
\end{figure}

Bifurcation analysis has been carried out on numerous classical nonlinear systems. In the Lorenz system, for instance, varying \(\rho\) can cause transitions from steady-state solutions to chaotic attractors. The Rössler system, Chen system, Chua’s circuit, and even time-delay systems like Mackey–Glass likewise exhibit rich cascades of period-doubling or other routes to chaos as parameters shift \cite{sparrow1982lorenz, rossler1976equation, chen1999yet, chua1986chaos, mackey1977oscillation}. These bifurcations demarcate fundamental changes in the dynamics—such as the sudden emergence of a strange attractor—capturing the essence of complexity in nonlinear phenomena.
\begin{figure}[!ht]
    \centering
    \includegraphics[width=0.8\linewidth]{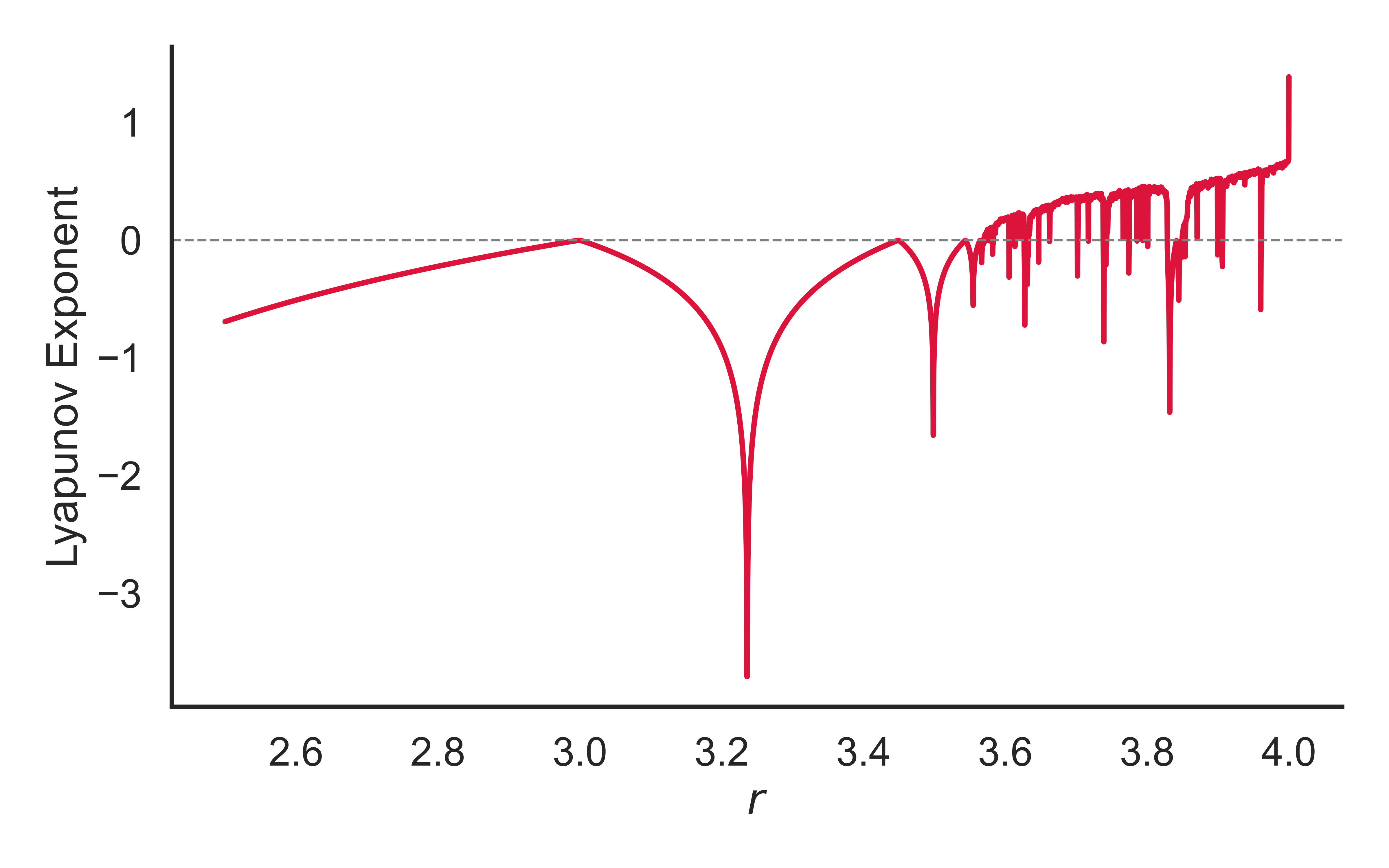}
    \caption{Lyapunov Exponent Spectrum of Logistic Map.}
    \label{fig:lyapunov_logistic}
\end{figure}
In \emph{Figure \ref{fig:lyapunov_logistic}}, we plot the \emph{largest Lyapunov exponent} \(\lambda\) of the logistic map  as a function of its parameter \(r\).  Mathematically, \(\lambda\) characterizes the exponential rate at which nearby trajectories diverge (positive \(\lambda\)) or converge (negative \(\lambda\)), thereby signaling the transition between stable and chaotic dynamics.  Here, we iteratively approximate
\begin{equation}
\lambda \;\approx\; \frac{1}{N_{\text{eff}}}\;\sum_{i=\text{trans}}^{n-1} 
\ln\ \!\Bigl\lvert r \,-\, 2\,r\,x_{i}\Bigr\rvert,
\end{equation}
which follows directly from the derivative \(\tfrac{d}{dx}\!\bigl(r\,x\,(1-x)\bigr)=r\,(1-2x)\).  As \(r\) increases beyond \(\approx 3.57\), \(\lambda\) becomes strictly positive, indicating the onset of chaos in the map.  From the perspective of echo state networks, studying such canonical chaotic regimes is essential: the logistic map’s well-known period-doubling route to chaos and the corresponding Lyapunov analysis provide a fundamental benchmark for understanding how reservoirs, trained on or inspired by chaotic processes, may encode, predict, or replicate increasingly complex dynamics.

 We perform a \emph{delay-coordinate embedding} of the logistic map at \(r=3.9\) in \emph{Figure \ref{fig:delay_embedding_logistic}}.  By taking the univariate time series \(\{x(t)\}\) and constructing the vectors
\(
\bigl(x(t),\,x(t+\tau),\,x(t+2\tau)\bigr),
\)
we embed the scalar dynamics into a higher-dimensional space, thus unfolding the system’s hidden geometric structure.  Underpinning this technique is the classical \emph{Takens’ theorem} \cite{takens1981detecting}, which ensures that for a generically chosen embedding dimension and delay, the reconstructed attractor preserves the underlying topological dynamics of the original system.  Practically, each point in the resulting scatter plot corresponds to \((x(t),\,x(t+\tau))\) on the axes, and is color-coded by the third coordinate \(x(t+2\tau)\).  This visualization offers a powerful way to analyze, characterize, and even forecast chaotic processes—an approach also relevant for understanding how echo state networks encode and process time-lagged dependencies in their internal state representations.
\begin{figure}[!ht]
    \centering
    \includegraphics[width=0.8\linewidth]{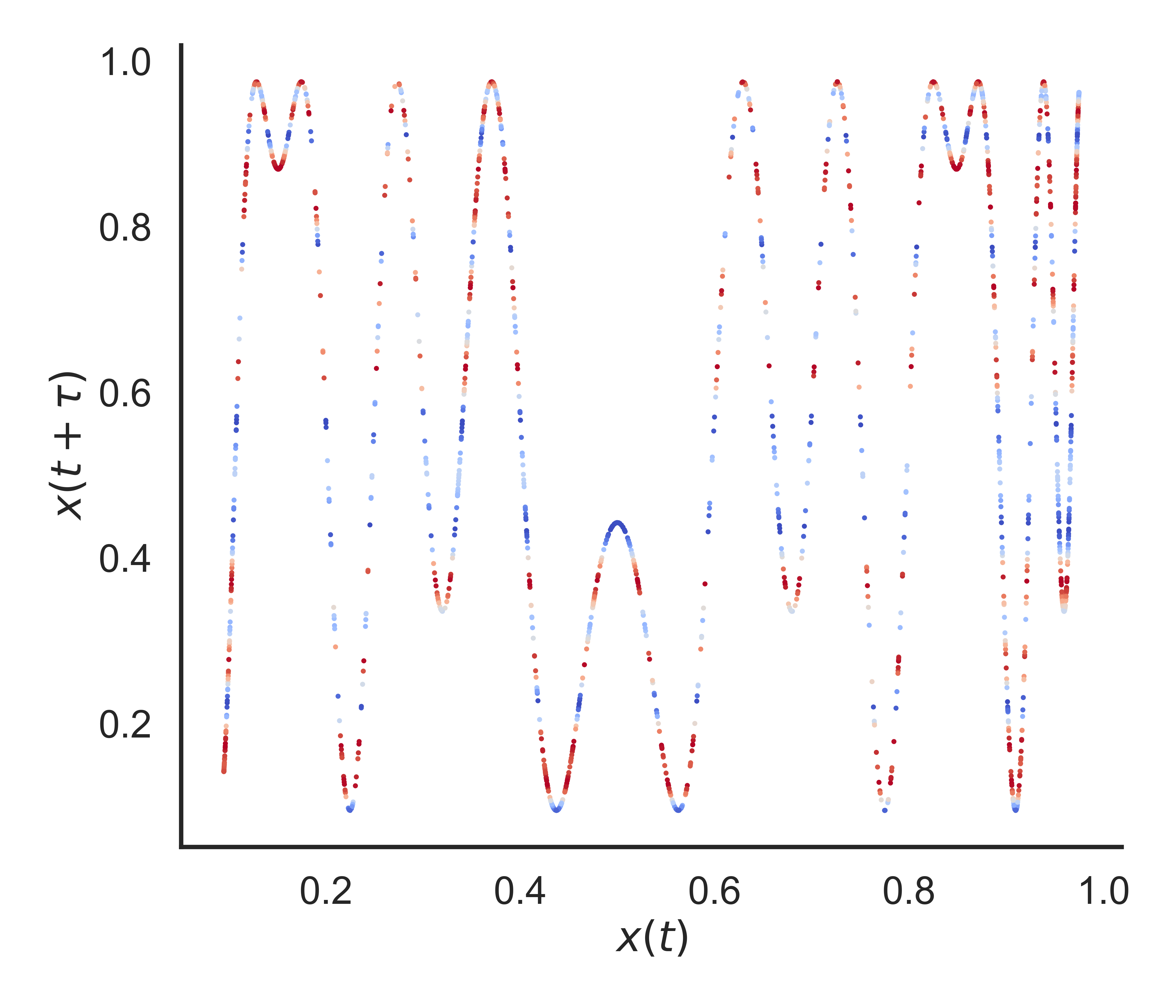}
    \caption{Delay-coordinate embedding of the logistic map ( \(r=3.9\) ), with each point \(\bigl(x(t),\,x(t+\tau)\bigr)\) colored by \(x(t+2\tau)\).}
    \label{fig:delay_embedding_logistic}
\end{figure}

We show a \emph{recurrence plot} of the logistic map at \(r=3.9\) by marking points \(\bigl(i,j\bigr)\) where \(\lvert x_i - x_j\rvert < \epsilon\) (see \emph{Figure \ref{fig:recurrence_plot_logistic}}).  This binary matrix highlights how the chaotic trajectory repeatedly visits similar regions of state space, revealing both local quasi-periodicity and irregular patterns indicative of sensitive dependence on initial conditions.  Meanwhile, \emph{Figure \ref{fig:ordinal_pattern_frequencies}} displays an \emph{ordinal pattern distribution} for length-\(3\) permutations of the same time series.  Here, each triplet of consecutive values is replaced by a pattern denoting the rank order \(\bigl(\arg\min,\,\ldots,\,\arg\max\bigr)\).  The resulting bar plot quantifies how often each distinct permutation appears, providing a purely order-based representation of the dynamical complexity.  Together, these two viewpoints—recurrence structure and ordinal patterns—offer complementary perspectives on the underlying chaos, each emphasizing different aspects of how state values recur or permute over time.

\begin{figure}[!ht]
    \centering
    \includegraphics[width=0.7\linewidth]{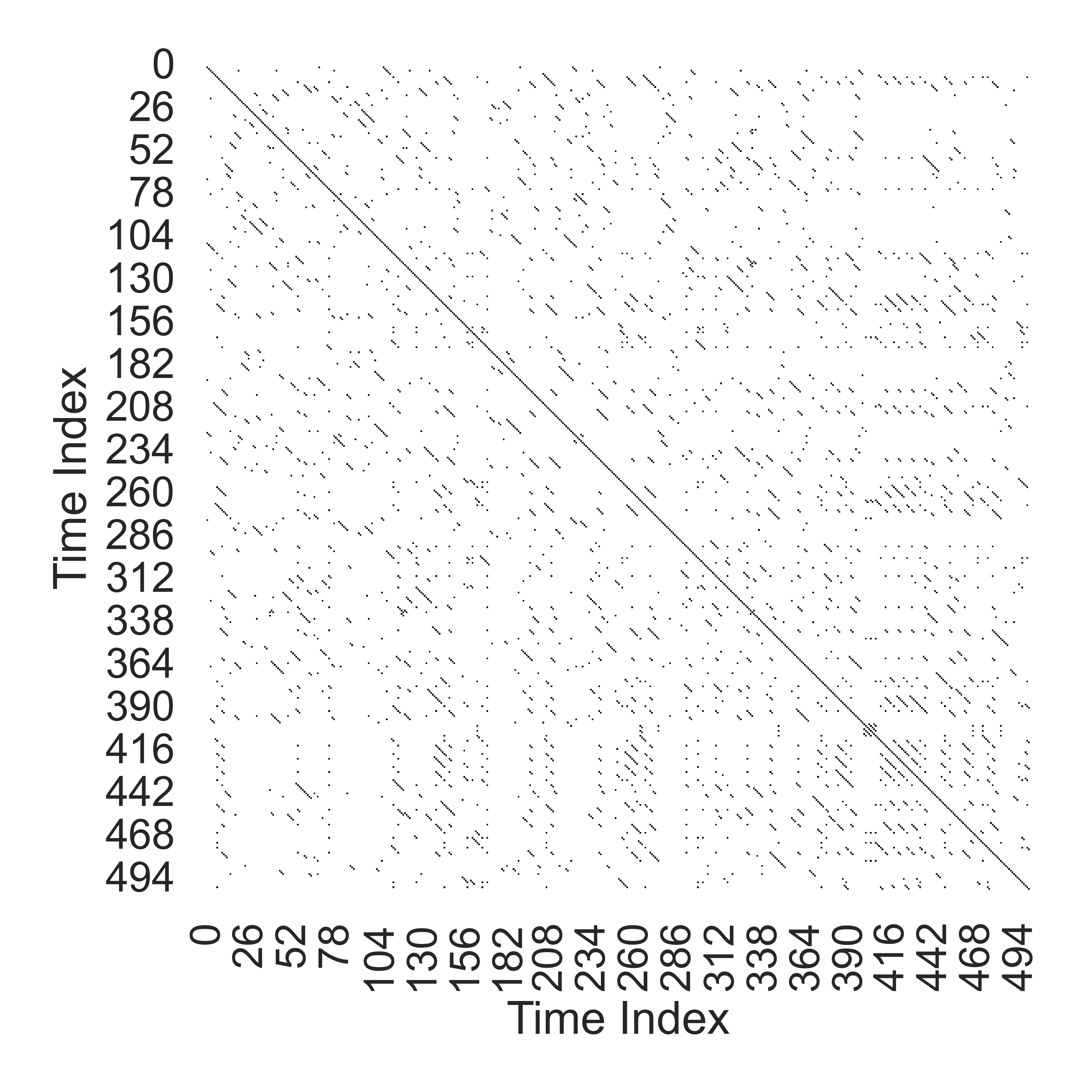}
    \caption{Recurrence Plot of Logistic Map (r = 3.9).}
    \label{fig:recurrence_plot_logistic}
\end{figure}

\begin{figure}[!ht]
    \centering
    \includegraphics[width=0.7\linewidth]{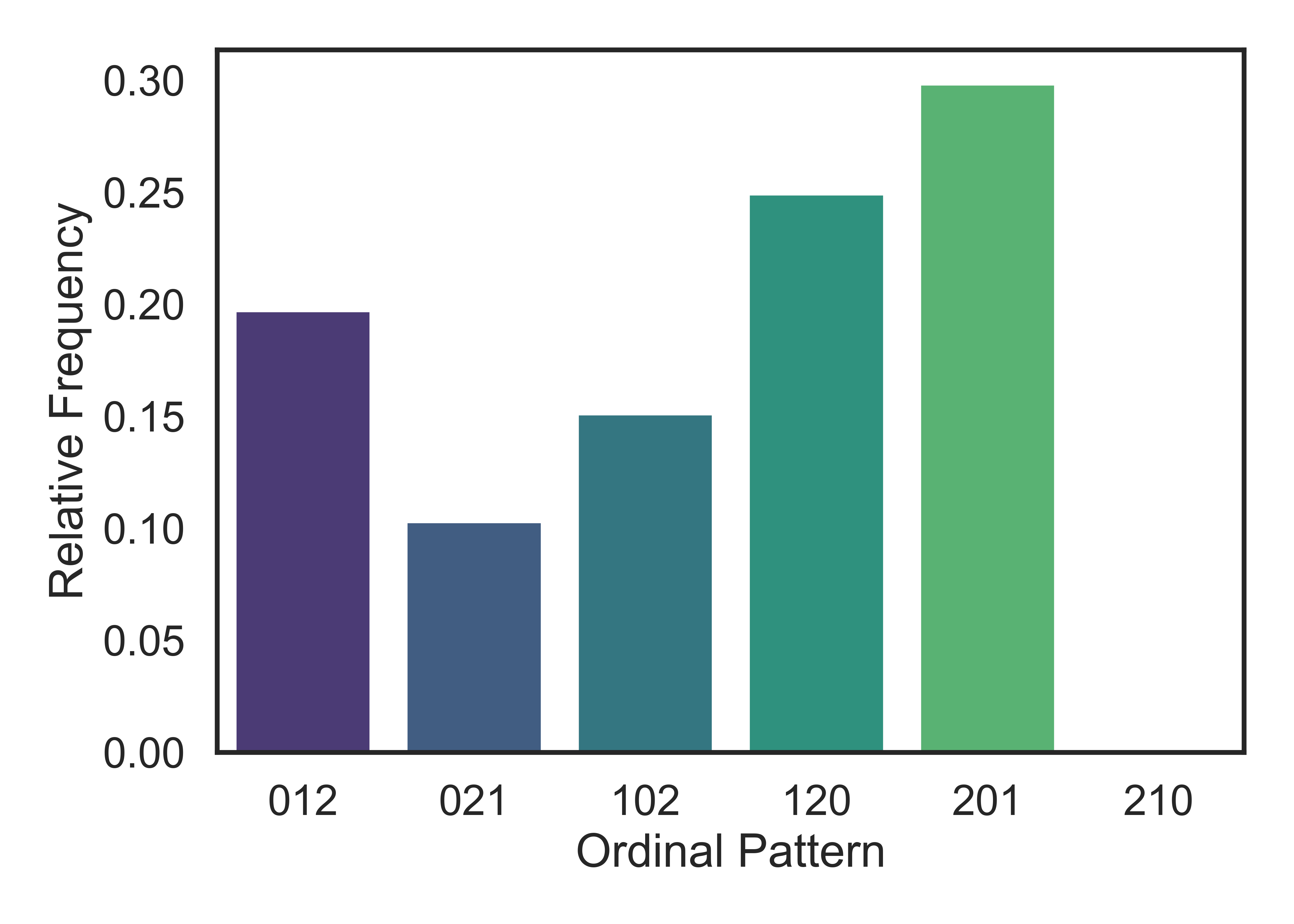}
    \caption{Ordinal Pattern Distribution (Logistic Map, d=3).}
    \label{fig:ordinal_pattern_frequencies}
\end{figure}

In \emph{Figure \ref{fig:mandelbrot_zoom}}, we illustrate a zoomed region of the Mandelbrot set, which is defined by the quadratic iteration \(z \mapsto z^{2} + c\) in the complex plane.  Specifically, \(c\in\mathbb{C}\) lies in the Mandelbrot set if the sequence
\(
z_{0}=0,\quad z_{n+1} \;=\; z_{n}^2 \;+\; c
\)
remains \emph{bounded} for all \(n\).  Equivalently, if \(\lvert z_{n}\rvert\) ever exceeds \(2\), we conclude that the orbit diverges, and hence \(c\) is not in the set.  The boundary of this set exhibits fractal complexity: every point on the boundary serves as a seed for infinitely many miniature replicas of the Mandelbrot shape, reflecting an unbounded degree of self-similarity under repeated magnifications.  From a dynamical-systems viewpoint, such fractal geometry underscores the interplay between stable and unstable behaviors within a simple nonlinear map, offering a quintessential example of emergent complexity.  In the context of echo state networks, understanding these delicate fractal boundaries—where orbits transition from bounded to unbounded—is conceptually relevant to how small variations in parameter space can yield significant shifts in a system’s long-term behavior.

\begin{figure}[!ht]
    \centering
    \includegraphics[width=0.99\linewidth]{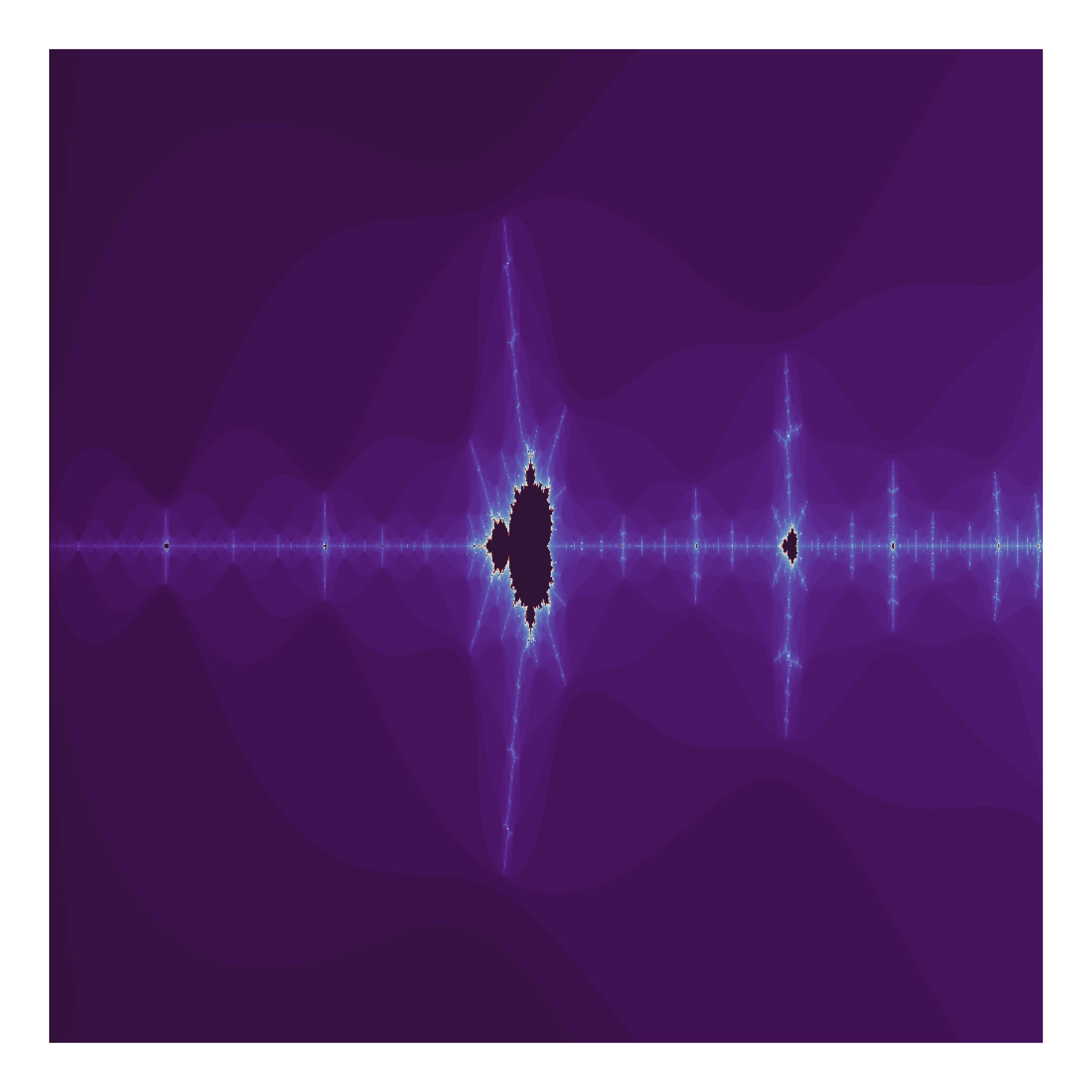}
    \caption{Zoomed Mandelbrot Set.}
    \label{fig:mandelbrot_zoom}
\end{figure}

\paragraph{Poincaré Maps.}
Whereas bifurcation analysis can reveal how equilibrium and periodic orbits change, \emph{Poincaré maps} (also called return maps) provide a powerful tool to reduce continuous-time systems to discrete iterations, facilitating the study of periodic and quasi-periodic orbits \cite{guckenheimer2013nonlinear, strogatz2018nonlinear}. 
Suppose \(\dot{\mathbf{x}} = \mathbf{f}(\mathbf{x})\) is a flow on \(\mathbb{R}^n\). A \emph{Poincaré section} \(\Sigma \subset \mathbb{R}^n\) is a \((n-1)\)-dimensional hypersurface transverse to the flow (so typical trajectories intersect \(\Sigma\) repeatedly). One defines the Poincaré map: 
\begin{equation}
  \mathbf{P}: \Sigma \to \Sigma,
\end{equation}
by tracking the intersection of trajectories with \(\Sigma\). For a point \(\mathbf{x}_0 \in \Sigma\), the trajectory \(\mathbf{x}(t)\) starting at \(\mathbf{x}_0\) will re-enter \(\Sigma\) after some minimal time \(T(\mathbf{x}_0)\). Set
\begin{equation}
  \mathbf{P}\bigl(\mathbf{x}_0\bigr) \;=\; \mathbf{x}\Bigl(T\!\bigl(\mathbf{x}_0\bigr)\Bigr).
\end{equation}

Thus, \(\mathbf{P}\) converts the continuous-time flow into a discrete map. In many systems, analyzing the fixed points or periodic points of \(\mathbf{P}\) is easier than working directly with the original flow. For example, if \(\mathbf{x}^*\) is a fixed point of \(\mathbf{P}\), that means
\begin{equation}
  \mathbf{P}\bigl(\mathbf{x}^*\bigr) = \mathbf{x}^*,
\end{equation}
which corresponds to a periodic orbit of the original system, because starting at \(\mathbf{x}^*\) on the section \(\Sigma\) returns there after one full period of the cycle. One can study the linearization \(D_{\mathbf{x}}\mathbf{P}\bigl(\mathbf{x}^*\bigr)\) to deduce stability properties of that orbit. If all eigenvalues of \(D_{\mathbf{x}}\mathbf{P}\bigl(\mathbf{x}^*\bigr)\) lie in the unit circle of the complex plane, the corresponding limit cycle of the flow is stable (asymptotically attracting); if at least one eigenvalue lies outside the unit circle, the cycle is unstable \cite{guckenheimer2013nonlinear, wiggins2003introduction}.

This viewpoint is especially invaluable for chaotic systems where direct time-domain analysis may be intractable. By constructing a 1D or 2D Poincaré map (depending on the dimension of the chosen section), complex dynamics such as period-doubling cascades, intermittency, or attractor merging can be identified. For instance, in the Lorenz system or the Rössler system, one commonly selects a plane \(\Sigma = \{\mathbf{x} \in \mathbb{R}^3 : h(\mathbf{x}) = \text{const}\}\) that is intersected repeatedly by the chaotic trajectory; successive intersections reveal the structure of the underlying map \(\mathbf{P}\), such as stable and unstable manifolds, and give insights into the fractal geometry of the attractor.  

Bifurcation diagrams constructed via Poincaré maps form a central tool in the numerical and theoretical exploration of chaotic phenomena. They show how fixed points of \(\mathbf{P}\) (equivalently, periodic orbits of the original flow) appear, collide, or disappear as parameters change, neatly encapsulating the main routes to chaos \cite{sparrow1982lorenz, ott2002chaos, strogatz2018nonlinear}.

From the standpoint of reservoir computing, especially ESNs, these approaches—bifurcation analysis and Poincaré sections—offer a language for discussing and designing high-dimensional recurrent networks whose internal dynamics capture a desired range of time-dependent behaviors. Whether the system remains in a stable regime or transitions to more complex oscillations or even low-level chaos can have profound consequences for the network’s representational capacity, memory, and computational performance. In effect, the dynamical systems viewpoint emphasizes how the \emph{internal} recurrence of an RC network must be balanced or tuned so that it does not saturate in an overly stable or overly chaotic region, a principle well-reflected in the notion of the “edge of chaos” for maximizing expressive power \cite{bertschinger2004real, boedecker2012information, Lukosevicius2012}.

Reservoir computing, particularly via ESNs, leverages insights from dynamical systems by harnessing the capacity of high-dimensional recurrent networks to emulate or approximate these complex behaviors. By appropriately embedding the input streams into the nonlinear state-space of the reservoir, RC approaches exploit the manifold structure of the underlying dynamical system and aim to provide robust time-series predictions, classification, or control in both linear and nonlinear regimes. The mixture of linear components (readouts) and nonlinear internal recurrence in RC draws heavily on the distinction between simpler linear systems and these more intricate nonlinear and chaotic examples, bridging theory with practical learning tasks \cite{jaeger2001echo, Lukosevicius2012, maass2002real}.

A central concept in dynamical systems theory is the classification of systems into chaotic and non-chaotic regimes. A dynamical system is said to be chaotic if it satisfies three key properties: (1) sensitive dependence on initial conditions, (2) topological mixing, and (3) dense periodic orbits \cite{strogatz2018nonlinear}. Mathematically, sensitive dependence implies that there exists a constant $\delta > 0$ such that, for any two points $\mathbf{x}_0, \mathbf{y}_0 \in \mathcal{M}$ and $\|\mathbf{x}_0 - \mathbf{y}_0\| < \epsilon$, there exists a time $t > 0$ such that $\|\phi_t(\mathbf{x}_0) - \phi_t(\mathbf{y}_0)\| > \delta$. This behavior is often quantified using Lyapunov exponents, $\lambda_i$, which measure the exponential divergence of nearby trajectories:
\begin{equation}
\lambda_i = \lim_{t \to \infty} \frac{1}{t} \ln \|\mathbf{J} \cdot \mathbf{v}_i\|,
\end{equation}
where $\mathbf{J}$ is the Jacobian of $\phi_t$ and $\mathbf{v}_i$ is an eigenvector of $\mathbf{J}$.

Non-chaotic systems, in contrast, exhibit stable and predictable behavior. For example, linear systems with eigenvalues of $A$ having strictly negative real parts converge to a fixed point. The phase space trajectories of such systems form well-defined patterns, such as fixed points, limit cycles, or tori, in contrast to the fractal-like structures associated with chaotic attractors.

The analysis of dynamical systems often relies on embedding theorems, which enable the reconstruction of system dynamics from observed time series. Takens' embedding theorem \cite{takens1981detecting} states that, for a generic smooth dynamical system $(\mathcal{M}, \phi_t)$ of dimension $d$, the \textit{delayed coordinate map} $\Psi: \mathcal{M} \to \mathbb{R}^m$, given by:
\begin{equation}
\Psi(\mathbf{x}) = \big(x(t), x(t + \tau), \ldots, x(t + (m-1)\tau)\big),
\end{equation}
is an embedding for $m > 2d$, where $\tau$ is the embedding delay. This result implies that the dynamics of the original system can be faithfully represented in a higher-dimensional Euclidean space, facilitating analysis from time-series data.

These foundational concepts are essential for understanding the behavior of reservoirs in RC. The reservoir itself can be viewed as a high-dimensional dynamical system that maps input sequences into a state space, preserving critical temporal information. The interplay between chaotic and non-chaotic dynamics in the reservoir determines its ability to process and encode temporal dependencies, making dynamical systems theory a cornerstone of reservoir computing.

\subsection{Mathematical Model of RC}

RC is fundamentally grounded in the theory of dynamical systems and linear algebra, leveraging a high-dimensional dynamical system, referred to as the reservoir, to encode temporal information from input sequences. Formally, the reservoir is modeled as a discrete-time dynamical system with a state-space representation. Let $\mathbf{u}(t) \in \mathbb{R}^k$ denote the input at time $t$, $\mathbf{x}(t) \in \mathbb{R}^n$ the reservoir state vector at time $t$, and $\mathbf{y}(t) \in \mathbb{R}^m$ the output. The dynamical evolution of the reservoir is given by the state update equation:
\begin{equation}
\mathbf{x}(t) = f\big(W_\text{in} \mathbf{u}(t) + W_\text{res} \mathbf{x}(t-1) + \mathbf{b}\big),
\end{equation}
where $W_\text{in} \in \mathbb{R}^{n \times k}$ is the input weight matrix,
$W_\text{res} \in \mathbb{R}^{n \times n}$ is the recurrent weight matrix,
$\mathbf{b} \in \mathbb{R}^n$ is the bias vector, and
 $f: \mathbb{R}^n \to \mathbb{R}^n$ is a pointwise nonlinear activation function, such as $\tanh$.

The output of the reservoir is computed via a linear readout layer:
\begin{equation}
\mathbf{y}(t) = W_\text{out} \mathbf{x}(t),
\end{equation}
where $W_\text{out} \in \mathbb{R}^{m \times n}$ is the output weight matrix. In standard RC models, $W_\text{out}$ is the only trainable parameter, while $W_\text{in}$, $W$, and $\mathbf{b}$ are initialized randomly and kept fixed.

\paragraph{State-Space Representation}

The reservoir can also be described in state-space form, which is commonly used in control theory and signal processing. Define the state vector $\mathbf{x}(t) \in \mathbb{R}^n$ as the internal representation of the system at time $t$. The state-space equations are:
\begin{equation}
\begin{aligned}
    \mathbf{x}(t+1) &= A \mathbf{x}(t) + B \mathbf{u}(t), \\
    \mathbf{y}(t) &= C \mathbf{x}(t),
\end{aligned}
\end{equation}
where
 $A = W_\text{res}$ represents the state transition matrix, incorporating the recurrent dynamics,
$B = W_\text{in}$ maps the input to the reservoir state space, and
 $C = W_\text{out}$ maps the reservoir states to the output.

This representation emphasizes the modular structure of the reservoir and its separation from the input and output mappings. The nonlinear function $f$ ensures that the reservoir exhibits rich dynamics, allowing it to encode complex temporal patterns.  The fixed nature of the reservoir weights simplifies the training process, as only the output weights $W_\text{out}$ are optimized. This separation of training from reservoir dynamics allows RC to achieve a balance between computational efficiency and expressiveness.

\paragraph{Echo State Property (ESP).}

Consider the discrete-time reservoir dynamical system 
\begin{equation}
\mathbf{x}(t+1) \;=\; F \bigl( \mathbf{W}_{\mathrm{res}}\;\mathbf{x}(t) \;+\; \mathbf{W}_{\mathrm{in}}\;\mathbf{u}(t+1) \;+\; \mathbf{b} \bigr),
\label{eq:rc_state}
\end{equation}
where
$\mathbf{x}(t) \in \mathbb{R}^N, 
\mathbf{u}(t) \in \mathbb{R}^m,
\mathbf{W}_{\mathrm{res}} \in \mathbb{R}^{N\times N}, 
\mathbf{W}_{\mathrm{in}} \in \mathbb{R}^{N\times m}, 
\mathbf{b} \in \mathbb{R}^N,$
and \(F: \mathbb{R}^N \to \mathbb{R}^N\) is a componentwise nonlinear activation (e.g.\ \(\tanh\), ReLU, etc.). Let \(\|\cdot\|\) be a suitable norm on \(\mathbb{R}^N\), and assume \(F\) is \emph{Lipschitz continuous} with constant \(L>0\), i.e.\ for any \(\mathbf{z}_1,\mathbf{z}_2 \in \mathbb{R}^N\),
\begin{equation}
\bigl\|F(\mathbf{z}_1) - F(\mathbf{z}_2)\bigr\| \;\le\; L \,\|\mathbf{z}_1 - \mathbf{z}_2\|.
\end{equation}
Define the \emph{spectral radius} of \(\mathbf{W}_{\mathrm{res}}\) by \(\rho(\mathbf{W}_{\mathrm{res}})\). 

\medskip

\begin{theorem}[Echo State Property \cite{jaeger2001echo, Lukosevicius2012}]
\label{thm:esp}
If \(\rho(\mathbf{W}_{\mathrm{res}}) < \frac{1}{L}\), then the reservoir system \eqref{eq:rc_state} possesses the \emph{echo state property}, namely:
\begin{enumerate}

 \item  For any bounded input sequence \(\{\mathbf{u}(t)\}\), there exists a unique \emph{reservoir trajectory} \(\{\mathbf{x}^*(t)\}\) consistent with that input, regardless of initial conditions.
 \item  If \(\mathbf{x}_1(t)\) and \(\mathbf{x}_2(t)\) are solutions driven by the same input sequence \(\{\mathbf{u}(t)\}\) but different initial states, then
\[
\lim_{t \to \infty} \|\mathbf{x}_1(t) - \mathbf{x}_2(t)\| \;=\; 0.
\]
\end{enumerate}
In other words, the influence of initial conditions vanishes, and the reservoir state is eventually determined uniquely by the input history.
\end{theorem}

\begin{proof}
Let \(\{\mathbf{x}_1(t)\}\) and \(\{\mathbf{x}_2(t)\}\) be two trajectories satisfying \eqref{eq:rc_state} with identical inputs \(\mathbf{u}(t)\) but different initial states \(\mathbf{x}_1(0)\neq \mathbf{x}_2(0)\). We define
\[
\boldsymbol{\delta}(t) \;=\; \mathbf{x}_1(t) \;-\; \mathbf{x}_2(t).
\]
Subtract the state-update equations for \(\mathbf{x}_1\) and \(\mathbf{x}_2\):
\[
\boldsymbol{\delta}(t+1) 
= 
F\Bigl(\mathbf{W}_{\mathrm{res}}\,\mathbf{x}_1(t) + \mathbf{W}_{\mathrm{in}}\,\mathbf{u}(t+1) + \mathbf{b}\Bigr)
- 
F\Bigl(\mathbf{W}_{\mathrm{res}}\,\mathbf{x}_2(t) + \mathbf{W}_{\mathrm{in}}\,\mathbf{u}(t+1) + \mathbf{b}\Bigr).
\]
By the Lipschitz continuity of \(F\) with constant \(L\),
\[
\|\boldsymbol{\delta}(t+1)\| 
\;\le\; 
L \,\bigl\|\mathbf{W}_{\mathrm{res}}\bigl(\mathbf{x}_1(t) - \mathbf{x}_2(t)\bigr)\bigr\|
\;\le\; 
L \;\|\mathbf{W}_{\mathrm{res}}\|\;\|\boldsymbol{\delta}(t)\|,
\]
where \(\|\mathbf{W}_{\mathrm{res}}\|\) is any chosen operator norm of \(\mathbf{W}_{\mathrm{res}}\). By submultiplicative properties of matrix norms, \(\|\mathbf{W}_{\mathrm{res}}\|\ge \rho(\mathbf{W}_{\mathrm{res}})\); thus, if \(\rho(\mathbf{W}_{\mathrm{res}})<\frac{1}{L}\), we can choose a norm with \(\|\mathbf{W}_{\mathrm{res}}\|\le r < \frac{1}{L}\). Then there exists a constant \(\kappa = L\,r < 1\) such that
\[
\|\boldsymbol{\delta}(t+1)\| \;\le\; \kappa \,\|\boldsymbol{\delta}(t)\|.
\]
Iterating from \(t=0\) to \(t=T\) yields
\[
\|\boldsymbol{\delta}(T)\|
\;\le\; \kappa^T\,\|\boldsymbol{\delta}(0)\|.
\]
Since \(\kappa<1\), \(\kappa^T \to 0\) as \(T\to\infty\). Therefore, 
\[
\lim_{T\to\infty} \|\mathbf{x}_1(T) - \mathbf{x}_2(T)\| \;=\; 0,
\]
which establishes uniqueness of the reservoir state for any bounded input sequence and proves the ESP.
\end{proof}

\paragraph{Fading Memory Property (FMP).}

In addition to the ESP, one typically requires that the reservoir possesses the \emph{fading memory property (FMP)} for input-driven computations. Informally, the FMP means that older inputs have exponentially diminishing influence on the present state, so the reservoir “forgets” distant past inputs in a controlled manner. A standard formalization involves \emph{fading-memory filters} \cite{Boyd1985, grigoryeva2018echo}:

\begin{definition}[Fading Memory Property]
\label{def:fmp}
A discrete-time operator \(\Phi: (\mathbb{R}^m)^\mathbb{Z} \to (\mathbb{R}^N)^\mathbb{Z}\), which maps an input sequence \(\{\mathbf{u}(t)\}_{t\in\mathbb{Z}}\) to a state/output sequence \(\{\mathbf{x}(t)\}_{t\in\mathbb{Z}}\), is said to have the \emph{fading memory property} if, for every pair of input sequences \(\mathbf{u}, \mathbf{v}\) that coincide on a large time interval \([T, \infty)\) (i.e.\ differ only in the remote past), the outputs \(\Phi(\mathbf{u})\) and \(\Phi(\mathbf{v})\) become arbitrarily close as \(T \to \infty\). Formally, for any \(\varepsilon>0\), there exists a \(\delta>0\) such that if \(\|\mathbf{u}(t)-\mathbf{v}(t)\| < \delta\) for all \(t\ge T\), then
\[
\sup_{t\ge T} \|\mathbf{x}_{\mathbf{u}}(t) - \mathbf{x}_{\mathbf{v}}(t)\| \;<\; \varepsilon,
\]
where \(\mathbf{x}_{\mathbf{u}}(t)\) and \(\mathbf{x}_{\mathbf{v}}(t)\) are the states generated by \(\Phi\) from inputs \(\mathbf{u}\) and \(\mathbf{v}\), respectively.
\end{definition}

In the reservoir setting \eqref{eq:rc_state}, if \(\rho(\mathbf{W}_{\mathrm{res}}) < \tfrac{1}{L}\) and \(F\) is Lipschitz, one can show that \(\mathbf{x}(t)\) depends most strongly on recent inputs, with the influence of earlier time steps decaying exponentially \cite{jaeger2001echo, grigoryeva2018echo}. Concretely, small input perturbations in the distant past produce negligible changes in the current state. Coupled with the ESP, the FMP ensures that the reservoir reliably encodes relevant temporal information while remaining robust to small perturbations or drifting initial conditions.

\begin{theorem}[Fading Memory Property]
\label{thm:fmp}
Consider the discrete-time reservoir system
$\mathbf{x}(t)$ given by \eqref{eq:rc_state},
where $\mathbf{x}(t) \in \mathbb{R}^N$, 
$\mathbf{u}(t) \in \mathbb{R}^m$, 
$\mathbf{W}_{\mathrm{res}} \in \mathbb{R}^{N\times N}$, 
$\mathbf{W}_{\mathrm{in}} \in \mathbb{R}^{N\times m}$, 
$\mathbf{b} \in \mathbb{R}^N$, and $F: \mathbb{R}^N \to \mathbb{R}^N$ is a componentwise Lipschitz function with Lipschitz constant $L>0$.  
Assume $\rho(\mathbf{W}_{\mathrm{res}}) < \frac{1}{L}$, where $\rho(\cdot)$ is the spectral radius.  
Then the corresponding input-to-state map $\Phi$, which takes any input sequence $\{\mathbf{u}(t)\}$ and produces the reservoir states $\{\mathbf{x}(t)\}$, possesses the \emph{fading memory property}: if two input sequences coincide or become arbitrarily close after some large time index, then their corresponding reservoir state trajectories also become arbitrarily close after that time.
\end{theorem}

\begin{proof}  
From the ESP argument, we know that if $\rho(\mathbf{W}_{\mathrm{res}}) < \tfrac{1}{L}$, then there exists a norm $\|\cdot\|$ such that $\|\mathbf{W}_{\mathrm{res}}\|\le r$ and $L\,r < 1$.  
For a single input sequence $\{\mathbf{u}(t)\}$, if $\mathbf{x}_1(t)$ and $\mathbf{x}_2(t)$ are two solutions with the same inputs but different initial states, their difference satisfies  
\[
\|\mathbf{x}_1(t+1) - \mathbf{x}_2(t+1)\|
\;\le\; 
L\,\|\mathbf{W}_{\mathrm{res}}\|\;\|\mathbf{x}_1(t) - \mathbf{x}_2(t)\|
\;\le\;
L\,r\;\|\mathbf{x}_1(t) - \mathbf{x}_2(t)\|.
\]
Hence $\|\mathbf{x}_1(t) - \mathbf{x}_2(t)\|\to 0$ exponentially fast, showing uniqueness of the trajectory in the limit (the ESP).

\medskip

Let $\mathbf{u}$ and $\mathbf{v}$ be two bounded input sequences, with reservoir state trajectories $\mathbf{x}_{\mathbf{u}}(t)$ and $\mathbf{x}_{\mathbf{v}}(t)$, respectively.  Suppose that for some large $T_0$, the inputs coincide (or are very close) for $t \ge T_0$.  Concretely:
\[
\sup_{t \,\ge\, T_0}\|\mathbf{u}(t) - \mathbf{v}(t)\| \;\le\; \delta
\quad(\text{for a chosen }\delta\ge 0).
\]
Since
\[
\mathbf{x}_{\mathbf{u}}(t+1)
=
F\!\Big(
\mathbf{W}_{\mathrm{res}}\;\mathbf{x}_{\mathbf{u}}(t)
+
\mathbf{W}_{\mathrm{in}}\;\mathbf{u}(t+1)
+
\mathbf{b}
\Big),
\quad
\mathbf{x}_{\mathbf{v}}(t+1)
=
F\!\Big(
\mathbf{W}_{\mathrm{res}}\;\mathbf{x}_{\mathbf{v}}(t)
+
\mathbf{W}_{\mathrm{in}}\;\mathbf{v}(t+1)
+
\mathbf{b}
\Big),
\]
subtract and apply the Lipschitz property:
\[
\|\mathbf{x}_{\mathbf{u}}(t+1) - \mathbf{x}_{\mathbf{v}}(t+1)\|
\;\le\;
L\,\bigl\|\mathbf{W}_{\mathrm{res}}(\mathbf{x}_{\mathbf{u}}(t)-\mathbf{x}_{\mathbf{v}}(t))
\;+\;
\mathbf{W}_{\mathrm{in}}(\mathbf{u}(t+1)-\mathbf{v}(t+1))
\bigr\|.
\]
Using the operator norms, let $\|\mathbf{W}_{\mathrm{res}}\|\le r$ and set $M=\|\mathbf{W}_{\mathrm{in}}\|$; then
\[
\|\mathbf{x}_{\mathbf{u}}(t+1) - \mathbf{x}_{\mathbf{v}}(t+1)\|
\;\le\;
L\,r\,\|\mathbf{x}_{\mathbf{u}}(t) - \mathbf{x}_{\mathbf{v}}(t)\|
\;+\;
L\,M\,\|\mathbf{u}(t+1)-\mathbf{v}(t+1)\|.
\]

\medskip

\emph{Case A: Perfect coincidence for $t \ge T_0$.}  
If $\mathbf{u}(t)=\mathbf{v}(t)$ for all $t \ge T_0$, then the second term above is zero for $t\ge T_0$.  Hence
\[
\|\mathbf{x}_{\mathbf{u}}(t+1) - \mathbf{x}_{\mathbf{v}}(t+1)\|
\;\le\;
(L\,r)\,\|\mathbf{x}_{\mathbf{u}}(t) - \mathbf{x}_{\mathbf{v}}(t)\|.
\]
Iterating from $t=T_0$ shows that 
\[
\|\mathbf{x}_{\mathbf{u}}(t) - \mathbf{x}_{\mathbf{v}}(t)\|
\;\le\;
(L\,r)^{\,t - T_0}\;\|\mathbf{x}_{\mathbf{u}}(T_0) - \mathbf{x}_{\mathbf{v}}(T_0)\|,
\]
which converges to $0$ as $t\to\infty$ since $L\,r<1$.  
Thus if the inputs match from $T_0$ onward, the state trajectories also become arbitrarily close, i.e.\ \emph{past differences are exponentially “forgotten.”}

\medskip

\emph{Case B: Small (rather than exact) differences for $t \ge T_0$.}  
If $\|\mathbf{u}(t)-\mathbf{v}(t)\|\le \delta$ for $t\ge T_0$, we get
\[
\|\mathbf{x}_{\mathbf{u}}(t+1)-\mathbf{x}_{\mathbf{v}}(t+1)\|
\;\le\;
L\,r\,\|\mathbf{x}_{\mathbf{u}}(t)-\mathbf{x}_{\mathbf{v}}(t)\|
+
L\,M\,\delta.
\]
A standard contraction-plus-constant argument (unfolding this recursion) implies that for any $t\ge T_0$,
\[
\|\mathbf{x}_{\mathbf{u}}(t)-\mathbf{x}_{\mathbf{v}}(t)\|
\;\le\;
(L\,r)^{\,t - T_0}\,\|\mathbf{x}_{\mathbf{u}}(T_0)-\mathbf{x}_{\mathbf{v}}(T_0)\|
\;+\;
\frac{L\,M\,\delta}{\,1 - L\,r\,}.
\]
Thus by choosing $\delta$ sufficiently small (and $T_0$ large enough), one can make $\|\mathbf{x}_{\mathbf{u}}(t)-\mathbf{x}_{\mathbf{v}}(t)\|$ as small as desired for all $t \ge T_0$.  
In other words, remote past input differences vanish exponentially, and recent slight differences lead only to a bounded shift in the states.  

\medskip

Hence the system has the \emph{fading memory property}:
for any $\varepsilon>0$, one can pick $\delta>0$ and $T_0$ such that whenever the inputs satisfy 
\[
\sup_{t\ge T_0}\|\mathbf{u}(t)-\mathbf{v}(t)\|<\delta,
\]
then
\[
\sup_{t \ge T_0}\|\mathbf{x}_{\mathbf{u}}(t)-\mathbf{x}_{\mathbf{v}}(t)\|<\varepsilon.
\]
\end{proof}

Together, ESP and FMP provide theoretical guarantees in reservoir computing. They ensure that (1) there is a unique, well-defined reservoir trajectory given a particular input sequence, and (2) the reservoir “forgets” remote history in a controlled manner, allowing the system to focus on the most relevant (recent) temporal patterns for the downstream learning task.

In \emph{Figure \ref{fig:esn_state_pca_flow}}, we project a sequence of reservoir states \(\{\mathbf{x}(t)\}\subset\mathbb{R}^{N}\) into the plane spanned by the first two principal components.  Each reservoir update follows an iterative nonlinear map
\begin{equation}
\mathbf{x}(t+1)
\;=\;
\tanh\Bigl(
    \alpha\,\mathbf{x}(t)
    \;+\;\boldsymbol{\xi}(t)
\Bigr),
\end{equation}
where \(\alpha\) is a contraction factor (here set to \(0.9\)) and \(\boldsymbol{\xi}(t)\sim\mathcal{N}(\mathbf{0},\sigma^{2}\mathbf{I})\) introduces a small stochastic input.  By performing principal component analysis (PCA) on the resulting high-dimensional trajectories, we retain the directions of largest variance and reduce the reservoir’s state manifold to a two-dimensional plane.  The coloration from blue to cyan encodes time progression, capturing how the system moves smoothly through its reduced-phase space.  Such dimensionality-reduction techniques are useful in  analysis, where reservoirs often have hundreds or thousands of neurons.  Viewing the states in a lower-dimensional embedding reveals whether the dynamics explore a narrow subset of configurations or spread broadly throughout the network’s state space—information central to understanding the echo property and the representation of temporal signals.

\begin{figure}[!ht]
    \centering
    \includegraphics[width=0.72\linewidth]{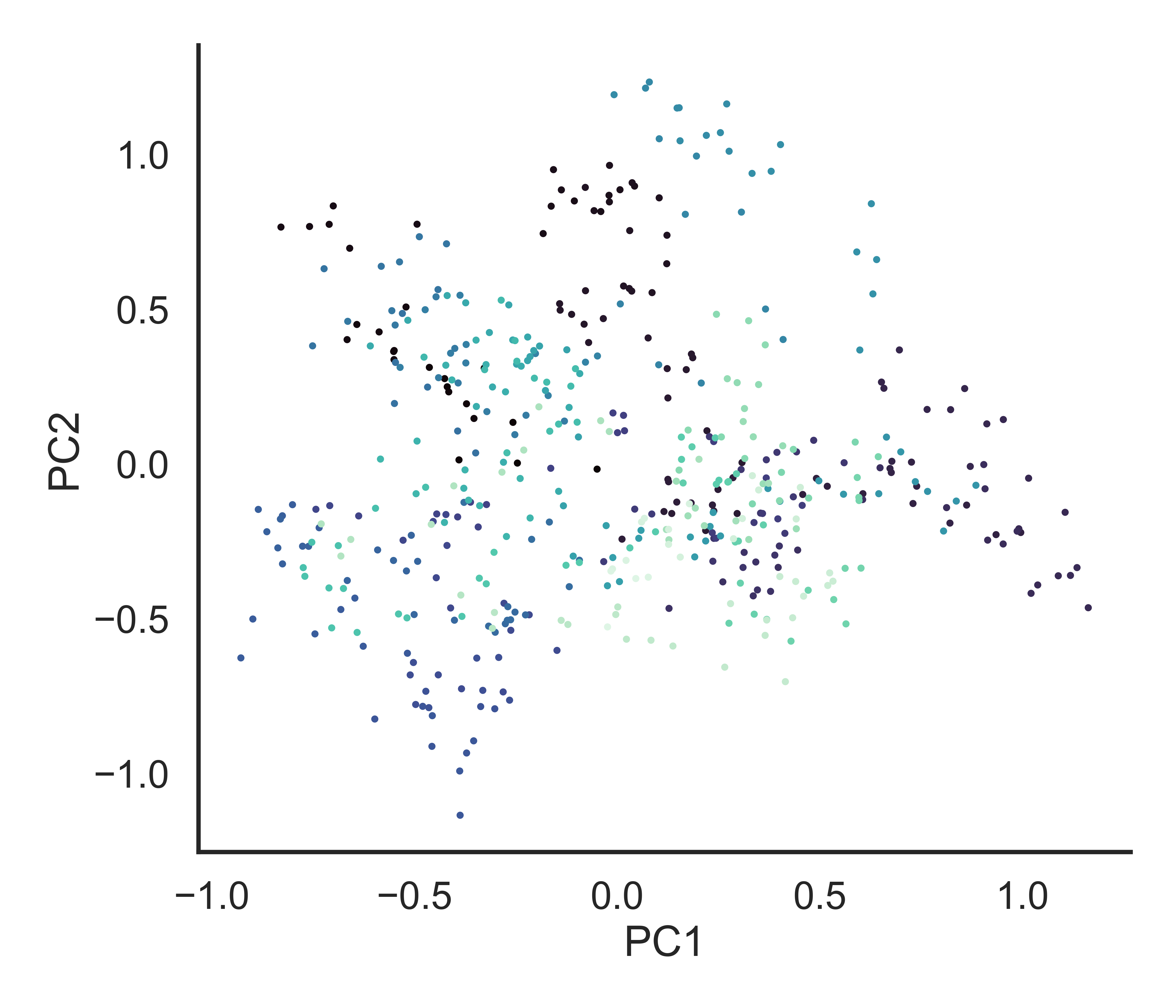}
    \caption{PCA-based two-dimensional projection of the reservoir’s high-dimensional state flow, with color indicating time from earliest (blue) to latest (cyan).  This reduced-phase portrait visually illustrates how the network’s states evolve under recurrent tanh dynamics perturbed by small random noise.}
    \label{fig:esn_state_pca_flow}
\end{figure}

We construct a two-dimensional delay embedding of the logistic map  at parameter \(r=3.8\) (see \emph{Figure \ref{fig:delay_embedding_curvature}}), where each point \(\bigl(x(t),\,x(t+\tau)\bigr)\) is color-coded by \(\lvert x^{\prime\prime}(t)\rvert\).  Here, \(x^{\prime\prime}(t)\) is approximated by a finite difference of the time series (i.e., a discrete second derivative), thus capturing the local \emph{curvature} or inflection in the trajectory.  In chaotic regimes, large curvature often signals rapid transitions between rising and falling branches, indicative of heightened sensitivity in the underlying map.  Delay embedding reveals important topological structure—how the sequence cycles, folds, and separates over time—while coloring by curvature highlights precisely where the orbit accelerates or decelerates most abruptly.  Within the framework of echo state networks, understanding such abrupt transitions and local maxima in curvature is essential for effectively encoding and predicting chaotic signals, since the reservoir must possess sufficient nonlinearity and memory to track these heightened twists and turns in the input dynamics.

\begin{figure}[!ht]
    \centering
    \includegraphics[width=0.8\linewidth]{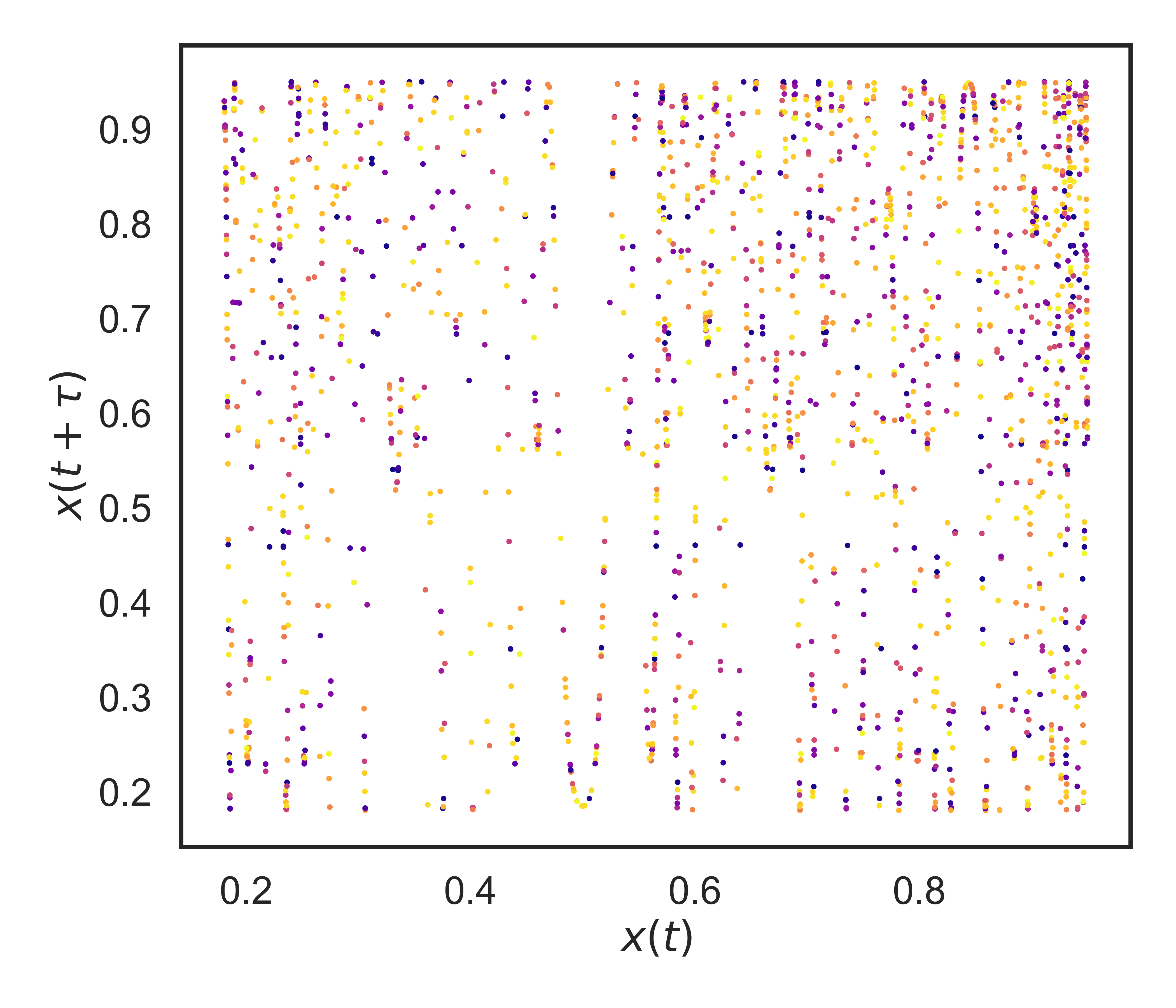}
    \caption{Two-dimensional delay embedding of the logistic map at \(r=3.8\), color-coded by the discrete second derivative \(\lvert x^{\prime\prime}(t)\rvert\).  Warmer hues reflect regions of higher curvature, pinpointing rapid changes in the orbit’s local slope.}
    \label{fig:delay_embedding_curvature}
\end{figure}

In \emph{Figure \ref{fig:spiral_attractor}}, we present a forced, damped oscillator with state equations
\begin{equation}
\begin{aligned}
&\dot{x} \;=\; \dot{x},\\
&\ddot{x} \;=\; -\,0.05\,\dot{x} \;-\; x \;+\; 0.01\,\sin\bigl(5\,t\bigr).
\end{aligned}
\end{equation}
A moderate damping term \(\bigl(-0.05\,\dot{x}\bigr)\) and a small sinusoidal forcing \(\bigl(0.01\,\sin(5t)\bigr)\) produce a \emph{spiral limit cycle} in the \((x,\dot{x})\)-plane.  As the flow evolves, trajectories converge onto a closed orbit whose geometry is slightly perturbed by the external drive, yielding a quasi-harmonic spiral structure with gradual inward drift.  In stark contrast to chaotic systems, such a limit cycle implies that small initial variations decay to the same stable loop, demonstrating phase coherence under forced oscillations.  From an echo state network perspective, having a known stable attractor of this form provides a valuable test for capturing and reproducing mild nonlinearity while maintaining consistent amplitude and phase—a property important in tasks requiring stable and predictable oscillatory outputs.

\begin{figure}[!ht]
    \centering
    \includegraphics[width=0.72\linewidth]{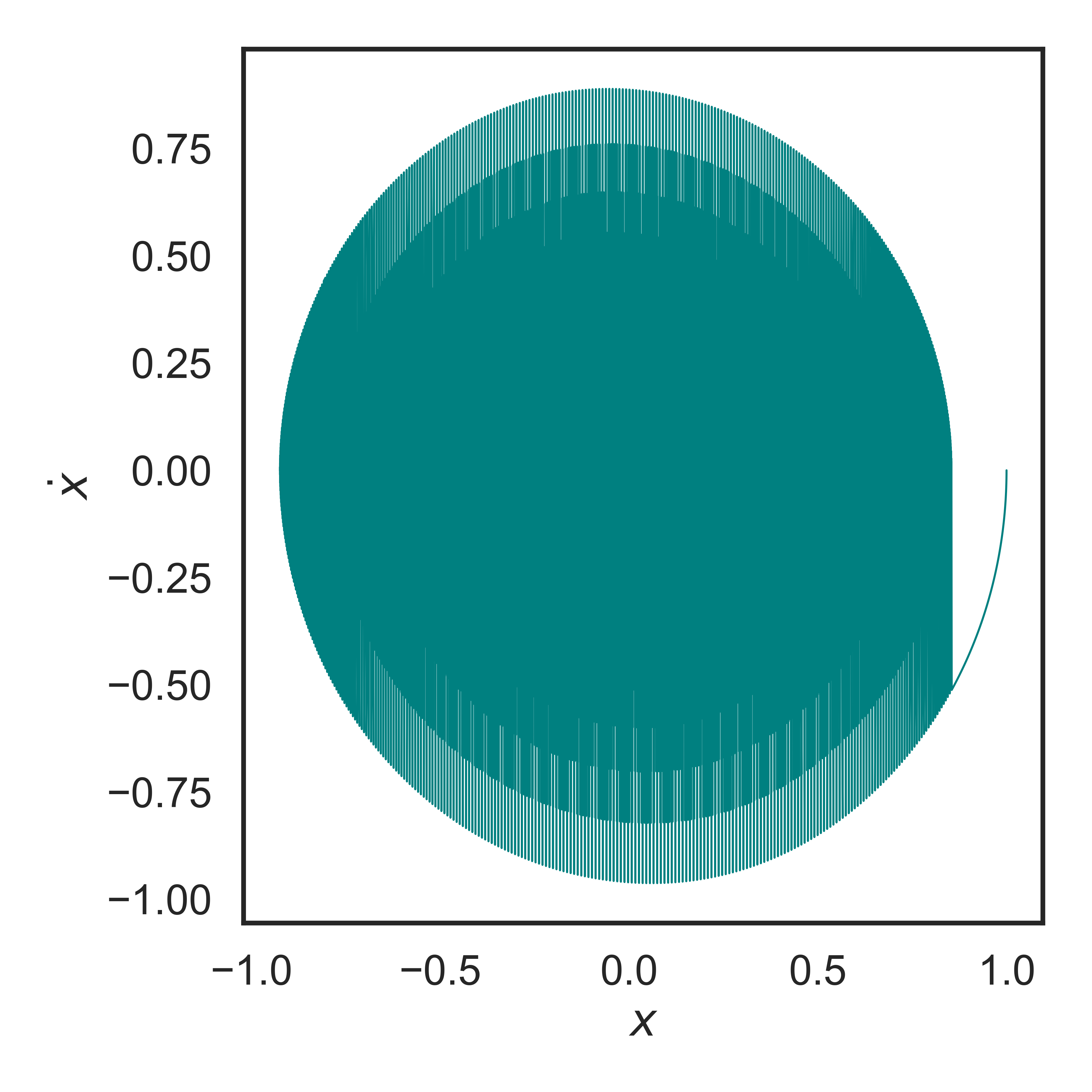}
    \caption{Spiral limit cycle of a forced, damped oscillator in the \((x,\,\dot{x})\)-plane, showing gradual convergence to a stable, spiral-like orbit under moderate damping and periodic forcing.}
    \label{fig:spiral_attractor}
\end{figure}

We  visualize how an 
 ESN  internally encodes the chaotic input \(x(t)\) from the Lorenz system in \emph{Figures \ref{fig:esn_state_pca_lorenz}, \ref{fig:esn_state_pca3_lorenz}}.  Specifically, if each reservoir state \(\mathbf{x}(t)\in \mathbb{R}^{N}\) evolves by:
\begin{equation}
\mathbf{x}(t+1)
\,=\;
\tanh\Bigl(
W\,\mathbf{x}(t)
\;+\;
W_{\mathrm{in}}\;x(t)
\Bigr),
\end{equation}
where \(W\) is scaled to have a spectral radius less than unity (here set to \(0.95\)) and \(W_{\mathrm{in}}\) maps the single-dimensional Lorenz drive into the reservoir.  We then reduce the high-dimensional state trajectory \(\{\mathbf{x}(t)\}_{t=0}^{T-1}\) to two principal components via PCA, yielding a low-dimensional embedding that captures the dominant variance directions of the reservoir’s internal dynamics.  Color-coding by \(x(t)\) (the same Lorenz component used to drive the network) reveals how regions in the projected reservoir manifold correspond to distinct input amplitudes, thus providing an interpretable view of how chaotic signals are represented and unfolded across the internal states of the ESN.  This geometric structure often underlies the network’s ability to perform tasks involving chaotic or time-varying data.

\begin{figure}[!ht]
    \centering
    \includegraphics[width=0.8\linewidth]{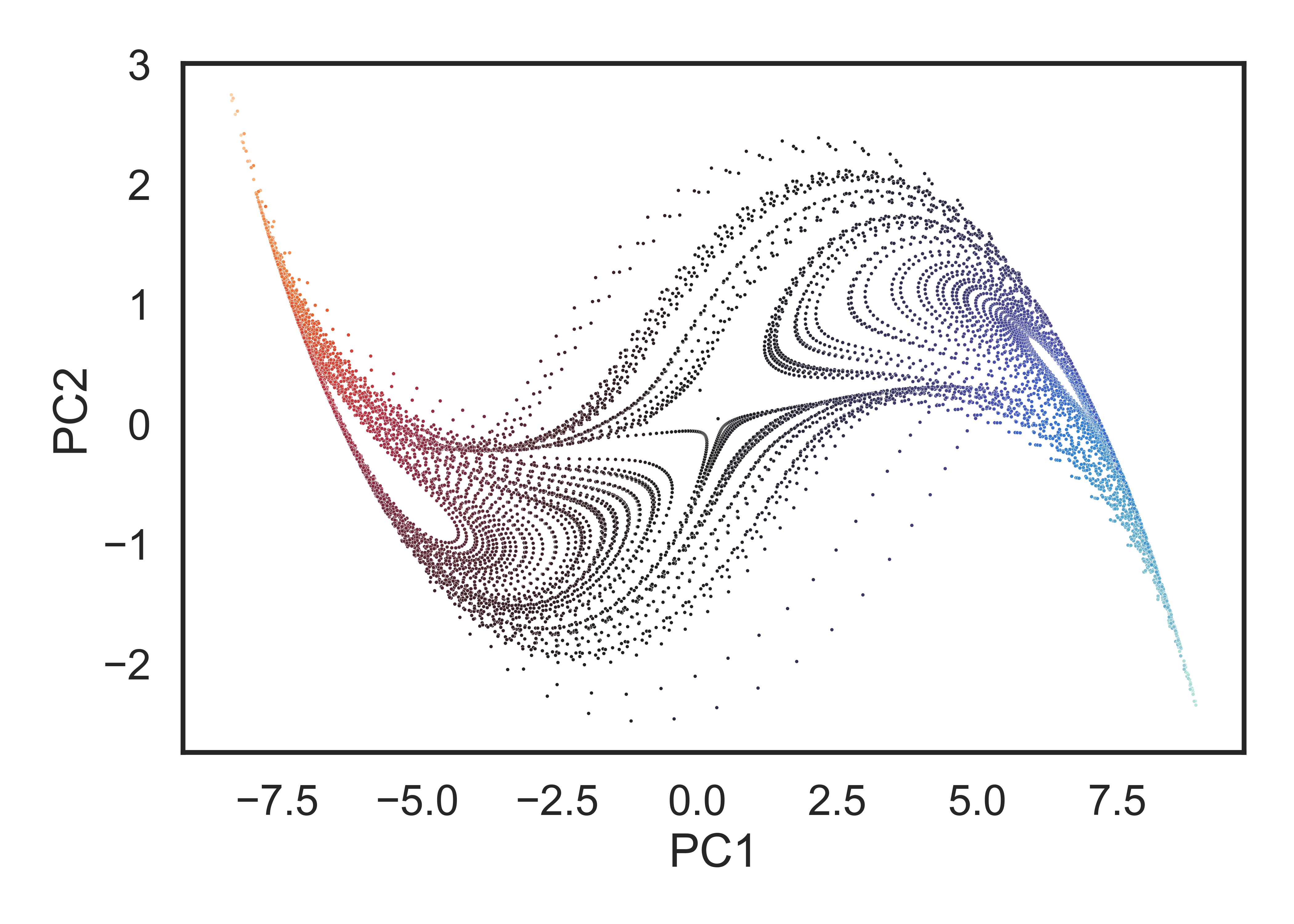}
    \caption{Two-dimensional PCA projection of an ESN’s reservoir states, driven by the Lorenz system’s \(x\)-component.  Points are colored according to the input amplitude \(x(t)\), illustrating the manifold geometry induced by the chaotic driving signal within the reservoir’s high-dimensional state space.}
    \label{fig:esn_state_pca_lorenz}
\end{figure}

\begin{figure}[!ht]
    \centering
    \includegraphics[width=0.99\linewidth]{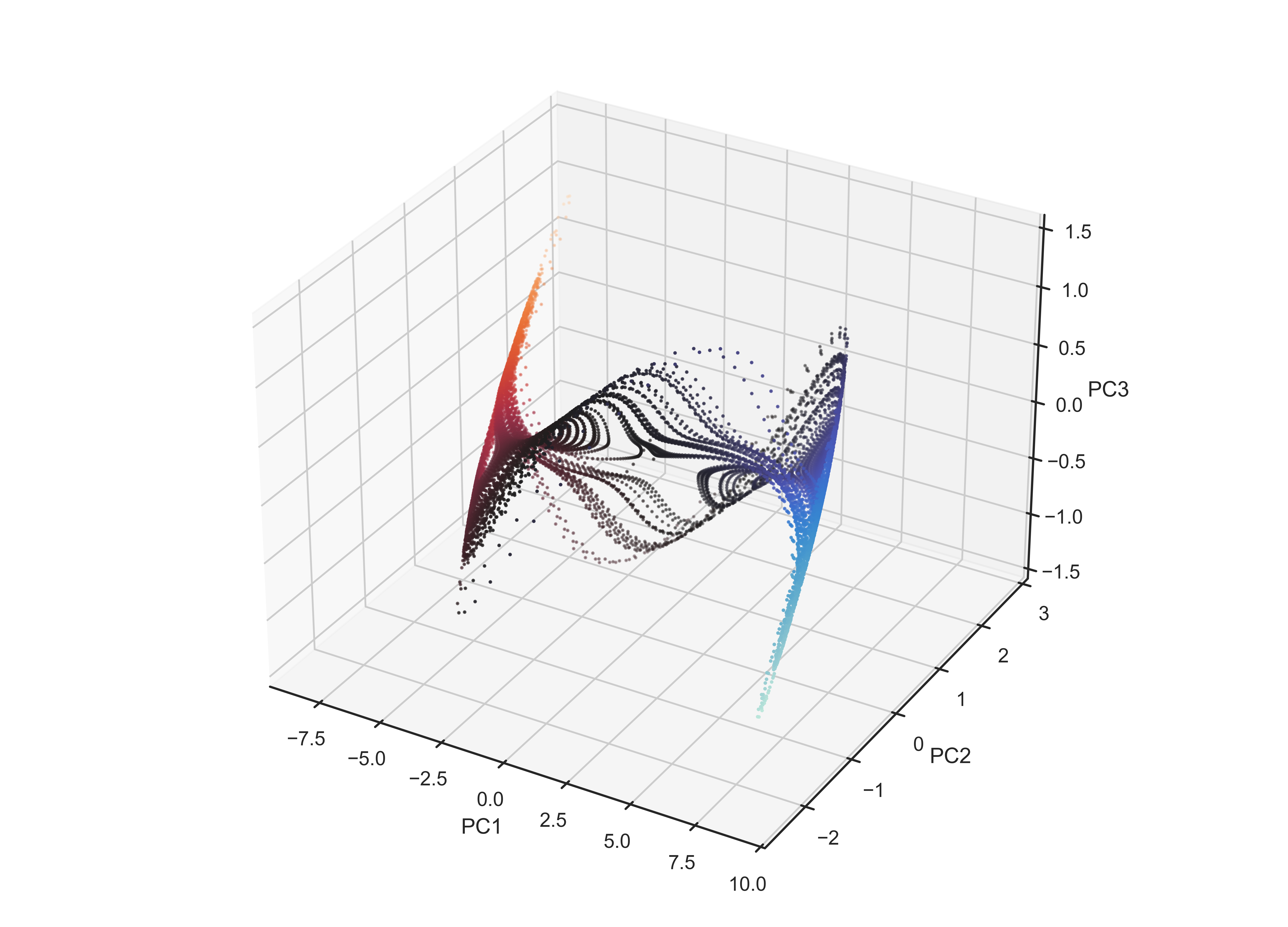}
    \caption{Three-dimensional PCA projection of an ESN’s reservoir states, driven by the Lorenz system’s \(x\)-component.  Points are colored according to the input amplitude \(x(t)\).}
    \label{fig:esn_state_pca3_lorenz}
\end{figure}

We can apply a \emph{symbolic derivative} transformation to an ESN neuron’s time series \(\{s(t)\}\) (see \emph{Figure \ref{fig:esn_symbolic_derivative}}).  Concretely, we compare successive states \(s(t+1) - s(t)\) against a small threshold \(\varepsilon\), classifying the local slope as either rising \((+1)\), falling \((-1)\), or neutral \((0)\).  This step-based encoding highlights rapid sign changes and stagnations in the neuron’s activation pattern, shedding light on how the network’s internal state transitions reflect input-driven or recurrent-driven fluctuations.  In \emph{Figure \ref{fig:esn_state_entropy}}, we further track the \emph{sliding window Shannon entropy} of the reservoir state distribution.  Within each window of length \(W\), we compute a histogram of states and apply the formula
\begin{equation}
H \;=\; -\,\sum_{k=1}^{B} p_{k}\,\ln p_{k},
\end{equation}
where \(p_{k}\) is the empirical probability mass in bin \(k\).  Higher entropy signifies a more dispersed or uniformly populated region of state space, reflecting greater diversity in neuronal activations, whereas lower entropy indicates tighter clustering and more predictable dynamics.  Both the symbolic derivative representation and the windowed entropy profile offer complementary perspectives on the reservoir’s nonlinear processing, elucidating how echoes of past inputs and internal recurrent effects conspire to yield a spectrum of dynamical motifs.

\begin{figure}[!ht]
    \centering
    \includegraphics[width=0.9\linewidth]{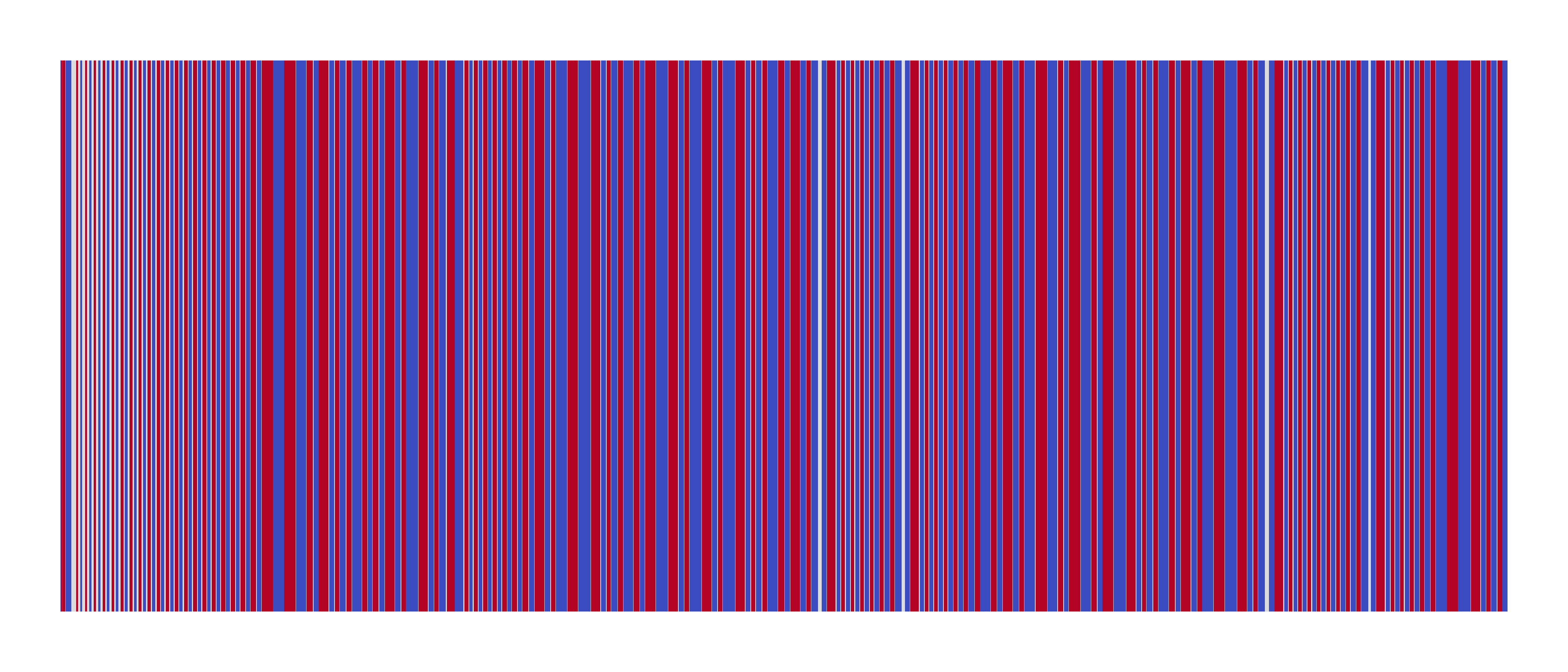}
    \caption{Symbolic Derivative Pattern of ESN Neuron 0.}
    \label{fig:esn_symbolic_derivative}
\end{figure}

\begin{figure}[!ht]
    \centering
    \includegraphics[width=0.8\linewidth]{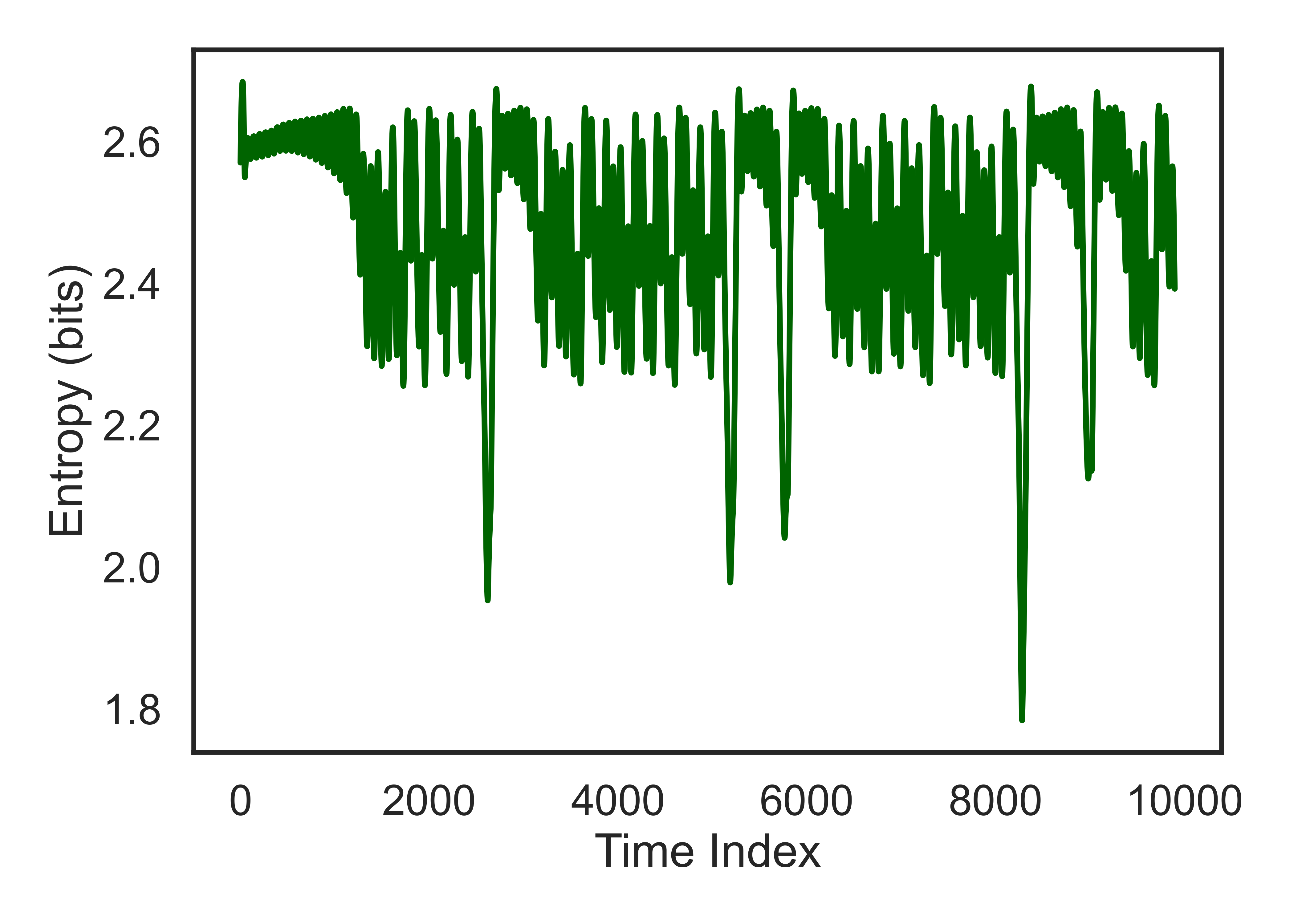}
    \caption{Shannon Entropy of ESN State (Sliding Window).}
    \label{fig:esn_state_entropy}
\end{figure}

We observe the \emph{echo response} of a 20-neuron reservoir subjected to a \(\delta\)-impulse at time \(t=10\) (see \emph{Figure \ref{fig:esn_echo_response}}).  Initially, the input \(u(t)\) is identically zero, until a single impulse of amplitude \(1.0\) perturbs the network.  Each row in the heatmap shows how a particular neuron’s activation evolves for subsequent timesteps.  The resulting transient wave of activations (darker or lighter regions) gradually decays over time, illustrating the reservoir’s capacity to \emph{retain} and \emph{fade out} the impulse as a short-term memory trace.  Such echo responses are central to ESN functionality, where relevant information in the input is mapped into a high-dimensional state space and left to reverberate—often crucial for tasks requiring the reconstruction or prediction of temporally extended signals.

\begin{figure}[!ht]
    \centering
    \includegraphics[width=0.99\linewidth]{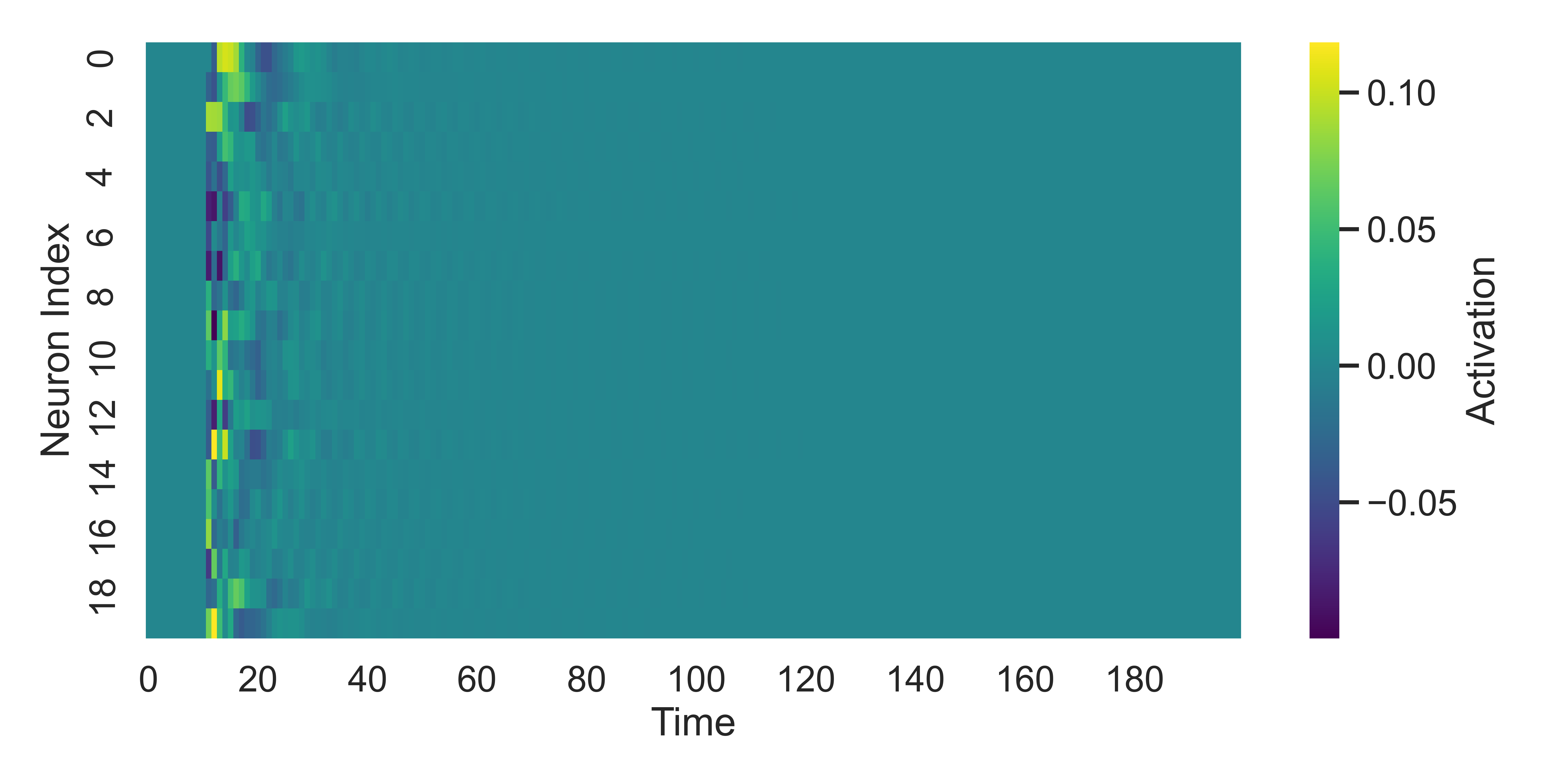}
    \caption{Echo Response of Reservoir to Delta Input.}
    \label{fig:esn_echo_response}
\end{figure}

In \emph{Figure \ref{fig:esn_input_state_correlation}}, we visualize the \emph{input–state cross-correlation matrix} for an ESN driven by a sinusoidal input \(u(t)=\sin(\omega\,t)\).  Specifically, for each neuron \(i\) and each delay \(\tau\in\{0,\dots,29\}\), we compute the Pearson correlation between \(\{\,X[t,\,i]\}_{t}\) and the shifted input signal \(\{\,u[t-\tau]\}_{t}\).  The resulting matrix thus quantifies how each neuron’s activation correlates with past (or near-past) values of the input, revealing the \emph{lag-specific memory} embedded within the reservoir dynamics.  In a well-configured ESN, different subsets of neurons specialize in different temporal offsets, enabling the network as a whole to preserve and process a broad temporal context.  This property plays a central role in tasks such as time-series forecasting, where the reservoir’s collective encoding of past input amplitudes (through these cross-correlations) underpins the system’s predictive power.

\begin{figure}[!ht]
    \centering
    \includegraphics[width=0.96\linewidth]{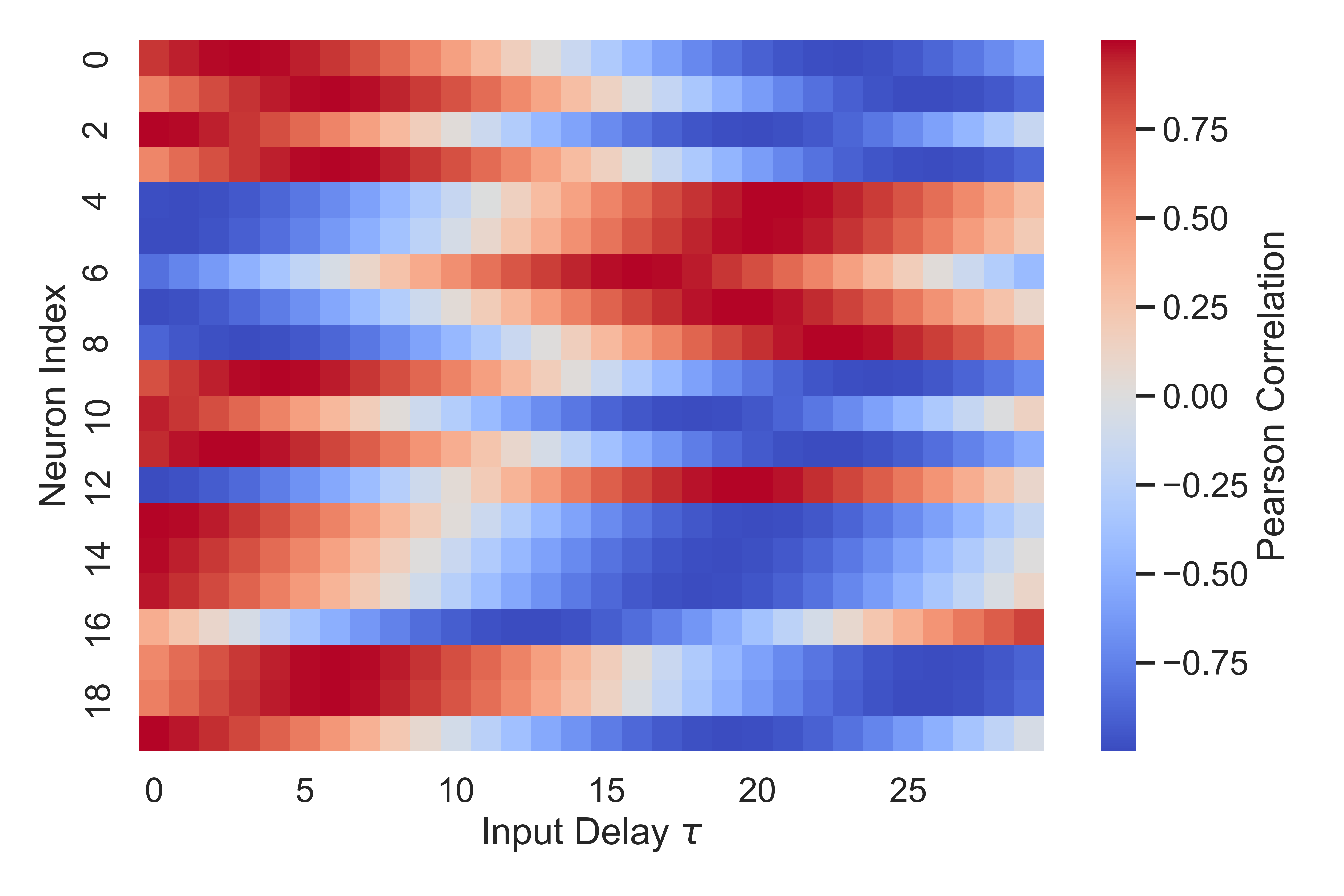}
    \caption{Input–State Cross-Correlation Matrix.}
    \label{fig:esn_input_state_correlation}
\end{figure}

In \emph{Figure \ref{fig:esn_activation_variance}}, we visualize the distribution of \emph{per-neuron activation variance} in an ESN.  For each neuron \(i\), we compute 
\begin{equation}
\mathrm{Var}\bigl[X_{\cdot,i}\bigr]
\;=\;
\frac{1}{T}\sum_{t=1}^{T}\Bigl(x_{t,i}\;-\;\bar{x}_{\cdot,i}\Bigr)^{2},
\end{equation}
where \(x_{t,i}\) is the activation of neuron \(i\) at time \(t\), and \(\bar{x}_{\cdot,i}\) is the empirical mean across all time steps.  Larger variance indicates a broader range of dynamic responses, implying that the neuron explores multiple regions of its activation function, while near-zero variance suggests that it either remains quiescent or saturates near a fixed output level.  In the broader context of reservoir computing, this variance distribution serves as a diagnostic tool for assessing how well the network’s representational capacity is being utilized: if too many neurons exhibit minimal or extreme variance, the ESN may underutilize its state space or risk numerical instability.

\begin{figure}[!ht]
    \centering
    \includegraphics[width=0.8\linewidth]{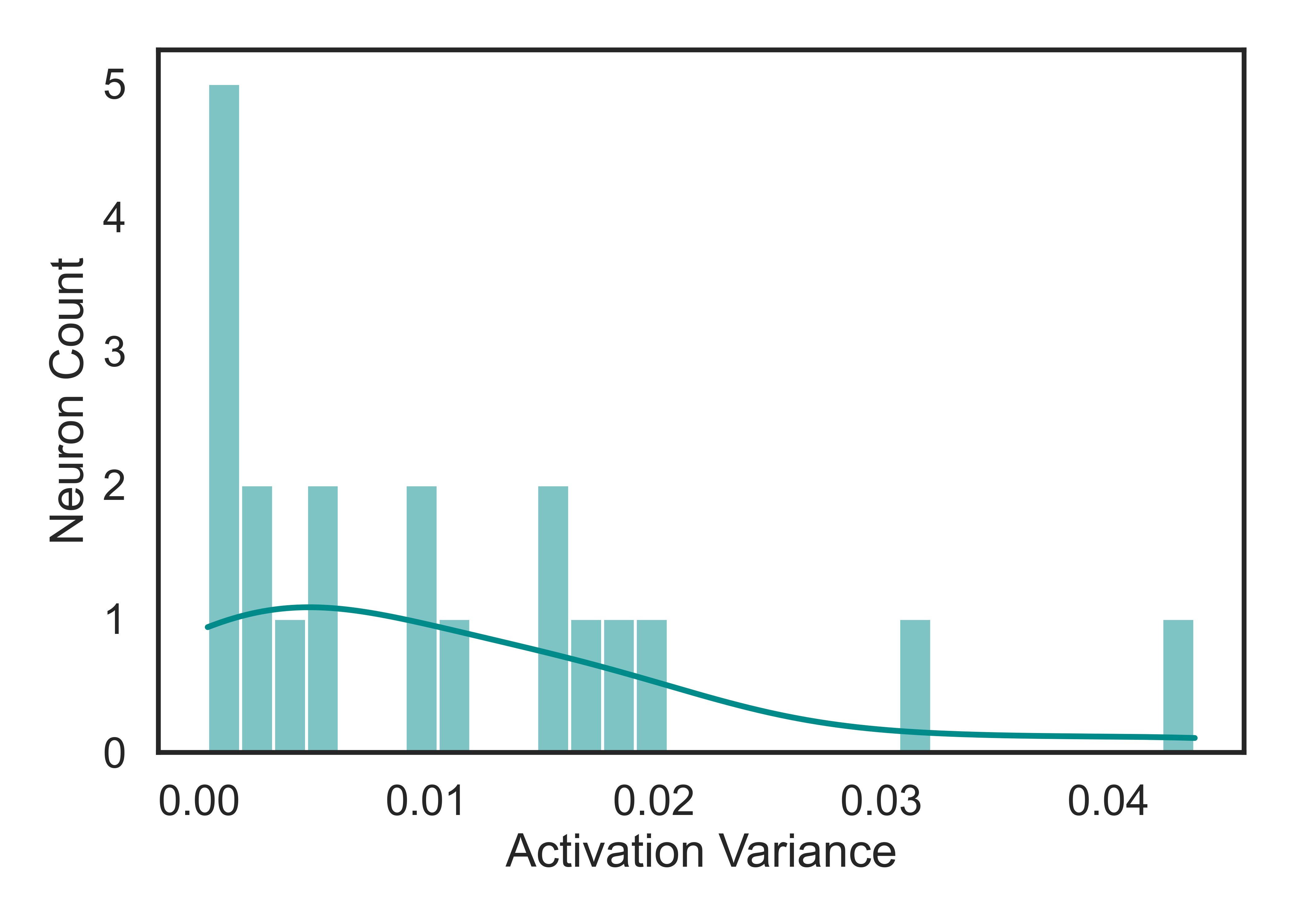}
    \caption{Input–State Cross-Correlation Matrix.}
    \label{fig:esn_activation_variance}
\end{figure}

In Figure \ref{fig:esn_recurrence_density}, we quantify how often the reservoir states revisit (or remain close to) previously visited regions of the state space.  Concretely, for each time point \(t\), we compute the fraction of points \(\mathbf{x}(s)\) (with \(s\) ranging from \(1\) to \(T\)) that lie within an \(\varepsilon\)-ball of \(\mathbf{x}(t)\).  Formally, letting \(\mathbf{x}(t)\) denote the reservoir state in \(\mathbb{R}^{N}\) at time \(t\), we define a \emph{recurrence density}
\begin{equation}
\text{RD}(t) \;=\; \frac{1}{T}\;\#\Bigl\{\,s:\bigl\lVert \mathbf{x}(s)-\mathbf{x}(t)\bigr\rVert < \varepsilon\Bigr\}.
\end{equation}
High values signify local clustering or extended “stickiness” in certain state-space neighborhoods, whereas low values imply frequent excursions away from prior orbits.  Such behavior can reveal whether the reservoir, under sinusoidal drive, forms quasi-stationary regions or transitions rapidly among diverse states.
\begin{figure}[!ht]
    \centering
    \includegraphics[width=0.8\linewidth]{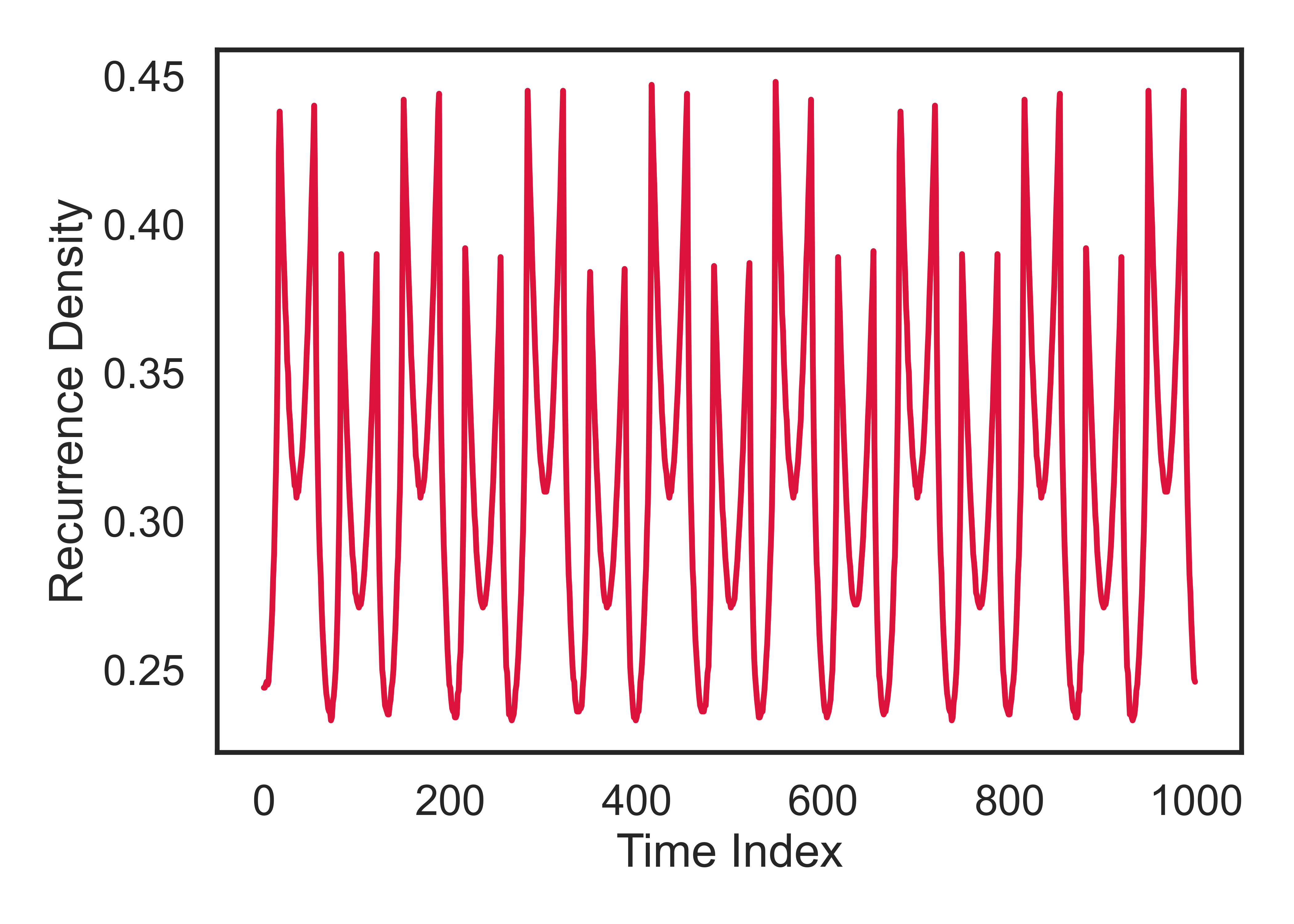}
    \caption{A time series of recurrence density in the reservoir, indicating how much each state \(\mathbf{x}(t)\) re-encounters itself (within distance \(\varepsilon\)) over the course of the simulation.  Peaks highlight stronger local re-visitation, tied to memory-like or attractor-like subdynamics.}
    \label{fig:esn_recurrence_density}
\end{figure}

\begin{figure}[!ht]
    \centering
    \includegraphics[width=0.8\linewidth]{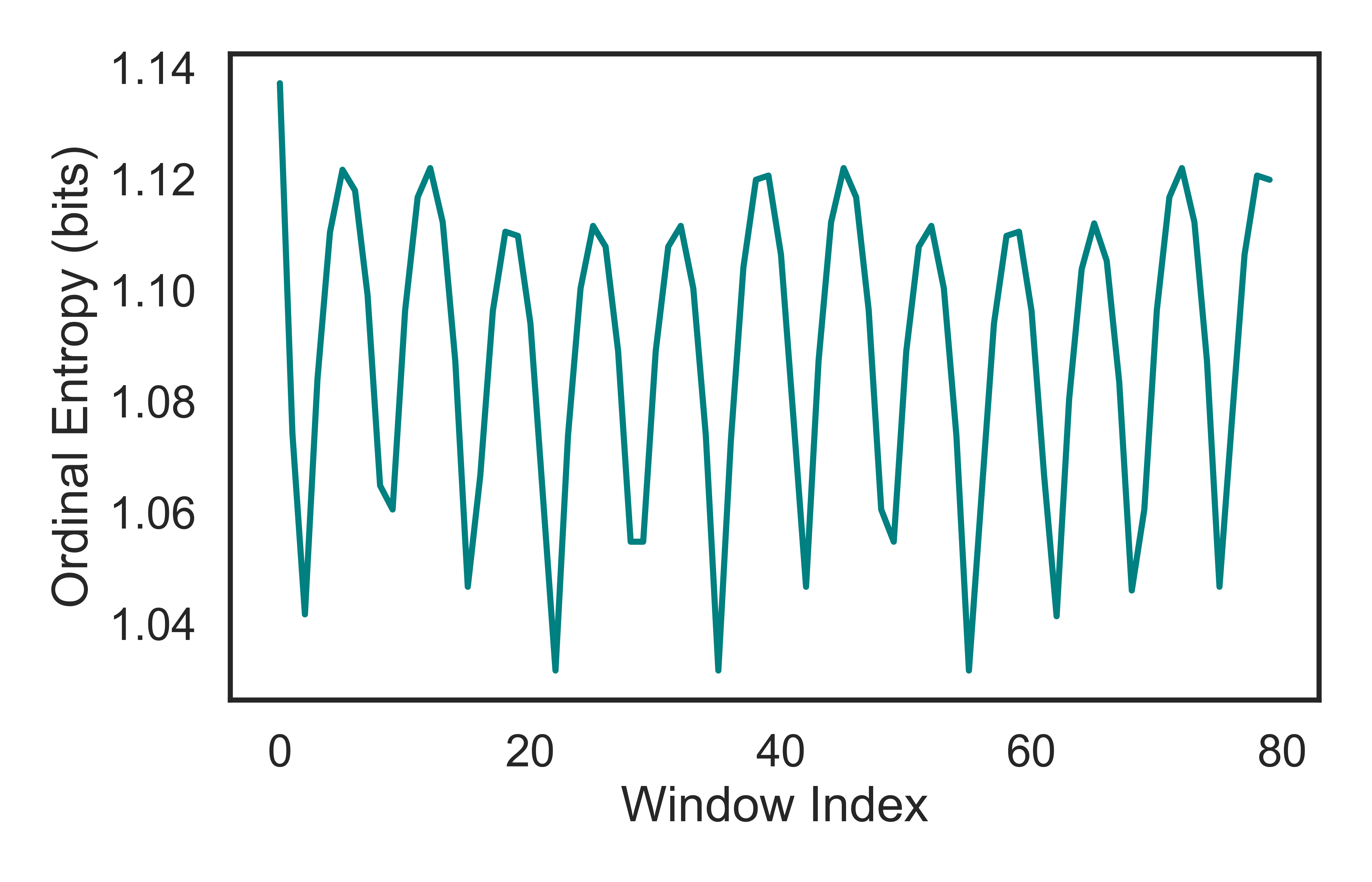}
    \caption{A sliding-window ordinal-entropy trace of the reservoir’s first neuron.  Fluctuations in this measure track time-varying complexity in the neuron’s local firing pattern, where higher entropy correlates with more diverse rank-order subsequences over each window.
}
    \label{fig:esn_symbolic_entropy}
\end{figure}

\begin{figure}[!ht]
    \centering
    \includegraphics[width=0.75\linewidth]{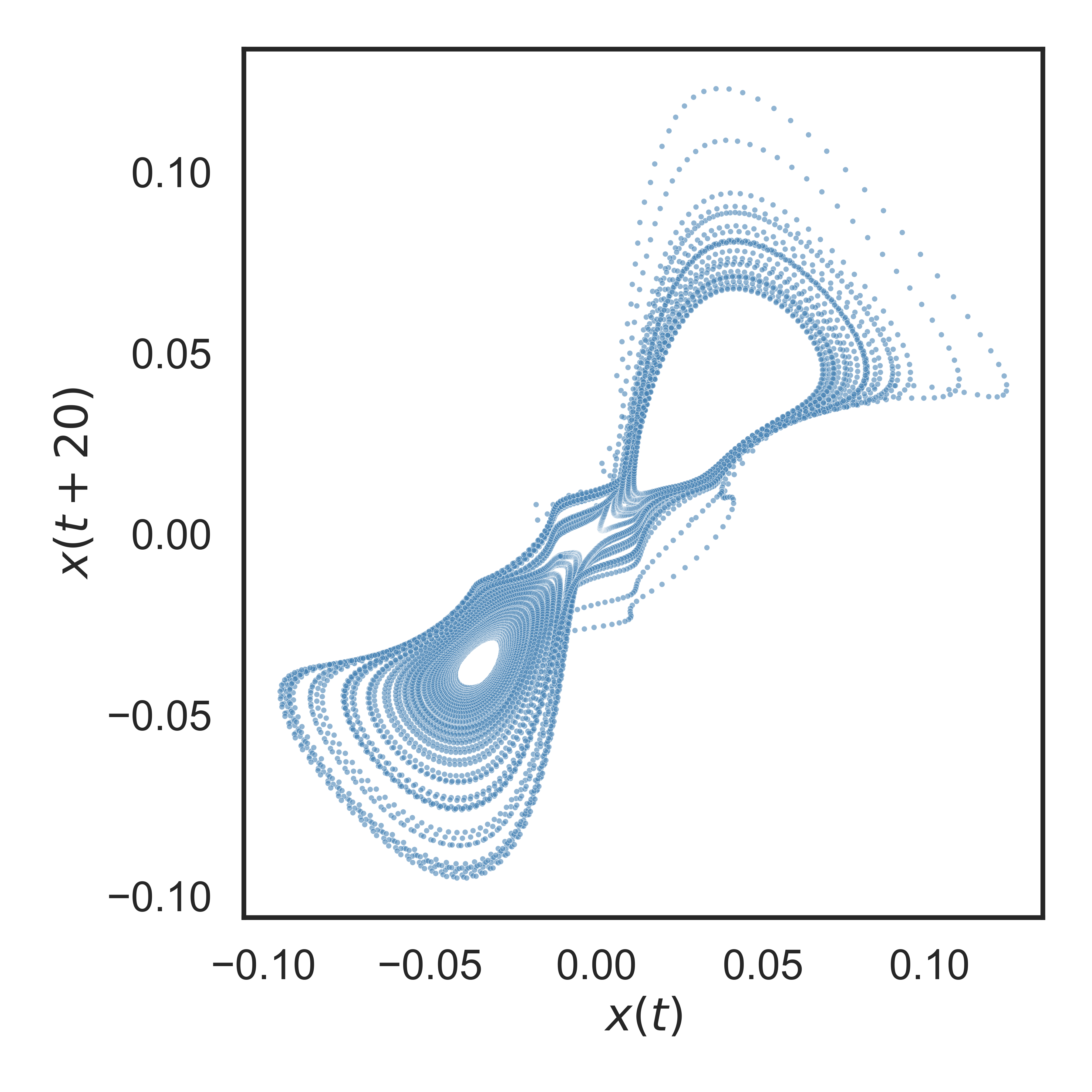}
    \caption{Two-dimensional delay map of a single reservoir neuron’s activation, plotting \(x(t)\) on the horizontal axis vs. \(x(t+\tau)\) on the vertical axis.  The resulting scatter reveals the local geometry of the neuron’s temporal evolution over a fixed lag \(\tau\).
}
    \label{fig:esn_return_map}
\end{figure}

\begin{figure}[!ht]
  \centering
  \begin{subfigure}{0.98\textwidth}
    \centering
    \includegraphics[width=\linewidth]{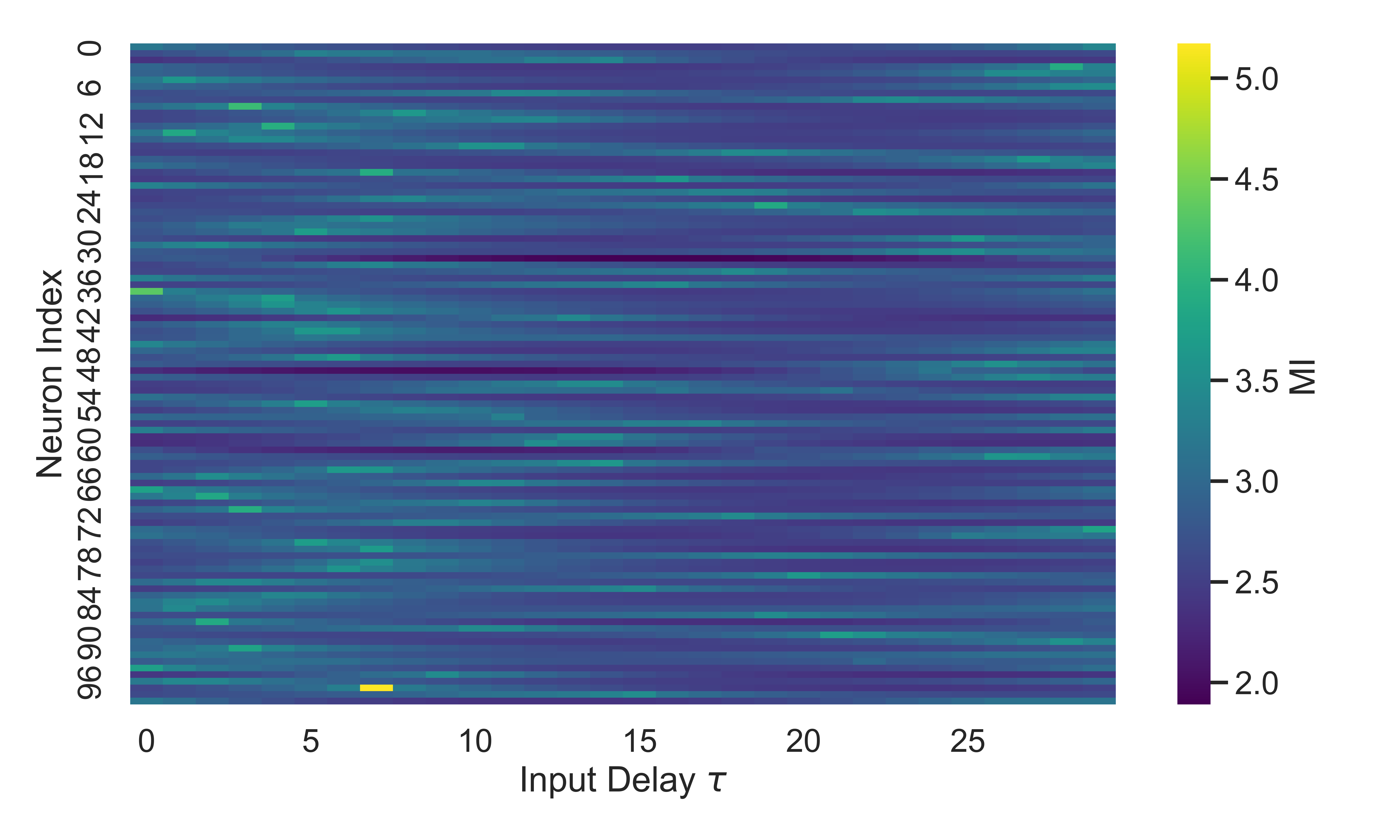}
    \caption{Heatmap of input–neuron mutual information.}
    \label{fig:input_neuron_map}
  \end{subfigure}

  \begin{subfigure}{0.88\textwidth}
    \centering
    \includegraphics[width=\linewidth]{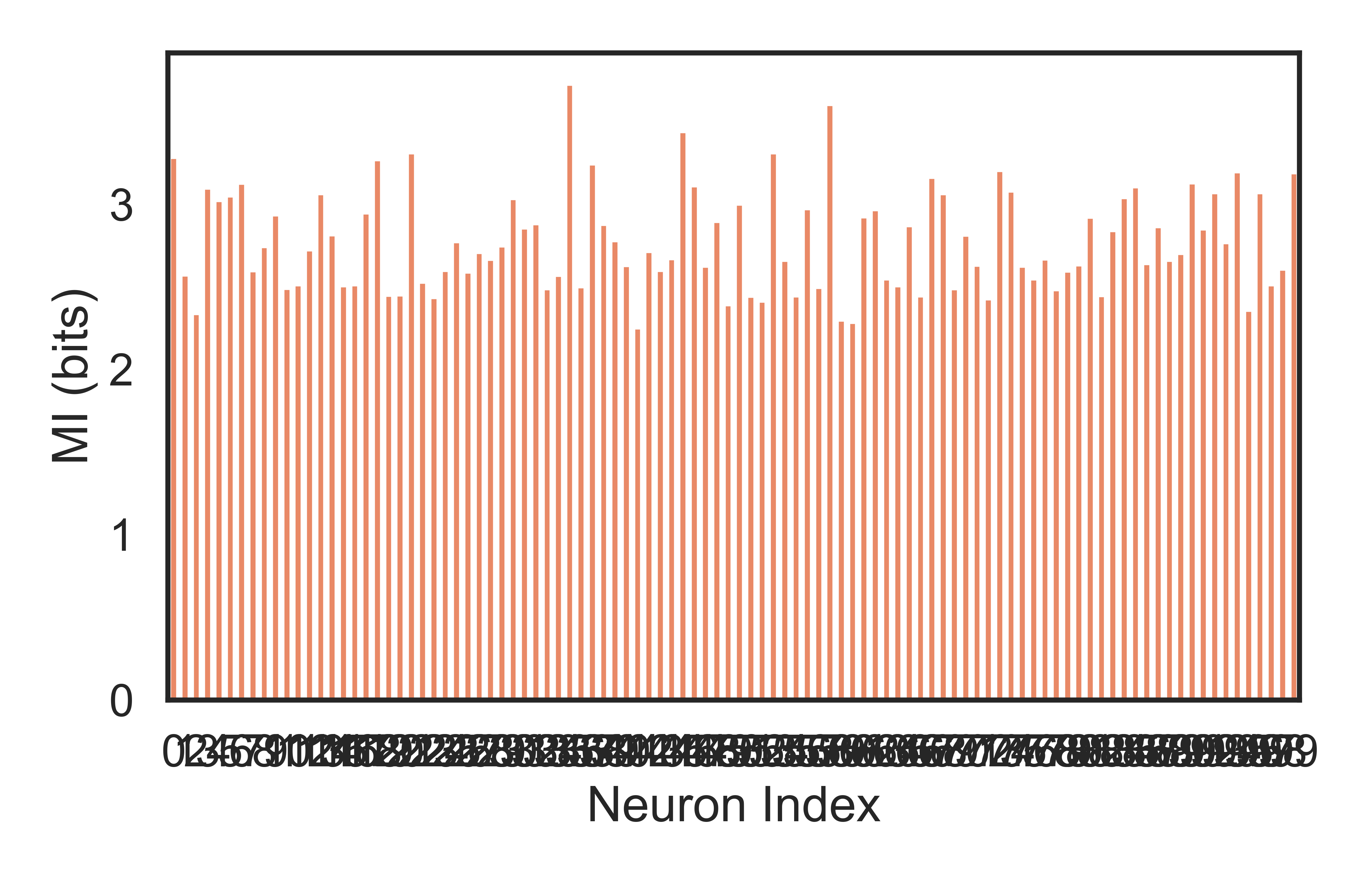}
    \caption{Neuron–output mutual information.}
    \label{fig:neuron_output_map}
  \end{subfigure}
  \caption{(a) Input-lag vs.\ neuron mutual information matrix. 
           (b) Bar plot of neuron–output mutual information, highlighting 
           which neurons strongly influence the reconstructed output.}
  \label{fig:mi_figures}
\end{figure}

In Figure \ref{fig:esn_symbolic_entropy}, we consider the first neuron’s time series \(\{x_{1}(t)\}\) within a sliding window of length \(W\), transform each segment into \emph{ordinal patterns} of dimension \(d=3\), and compute the resulting \emph{symbolic} (permutation) entropy:
\begin{equation}
H_{\mathrm{ord}}(S)
\;=\;
-\;\sum_{\pi\in\Pi_{d}} \;p_{\pi}\;\log_{2}\!\bigl(p_{\pi}\bigr),
\end{equation}
where \(\Pi_{d}\) is the set of all permutations of \(\{0,\dots,d-1\}\), and \(p_{\pi}\) is the empirical probability of pattern \(\pi\).  Higher entropy indicates more uniform coverage of available rank-order patterns, signifying richer, more aperiodic activity in the reservoir neuron, whereas lower entropy suggests repetitive or highly structured firing sequences.  Such symbolic analyses probe the \emph{intrinsic complexity} of neuronal activations, connecting straightforward statistical measures to the deeper chaotic or quasi-periodic dynamics often exploited by echo state networks.

In \emph{Figure \ref{fig:esn_return_map}}, we construct a \emph{delay return map} for a single reservoir neuron’s activation \(x(t)\) by plotting \(\bigl(x(t),\,x(t+\tau)\bigr)\) for a fixed time lag \(\tau\).  The dataset
\(
\Bigl\{\bigl[x(t),\,x(t+\tau)\bigr]\Bigr\}_{t=0}^{T-\tau}
\)
provides a two-dimensional projection of the neuron’s temporal evolution.  Such a map is a classical diagnostic in nonlinear dynamics: periodic or quasi-periodic sequences often yield closed or near-closed loops, while chaos or other complex behavior tends to fill out a region in the plane with intricate geometry.  In the context of echo state networks, examining the delay return map of an individual neuron highlights how that neuron internally “encodes” short-term correlations of the input or recurrent feedback.  Observing whether \(\bigl[x(t),\,x(t+\tau)\bigr]\) forms a narrow band, a loop, or a diffuse cloud can inform us about the neuron’s active regime (e.g., near saturation versus exploring the entire activation range) and thus its potential contribution to the overall memory capacity or predictive power of the reservoir.

In \emph{Figure \ref{fig:mi_figures}}, we measure the \emph{mutual information} (MI) flowing from a time-lagged sinusoidal input \(u(t-\tau)\) to each reservoir neuron’s state \(x_{i}(t)\), assembling a matrix that spans \(\tau=1\) to \(\tau=\max_{\tau}\) and \(i=1,\dots,N\).  Each cell in that matrix thus quantifies how strongly neuron \(i\) encodes or “remembers” the input signal \(\tau\) steps in the past.  Since different neurons typically specialize in different temporal offsets, this plot uncovers how the ESN’s high-dimensional state collectively captures a range of time delays, thereby forming a robust short-term memory mechanism.  We then train a simple linear readout via ridge regression to map the entire reservoir state \(\mathbf{x}(t)\) back to the original input \(u(t)\), and compute MI between each neuron’s activation and the resulting output \(\hat{y}(t)\).  Neurons with higher output MI are more influential in driving the readout to reconstruct the input.  In essence, these two figures reveal that certain neuron–delay combinations exhibit strong correlations with past inputs, while others play a more pivotal role in final output prediction, illustrating both the memory diversity and the partial redundancy that characterizes a well-structured reservoir.

\section{Architectures and Variants}\label{sec:archi}
\subsection{Echo State Networks}
ESNs, pioneered by Jaeger et al. \cite{jaeger2001echo}, sought to address the notorious challenges and inefficiencies in training RNNs. Since their inception, RNNs have been used for sequential pattern recognition, utilizing feedback connections and gradient-based learning methods like backpropagation through time (BPTT) \cite{rumelhart1985learning} and real-time recurrent learning (RTRL) \cite{williams1989learning, catfolis1993method}. However, these methods fail to handle long-term dependencies, require significant computational resources, and experience convergence issues due to bifurcations. 

The ESN model employs a reservoir of leaky-integrated discrete-time continuous-value units for RNN-based computation. At any moment in time \( t \), the network operates on three principal components: the input vector \( \mathbf{u}(t) \in \mathbb{R}^k \), the reservoir state vector \( \mathbf{x}(t) \in \mathbb{R}^n \) and the output vector \( \mathbf{y}(t) \in \mathbb{R}^m \). The reservoir state \( \mathbf{x}(t) \) evolves based on nonlinear functions of the current input \( \mathbf{u}(t) \), the preceding reservoir state \( \mathbf{x}(t-1) \) and optionally the prior output \(\mathbf{y}(t-1)\).  Formally, this update is described as
\begin{equation}
\mathbf{x}(t+1) = (1-\alpha)\mathbf{x}(t) + \alpha f(W_{\text{in}} \mathbf{u}(t) + W \mathbf{x}(t) + W_{\text{back}}\mathbf{y}(t-1) + \mathbf{b}),
\end{equation}
where \( f \) represents a nonlinear activation function, \(\alpha \in (0,1] \) represents the leaking rate and \(\mathbf{b} \in \mathbb{R}^n \) is the bias vector. The matrix \( W \in \mathbb{R}^{n \times n} \) represents the internal connectivity of the reservoir, governing the interactions between its internal units, \( W_{\text{in}} \in \mathbb{R}^{n \times k} \) encapsulates the weights governing the connections between the input and the reservoir, and \( W_{\text{back}} \in \mathbb{R}^{n \times m} \) denotes the connections that feedback from the output units to the internal units. The leaking rate  $\alpha$ is tuned to strike a balance between memory retention and adaptability. A low $\alpha$ retains more historical information, while a high $\alpha$ makes the model more responsive to recent inputs. ESNs can also be used without leaky integration ($\alpha$ = 1) and without feedback from the output to the reservoir, simplifying the update equation to
\begin{equation}
\mathbf{x}(t+1) = f(W_{\text{in}} \mathbf{u}(t) + W \mathbf{x}(t) + \mathbf{b}).
\end{equation}
The output vector $ \mathbf{y}(t)$ is computed as a linear transformation of the reservoir state $\mathbf{x}(t)$, \begin{equation}
\mathbf{y}(t) = W_{\text{out}} \mathbf{x}(t),
\end{equation}
where \( W_{\text{out}} \in \mathbb{R}^{m \times n} \) represents the readout weight matrix, which defines the mapping from the reservoir state to the output.  \( W_{\text{out}} \) constitutes the only set of trainable parameters within the network, making the training process significantly more computationally efficient compared to RNNs, where all weight matrices are typically optimized. 

\begin{figure}[!ht]
    \centering
    \captionsetup{justification=centering}
    \includegraphics[width=0.9\textwidth]{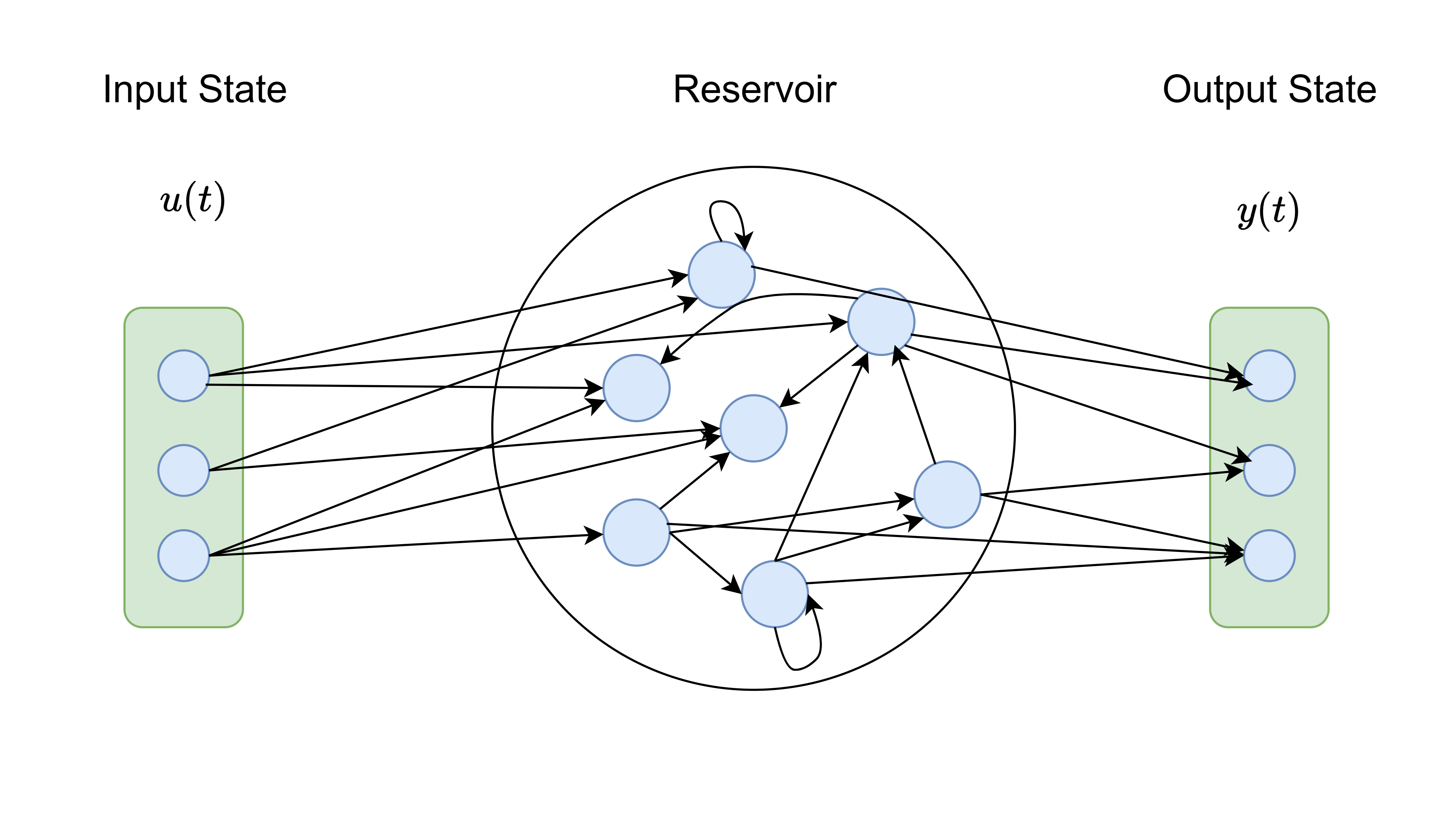}
    \caption{An illustration of a classical Echo State Network. Here $u(t)$ is the input state and $y(t)$ is the output state, the input state is acted upon by $W_{in}$ and the output state is obtained when the reservoir state, $x(t)$ is acted upon by $W_{out}$.}
    \label{fig:Echo State Network}
\end{figure}

The readout weights $W_{\text{out}}$ are computed through linear regression by minimizing the mean squared error (MSE) between the predicted output \( \mathbf{y}(t) \) and the target output \( \mathbf{y}_{\text{target}}(t) \)
\begin{equation}
\text{MSE} = \frac{1}{T - T_{warm-up}} \sum_{t = T_{warm-up} + 1}^{T} \left( \mathbf{y}(t) - \mathbf{y}_{\text{target}}(t) \right)^2,
\end{equation}
where \( T \) is the total number of time steps and \( T_{\text{warm-up}} \) is the transient period required for the reservoir state $\mathbf{x}(t)$ to satisfy the ESP, ensuring it depends solely on the input history rather than the initial state  $\mathbf{x}(0)$. During this warm-up period, the outputs of the reservoir are not used for training or prediction. 

The MSE can be reformulated in matrix form for the sake of compactness. The reservoir states collected after the warm-up period can be stacked into a matrix
\begin{equation}
\mathbf{X} = [\mathbf{x}(T_{warm-up}+1) \ \ \mathbf{x}(T_{warm-up}+2) \ \ ...\ \ \mathbf{x}(T)]^\top,
\end{equation}
where $\mathbf{X} \in \mathbb{R}^{(T - T_{warm-up}) \times n}$. Similarly, the target outputs over the training period can be represented as
\begin{equation}
\mathbf{Y}_{target} = [\mathbf{y}_{target}(T_{warm-up}+1) \ \ \mathbf{y}_{target}(T_{warm-up}+2) \ \ ...\ \ \mathbf{y}_{target}(T)]^\top,
\end{equation}
where $\mathbf{Y}_{target} \in \mathbb{R}^{(T - T_{warm-up}) \times m}$. The predicted output  $\mathbf{Y} \in \mathbb{R}^{(T - T_{warm-up}) \times m}$ can then be expressed as
\begin{equation}
\mathbf{Y} = \mathbf{X} W_{\text{out}}^\top. 
\end{equation}
This allows the MSE to be compactly written as
\begin{equation}
\text{MSE} = \frac{1}{T - T_{warm-up}} \|\mathbf{Y} - \mathbf{Y}_{target} \|^2
\end{equation}
To minimize the MSE, the optimal readout weights $W_{\text{out}}^*$ are computed as
\begin{equation}
W_{\text{out}}^* = \underset{W_{\text{out}}}{\arg \min} \ \|\mathbf{X} W_{\text{out}}^\top - \mathbf{Y}_{target} \|^2
\end{equation}
This leads to the closed-form solution \begin{equation}
W_{\text{out}}^* =  \mathbf{X}^\dagger \mathbf{Y}_{\text{target}},
\end{equation}
where the pseudo-inverse \( \mathbf{X}^\dagger = (\mathbf{X}^\top \mathbf{X})^{-1} \mathbf{X}^\top \). With $L_2$-regularization applied, the updated expression for $W_{\text{out}}^*$ is given by
\begin{equation}
W_{\text{out}}^* = \underset{W_{\text{out}}}{\arg \min} \ (\|\mathbf{X} W_{\text{out}}^\top - \mathbf{Y}_{target} \|^2 + \lambda\|W_{\text{out}}\|^2),
\end{equation}
where $\lambda>0$ is the regularization parameter. The analytical solution for the ridge regression yields
\begin{equation}
W_{\text{out}}^* = (\mathbf{X}^\top\mathbf{X} + \lambda I )^{-1} \mathbf{X}^\top \mathbf{Y}_{\text{target}},
\end{equation}
where $I \in \mathbb{R}^{n \times n}$ is the identity matrix. The regularization term $\lambda I$ ensures $\mathbf{X}^\top\mathbf{X} + \lambda I$ is invertible, thus guaranteeing a unique solution while also penalizing overfitting. This modified objective function ensures that the model not only minimizes the error on the training data but also prevents excessively large weights, thereby improving stability and generalization \cite{ehlers2025improving,Jaurigue2024, Hulser2022}.

In ESNs, the reservoir’s internal connections \( W \), the input weights \( W_{\text{in}} \) and the feedback weights \( W_{\text{back}}\) are randomly initialized and remain fixed during training. This fixed-reservoir design streamlines training by decoupling the reservoir's dynamic behavior from the optimization of the output weights \( W_{\text{out}} \).  This achieves an optimal balance between computational efficiency and dynamic expressiveness, making ESNs well-suited for a variety of sequence modeling tasks.

\subsection{Liquid State Machines}
LSMs, introduced by Maass et al. \cite{maass2002real}, extend the principles of reservoir computing to the domain of spiking neural networks (SNNs) \cite{maass1997networks, malcolm2023snn}. SNNs, considered the \textit{third generation of artificial neural networks}, outperform traditional neural networks by mimicking biological processes in the human brain on many tasks. The development of SSNs was inspired by the flexibility of the neocortex, a key part of the brain responsible for functions such as sensory perception and motor control \cite{kumarasinghe2021brain}.  In contrast to traditional neural networks which require constant weight adjustments via backpropagation, SNNs activate neurons only when they reach a threshold, thus reducing both energy usage and processing time.  Unlike ESNs, which feature continuous-valued activations, LSMs use a reservoir composed of spiking neurons, whose dynamics are governed by discrete events or spikes. These liquid neurons interact in a fluid-like reservoir, where the state of the system changes in response to the input.  The state of the reservoir at time $t$ is represented by the vector $\mathbf{x}(t) \in \mathbb{R}^n$, where each component corresponds to the firing activity of a neuron. The evolution of the reservoir state is described by:
\begin{equation}
\mathbf{x}(t+1) = f_\text{spike}(W_\text{in} \mathbf{u}(t) + W _\text{res}\mathbf{x}(t)),
\end{equation}
where $f_\text{spike}$ models the spiking dynamics of neurons, such as integrate-and-fire \cite{dutta2017leaky} or Hodgkin-Huxley models \cite{hodgkin1952quantitative}. The input weight matrix $W_\text{in}$ and recurrent weight matrix $W_\text{res}$ are initialized randomly and remain fixed, as in ESNs.

The theoretical foundation of LSMs lies in their ability to approximate any fading memory functional with arbitrary precision, provided the reservoir is sufficiently high-dimensional and exhibits rich dynamics \cite{maass2002real}. This universality is mathematically guaranteed by the separation property, which states that distinct input sequences produce distinct trajectories in the reservoir state space. Additionally, the fading memory property ensures that the reservoir's response to inputs decays over time, making it suitable for processing temporal information.

A key advantage of LSMs is their biological plausibility, as their structure and dynamics mimic neural circuits in the brain. The spiking activity of neurons enables LSMs to encode temporal information in a sparse and efficient manner, making them well-suited for real-time processing tasks such as speech recognition, motor control, and robotics. The nonlinearity introduced by spiking dynamics enriches the reservoir's representational power, allowing it to capture complex temporal patterns.

However, the use of spiking neurons introduces additional mathematical challenges in analyzing the dynamics of LSMs. The reservoir dynamics depend on the precise timing of spikes, which are influenced by both the input and the recurrent connections. This requires advanced tools from dynamical systems theory and stochastic processes to model and understand. Despite these complexities, LSMs have demonstrated strong performance in tasks requiring temporal precision and have inspired further research into neuromorphic computing.

Together, ESNs and LSMs represent two complementary approaches to reservoir computing, differing in their mathematical formulations and computational properties. While ESNs focus on simplicity and efficiency, LSMs emphasize biological realism and temporal precision, making both models foundational to the broader RC paradigm.

\subsection{Structured Reservoirs}

While the canonical  ESN architecture employs a random reservoir with a large, sparsely connected recurrent weight matrix \(\mathbf{W}_{\text{res}}\), a variety of alternative topologies have been proposed to enhance theoretical tractability or reduce complexity. These \emph{structured reservoirs} typically replace fully random connectivity with more constrained, deterministic connection patterns that can still support the echo state property and rich dynamical behavior. 

\paragraph{Cycle Reservoirs.}
A \emph{cycle reservoir}, often called a \emph{ring reservoir}, arranges the reservoir neurons in a single directed cycle. Concretely, one may set \(\mathbf{W}_{\text{res}}\) to be an \(N\times N\) matrix such that:
\begin{equation}
    W_{\text{res}}(i, (i+1)\,\bmod\,N) \;=\; \alpha, 
    \quad 
    \text{and} 
    \quad
    W_{\text{res}}(i, j) \;=\; 0 
    \; \text{for all other entries},
\end{equation}
where \(\alpha\) is a scalar regulating the spectral radius. Thus, each neuron sends activation only to the next neuron in a unidirectional loop, and the adjacency matrix is essentially a single off-diagonal of constant \(\alpha\). The cycle reservoir was introduced in part to provide a reduced parameter footprint and to simplify the analysis of state updates. Despite this simplicity, cycle reservoirs often exhibit acceptable performance on time-series prediction tasks as long as \(\alpha\) is chosen to ensure the echo state property \cite{Lukosevicius2012, gallicchio2017deep}.

In such a cycle, the time evolution of the reservoir state \(\mathbf{x}(t) = [x_1(t), \dots, x_N(t)]^T\) is essentially a discrete shift on the ring:
\begin{equation}
    x_{i}(t+1) = f\bigl(\alpha \, x_{i-1}(t) + \mathbf{W}_{\text{in}}(i, :) \,\mathbf{u}(t)\bigr),
\end{equation}
where indices are taken modulo \(N\). Because there is a single feedback loop, the reservoir’s memory and nonlinear transformations are closely tied to \(\alpha\) and the form of \(f(\cdot)\). If \(\alpha\) is too large, the network becomes prone to unstable oscillations or chaotic behavior, whereas too small an \(\alpha\) yields insufficient memory. Intermediate values near the “edge of stability” often yield the best performance.

\paragraph{Cycle Reservoirs with Jumps.}
A straightforward generalization of the cycle reservoir adds \emph{jump connections}, sometimes called “skip connections” or “leapfrog” couplings, that allow for longer-range interactions around the ring. A typical construction might assign an integer jump length \(d > 1\), creating edges of the form:
\begin{equation}
    x_{i}(t+1) = f\bigl(\alpha \, x_{i-d}(t) + \mathbf{W}_{\text{in}}(i, :)\,\mathbf{u}(t)\bigr),
\end{equation}
again respecting modulo \(N\) indexing. One can also include multiple jumps \(\{d_1, d_2, \dots\}\) or randomize their distribution to gain additional expressiveness \cite{wyffels2008improving, gallicchio2017deep}.

These extended links help the reservoir capture a broader set of temporal dependencies and can improve the \emph{separation property} by ensuring that different parts of the input history combine in more varied ways. A cycle with jumps also mitigates the resonance-like effects seen in pure cycles, where signals propagate in an overly regular pattern. Under typical activation functions like \(\tanh\), one still requires that the spectral radius (now reflecting the effective “jumped” adjacency) remains below unity to preserve the echo state property, but the presence of multiple or larger jumps can lead to richer internal dynamics.

\begin{figure}[!ht]
    \centering
    \includegraphics[width=0.9\linewidth]{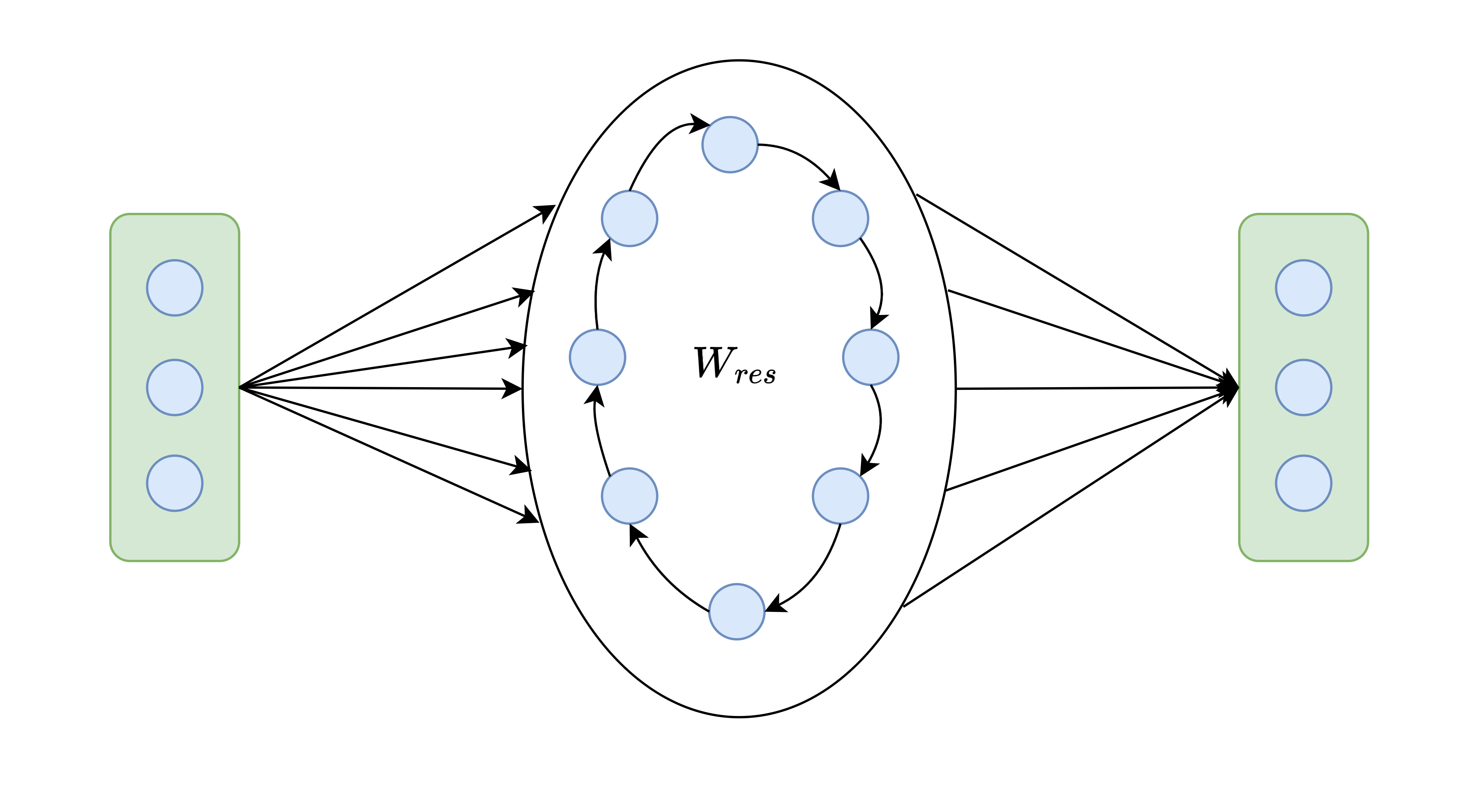}
    \caption{An illustration of a Simple Cycle Reservoir.}
    \label{fig:Simple Cycle Reservoir}
\end{figure}

\paragraph{Minimum Complexity Interaction Echo State Networks.}

The Minimum Complexity Interaction Echo State Network (MCI-ESN), introduced by Liu, Xu, and Li \cite{liu2024minimum}, is a recent architectural variant of reservoir computing that seeks to overcome the representational limitations of simple cycle reservoirs while preserving computational efficiency and analytical tractability. The key innovation is the design of a composite reservoir consisting of two minimally interacting simple cycle sub-reservoirs. Each sub-reservoir is a unidirectional cycle of \( N \) neurons, with interaction achieved through only two interconnections—one unidirectional edge from a neuron in the first cycle to a neuron in the second, and vice versa.

\begin{figure}[!ht]
    \centering
    \includegraphics[width=0.9\textwidth]{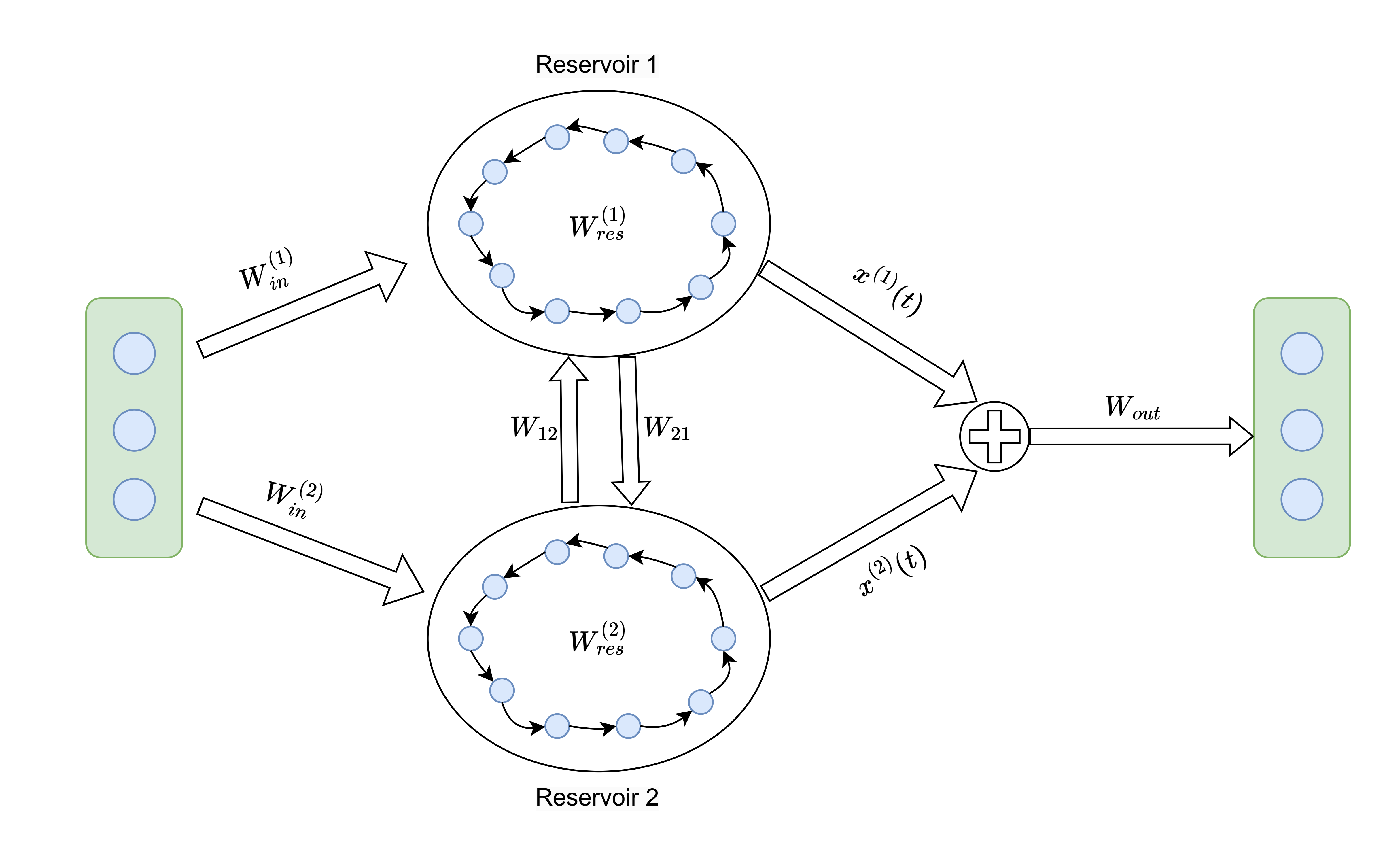}
    \caption{An illustration of the Minimum Complexity Interaction ESN.}
    \label{fig:Minimum Complexity Interaction (MCI) Echo State Networks}
\end{figure}

Let \(\mathbf{x}^{(1)}(t), \mathbf{x}^{(2)}(t) \in \mathbb{R}^N\) denote the reservoir states of the two sub-reservoirs at time \(t\). The update equations for the MCI-ESN are given by:
\begin{align}
    \mathbf{x}^{(1)}(t+1) &= f\left(\mathbf{W}^{(1)}_{\text{res}} \mathbf{x}^{(1)}(t) + \mathbf{W}^{(1)}_{\text{in}} \mathbf{u}(t) + \mathbf{W}_{12} \mathbf{x}^{(2)}(t) \right), \\
    \mathbf{x}^{(2)}(t+1) &= f\left(\mathbf{W}^{(2)}_{\text{res}} \mathbf{x}^{(2)}(t) + \mathbf{W}^{(2)}_{\text{in}} \mathbf{u}(t) + \mathbf{W}_{21} \mathbf{x}^{(1)}(t) \right),
\end{align}
where \(\mathbf{u}(t) \in \mathbb{R}^m\) is the input signal at time \(t\),
 \(\mathbf{W}^{(1)}_{\text{res}}, \mathbf{W}^{(2)}_{\text{res}} \in \mathbb{R}^{N \times N}\) are the cyclic reservoir weight matrices (each representing a simple directed ring),
 \(\mathbf{W}^{(1)}_{\text{in}}, \mathbf{W}^{(2)}_{\text{in}} \in \mathbb{R}^{N \times m}\) are the input weight matrices,
 \(\mathbf{W}_{12}, \mathbf{W}_{21} \in \mathbb{R}^{N \times N}\) are sparse inter-reservoir connection matrices with only a single non-zero entry each, indicating a single neuron-to-neuron interaction between the sub-reservoirs, and
 \(f(\cdot)\) is a component-wise activation function, typically chosen as \(\tanh\).

This coupling topology significantly increases the expressive power of the network while maintaining low connectivity and structural simplicity. Importantly, only \(2N + 2\) non-zero entries exist in the total recurrent weight matrix, in contrast to the \(\mathcal{O}(N^2)\) connections in dense ESNs. The full reservoir state is given by concatenation:
\begin{equation}
    \mathbf{x}(t) = 
    \begin{bmatrix}
    \mathbf{x}^{(1)}(t) \\
    \mathbf{x}^{(2)}(t)
    \end{bmatrix}
    \in \mathbb{R}^{2N}.
\end{equation}
The readout is then formed as in classical ESNs:
\begin{equation}
    y(t) = \mathbf{W}_{\text{out}} \, \mathbf{x}(t),
\end{equation}
where \(\mathbf{W}_{\text{out}} \in \mathbb{R}^{1 \times 2N}\) is typically learned by ridge regression.

A sufficient condition for the \emph{echo state property} of the MCI-ESN is provided in \cite{liu2024minimum}. Specifically, under the assumption that \(f(\cdot)\) is Lipschitz continuous with Lipschitz constant \(L_f\), the ESP holds if the effective spectral norm of the joint recurrent matrix \(\mathbf{W}_{\text{eff}}\) satisfies
\begin{equation}
    L_f \, \|\mathbf{W}_{\text{eff}}\| < 1,
\end{equation}
where \(\mathbf{W}_{\text{eff}}\) is the block matrix:
\begin{equation}
    \mathbf{W}_{\text{eff}} = 
    \begin{bmatrix}
    \mathbf{W}^{(1)}_{\text{res}} & \mathbf{W}_{12} \\
    \mathbf{W}_{21} & \mathbf{W}^{(2)}_{\text{res}}
    \end{bmatrix}.
\end{equation}
Given the sparse structure of \(\mathbf{W}_{12}\) and \(\mathbf{W}_{21}\), the norm \(\|\mathbf{W}_{\text{eff}}\|\) can be tightly bounded, facilitating easy ESP verification during design.

Empirical evaluations in the original paper show that MCI-ESNs outperform traditional simple cycle reservoirs on both chaotic time-series prediction and classification tasks. For example, in predicting multivariate chaotic systems like the Rössler and Mackey–Glass series, the MCI-ESN achieved lower normalized root mean square error (NRMSE) compared to single-cycle and random ESNs. For time-series classification (e.g., on the Japanese vowels and spoken Arabic digits datasets), MCI-ESNs achieved superior accuracy, attributed to the richer dynamic interactions enabled by the inter-reservoir connections \cite{liu2024minimum}.

This architecture thus strikes a balance between \emph{structural minimalism} and \emph{functional expressiveness}. With only two additional interconnections beyond the cycle structure, the MCI-ESN significantly enhances the separation and memory capacities of its internal dynamics, demonstrating that careful topological design can yield gains in both theoretical tractability and practical performance.
By reducing the reservoir connectivity to a deterministic pattern—often combining cycles, shifts, or permutations with simple scaling—these variants highlight that the core functionality of ESNs does not strictly require large random graphs. The critical requirements are:
1.) A suitably chosen spectral radius \(\rho(\mathbf{W}_{\text{res}})\) below or near unity.
2.) A sufficient “mixing” or “separation” capacity within the network so that distinct input patterns generate distinguishable trajectories in state space.
3.) A nonlinear activation that, in conjunction with the reservoir topology, provides rich transformations of input streams.

These structured reservoirs (cycle-based or minimum-complexity designs) thus represent compelling alternatives to dense random matrices. They often yield clearer mathematical analysis of echo state properties, simpler hyperparameter tuning procedures, and competitive or improved performance on real-world tasks. As reservoir computing continues to evolve, such architectures underscore the versatility of the ESN concept and the importance of balancing complexity, interpretability, and computational cost.

\paragraph{Small-World Reservoirs.}
A further variation on reservoir design draws inspiration from \emph{small-world} graphs, originally popularized through the work of Watts and Strogatz \cite{watts1998collective}. In such graphs, most nodes are not directly connected to each other, yet the average path length between nodes remains small due to a sparse set of long-range links. This property can be implemented in a reservoir by starting with a regular or locally connected ring (as in a cycle reservoir) and then randomly \emph{rewiring} a fraction of the edges to connect distant nodes \cite{maslennikov2017small, kim2007small}. 

Formally, one might define an adjacency matrix \(\mathbf{W}_{\text{res}}\) by initially linking each neuron only to its \(k\) nearest neighbors in a cycle-like structure, and then, with probability \(p\), rewire each connection to a randomly chosen node elsewhere in the network. The result is a “small-world” topology characterized by:
1.) High clustering coefficients (local connectivity akin to a ring).
2.) Small average path length across the entire network, thanks to the randomly placed shortcuts.

Such a construction can enhance a reservoir’s ability to propagate information quickly across the state space while preserving significant local structure. Empirical evaluations on time-series and classification tasks suggest that small-world reservoirs can achieve improved performance over purely random or purely local topologies, often attributed to the balance between \emph{local recursion} (enhancing memory of nearby states) and \emph{long-range interactions} (improving separation of different input trajectories) \cite{kim2007small, maslennikov2017small}.

From a dynamical systems perspective, small-world reservoirs can lie between purely regular and purely random networks, offering control over measures such as clustering, path length, and degree distribution. The \emph{spectral radius} and echo state property still remain central to ensuring that the reservoir does not drift into pathological regimes (either too stable or too chaotic). Nonetheless, numerical experiments have demonstrated that moderate rewiring probabilities \(p\) can improve the memory capacity and nonlinear transformation capabilities of ESNs without necessitating the dense connectivity of standard random reservoirs \cite{maslennikov2017small, Lukosevicius2012}.
Overall, small-world reservoirs underscore the idea that \emph{not all random graphs are created equal}: introducing controlled structural biases—in this case, small-world characteristics—can enhance reservoir performance while still maintaining analytical tractability and relatively low-dimensional parameter tuning.

\subsection{Hybrid Architectures}
While traditional RC models like ESNs and LSMs have demonstrated strong capabilities in temporal processing, their limitations in scalability, adaptability, and efficiency have led to the development of hybrid and advanced architectures. These models integrate RC with other neural network paradigms, computational frameworks, or novel design principles to enhance performance across diverse applications. Hybrid architectures combine reservoirs with deep learning, parallelism, or attention mechanisms. Additionally, emerging approaches like quantum and optical reservoir computing push the boundaries of efficiency and computational power.

\subsubsection{Parallel Reservoir Computing.} Parallel Reservoir Computing \cite{pathak2018model} is inspired by Convolutional Neural Networks (CNNs). Similar to how CNNs extract local features from input data, this approach utilizes a distributed architecture comprising multiple moderately sized, parallel reservoirs.  Each reservoir focuses on processing a specific local region of the spatiotemporal system, generating predictions based on its localized input. By exploiting the inherent spatial structure of the data, this approach enables efficient and scalable handling of high-dimensional systems, where complex global dynamics arise from localized interactions.

In this framework, the spatiotemporal system with some periodic boundary producing a multivariate time series is collectively denoted by the vector \( \mathbf{u}(t) \in \mathbb{R}^N \). To manage the complexity of the system, the \( N \) variables \( u_j(t) \) are partitioned into \( m \) groups, each containing \( q \) spatially contiguous variables such that \( m \times q = N\). The state of the spatial points within each group is represented by \( \mathbf{y}_i(t)  \in \mathbb{R}^{q} \ and \  i\in\{1, 2, \dots, m\} \), and are given as:
\begin{equation}
\begin{aligned}
    \mathbf{y}_1(t) &= \left( u_1(t), u_2(t), \dots, u_q(t) \right)^T,  \\
    \mathbf{y}_2(t) &= \left( u_{q+1}(t), u_{q+2}(t), \dots, u_{2q}(t) \right)^T, \\
    &\quad \vdots \\
    \mathbf{y}_m(t) &= \left( u_{(m-1)q+1}(t), u_{(m-1)q+2}(t), \dots, u_{N}(t) \right)^T.
\end{aligned}
\end{equation}

A set of parallel reservoir architecture is constructed consisting of $m$ reservoirs \(\{R_1, R_2, \dots, R_m\}\) and each group of time series \( \mathbf{y}_i(t) \) is predicted by a dedicated reservoir \( R_i \), which is characterized by its internal adjacency matrix \( \mathbf{W}_i \in \mathbb{R}^{ n_i \times n_i}\), an internal state vector \( \mathbf{x}_i(t)  \in \mathbb{R}^{n_i} \), and an input weight matrix \( \mathbf{W}_{\text{in},i}  \in \mathbb{R}^{n_i \times \\(q+b)} \). The input to the \( i \)-th reservoir, denoted by \( \mathbf{v}_i(t)  \in \mathbb{R}^{(q+b)} \), is designed to include not only the spatial points within the \( i \)-th group but also two adjacent buffer regions of size \( b \) on either side. This design ensures that each reservoir has access to a localized spatial context, which is critical for capturing the spatiotemporal dependencies inherent in the system. The input vector \( \mathbf{v}_i(t) \) is formally defined as:
\begin{equation}
\mathbf{v}_i(t) = (u_{(i-1)q-b+1}(t), u_{(i-1)q-b+2}(t), \dots, u_{iq+b}(t))^T,
\end{equation}
where the subscript \( j \) in \( u_j(t) \) is taken modulo \( N \) to respect the periodic boundary conditions. The locality parameter \( l \) determines the extent of overlap between adjacent reservoirs, facilitating smooth transitions and information sharing across neighboring regions. This overlapping structure is reminiscent of the translational invariance employed by CNNs, where local features are processed independently but collectively contribute to a global understanding of the system.
\begin{figure}[!ht]
    \centering
    \captionsetup{justification=centering}
    \includegraphics[width=0.99\textwidth]{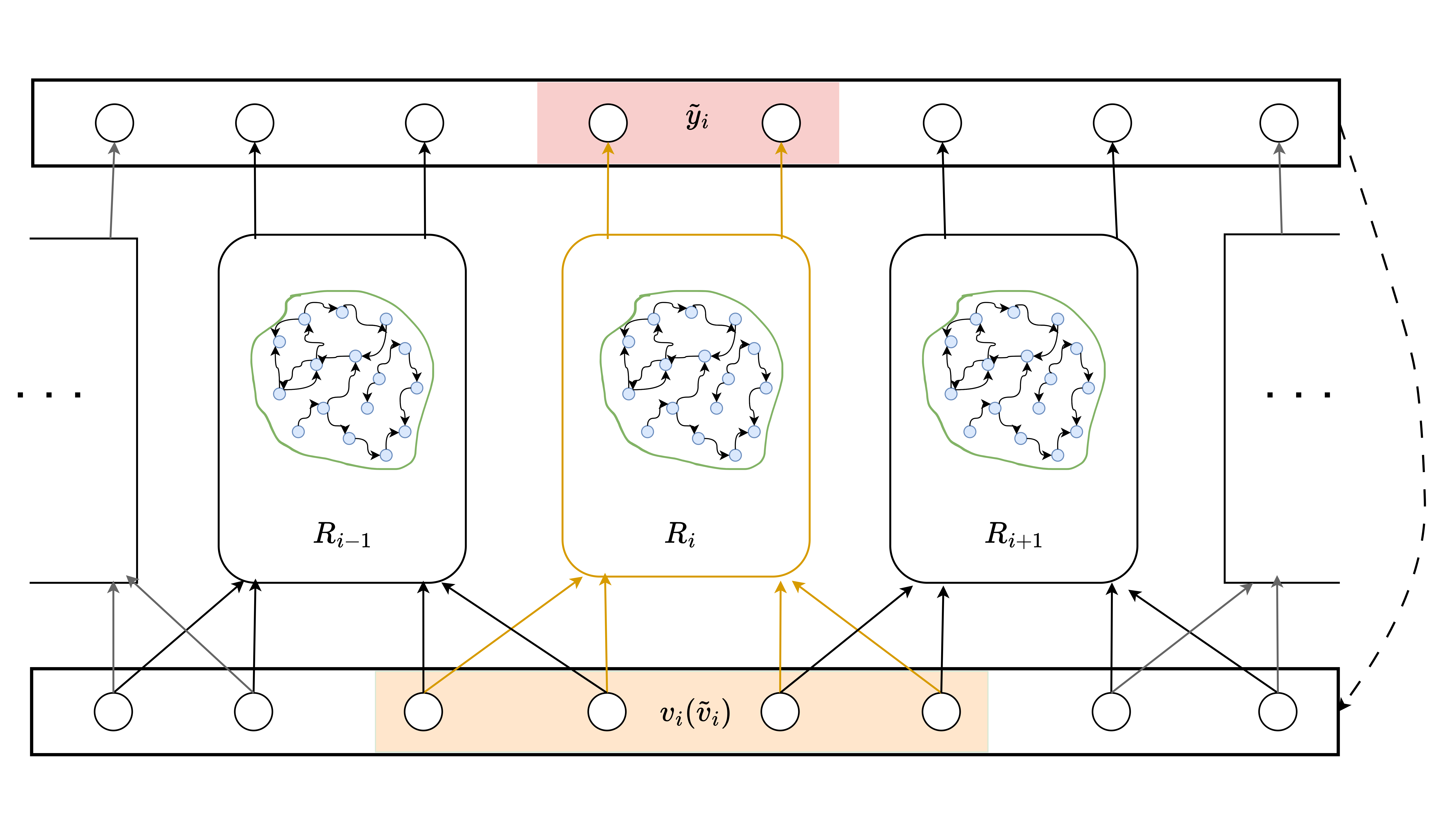}
    \caption{An illustration of Parallel Reservoir Computing.(adapted from \cite{pathak2018model}). The orange shaded region represents the input, $v_{i}$ represents input during training and $\tilde{v}_{i}$ during prediction. The dotted line represents the feedback during prediction phase.
}
    \label{fig:Parallel Reservoir Computing}
\end{figure}
The evolution of each reservoir's internal state is governed by a nonlinear update equation driving the system's dynamics during both training and prediction phases:
\begin{equation}
\mathbf{x}_i(t + 1) = f\left(\mathbf{W}_i \mathbf{x}_i(t) + \mathbf{W}_{\text{in},i} \mathbf{v}_i(t)\right), \quad 1 \leq i \leq g.
\end{equation}
Here, \(\mathbf{x}_i(t)\) represents the internal state vector of the \(i\)-th reservoir at time \(t\), encoding the reservoir's dynamical information. The updated state \(\mathbf{x}_i(t + 1)\) is computed for the next time step using the reservoir's adjacency matrix \(\mathbf{W}_i\), which defines the fixed internal connections, and the input weight matrix \(\mathbf{W}_{\text{in},i}\), which maps the input vector \(\mathbf{v}_i(t)\) to the reservoir's state space and \(f\) is a non-linear activation function.

Once the reservoir states are updated, the \textit{readout phase} computes the output of the system. The readout is typically a linear combination of the reservoir states, often augmented with additional terms to capture higher-order interactions. For the \(i\)-th reservoir, the output \(\mathbf{y}_i(t)\) is given by:
\begin{equation}
\mathbf{y}_i(t) = \mathbf{W}_{\text{out},i} \mathbf{x}_i(t),
\end{equation}
where \(\mathbf{W}_{\text{out},i} \in \mathbb{R}^{ q \times n_i}\) is the output weight matrix, which is trained to map the reservoir states to the desired output. During training, \(\mathbf{W}_{\text{out},i}\) is optimized to minimize the prediction error, ensuring that the reservoir's output closely matches the target dynamics as discussed in earlier  \subsectionautorefname. This readout mechanism allows the reservoir to generate accurate predictions for the local region of the spatiotemporal system it is responsible for, while the collective outputs of all \(g\) reservoirs provide a global prediction for the entire system. The simplicity of the readout phase, combined with the rich dynamics of the reservoir states, makes this approach both computationally efficient and highly effective for modeling complex systems. The architecture is given in Figure \ref{fig:Parallel Reservoir Computing}, where the red shaded region is the output and the orange shaded region is the input.

\subsubsection{Deep Reservoir Computing.} The ability of deep learning models to learn hierarchical data representations at varying levels of abstraction has increasingly drawn the attention of the machine learning community \cite{lecun2015deep}. Akin to deep learning architectures, Deep Reservoir Computing \cite{gallicchio2017deep} has emerged as a promising extension of RC by introducing multiple stacked reservoir layers to enable deeper hierarchical representations. This approach distributes temporal processing across layers, enhancing the network’s capacity to extract multi-scale features while retaining the computational efficiency inherent to RC frameworks. DeepESNs \cite{gallicchio2017deep} extend traditional ESNs by arranging several reservoir layers in a linear stack, where each layer directly feeds its output to the next. The state update equations for deepESNs can be formulated as:
\begin{equation}
\mathbf{x}^{(l)}(t+1) = 
\begin{cases}
(1-\alpha^{(l)})\mathbf{x}^{(l)}(t) + \alpha^{(l)} f \left( W^{(l)}_{\text{in}} \mathbf{u}(t+1) + W^{(l)}\mathbf{x}^{(l)}(t) + \mathbf{b}^{(l)} \right), \quad &\text{for } l = 1, 
\\
(1-\alpha^{(l)})\mathbf{x}^{(l)}(t) + \alpha^{(l)} f \left( W^{(l)}_{\text{in}} \mathbf{x}^{(l-1)}(t+1) + W^{(l)}\mathbf{x}^{(l)}(t) + \mathbf{b}^{(l)} \right), \quad &\text{for } l > 1, 
\end{cases}
\end{equation}
where the superscript $(l)$ denotes the parameters and hyperparameters associated with the $l$-th layer of the network. The readout is a linear aggregation of reservoir states at each intermediate layer
\begin{equation}
\mathbf{y}(t) = W_\text{out} [\mathbf{x}^{(1)}(t) \ \ \mathbf{x}^{(2)}(t) \ \ ... \mathbf{x}^{(L)}(t)]^\top,
\end{equation}
$L$ denoting the total number of reservoir layers.

\emph{DeepESN Input to All (deepESN-IA)}  \cite{gallicchio2017deep}, an enhancement over the deepESN architecture, augments each layer's input by concatenating the external input $\mathbf{u}(t)$ with the previous layer’s state
\begin{equation}
\mathbf{x}^{(l)}(t+1) = 
\begin{cases}
(1-\alpha^{(l)})\mathbf{x}^{(l)}(t) + \alpha^{(l)} f \left( W^{(l)}_{\text{in}} \mathbf{u}(t+1) + W^{(l)}\mathbf{x}^{(l)}(t) + \mathbf{b}^{(l)} \right), \quad &\text{for } l = 1, 
\\
(1-\alpha^{(l)})\mathbf{x}^{(l)}(t) + \alpha^{(l)} f \left( W^{(l)}_{\text{in}} [\mathbf{u}(t+1)\ \mathbf{x}^{(l-1)}(t+1)]^\top + W^{(l)}\mathbf{x}^{(l)}(t) + \mathbf{b}^{(l)} \right), \quad &\text{for } l > 1.
\end{cases}
\end{equation}

\begin{figure}[!ht]
    \centering
    \captionsetup{justification=centering}
    \includegraphics[width=0.8\textwidth]{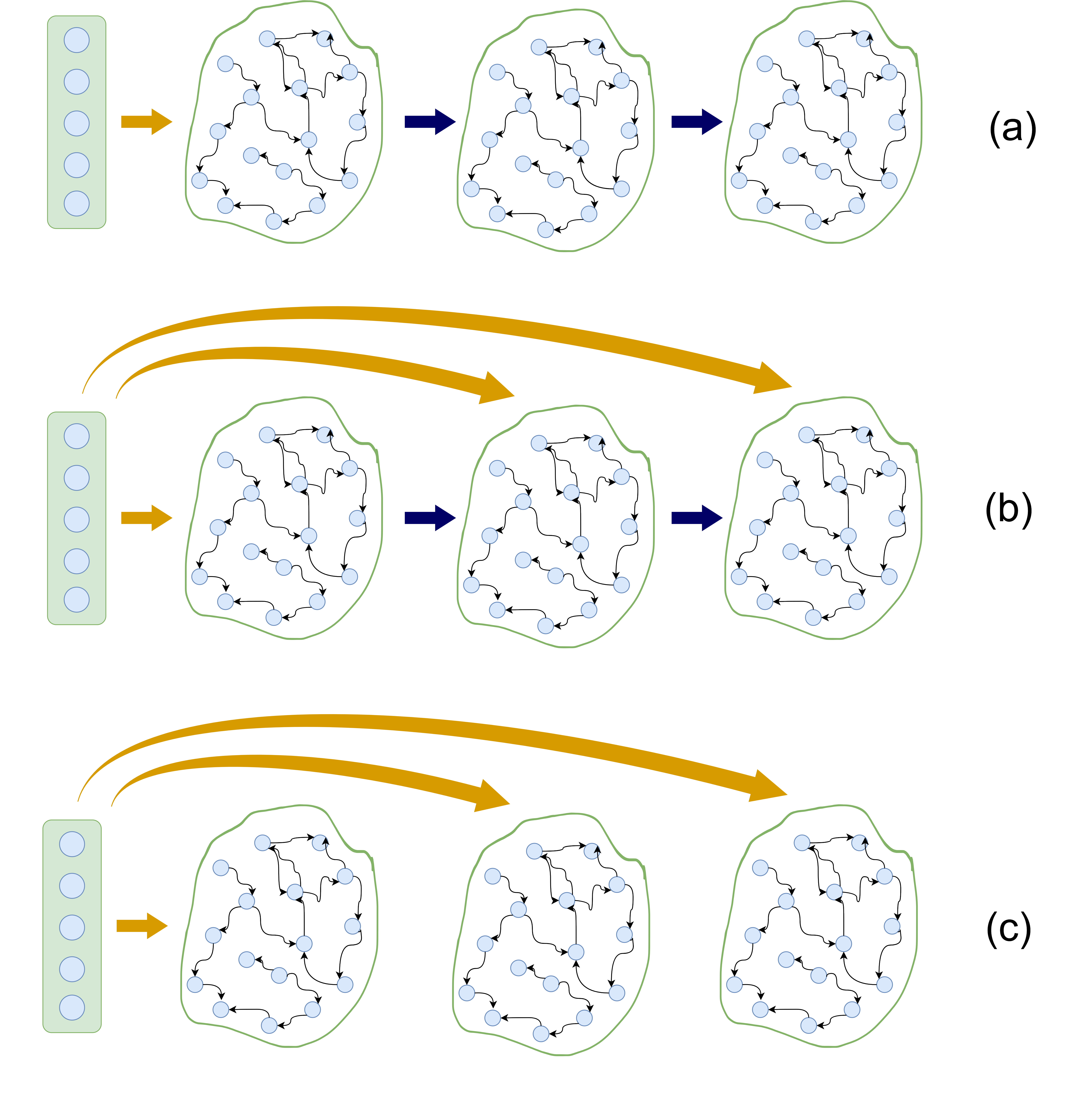}
    \caption{Deep Reservoir Computing. (a) DeepESN, (b) DeepESN-IA, (c) GroupedESN (adapted from \cite{gallicchio2017deep}).}
    \label{fig:Deep Reservoir Computing}
\end{figure}

\emph{GroupedESNs} \cite{gallicchio2017deep} introduce an alternative by partitioning the reservoir into multiple independent sub-reservoirs, each of which operates solely on the input signal
\begin{equation}
\mathbf{x}^{(l)}(t+1) = (1-\alpha^{(l)})\mathbf{x}^{(l)}(t) + \alpha^{(l)} f \left( W^{(l)}_{\text{in}} \mathbf{u}(t+1) + W^{(l)}\mathbf{x}^{(l)}(t) + \mathbf{b}^{(l)} \right).
\end{equation}

DeepESNs \cite{gallicchio2017deep} also employ Intrinsic Plasticity (IP), a form of unsupervised reservoir adaptation technique \cite{lukovsevivcius2009reservoir,wyffels2008improving, Steil2007, triesch2005gradient}. It aims to maximize the entropy of reservoir outputs by minimizing the Kullback–Leibler divergence \cite{van2014renyi} between the empirical output distribution and a Gaussian distribution. Using a gradient descent method, the gain $g$ and bias $b$ of the activation function (e.g., $\tanh$) are adapted locally for each unit. For a unit with input $x_{\text{net}}$ and output $\tilde{x} = \tanh(gx_{\text{net}} + b)$, the IP update rules are: 
\begin{equation}
    \Delta b = -\eta \left( -\left(\frac{\mu}{\sigma^2}\right) + \left(\frac{\tilde{x}}{\sigma^2}\right)\left(2\sigma^2 + 1 - \tilde{x}^2 + \mu \tilde{x}\right) \right), \quad \Delta g = \frac{\eta}{g} + \Delta b \, x_{\text{net}},
\end{equation}
where $\mu$ and $\sigma$ are the mean and standard deviation of the target Gaussian, and $\eta$ is the learning rate. \\

\noindent Thus, Deep RC enhances temporal abstraction through hierarchical depth, adaptive plasticity, and modular architectures, balancing expressivity with efficiency. Its advancements position it as a robust framework for scalable and adaptive neural computation.

\subsubsection{Attention-Enhanced Reservoirs.} Work \cite{koester2024attention} explores the integration of the attention mechanism within its RC framework. Introducing an attention layer as the output enables the system to prioritize specific temporal features, offering a more context-sensitive understanding of the input data. The study implements the reservoir as a delay-based photonic computer using a semiconductor laser with optical feedback. The reservoir states  $\mathbf{x}(t) \in \mathbb{R}^n$ simultaneously act as queries, keys, and values within the self-attention mechanism. The attention weights $w_\text{att}(t) \in \mathbb{R}^n$, which determine the significance of each reservoir node at a given time, are derived from the attention layer weight matrix  $W_\text{att} \in \mathbb{R}^{n \times n}$, 
\begin{equation}w_\text{att}(t) = W_\text{att} \mathbf{x}(t)
\end{equation}
with $W_\text{att}$ trained using a gradient descent algorithm. While $W_\text{att}$ remains static post-training, the time dependence of $w_\text{att}(t)$ arises from the evolving reservoir states. The attention weights and reservoir states are linearly combined to obtain the output signal $\mathbf{y}(t)$,
\begin{equation}
\mathbf{y}(t) = w_\text{att}^T(t)\mathbf{x}(t)
\end{equation}
thus amplifying the most relevant states for a given time step. 

\subsection{Adaptive Reservoirs}

While the classical reservoir computing framework keeps the internal weights \(\mathbf{W}_\mathrm{res}\) fixed, several extensions have been proposed to \emph{adapt} or partially train the reservoir. Such adaptive reservoirs can dynamically reorganize their network structure or connection strengths to maintain optimal performance in changing environments \cite{Jaeger2007,schrauwen2007overview}.

A specific class of adaptable reservoirs \cite{panahi2024adaptable} introduce an augmented parameter input channel. While traditional reservoirs learn only from time-series data, these reservoirs can learn from both the time-series data and an external bifurcation parameter. Apart from $W_\text{in}$, an additional matrix $W_\text{param}$ comes into play which connects the bifurcation parameter to the reservoir nodes. As shown in Figure \ref{fig:Adaptable Reservoir Computing}, the architecture  includes an additional bifurcation parameter, \( W_\text{param} \), that connects the reservoir to the parameter dynamics. If we represent the bifurcation parameter by $p$, then overall, the update equation is given by:
\begin{equation}
\mathbf{x}(t+1) = f(W_{\text{in}} \mathbf{u}(t) + W_{\text{res}} \mathbf{x}(t) + W_{\text{param}} \mathbf{p} + \mathbf{b}).
\end{equation}

Let \(\mathbf{\Theta}\) represent an augmented set of parameters for the reservoir, such as recurrent weights, node biases, or intrinsic plasticity parameters. Then the update rule becomes:
\begin{equation}
    \mathbf{x}(t+1) = f\bigl(\mathbf{W}_\mathrm{res}(\mathbf{\Theta})\,\mathbf{x}(t) + \mathbf{W}_\mathrm{in}(\mathbf{\Theta})\,\mathbf{u}(t+1)\bigr).
\end{equation}
An additional adaptation law adjusts \(\mathbf{\Theta}\) over time. For instance, gradient-based or Hebbian-like rules can be used to optimize specific objectives (e.g., memory capacity, separation property, or sparsity).

\paragraph{Types of Adaptive Mechanisms.}
\begin{enumerate}
    \item \textit{Intrinsic plasticity}: Adjust the neuron nonlinear activation parameters (e.g., slope or bias) to normalize the firing rates or maintain certain dynamic ranges \cite{Steil2007}.
    \item \textit{Structural plasticity}: Dynamically create or prune connections in \(\mathbf{W}_\mathrm{res}\) to achieve efficient representations, often guided by synaptic rewiring algorithms \cite{Butcher2013}.
    \item \textit{Online gradient-based approaches}: Temporarily unroll the reservoir dynamics and compute approximate gradients for \(\mathbf{W}_\mathrm{res}\), often in a truncated-backprop-through-time or real-time recurrent learning (RTRL)-like framework \cite{williams1989learning,Triefenbach2013}.
\end{enumerate}

\paragraph{Advantages and Challenges.}
Adaptive reservoirs can maintain high performance in non-stationary tasks and automatically find favorable operating regimes (e.g., near the “edge of chaos”). However, fully training an RNN can be computationally expensive. Partial or constrained adaptivity (e.g., only adapting gain factors or select connections) is often a practical compromise \cite{Jaeger2007,Lukosevicius2012}.

\begin{figure}[!ht]
    \centering
    \includegraphics[width=0.9\textwidth]{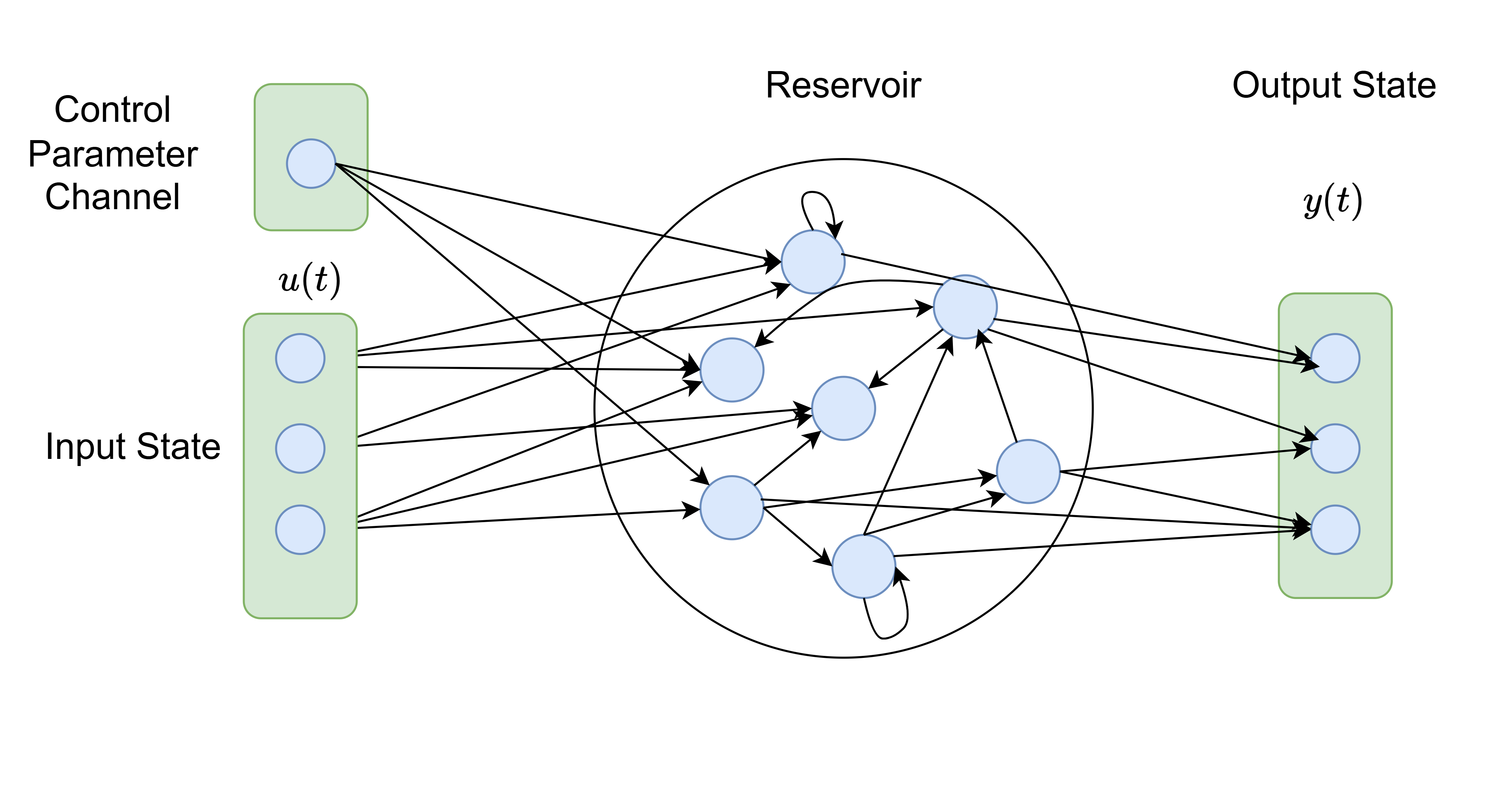}
    \captionsetup{justification=centering}  
    \caption{Adaptable Reservoirs, unlike traditional reservoirs add an additional bifurcation parameter  that is connected to the reservoir via $W_\text{param}$.}
    \label{fig:Adaptable Reservoir Computing}
\end{figure}

\section{Analysis of Reservoir Dynamics}\label{sec:analysis}

RC hinges on the dynamics of the reservoir component, which is  a high-dimensional RNN with fixed or partially trained internal connections. The nature of these internal dynamics is critical: it determines both the \emph{memory} of past inputs retained in the state of the reservoir and the \emph{transformational capabilities} that allow the system to separate and map input signals to linearly separable representations. In this section, we discuss three fundamental aspects of reservoir dynamics: (1) reservoir capacity, (2) representation power and universal approximation, and (3) generalization bounds and robustness analysis.

\subsection{Reservoir Capacity}

\paragraph{Memory Capacity and Related Metrics.}
The most widely studied notion of \emph{reservoir capacity} pertains to the system’s ability to retain and utilize past information \cite{jaeger2001echo,Jaeger2002,Lukosevicius2012}. Formally, an RC system (e.g., an echo state network, ESN) can be viewed as a dynamical system
\begin{equation}
    \mathbf{x}(t+1) = f\bigl(\mathbf{W}_{\text{res}} \mathbf{x}(t) + \mathbf{W}_{\text{in}} \mathbf{u}(t+1)\bigr),
    \label{eq:reservoir_state_update}
\end{equation}
where 
\(
    \mathbf{x}(t) \in \mathbb{R}^N \quad \text{is the reservoir state at time } t,
\)
\(
    \mathbf{u}(t) \in \mathbb{R}^m \quad \text{is the external input at time } t,
\)
\(
    \mathbf{W}_{\text{res}} \in \mathbb{R}^{N\times N} \quad \text{and} \quad \mathbf{W}_{\text{in}} \in \mathbb{R}^{N\times m} \quad \text{are the (fixed) reservoir and input weight matrices,}
\)
\(
    F(\cdot) \) is  a componentwise nonlinear activation function (e.g.  $\tanh$).
The system’s output at time \(t\) is given by
\(
    y(t) = \mathbf{W}_{\text{out}} \, \mathbf{x}(t) \,,
\)
where \(\mathbf{W}_{\text{out}} \in \mathbb{R}^{1\times N}\) (or higher-dimensional if the output is multi-dimensional) is usually the only trainable parameter.

\textit{Memory capacity (MC)} is commonly measured by how well the reservoir can linearly reconstruct delayed versions of the input. Specifically, consider a scalar input \(u(t)\). For a delay \(\tau \ge 0\), define the target signal \(d_{\tau}(t) = u(t-\tau)\). The readout is trained (in a mean-square sense) to approximate \(d_{\tau}(t)\) from \(\mathbf{x}(t)\):
\begin{equation}
    d_{\tau}(t) \approx \mathbf{W}_{\text{out}}^{(\tau)} \, \mathbf{x}(t).
\end{equation}
The memory capacity for delays \(\tau = 0,1,2,\ldots,\tau_{\max}\) is given by:
\begin{equation}
    \mathrm{MC} = \sum_{\tau=0}^{\tau_{\max}} \mathrm{NRMSE}\bigl(d_{\tau}(t), \hat{d}_{\tau}(t)\bigr),
\end{equation}
where \(\mathrm{NRMSE}\) denotes some normalized error measure, or, equivalently, one measures the correlation between \(\hat{d}_{\tau}(t)\) and \(d_{\tau}(t)\). In practice, the memory capacity often saturates at or below \(N\) (the reservoir size), because each of the \(N\) reservoir state components can be viewed (roughly) as carrying one dimension of memory \cite{Jaeger2002,White2004}. This bound holds only under
the conditions of i.i.d. input and linear output units, as the temporal dependencies between
inputs potentially affect MC estimation. A well-tuned spectral radius \(\rho(\mathbf{W}_{\text{res}})\) (i.e., the largest absolute eigenvalue of \(\mathbf{W}_{\text{res}}\)) is also critical to achieving high memory capacity.

However, the $\tau$ -delay MC estimation may not fully capture the complexity of nonlinear
systems. Nonlinearities in the reservoir can lead to behaviors that cannot be accurately reflected
by a simple linear correlation between the output and delayed input. As such, while MC is a
valuable measure, it doesn’t always offer a complete representation of a reservoir’s ability to
model and recall complex temporal relationships

In \emph{Figure \ref{fig:esn_memory_capacity}}, we examine the \emph{memory capacity} of an ESN by quantifying how well it can reconstruct input signals of increasing temporal delay.  Concretely, an i.i.d.\ random sequence \(\{u(t)\}_{t=0}^{T-1}\) is fed into a reservoir whose state evolves as
\(
\mathbf{x}(t+1)
\;=\;
\tanh\Bigl(
    W_{\mathrm{res}}\,\mathbf{x}(t)
    \;+\;
    W_{\mathrm{in}}\;u(t)
\Bigr),
\)
where \(W_{\mathrm{res}}\) has been normalized to satisfy the spectral radius condition (i.e., \(\|\lambda(W_{\mathrm{res}})\|_{\max} < 1\)).  For each delay \(\tau\in\{1,2,\dots\}\), we attempt to recover \(u(t-\tau)\) from \(\mathbf{x}(t)\) using a linear regressor and measure the \(R^2\) fit score.  As \(\tau\) increases, the reservoir must rely on subtler echoes of past inputs, so the recall quality eventually decays.  Aggregating these scores over \(\tau\) provides a classical indicator of the ESN’s short-term memory depth, central to applications requiring multi-step prediction or reconstitution of past driving signals.  In practice, the maximum useful \(\tau\) for which the reservoir can maintain a high recall score often guides design choices, such as the reservoir’s spectral radius and connectivity patterns, to match the targeted temporal dependencies of a given task.

\begin{figure}[!ht]
    \centering
    \includegraphics[width=0.76\linewidth]{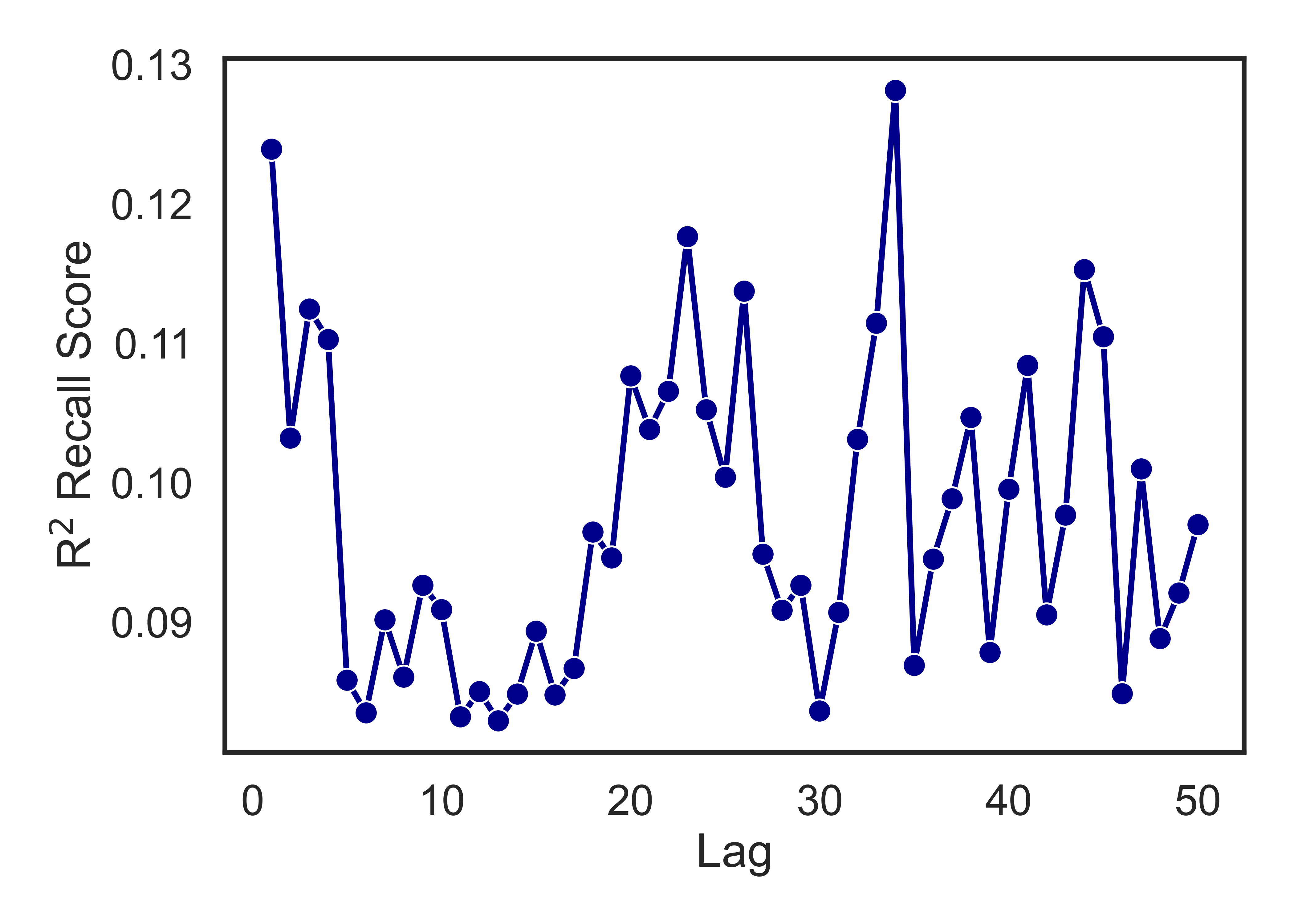}
    \caption{Memory capacity curve of an echo state network, showing the \(R^2\) recall score for increasingly large delays \(\tau\).  Higher recall at larger \(\tau\) indicates a reservoir architecture well-suited to retaining past information over longer horizons.}
    \label{fig:esn_memory_capacity}
\end{figure}

\paragraph{Separation Property.}
Another important notion is the \emph{separation property} \cite{maass2002real}, which states that distinct input trajectories should map to sufficiently distinguishable trajectories in the reservoir state space. Formally, for two different input streams \(\{u_1(t)\}\) and \(\{u_2(t)\}\), the induced reservoir states \(\{\mathbf{x}_1(t)\}\) and \(\{\mathbf{x}_2(t)\}\) should remain distinguishable (with high probability). The separation property underpins not just memory capacity, but also the reservoir’s capacity to encode complex temporal patterns.

\paragraph{Capacity for Nonlinear Transformations.}
While memory capacity was originally defined for linear readouts, it can be extended to capture the reservoir’s ability to represent and reconstruct more general transformations of past inputs \cite{Dambre2012,grigoryeva2018echo}. For example, \emph{polynomial capacity} measures how well the reservoir can approximate polynomials of past inputs. Such metrics correlate with a reservoir’s potential for higher-order signal processing tasks.

\subsection{Representation Power and Universal Approximation}

Reservoir computing frameworks derive much of their power from the dynamical regime in which the reservoir operates. The interplay between the \emph{echo state property} and the \emph{rich nonlinearity} of the recurrent layer results in a broad class of representable functions. Several theoretical results support the idea that RC-based networks are universal approximators of dynamical systems under mild conditions \cite{maass2002real,grigoryeva2018echo,Chen2019}.

\paragraph{Universal Approximation.}
Universal approximation refers to the ability of an RC system to approximate any continuous function, analogous to the Universal Approximation Theorem (UAT) in feed-forward networks  (FFNs). The UAT states that an FFN with at least one hidden layer, using a sufficiently large number of neurons and a nonlinear activation function, can approximate any continuous function on a compact domain arbitrarily well \cite{cybenko1989approximation}. Formally, for any continuous function $f: \mathbb{R}^n \to \mathbb{R}$ and any arbitrarily small $\epsilon>0$ there exists an FFN with a single hidden layer, and a finite number of neurons that can approximate $f(x)$ such that
\begin{equation}
\sup_{x\in K}\ |\ g(x) - f(x)\ | < \epsilon,
\end{equation}
where $K$ is a compact subset of $\mathbb{R}^n$ and $g(x)$ is the function represented by the FFN.

From a theoretical perspective, an RC system with sufficient reservoir size \(N\) and a universal set of activation functions \(F\) can approximate, to arbitrary precision, any fading-memory filter \cite{Boyd1985,grigoryeva2018echo}. In discrete-time terms, a fading-memory filter is an operator \(\Phi\) on input sequences that depends on past inputs but in a way that diminishes with increasing delay. Formally, a filter \(\Phi\) is said to have fading memory if:
\begin{equation}
    \lim_{\delta \to \infty} \|\Phi(\mathbf{u}) - \Phi(\mathbf{u}')\| = 0
\end{equation}
whenever \(\mathbf{u}\) and \(\mathbf{u}'\) differ only in the remote past (by more than some \(\delta\) steps).

The reservoir’s recurrent architecture naturally encapsulates a fading-memory-like mechanism, because the state \(\mathbf{x}(t)\) (assuming the ESP holds) is an effectively compressed summary of recent inputs. With a linear readout, one can then form arbitrary linear combinations of a sufficiently rich set of basis functions encoded in \(\mathbf{x}(t)\). As a result, RC networks can approximate a broad class of time-dependent functions, thus granting them \emph{universal} representation power within the space of fading-memory or stable dynamical systems \cite{maass2002real,grigoryeva2018echo}.

\paragraph{Implications for Architectural Design.}
In practice, achieving near-universal approximation in finite networks requires careful tuning of the spectral radius \(\rho(\mathbf{W}_{\text{res}})\), the reservoir size \(N\), the sparsity and distribution of \(\mathbf{W}_{\text{res}}\), the input scaling and reservoir biases.
There is a \textit{trade-off}: larger and more expressive reservoirs can approximate more complex functions, but are more challenging to train efficiently and can suffer from overfitting if the readout is not regularized \cite{Rodan2011,Lukosevicius2012}.

\subsection{Generalization Bounds and Robustness Analysis}

\paragraph{Generalization in Temporal Learning.}
In the previous sections, we examined the memory capacity, representational capabilities, and universal approximation properties of RC. However, recent studies have revealed key limitations. Optimizing memory capacity does not always translate to better prediction performance. Moreover, while universal approximation guarantees low training errors, it does not ensure good generalization, meaning reservoirs can fit training data precisely but may fail on unseen inputs. Unlike feedforward neural networks, where classical VC-dimension arguments or Rademacher complexity bounds \cite{bartlett2002rademacher} are often applied, recurrent networks (including reservoir-based models) require specialized analyses of capacity and generalization in the time domain \cite{Sontag1995,Buesing2010}. One line of work focuses on bounding the error in predicting time-varying signals. If the reservoir states remain in a bounded region (under the ESP) and the readout is a regularized linear map, one can derive performance guarantees in terms of the magnitude of \(\mathbf{W}_{\text{res}}\),
the norm of \(\mathbf{W}_{\text{out}}\), and
the input signal’s energy or bandwidth.
For instance, classical stability analyses show that if \(\rho(\mathbf{W}_{\text{res}}) < 1\) and the state-update function \(F(\cdot)\) is Lipschitz, then perturbations to the input or readout weights lead to proportionally bounded changes in the reservoir state and output \cite{Jaeger2002,grigoryeva2018echo}.

\paragraph{Robustness to Noise and Perturbations.}
Reservoir networks often exhibit a certain robustness to noise because the high-dimensional representation in the reservoir can be thought of as a projection that separates distinct input streams, and
    the ESP ensures that small perturbations from the distant past eventually vanish from the state representation.
Mathematically, consider a noisy input stream \(u(t) + \epsilon(t)\). Under mild assumptions (e.g., \(\|\epsilon(t)\|\) is bounded and small), one can show that:
\begin{equation}
    \|\mathbf{x}_{\text{noisy}}(t) - \mathbf{x}_{\text{clean}}(t)\| \leq \alpha^t \|\mathbf{x}(0) - \mathbf{x}'(0)\| + \sum_{k=0}^t \alpha^{t-k}\|\epsilon(k)\|,
\end{equation}
where \(\alpha < 1\) if the spectral radius is appropriately constrained and the activation function \(F(\cdot)\) is contractive in a region around the operating point \cite{jaeger2001echo,grigoryeva2018echo}. This implies exponential decay of initial-condition differences and linear accumulation of new noise. Consequently, reservoir networks can continue to function effectively in low-noise environments and degrade gracefully in moderate-noise ones.

\paragraph{Regularization and Overfitting.}
Even though the reservoir weights are typically kept fixed (or partially trained in some extensions), overfitting can still occur at the readout layer when the number of reservoir neurons \(N\) is large relative to the amount of training data \cite{Lukosevicius2012}. Standard techniques such as ridge regression or sparsity constraints on \(\mathbf{W}_{\text{out}}\) help control the effective capacity of the readout, thereby preventing overfitting and improving generalization performance. In some analyses, the effective degrees of freedom of the reservoir are bounded not just by \(N\) but by the rank and connectivity of \(\mathbf{W}_{\text{res}}\) \cite{Rodan2011}.

\paragraph{Open Challenges.}
Despite the progress in theoretical understanding, several open questions remain:
\begin{itemize}
    \item How to precisely quantify the interaction between reservoir topology and generalization bounds for arbitrary input statistics.
    \item Extensions to non-stationary or adversarial input streams, where stable memory and robust representation may require adaptive reservoirs \cite{Jaeger2007}.
    \item Formalizing the trade-off between expressiveness (e.g., universal approximation of richer classes of dynamical systems) and computational or training complexity.
\end{itemize}
\section{Training Dynamics}\label{sec:training}
\subsection{Regularization Techniques}

Regularization plays a crucial role in RC, ensuring stability, preventing overfitting and improving generalization, particularly when dealing with large-scale or noisy datasets. Unlike traditional neural networks, where regularization techniques such as dropout and weight decay are common, RC primarily applies regularization to the readout layer while maintaining the reservoir dynamics unchanged.

\subsubsection{Ridge Regression in Readout Training}
Since reservoir computing relies on a fixed randomly initialized reservoir, training is limited to optimizing the output weights \( W_{\text{out}} \). A common approach for regularization is \textit{Ridge Regression with Tikhonov regularization} \cite{hoerl1970ridge, tikhonov1977solutions}, which adds an \( L_2 \)-penalty to the squared loss function:
\begin{equation}
W_{\text{out}}^* = \underset{W_{\text{out}}}{\arg \min} \ (\|\mathbf{X} W_{\text{out}}^\top - \mathbf{Y}_{target} \|^2 + \beta\|W_{\text{out}}\|^2),
\end{equation}
where \( \mathbf{Y}_{target} \) is the target output, \( \mathbf{X} \) is the reservoir state vector, and \( \beta \) is the regularization coefficient that controls the penalty on the magnitude of the weights. 
Ridge regression ensures numerical stability and prevents overfitting when the number of reservoir neurons is large.

The traditional approach employs a linear readout function to map the high-dimensional reservoir states to the desired output. However, this linear approximation may be insufficient for capturing the complexities inherent in certain dynamical systems. To address this limitation, Ohkubo and Inubushi introduced a generalized readout mechanism that incorporates nonlinear combinations of reservoir variables, such as quadratic and cubic terms \cite{ohkubo2024reservoir}. This approach is grounded in the mathematical theory of generalized synchronization, which posits the existence of a smooth, differentiable mapping between the reservoir states and the target outputs. By expanding this mapping into a Taylor series, the generalized readout effectively approximates higher-order terms, thereby enhancing the RC's ability to model complex, nonlinear dynamics. Empirical studies, particularly on chaotic systems like the Lorenz and Rössler attractors, have demonstrated that this method not only improves prediction accuracy but also enhances robustness over extended prediction horizons \cite{ohkubo2024reservoir}. Notably, despite the increased complexity of the readout function, the training process still remains as linear regression, preserving the computational efficiency characteristic of traditional RC methods.

\subsubsection{Noise Based Regularization}

An alternative regularization strategy, introduced in \cite{jaeger2001echo}, involves perturbing the training state vectors with additive noise prior to training the readout layer. Formally, the perturbed state at time $t$ is defined as:
\begin{equation}
    x_p(t) = x(t) + \sigma \cdot \eta(t),
\end{equation}
where $x(t)$ denotes the original reservoir state, $\eta(t)$ is a zero-mean unit-variance noise term (e.g., drawn from $\mathcal{N}(0, 1)$), and $\sigma$ controls the noise amplitude.
However, this method exhibits behavior and outcomes closely aligned with those of Ridge Regression, particularly for the experiments conducted in this study \cite{wyffels2008stable}. 


\subsubsection{Pruning based Regularization}

Pruning is a regularization technique that reduces the number of readout connections, improving generalization, lowering computational cost, and highlighting important features \cite{dutoit2009pruning}. Only readout weights are modified to limit disturbance. In the absence of feedback from output to reservoir, pruning does not affect reservoir states.
Let $\mathcal{P}$ be the set of pruned indices. The pruned readout vector is:
\begin{equation}
    w'_i = 
    \begin{cases}
        0 & \text{if } i \in \mathcal{P}, \\
        w_i & \text{otherwise}.
    \end{cases}
\end{equation}

Pruning introduces bias but can reduce test error. The challenge lies in selecting both the number and identity of connections to prune. Validation error is used to determine when to stop pruning. As exhaustive search is intractable, three sub-optimal methods are employed.

\paragraph{Backward Selection (BS).}

BS is a greedy algorithm that iteratively removes the connection whose elimination least increases validation error. Due to its binary nature, BS may perform sub-optimally with highly correlated inputs. Hence, BS is combined with ridge regression (RR) to form Weighted Ridge Regression (WRR), where connections are either pruned or shrunk:
\begin{equation}
    \mathbf{W}_{\text{WRR}} = \arg\min_{\mathbf{W}} \|\mathbf{Y} - \mathbf{XW}\|^2 + \beta \mathbf{W}^T \mathbf{D} \mathbf{W},
\end{equation}
where $\mathbf{D}$ is a diagonal matrix with weights reflecting the sensitivity of the output to each neuron.

\paragraph{Least Angle Regression (LAR).}

LAR is a forward-selection strategy that incrementally adds the readout connection most correlated with the current residual error. It combines sparsity and shrinkage, operating along an $\ell_1$-regularized path with an analytically computable solution, reducing computational complexity compared to full regression.

\paragraph{Random Deletion (RanD).}

RanD removes readout connections randomly and serves as a performance baseline. It can also be combined with RR to assess the role of randomness under regularization.

\subsubsection{Effective Number of Readout Connections}

Shrinkage methods like RR do not directly reduce the number of connections but lower their effective number, as introduced in \cite{zou2005regularization}. For RR, the effective number of parameters is:
\begin{equation}
    d_{\text{eff}} = \sum_{i} \frac{\sigma_i^2}{\sigma_i^2 + \beta},
\end{equation}
where $\sigma_i$ are the singular values of $\mathbf{X}$ and $\beta$ is the ridge parameter. For LAR, this effective number approximates the number of active connections. From a practical standpoint, both actual and effective counts matter, as each active connection contributes to computational cost.

Regularization techniques in reservoir computing help models achieve enhanced robustness and generalization. Future research may focus on adaptive and biologically inspired regularization mechanisms to further optimize reservoir dynamics in high-dimensional tasks.

\subsection{Hyperparameter Tuning}
Hyperparameter tuning in RC is crucial for achieving good performance without overcomplicating the architecture or incurring undue computational cost \cite{Jaeger2002,Lukosevicius2012}. While the reservoir is often constructed with fixed internal weights \(\mathbf{W}_\mathrm{res}\), there are several key hyperparameters that significantly influence the dynamics:

\begin{itemize}
    \item \textit{Spectral radius} \(\rho(\mathbf{W}_\mathrm{res})\): The eigenvalue of \(\mathbf{W}_\mathrm{res}\) with the largest absolute value. A common rule of thumb is to choose \(\rho(\mathbf{W}_\mathrm{res}) < 1\) to ensure ESP, although optimal values can deviate from this in practice \cite{Jaeger2002,White2004}.
    \item \textit{Input scaling and bias}: Scaling of the input weight matrix, $W_{in}$ determines how strongly the input signal influences the reservoir states. Small input scaling keeps reservoir neurons in a near-linear regime, aiding smooth time series tasks, while large scaling enhances nonlinearity but risks instability and information loss due to saturation. To reduce hyperparameters, a single global scaling factor is typically applied to $W_{in}$. However, scaling the bias column separately can improve performance, as it influences neuron activations. Additionally, scaling each input channel individually helps when features contribute unequally to the task. Proper input preprocessing, such as normalization, enhances approximation accuracy, while feature engineering techniques like  PCA optimize performance.
    \item \textit{Reservoir size} \(N\): Generally, increasing the reservoir size enhances performance by providing a richer space of internal representations, making it easier to find a suitable linear combination of reservoir states to approximate the target function, $\mathbf{y}_{\text{target}}(t)$ . In challenging tasks, a larger reservoir is preferable, provided regularization techniques, such as ridge regression, are used to mitigate overfitting.  However, an excessively large reservoir can lead to overparameterization and inefficiencies if the task is simple or the available training data is limited. In practice, reservoir size should be as large as computational resources allow, with parameter tuning ensuring optimal performance.
    \item \textit{Connectivity structure}: Sparsity refers to the fraction of nonzero connections in the recurrent weight matrix, $W_{res}$. Too dense a network may lead to overly correlated states, while too sparse a network may underrepresent input signals \cite{Verstraeten2007}. Sparse connectivity tends to reduce redundant interactions, helping the reservoir maintain diverse activations without excessive correlation among neurons. Consequently, sparse reservoirs perform slightly better in many practical tasks. Moreover, if each reservoir neuron connects to only a few other neurons regardless of the total reservoir size, the computational cost of updating the reservoir states drastically reduces. Efficient sparse matrix representations in modern programming environments, such as Python’s SciPy and MATLAB, enable fast reservoir updates with minimal implementation effort. The nonzero elements of $W$ are typically drawn from uniform or Gaussian distributions, yielding similar performance in both cases. 
    \item \textit{Leaking rate} \(\alpha\) (for leaky integrator neurons): In leaky reservoirs, the update rule is often modified to
    \[
        \mathbf{x}(t+1) = (1 - \alpha)\,\mathbf{x}(t) \;+\; \alpha \, f\bigl(\mathbf{W}_\mathrm{res}\,\mathbf{x}(t) + \mathbf{W}_\mathrm{in}\,\mathbf{u}(t+1)\bigr),
    \]
    where \(0 < \alpha \le 1\) controls how quickly the reservoir state responds to new inputs \cite{Jaeger2007}.
\end{itemize}

\paragraph{Automated Tuning Approaches.}
Methods such as grid search, random search, or more sophisticated Bayesian optimization can be used to explore the hyperparameter space \cite{Lukosevicius2012}. Because the cost of training an RC system is relatively low (the readout training is typically a single-step linear regression), many practitioners can afford to test multiple configurations quickly.

\paragraph{Relation to Theoretical Metrics.}
Hyperparameter selection often targets maximizing \emph{memory capacity} or \emph{information processing capacity} of the reservoir \cite{Dambre2012}. For instance, one might seek a spectral radius that is large enough to foster rich (but stable) dynamics. Another consideration is the \emph{echo state property}; ensuring \(\rho(\mathbf{W}_\mathrm{res}) < 1\) is a common guideline, though in some tasks values slightly above unity can be experimentally beneficial \cite{Lukosevicius2012}.

\subsection{Computational Overhead and Trade-Offs}
RC frameworks, such as ESNs and LSMs, offer a relatively lightweight training procedure compared to fully trained recurrent neural networks. However, careful consideration of computational overhead, trade-offs in design choices, and resource requirements is crucial for effective deployment.

\paragraph{Complexity of State Update.}
Let \(\mathbf{W}_\mathrm{res} \in \mathbb{R}^{N \times N}\) and \(\mathbf{W}_\mathrm{in} \in \mathbb{R}^{N \times m}\). At each discrete time step \(t\), updating the reservoir state
    $\mathbf{x}(t+1)$ 
requires \(\mathcal{O}(N^2)\) operations if \(\mathbf{W}_\mathrm{res}\) is dense. If the reservoir matrix has a sparsity factor \(s\), the update can often be executed in \(\mathcal{O}(s N)\). Hence, many practical RC implementations employ sparse or structured \(\mathbf{W}_\mathrm{res}\) to reduce computation \cite{Jaeger2002,Verstraeten2007}.

\paragraph{Complexity of Readout Training.}
Training the output weights \(\mathbf{W}_\mathrm{out} \in \mathbb{R}^{d \times N}\) (for \(d\)-dimensional output) usually involves solving a linear regression:
\begin{equation}
    \min_{\mathbf{W}_\mathrm{out}}\; \bigl\|\mathbf{Y} - \mathbf{W}_\mathrm{out}\,\mathbf{X}\bigr\|^2 \;+\; \beta\,\bigl\|\mathbf{W}_\mathrm{out}\bigr\|^2,
    \label{eq:ridge_regression}
\end{equation}
where \(\mathbf{X} \in \mathbb{R}^{N \times T}\) is the matrix of reservoir states collected over \(T\) time steps, \(\mathbf{Y} \in \mathbb{R}^{d \times T}\) the corresponding target outputs, and \(\beta\) a regularization parameter. The solution of \eqref{eq:ridge_regression} via the normal equations has cost \(\mathcal{O}(\min(N^2 T,\, N T^2))\). When \(T \gg N\), this reduces to \(\mathcal{O}(N^2 T)\). More efficient algorithms (e.g., QR decomposition or incremental methods) may improve practical performance \cite{Lukosevicius2012}.

\paragraph{Trade-Offs in Model Size and Accuracy.}
Larger reservoirs (increasing \(N\)) tend to yield richer dynamics and potentially higher capacity. However, the computational cost of each timestep increases, as does memory usage. This trade-off is often mitigated by employing sparse reservoir matrices, 
   low-rank approximations or random features \cite{Rodan2011}, and
   techniques like “conceptors” to prune or refine reservoir states \cite{Jaeger2014}.

\subsection{Forecasting Setups: Open-Loop and Closed-Loop Reservoir Dynamics}\label{sec:forecasting_setups}

In the context of time-series modeling and prediction with reservoir computing, two fundamental forecasting paradigms are employed: the \textit{open-loop} (also known as teacher-forced or driven) setup and the \textit{closed-loop} (or autonomous) setup. Each represents a distinct way of interacting with the input stream and has profound implications for training, stability, and long-term forecasting performance. While the core reservoir architecture remains unchanged, the way in which external inputs are injected or replaced during inference fundamentally alters the system dynamics.

\paragraph{Open-Loop (Teacher-Forced) Forecasting.}
The open-loop setup corresponds to the scenario in which the reservoir receives ground-truth inputs at every time step during training and optionally also during testing. The system is thus ``driven" by the true signal, and prediction is performed stepwise by feeding each true input \( \mathbf{u}(t) \) to the network and generating the output \( \hat{\mathbf{y}}(t) \) using a trained readout layer. Formally, the reservoir dynamics in open-loop mode are:
\begin{equation}
    \mathbf{x}(t+1) = f\left( \mathbf{W}_{\text{res}}\, \mathbf{x}(t) + \mathbf{W}_{\text{in}}\, \mathbf{u}(t+1) \right),
\end{equation}
\begin{equation}
    \hat{\mathbf{y}}(t+1) = \mathbf{W}_{\text{out}}\, \mathbf{x}(t+1).
\end{equation}
Here, the input \( \mathbf{u}(t+1) \) is always the true next input (or part of the sequence), and the output is calculated directly from the current reservoir state.

This mode is particularly effective when the aim is to perform short-term prediction, classification, or filtering tasks where the ground truth is always available. Because the input is continuously reset to its true value, the reservoir trajectory stays close to the data manifold, preventing divergence due to accumulated prediction error. The training problem reduces to a straightforward linear regression over the reservoir states, and optimization can often be carried out via ridge regression or similar techniques.

Open-loop forecasting is also equivalent to the so-called ``teacher forcing" method commonly used in recurrent neural network training. It ensures convergence and stability during training but suffers from a potential mismatch between training and inference dynamics if the trained model is later used in a closed-loop (generative) context.

\paragraph{Closed-Loop (Autonomous) Forecasting.}
In the closed-loop or autonomous forecasting mode, the reservoir must generate future inputs based solely on its own previous outputs. This setup is required when the goal is to generate free-running predictions or simulate a learned dynamical system over an extended horizon. The input at each step is no longer the true \( \mathbf{u}(t) \), but the model's own prediction \( \hat{\mathbf{y}}(t) \), which is fed back into the input channel:
\begin{equation}
    \hat{\mathbf{u}}(t) = \hat{\mathbf{y}}(t),
\end{equation}
\begin{equation}
    \mathbf{x}(t+1) = f\left( \mathbf{W}_{\text{res}}\, \mathbf{x}(t) + \mathbf{W}_{\text{in}}\, \hat{\mathbf{u}}(t) \right),
\end{equation}
\begin{equation}
    \hat{\mathbf{y}}(t+1) = \mathbf{W}_{\text{out}}\, \mathbf{x}(t+1).
\end{equation}

This configuration transforms the ESN into a self-contained autonomous system—essentially a recursive map whose future evolution depends entirely on its internal state and output feedback. It is commonly used in tasks such as modeling chaotic systems (e.g., Lorenz, Rössler, Mackey–Glass), where the goal is to emulate the entire attractor or produce long-range forecasts from an initial condition \cite{jaeger2001echo, pathak2018model}.

Closed-loop forecasting introduces several mathematical and practical challenges. Since each predicted output becomes the input at the next step, errors can compound over time, and the reservoir trajectory may diverge from the original data manifold. This phenomenon, often referred to as \emph{error accumulation} or \emph{trajectory drift}, is especially pronounced in chaotic systems due to their sensitive dependence on initial conditions.

Moreover, stability in the autonomous mode is not guaranteed even if the reservoir exhibits the echo state property in the driven case. The system becomes a composite dynamical map:
\begin{equation}
    \mathbf{x}(t+1) = f\left( \mathbf{W}_{\text{res}}\, \mathbf{x}(t) + \mathbf{W}_{\text{in}}\, \mathbf{W}_{\text{out}}\, \mathbf{x}(t) \right) =: \mathcal{F}(\mathbf{x}(t)),
\end{equation}
where the nonlinearity of \( f \) and the composition of matrices makes the overall update rule highly sensitive to the choice of spectral radius, input scaling, and readout weights.

Recent theoretical work has investigated conditions under which the closed-loop ESN preserves stability or converges to the true attractor. For example, results from Grigoryeva and Ortega \cite{grigoryeva2018echo} have shown that echo state networks can approximate fading memory filters and autonomous dynamical systems under specific constraints. Meanwhile, empirical studies such as \cite{pathak2018model,Gupta2022ModelfreeFO} have demonstrated that closed-loop ESNs (or their variants like time-delay reservoirs and hybrid architectures) can model complex spatiotemporal systems if initialized and regularized appropriately.

\paragraph{Comparison and Trade-offs.}
The distinction between open-loop and closed-loop forecasting in reservoir computing reflects a deeper question about supervised learning versus generative modeling. Open-loop setups provide stability and tractability for training, but may fail in deployment if autonomous generation is required. Conversely, closed-loop setups offer powerful tools for simulation and long-term forecasting, but require careful handling of stability, initialization, and regularization.
In practice, one common strategy is to train in open-loop mode and then evaluate performance in closed-loop mode—sometimes including curriculum learning schemes or hybrid loss functions that gradually mix predicted and true inputs (e.g., scheduled sampling). In the context of chaotic dynamical systems, accurate one-step prediction (open-loop) does not necessarily translate into good long-range behavior (closed-loop), so evaluation must be task-specific.

\subsection{Performance Metrics for Forecasting}\label{sec:metrics}

Evaluating the performance of RC models—especially in time-series forecasting tasks—requires metrics that capture not only pointwise prediction accuracy but also the temporal quality and stability of long-term predictions. The most common metrics include the \textit{Normalized Root Mean Squared Error} (NRMSE) and the \textit{Valid Prediction Time} (VPT). Other measures such as the \textit{Correlation Coefficient}, \textit{Spectral Similarity}, and \textit{Lyapunov Exponent Deviation} are also used to assess the quality of trajectory matching and dynamical fidelity. These metrics vary in their sensitivity to error scaling, temporal alignment, and chaotic divergence, and are often chosen based on the properties of the target system.

\paragraph{Normalized Root Mean Squared Error.}
The NRMSE is a scale-invariant measure of the deviation between predicted values and ground truth. Given a true time series \(\{y(t)\}_{t=1}^T\) and predicted sequence \(\{\hat{y}(t)\}_{t=1}^T\), the NRMSE is defined as:
\begin{equation}
    \mathrm{NRMSE} = \frac{\sqrt{\frac{1}{T} \sum_{t=1}^{T} \left(y(t) - \hat{y}(t)\right)^2}}{\sigma_y},
\end{equation}
where \(\sigma_y\) is the standard deviation of the true signal:
\[
    \sigma_y = \sqrt{\frac{1}{T} \sum_{t=1}^{T} \left(y(t) - \bar{y}\right)^2}, \quad \bar{y} = \frac{1}{T} \sum_{t=1}^T y(t).
\]
The NRMSE normalizes the prediction error, enabling comparison across signals of different amplitudes and variances. It is widely used in RC literature due to its intuitive interpretation and sensitivity to both bias and variance of predictions \cite{Lukosevicius2012, jaeger2004harnessing}.

\paragraph{Valid Prediction Time.}
While NRMSE evaluates accuracy over a full window, it does not capture the temporal horizon over which predictions remain valid. The \textit{Valid Prediction Time} (VPT) quantifies how long an autonomous prediction \(\hat{y}(t)\) remains sufficiently close to the true trajectory \(y(t)\) before diverging beyond a pre-defined threshold \(\varepsilon\). Define
\begin{equation}
    VPT(\varepsilon) = \max \left\{ \tau \in \mathbb{N} \;\middle|\; \frac{|\hat{y}(t) - y(t)|}{\sigma_y} < \varepsilon \quad \text{for all } t \in [0, \tau] \right\}.
\end{equation}
Typical values for \(\varepsilon\) range between 0.1 and 0.2, depending on the signal-to-noise ratio and the specific application \cite{pathak2018model, lu2018attractor}. VPT is particularly informative when forecasting chaotic systems, where even small prediction errors grow rapidly due to sensitivity to initial conditions. A high VPT implies that the model accurately reproduces the target dynamics for a meaningful time horizon before diverging.

\paragraph{Correlation Coefficient.}
Another standard metric is the Pearson correlation coefficient \(\rho\) between the predicted and true signals:
\begin{equation}
    \rho(y, \hat{y}) = \frac{\sum_{t=1}^T (y(t) - \bar{y}) (\hat{y}(t) - \bar{\hat{y}})}{\sqrt{\sum_{t=1}^T (y(t) - \bar{y})^2} \sqrt{\sum_{t=1}^T (\hat{y}(t) - \bar{\hat{y}})^2}}.
\end{equation}
This metric captures the linear correlation between signals and is bounded between \(-1\) and \(1\), with \(\rho = 1\) indicating perfect alignment. It is commonly used as a complementary metric to NRMSE in benchmarking RC models \cite{Lukosevicius2012, jalalvand2015real,Moses2018}.

\paragraph{Attractor Deviation.}
To measure how closely two trajectories align in state space, we use Attractor Deviation (ADev), following \cite{Zhai2023ResonanceML}. The space is discretized into uniform grid cells, each assigned a binary value: $1$ if a trajectory passes through it, $0$ otherwise. ADev is computed as the number of cells visited by only one of the trajectories, specifically for 3D-phase space \emph{ADev} is given as:
\begin{equation} 
    ADev = \sum_{i, j, k} \left| C_{ijk}^{true} - C_{ijk}^{pred} \right| \end{equation}
This captures topological differences between the two paths. The choice of grid resolution is critical—too coarse may overlook meaningful differences, while too fine may exaggerate minor deviations.

\paragraph{Spectral Similarity (Power Spectral Density Matching).}
To assess whether the reservoir captures the frequency content of a chaotic or periodic signal, one may compare the power spectral densities (PSDs) of the true and predicted trajectories:
\(
    \mathrm{PSD}_y(\omega), \mathrm{PSD}_{\hat{y}}(\omega).
\)
A similarity measure may be defined using the \(L^2\)-norm:
\begin{equation}
    D_{\text{PSD}} = \left( \int_{\omega_{\min}}^{\omega_{\max}} \left| \mathrm{PSD}_y(\omega) - \mathrm{PSD}_{\hat{y}}(\omega) \right|^2 \, d\omega \right)^{1/2}.
\end{equation}
This measure evaluates whether the model reproduces not only pointwise values but also the oscillatory behavior of the target system. It is useful when exact trajectory matching is not possible (as in chaotic systems), but capturing long-term statistical properties is desired \cite{pathak2017using, carroll2001quantifying,Toker2020}.

In \emph{Figure \ref{fig:lorenz_psd}}, we depict PSD of the \(z\)-component of the Lorenz system \(\dot{x}=\sigma\,(y-x),\,\dot{y}=x\,(\rho-z)-y,\,\dot{z}=x\,y-\beta\,z\) under standard chaotic parameters \(\{\sigma=10,\,\rho=28,\,\beta=8/3\}\).  By applying Welch’s method to a finely sampled trajectory, one observes a broad distribution in frequency space rather than sharp spectral lines, reflecting the complex, quasi-periodic structure typical of a chaotic flow.  In particular, the absence of narrow peaks and the slow decay of energy across a wide frequency range are hallmarks of sensitive dependence on initial conditions, manifesting here as a broad continuum in the spectrum.  From an echo state network perspective, such wideband spectral content emphasizes the challenge of capturing or predicting chaotic signals: the reservoir must be capable of accommodating a rich frequency composition over extended time intervals, thereby requiring both robust transient memory and nonlinear encoding capacities.

\begin{figure}[!ht]
    \centering
    \includegraphics[width=0.8\linewidth]{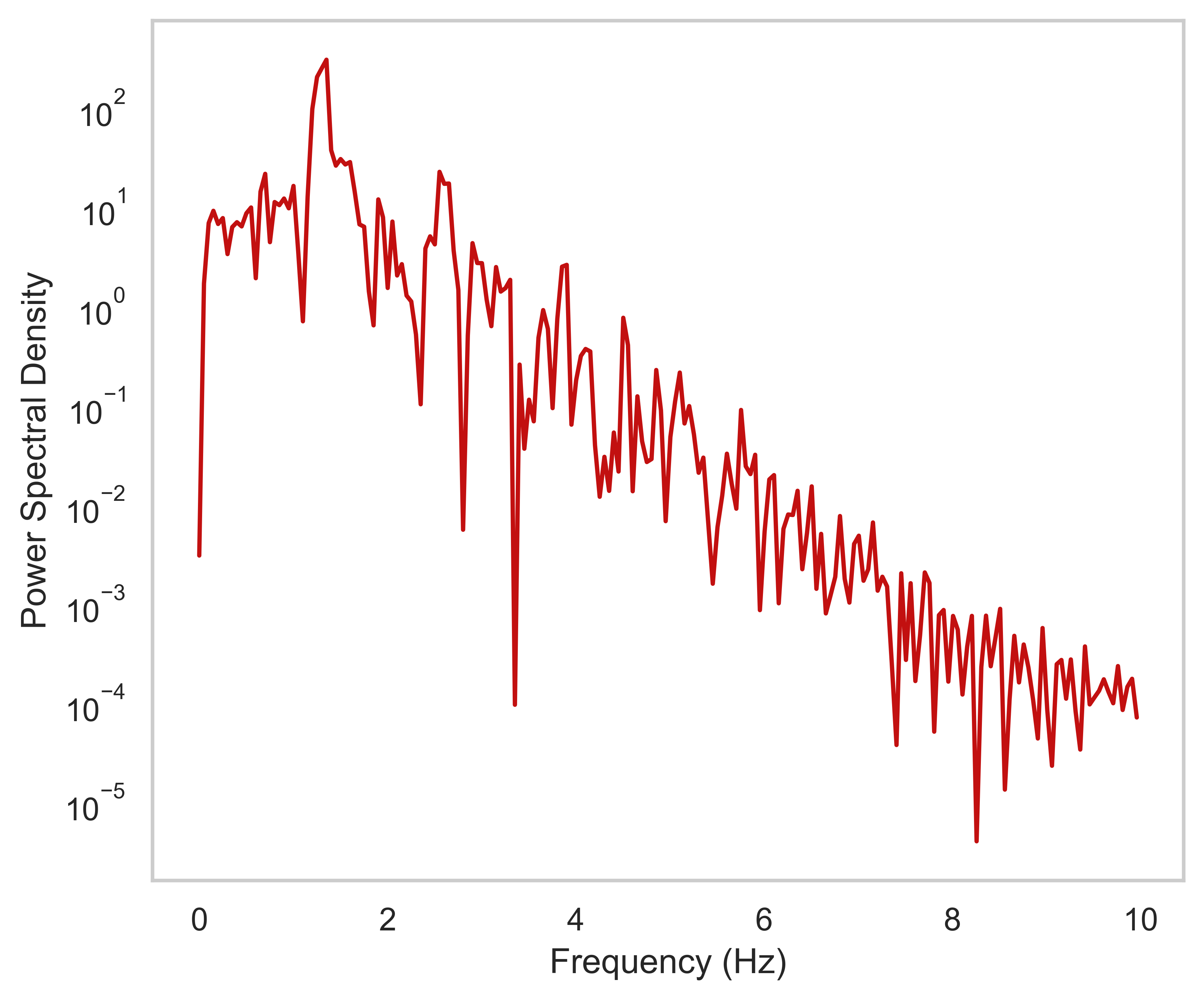}
    \caption{Power spectral density of the \(z(t)\) trajectory of the Lorenz system, computed via Welch’s method.  The broad, continuous frequency distribution indicates chaotic behavior, lacking the discrete lines characteristic of purely periodic or quasi-periodic dynamics.}
    \label{fig:lorenz_psd}
\end{figure}  

\paragraph{Lyapunov Exponent Deviation.}
For chaotic system emulation, it is often informative to compare the maximal Lyapunov exponent \(\lambda_{\max}\) of the ground truth and the predicted dynamics. Let \(\lambda_{\max}^y\) and \(\lambda_{\max}^{\hat{y}}\) denote the exponents of the true and model-generated trajectories, respectively. A useful metric is their difference:
\begin{equation}
    \Delta \lambda = \left| \lambda_{\max}^y - \lambda_{\max}^{\hat{y}} \right|.
\end{equation}
Small \(\Delta \lambda\) indicates that the reservoir correctly mimics the divergence rate of nearby trajectories, a hallmark of dynamical fidelity in chaotic regimes \cite{pathak2017using, parlitz1992prediction}.

\paragraph{Mean Absolute Error  and Mean Absolute Percentage Error}
Two classic yet informative measures are the Mean Absolute Error (MAE) and the Mean Absolute Percentage Error (MAPE). For a ground-truth time series \(\{y(t)\}\) and predictions \(\{\hat{y}(t)\}\):
\begin{equation}
\mathrm{MAE} \;=\; \frac{1}{T}\,\sum_{t=1}^T \bigl| y(t)\,-\,\hat{y}(t)\bigr|,
\quad
\mathrm{MAPE} \;=\; \frac{100\%}{T} \,\sum_{t=1}^T \left|\frac{y(t)-\hat{y}(t)}{y(t)}\right|.
\end{equation}
MAE is conceptually simple (absolute differences) and not as sensitive to outliers as squared-error metrics. MAPE directly measures relative error, which is especially useful when the target values vary by several orders of magnitude or when a scale-independent measure is needed.  
These metrics typically complement NRMSE in robust evaluations, since they highlight absolute or percentage-based deviations that might be masked when using a variance-based denominator.

\paragraph{Coefficient of Determination (\(R^2\))}
The coefficient of determination, \(R^2\), is common in regression analysis and sometimes reported in time-series contexts. It is defined as:
\begin{equation}
R^2 
\;=\;
1 \;-\;
\frac{\sum_{t=1}^T \bigl(y(t) \;-\;\hat{y}(t)\bigr)^2}
{\sum_{t=1}^T \bigl(y(t) \;-\;\bar{y}\bigr)^2},
\end{equation}
where \(\bar{y}\) is the mean of the true data. An \(R^2\) close to \(1\) indicates that the model explains most of the variance of the target series, while values near \(0\) or negative indicate poor explanatory power. Compared to the correlation coefficient \(\rho\), \(R^2\) is more sensitive to both linear fit and bias, thus providing a different perspective on model quality.

\paragraph{Dynamic Time Warping Distance}
When exact time alignment matters (e.g.\ signals may have phase shifts or local time distortions), Dynamic Time Warping (DTW) is a popular metric. DTW finds an optimal match between two time series by allowing local stretching or compressing of the time axis. The resulting DTW distance quantifies how much “warping” is needed:
\begin{equation}
\mathrm{DTW}\bigl(\mathbf{y}, \mathbf{\hat{y}}\bigr) 
\;=\;
\min_{\pi \in \mathcal{P}} \sum_{(t,t') \in \pi} d\ \!\bigl(\,y(t),\,\hat{y}(t')\bigr),
\end{equation}
where \(\pi\) ranges over all possible time-alignment paths, and \(d(\cdot,\cdot)\) is a local distance (e.g.\ absolute difference). DTW is particularly useful in reservoir computing for speech, gesture, or any signals with possible timing shifts, as it evaluates similarity independent of uniform time-lag alignment.

\paragraph{Permutation Entropy}
To evaluate whether a model preserves the complexity and ordinal patterns of a time series, Permutation Entropy (PE) can be used. PE quantifies the diversity of ordinal patterns in a sliding window of length \(D\):
Partition each local window \((y(t), \dots, y(t+D-1))\) into a permutation of \(\{1, 2, \dots, D\}\) based on ascending order.
 Count frequencies of each permutation pattern.
 Compute Shannon entropy of the permutation distribution.
\begin{equation}
H_{\text{perm}} \;=\; -\sum_{i=1}^{D!} p_i \,\ln (p_i),
\end{equation}
where \(p_i\) is the frequency of the \(i\)-th permutation pattern. Comparing PE values (or distributions of ordinal patterns) between the true and predicted series reveals whether the forecast preserves the underlying dynamical complexity and short-scale ordering of data—highly relevant for nonlinear or chaotic signals.

\paragraph{Distributional Similarity (Kullback–Leibler Divergence)}
When the goal is to match the statistical distribution of outputs rather than exact trajectories, a Kullback–Leibler (KL) Divergence between the empirical distributions of \(\{y(t)\}\) and \(\{\hat{y}(t)\}\) can be informative:
\begin{equation}
D_{\text{KL}}\bigl(P\|Q\bigr) 
\;=\;
\int p(z)\,\ln\ \!\Bigl(\tfrac{p(z)}{q(z)}\Bigr)\,dz,
\end{equation}
where \(p\) and \(q\) are probability density functions (estimated from \(y(t)\) and \(\hat{y}(t)\) respectively). If \(D_{\text{KL}}\) is small, the two distributions share similar statistical characteristics (mean, variance, higher moments), even if the forecasts do not match point by point. This is especially relevant in chaotic or stochastic tasks, where matching \emph{distributional} properties may suffice.

\paragraph{Diebold–Mariano Statistic}
To compare two forecasts \(\hat{y}_A(t)\) and \(\hat{y}_B(t)\) made by different models (e.g., two reservoir configurations) against the same ground truth \(y(t)\), the Diebold–Mariano test provides a statistical measure of whether one forecast is significantly more accurate. Let \(e_{A}(t)\) and \(e_{B}(t)\) be the respective error series (e.g.\ absolute or squared errors). The DM statistic tests the null hypothesis that the two forecasts have the same expected accuracy. It is computed from a time series of “loss differentials” \(\{d(t)=\ell(e_A(t)) - \ell(e_B(t))\}\) for some chosen loss function \(\ell\). Significance levels indicate whether model A outperforms model B in a statistically meaningful sense, an important consideration when differences in NRMSE or MAE are small.

\vspace{1em}
In practice, different metrics highlight different aspects of forecasting performance. NRMSE emphasizes pointwise accuracy, VPT emphasizes time-localized stability, PSD similarity focuses on spectral fidelity, and Lyapunov deviation quantifies dynamical consistency. A comprehensive evaluation of a reservoir model—especially in the context of nonlinear and chaotic systems—often combines multiple of these measures to obtain a robust characterization.

\subsection{Resource Requirements}
\paragraph{Memory Footprint.}
The principal memory usage arises from storing the reservoir weight matrix \(\mathbf{W}_\mathrm{res}\) of size \(N \times N\), and
    the collected reservoir states \(\mathbf{X}\) if offline batch training is used (size \(N \times T\)).
Online training procedures can reduce memory load by discarding reservoir states after each update \cite{Lukosevicius2012}.

\paragraph{Latency and Real-Time Considerations.}
Because the state update in an RC system is straightforward (and often can be parallelized), many real-time applications (e.g., signal processing and control) can run the RC update at the same rate as input sampling \cite{Verstraeten2007}. When the reservoir is small or sparse, such updates can be performed with low latency on embedded systems.

\paragraph{Hardware Acceleration.}
Some recent work explores specialized hardware (e.g., neuromorphic chips, FPGAs) for implementing reservoir architectures. Exploiting parallelism in the recurrent connections or the spiking nature of Liquid State Machines can further reduce power consumption and increase speed \cite{Hammer2019}.

\subsection{Experiments}
To evaluate the effectiveness and generalizability of the proposed methods outlined in \S \ref{sec:archi}, we present an extensive series of experiments across a diverse range of dynamical systems. These systems, selected to represent varying degrees of complexity and nonlinearity, are described in detail in \S \ref{sec:analysis}. Various reservoir models are evaluated on these systems in both open-loop (teacher-forced) forecasting and closed-loop (autoregressive) forecasting settings.

All datasets used in our experiments consist of time series sampled from their respective continuous-time dynamical systems. These sequences were generated by numerically integrating the governing ordinary differential equations (ODEs) using the classical fourth-order Runge–Kutta (RK4) method \cite{butcher2008numerical}. For each system, we employed an appropriate integration step size to ensure a stable and accurate representation of the dynamics. During the sampling process, we discard an initial portion of the simulated data—commonly referred to as the sampling washout period—to eliminate the influence of transient dynamics that arise from arbitrary or non-representative initial conditions. This approach ensures that the retained data reflects the long-term behavior of the system on or near its attractor, rather than being dominated by initialization artifacts. As part of the experimental setup, Table~\ref{tab:dataset_characteristics} lists the key attributes of each system, such as integration parameters and dimensionality. 
For each dataset, $10,000$ time steps were sampled. The initial $2,000$ steps were discarded as a washout period to eliminate transients arising from initial conditions. Of the remaining $8,000$ steps, $4,500$ were used for training (including a $100$-step warm-up phase), while the remaining $3,500$ steps were reserved for testing. Table~\ref{tab:LLE_LT_datasets} details the maximum Lyapunov Exponent $\lambda_{max}$ and Lyapunov time for each dynamical system.

\begin{table}[!ht]
\centering
\caption{Dataset Characteristics: Input/output dimensions, parameters, step sizes, and initial conditions. All datasets are simulated using the RK method.}

\label{tab:dataset_characteristics}
\renewcommand{\arraystretch}{1.5} 
\resizebox{0.8\textwidth}{!}{%
\begin{tabular}{lccccc}
\toprule
\textbf{Dataset} & \textbf{$d_{in}$} & \textbf{$d_{out}$} & \textbf{Parameters} & \textbf{Step Size} & \textbf{Initial Value} \\
\midrule
Lorenz            & 3 & 3 & $\sigma=10, \rho=28, \beta=\tfrac{8}{3}$ & 0.02& \{1, 1, 1\} \\
R\"{o}ssler       & 3 & 3 & $a=0.2, b=0.2, c=5.7$ & 0.02 & \{1, 1, 1\} \\
Chen              & 3 & 3 & $a=35, b=3, c=28$ & 0.02 & \{1, 1, 1\} \\
Chua              & 3 & 3 & \makecell{$\alpha=9, \beta=14,$ \\$  m_0= -1.143, m_1=-0.714$} & 0.02 & \{1, 1, 1\} \\
\bottomrule
\end{tabular}
}
\end{table}

\begin{figure}[htbp]
    \centering

    \begin{subfigure}[t]{0.23\textwidth}
        \includegraphics[width=\linewidth]{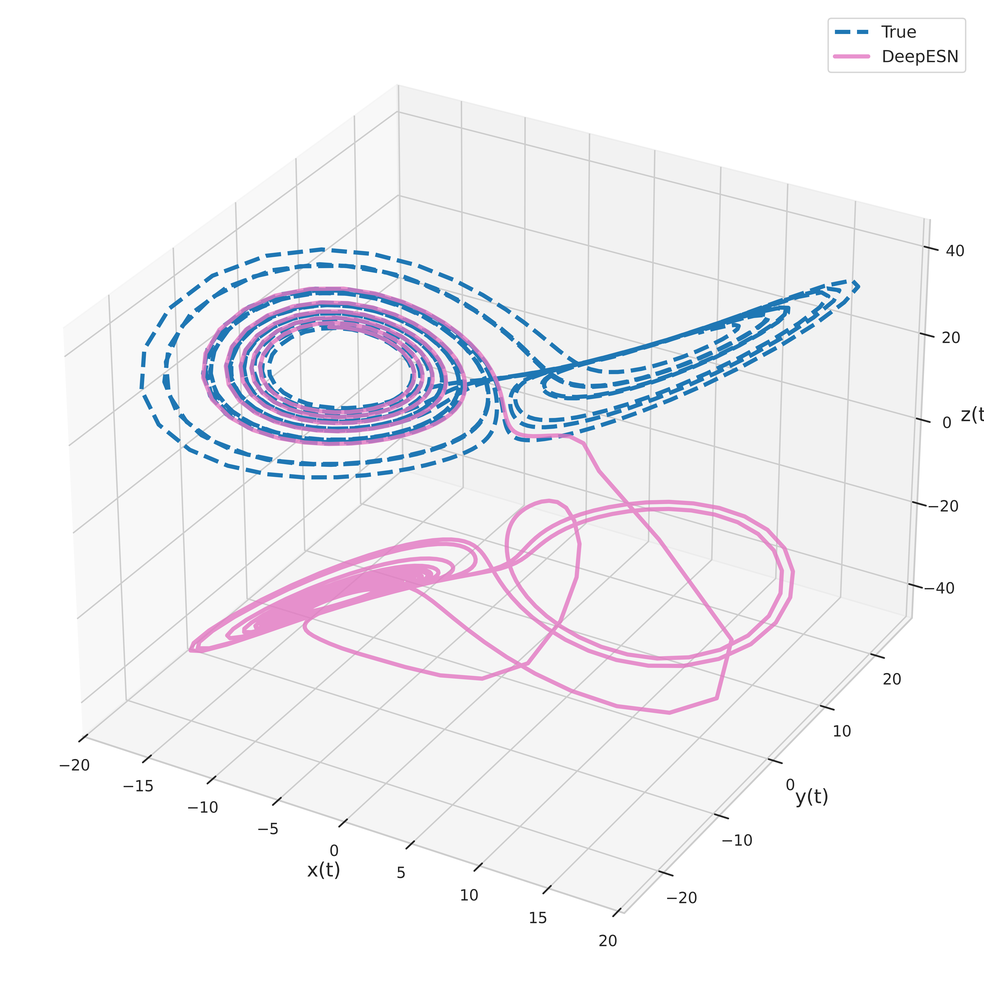}
        \caption{DeepESN}
    \end{subfigure}
    \begin{subfigure}[t]{0.23\textwidth}
        \includegraphics[width=\linewidth]{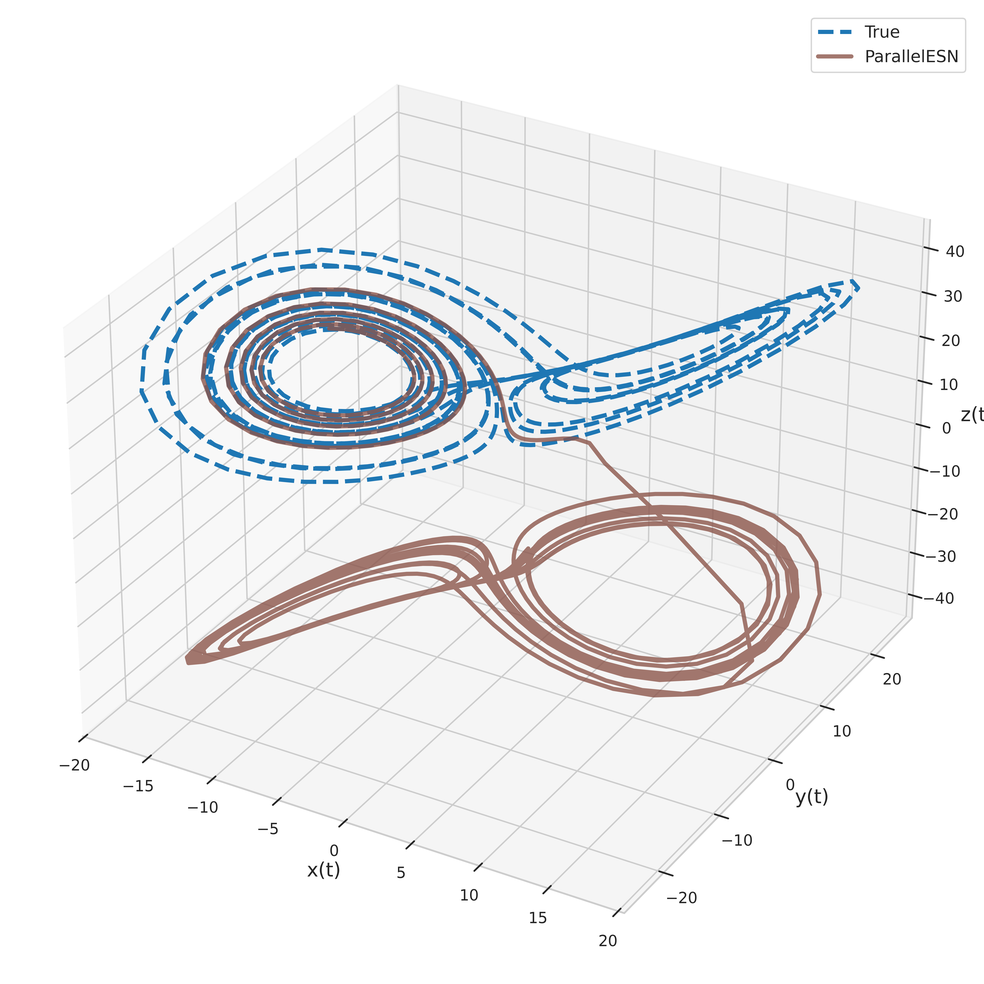}
        \caption{Parallel ESN}
    \end{subfigure}
    \begin{subfigure}[t]{0.23\textwidth}
        \includegraphics[width=\linewidth]{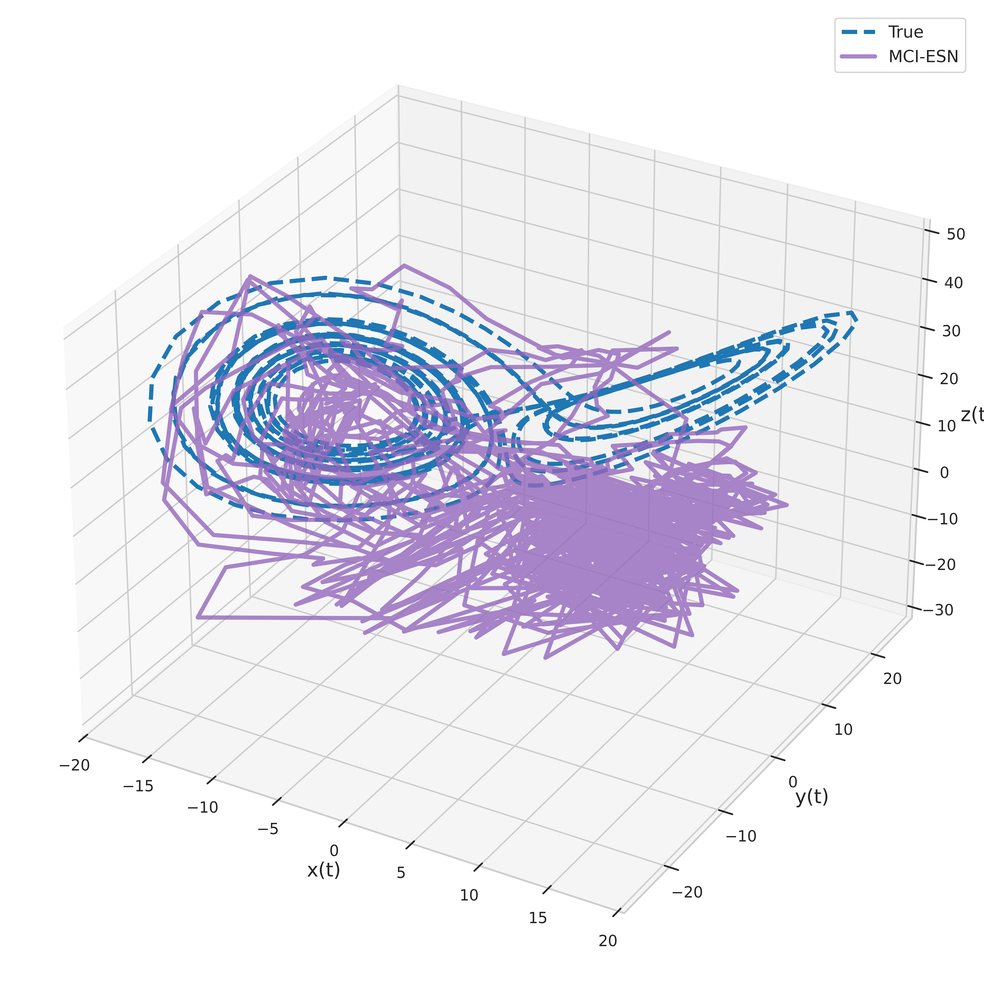}
        \caption{MCI-ESN}
    \end{subfigure}
    \begin{subfigure}[t]{0.23\textwidth}
        \includegraphics[width=\linewidth]{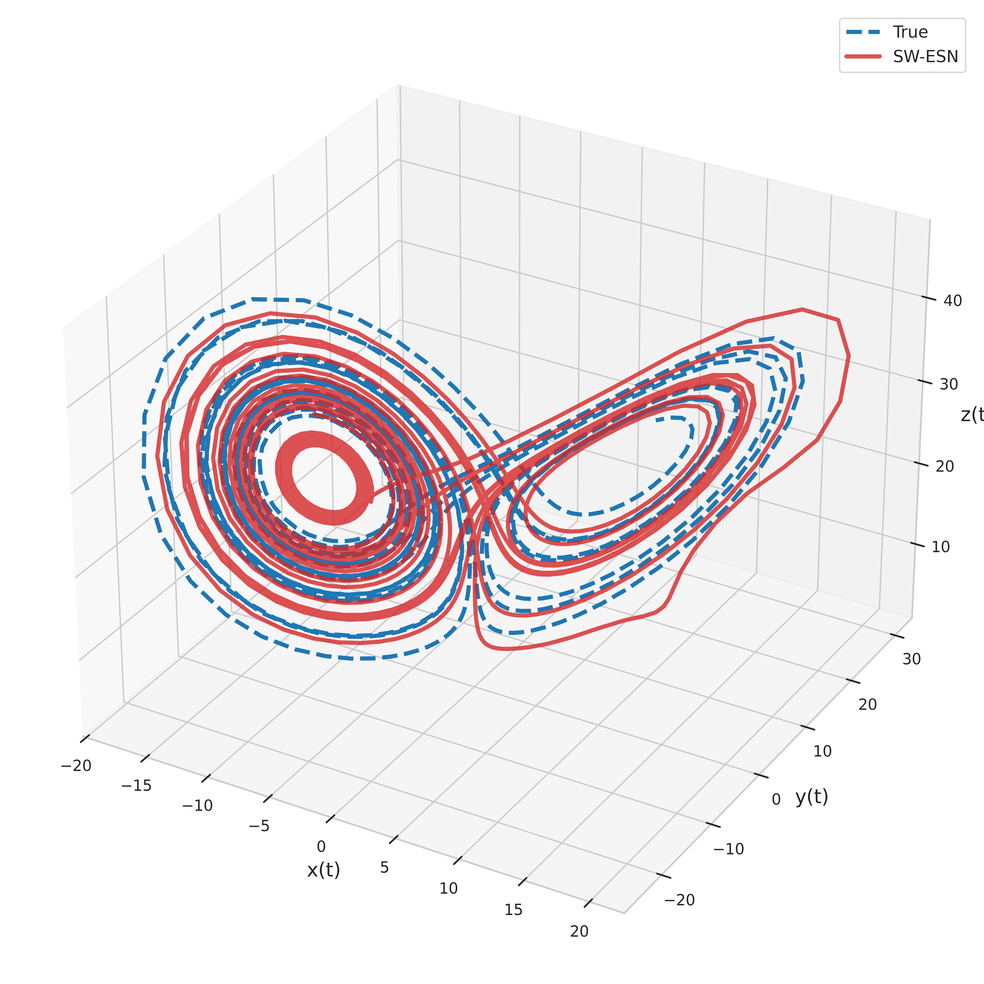}
        \caption{SW-ESN}
    \end{subfigure}

    \vspace{0.4cm}

    \begin{subfigure}[t]{0.23\textwidth}
        \includegraphics[width=\linewidth]{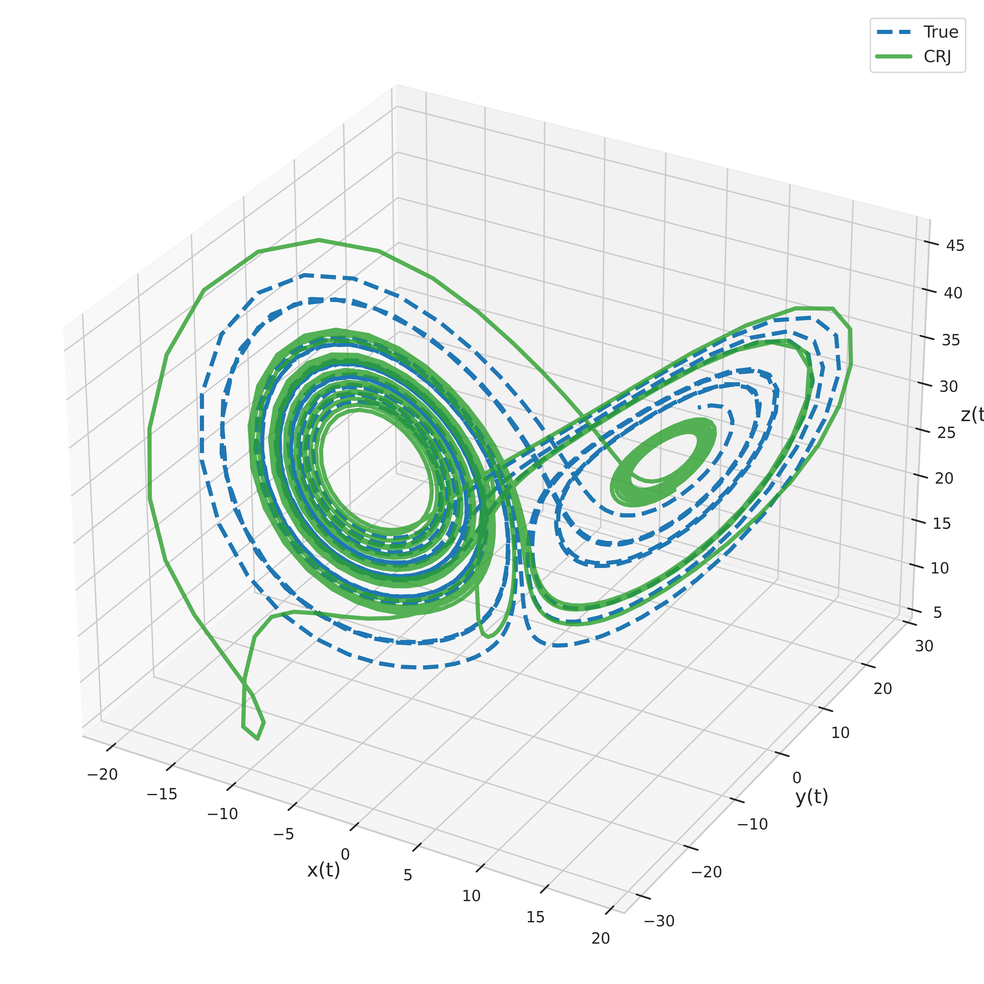}
        \caption{CRJ}
    \end{subfigure}
    \begin{subfigure}[t]{0.23\textwidth}
        \includegraphics[width=\linewidth]{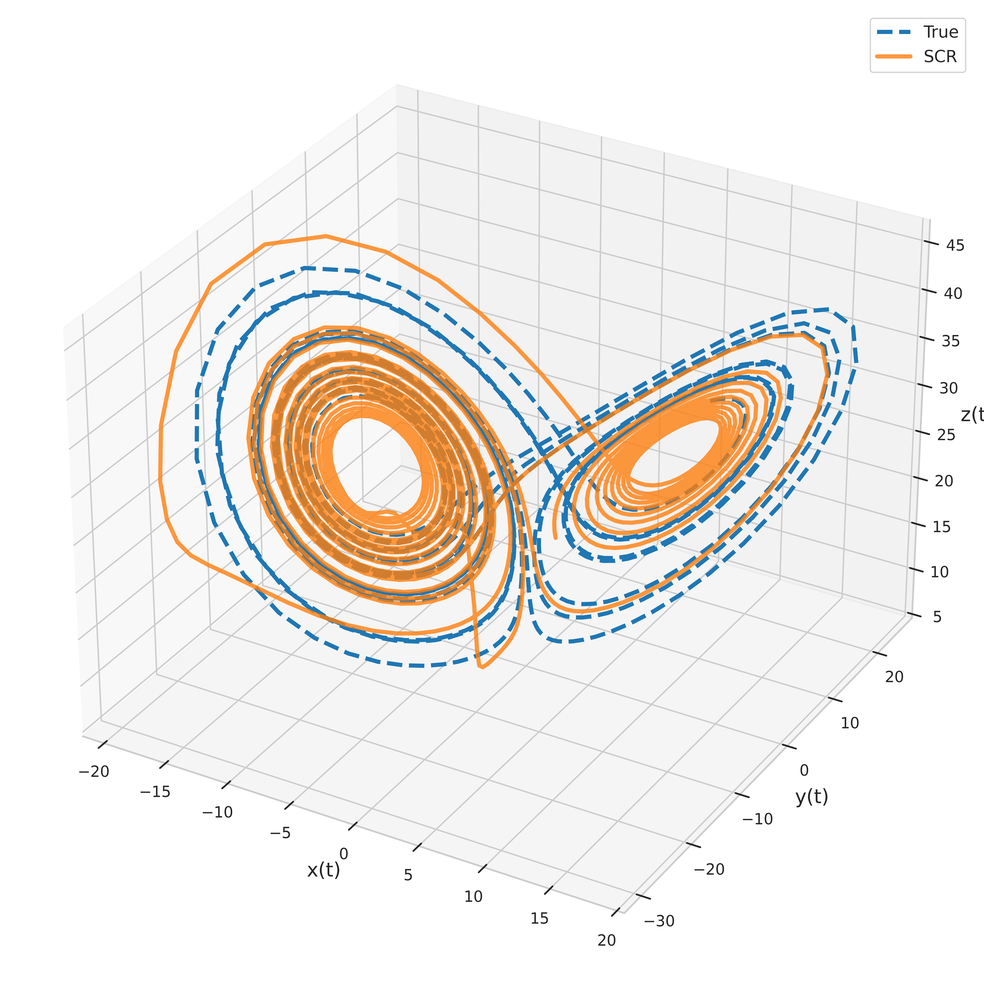}
        \caption{SCR}
    \end{subfigure}
    \begin{subfigure}[t]{0.23\textwidth}
        \includegraphics[width=\linewidth]{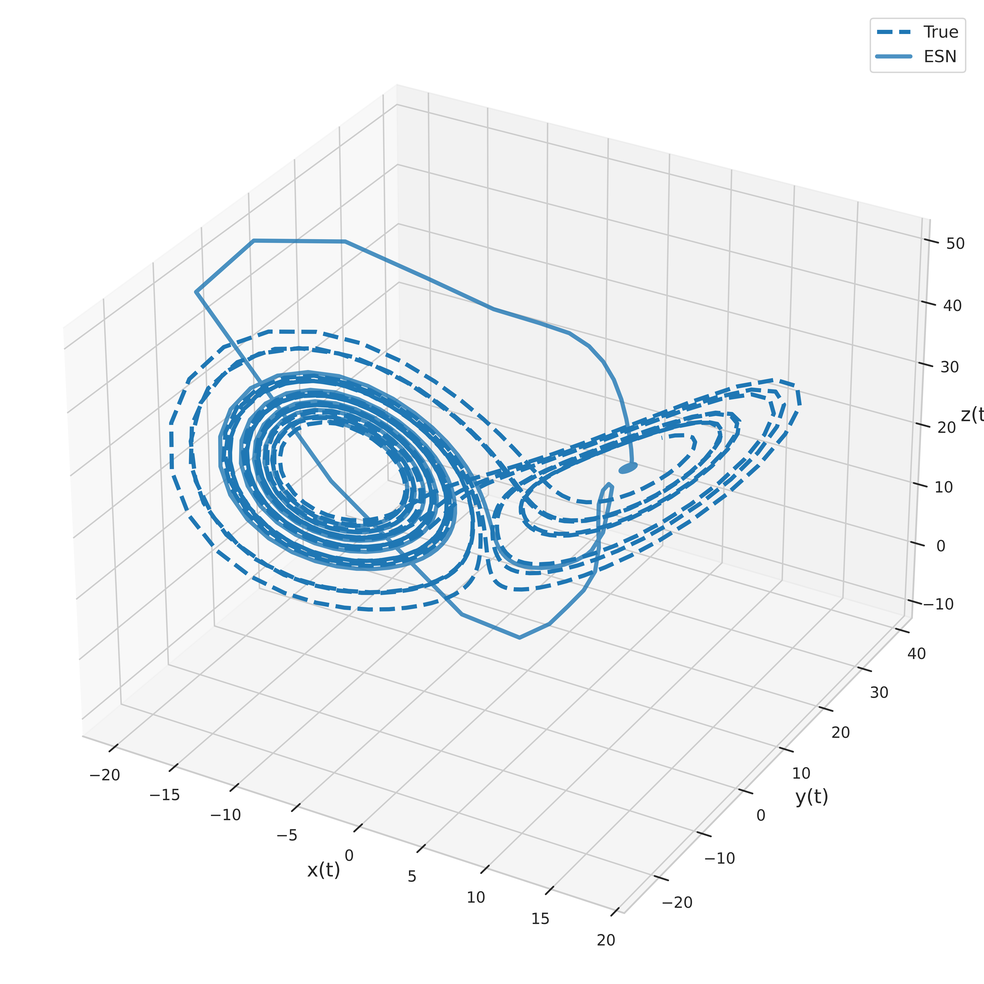}
        \caption{Vanilla ESN}
    \end{subfigure}

    \caption{3D Phase portraits for Lorenz system predicted by different reservoir architectures in closed-loop forecastings.}
    \label{fig:lorenz-phase_cl}
\end{figure}

\begin{figure}[htbp]
    \centering

    \begin{subfigure}[t]{0.23\textwidth}
        \includegraphics[width=\linewidth]{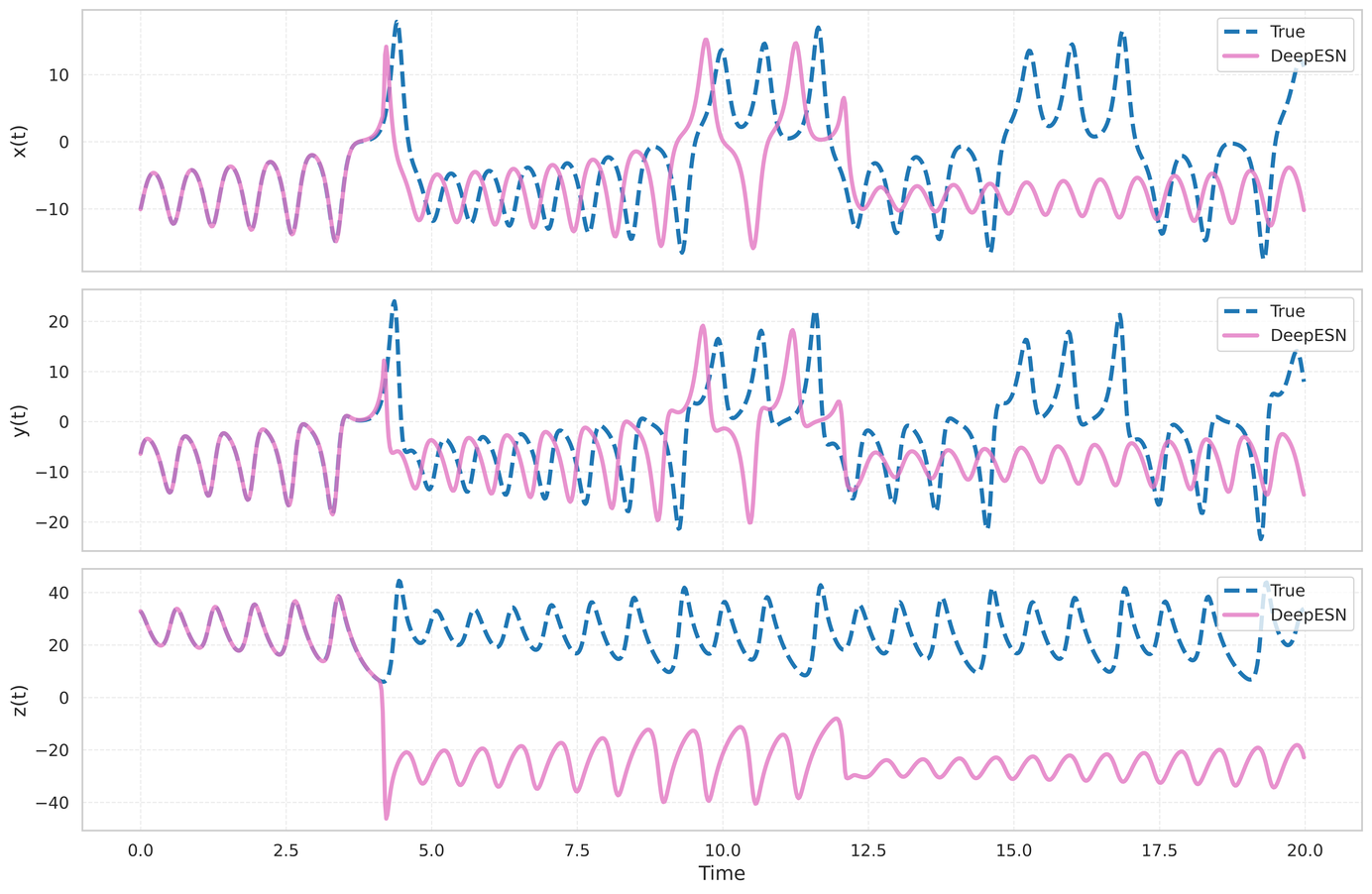}
        \caption{DeepESN}
    \end{subfigure}
    \begin{subfigure}[t]{0.23\textwidth}
        \includegraphics[width=\linewidth]{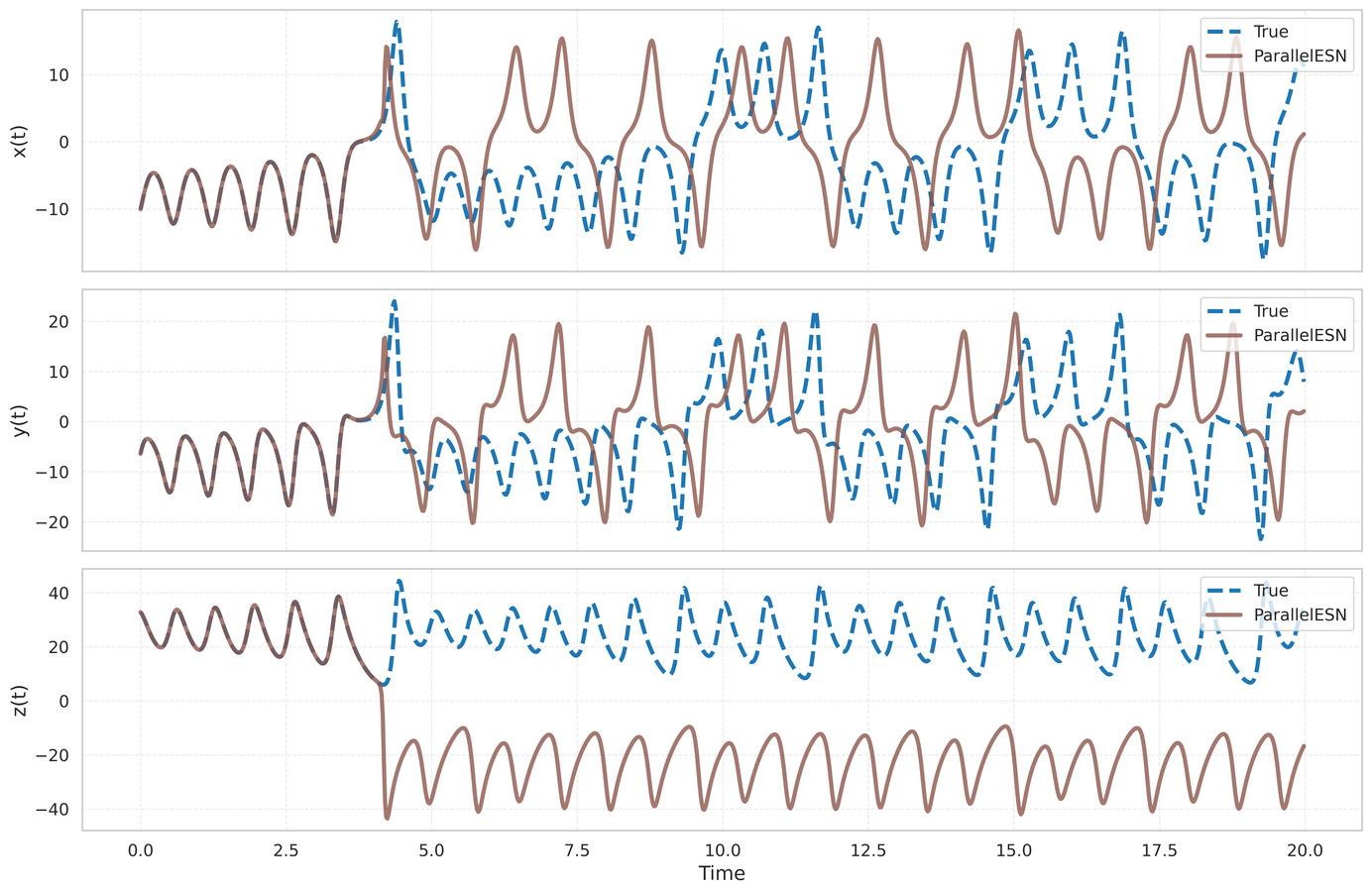}
        \caption{ParallelESN}
    \end{subfigure}
    \begin{subfigure}[t]{0.23\textwidth}
        \includegraphics[width=\linewidth]{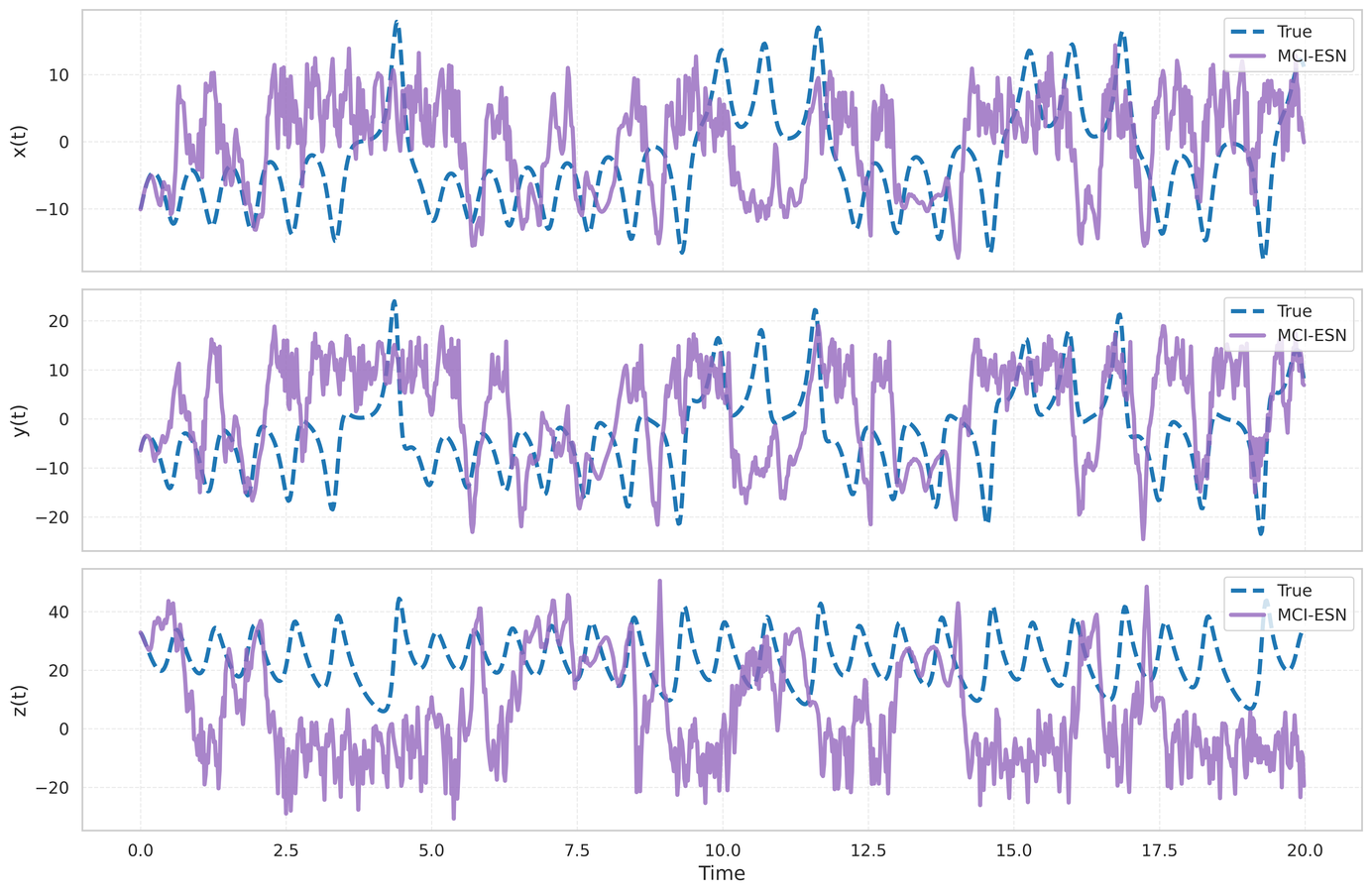}
        \caption{MCI-ESN}
    \end{subfigure}
    \begin{subfigure}[t]{0.23\textwidth}
        \includegraphics[width=\linewidth]{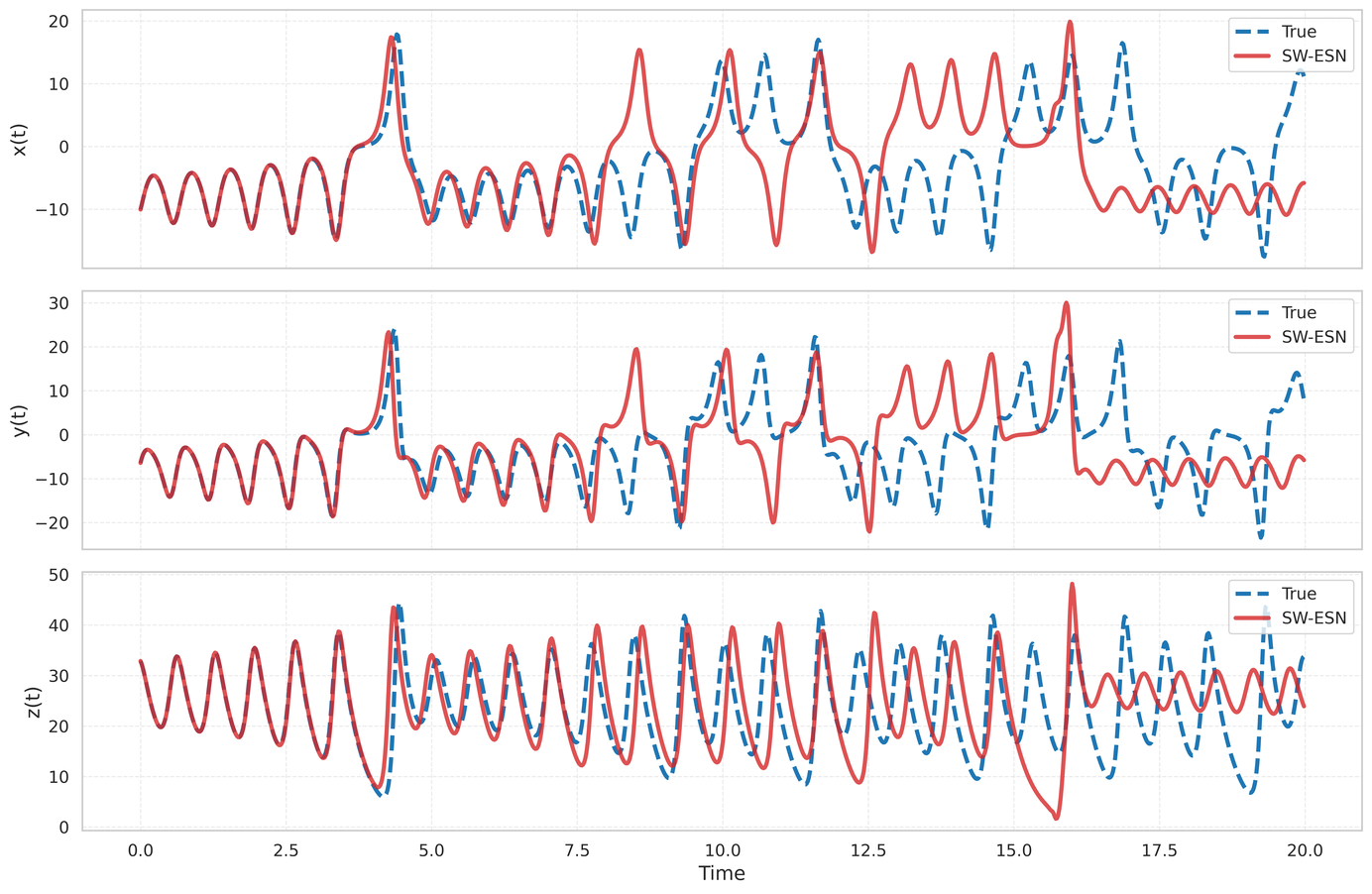}
        \caption{SW-ESN}
    \end{subfigure}

    \vspace{0.4cm}

    \begin{subfigure}[t]{0.23\textwidth}
        \includegraphics[width=\linewidth]{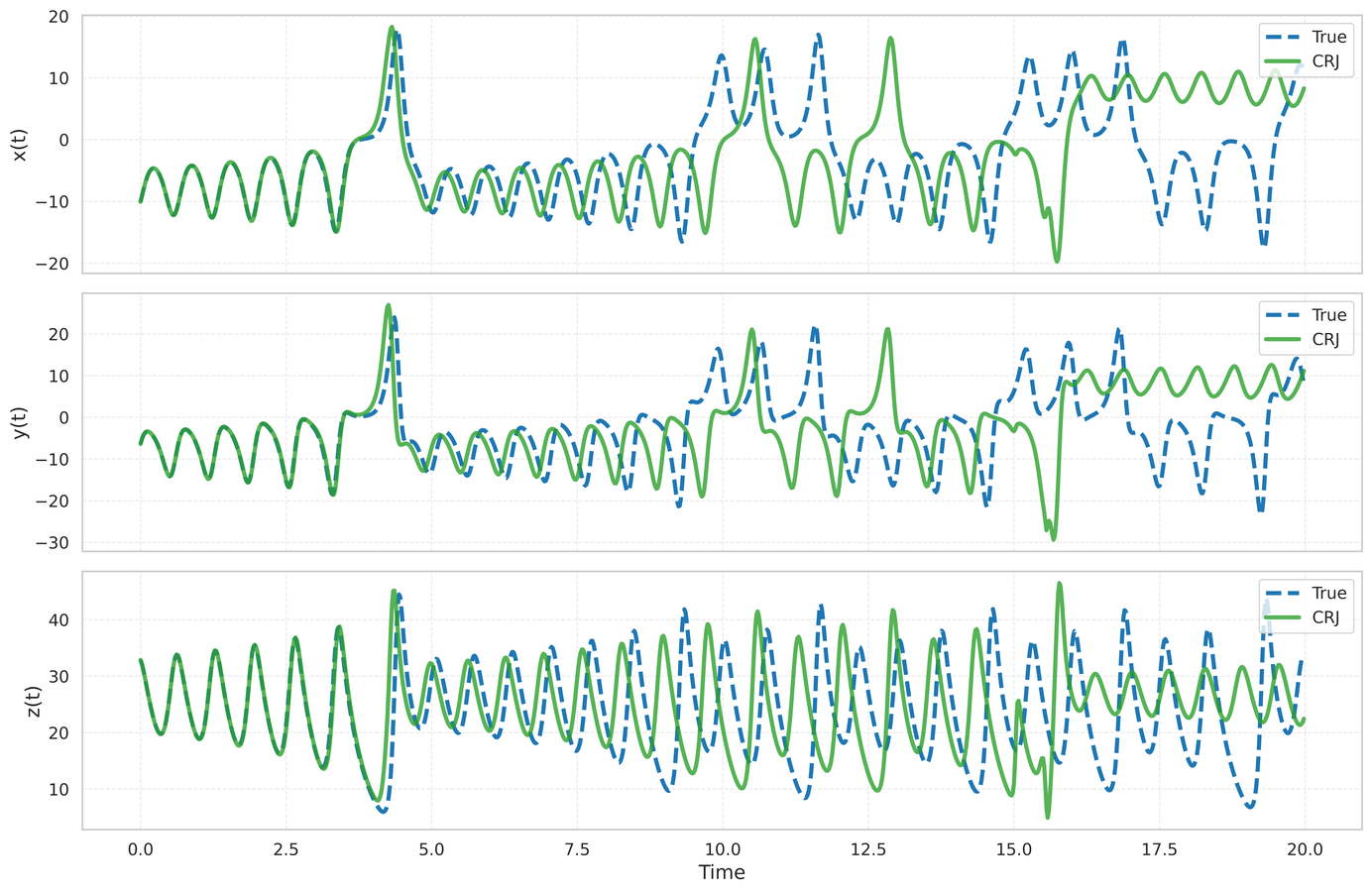}
        \caption{CRJ}
    \end{subfigure}
    \begin{subfigure}[t]{0.23\textwidth}
        \includegraphics[width=\linewidth]{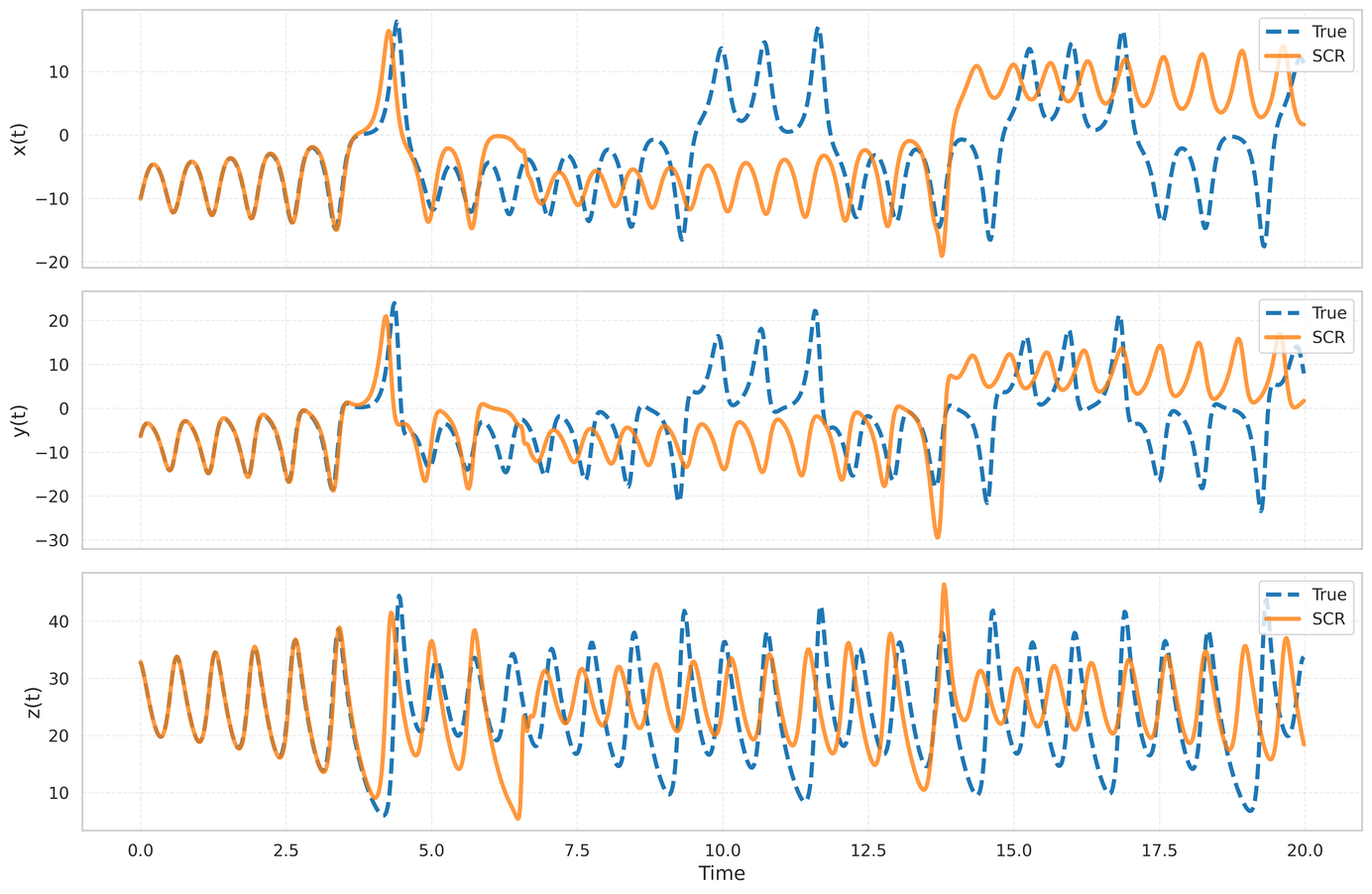}
        \caption{SCR}
    \end{subfigure}
    \begin{subfigure}[t]{0.23\textwidth}
        \includegraphics[width=\linewidth]{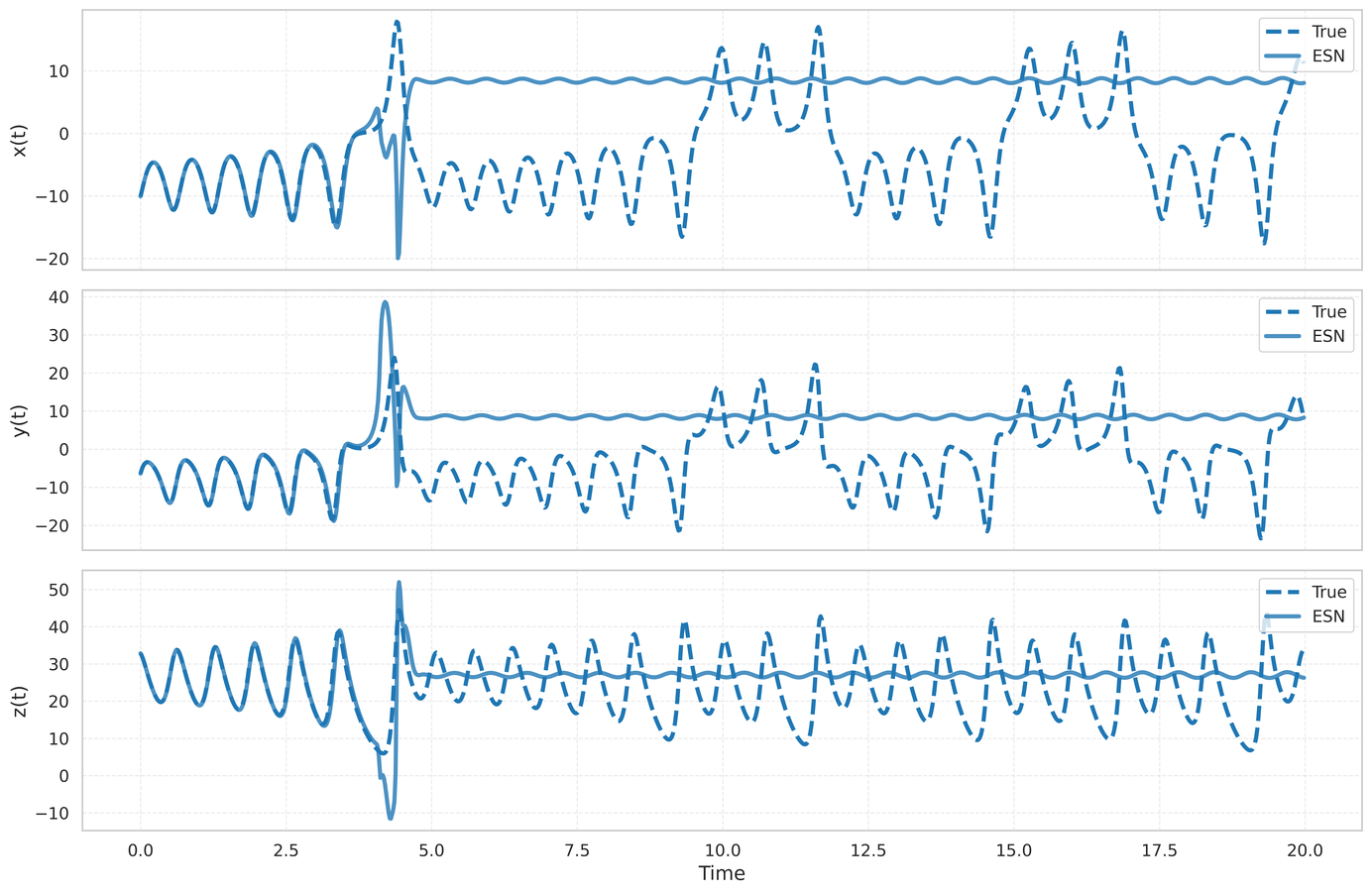}
        \caption{Vanilla ESN}
    \end{subfigure}

    \caption{Comparison of true and predicted trajectories for the Lorenz system across various reservoir models.}
    \label{fig:lorenz-pred_cl}
\end{figure}

\begin{figure}[htbp]
    \centering

    \begin{subfigure}[t]{0.23\textwidth}
        \includegraphics[width=\linewidth]{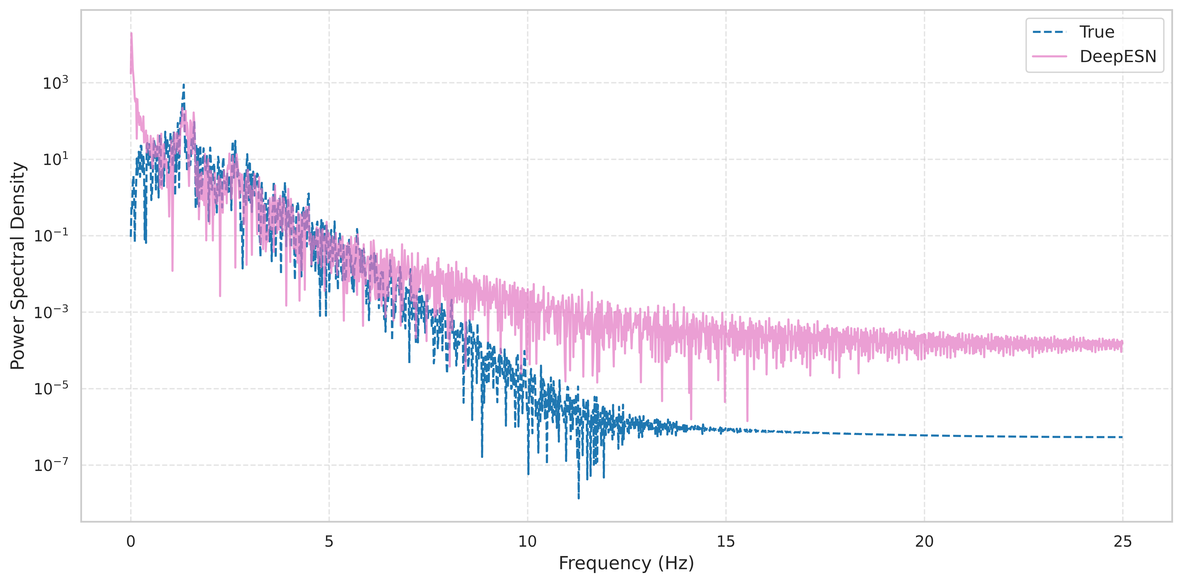}
        \caption{DeepESN}
    \end{subfigure}
    \begin{subfigure}[t]{0.23\textwidth}
        \includegraphics[width=\linewidth]{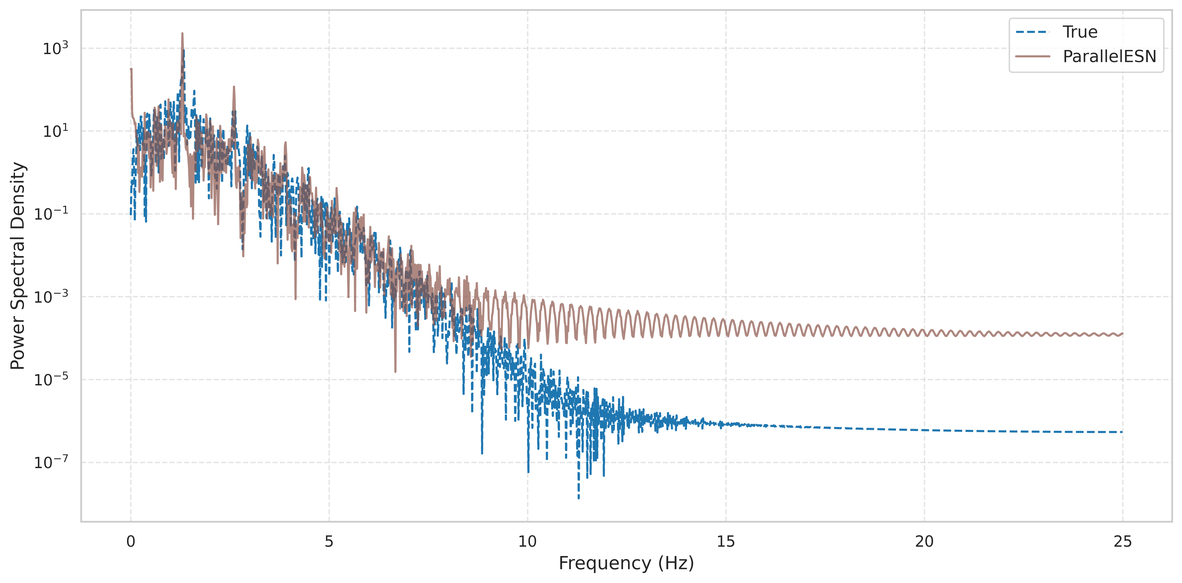}
        \caption{Parallel ESN}
    \end{subfigure}
    \begin{subfigure}[t]{0.23\textwidth}
        \includegraphics[width=\linewidth]{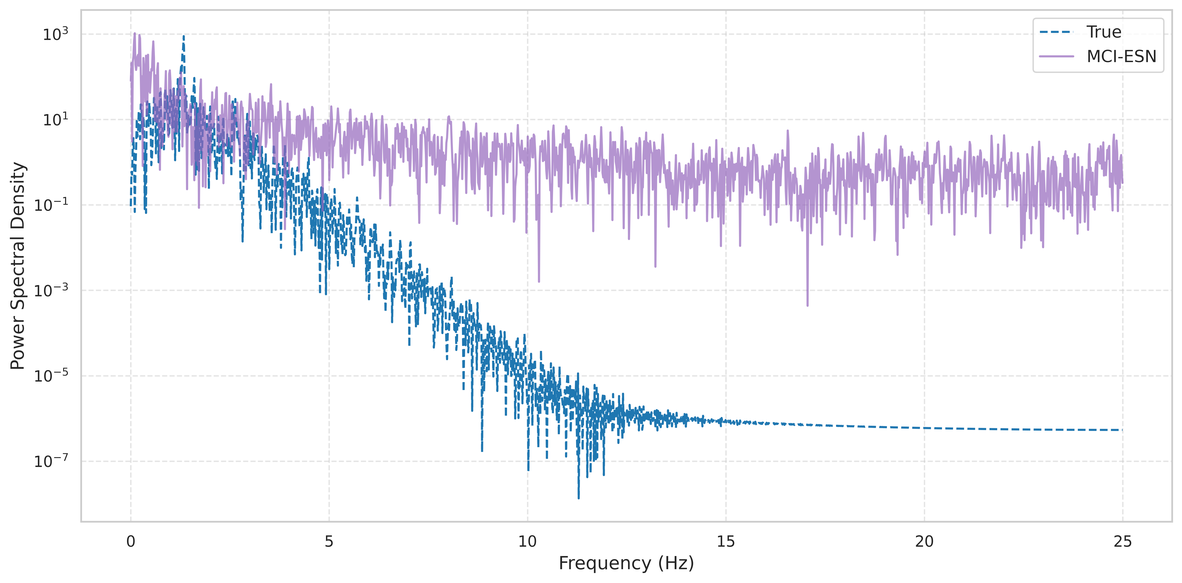}
        \caption{MCI-ESN}
    \end{subfigure}
    \begin{subfigure}[t]{0.23\textwidth}
        \includegraphics[width=\linewidth]{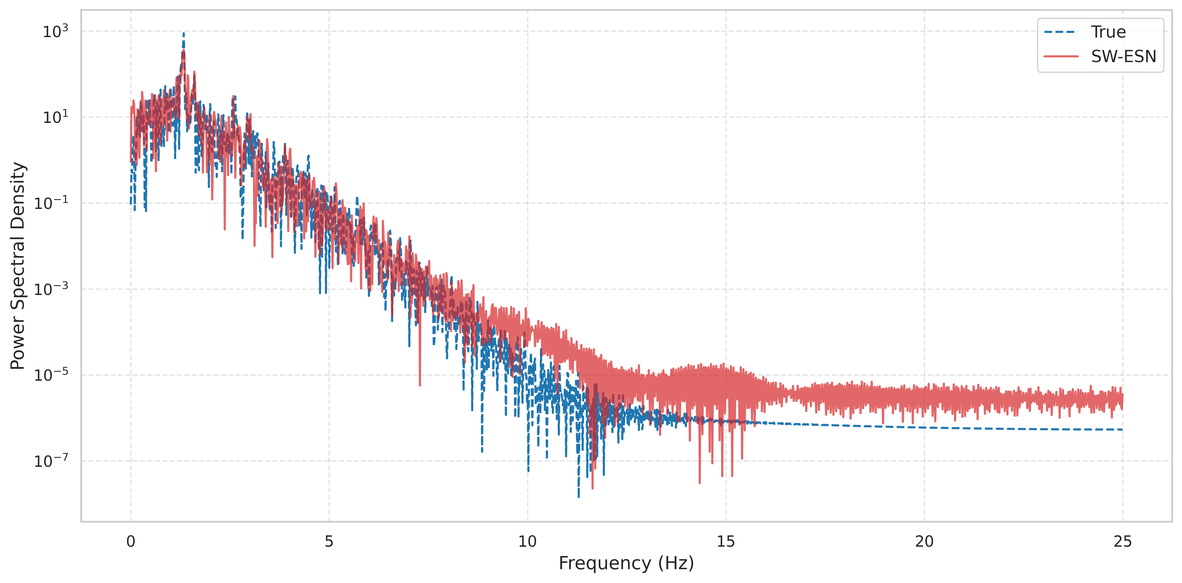}
        \caption{SW-ESN}
    \end{subfigure}

    \vspace{0.4cm}

    \begin{subfigure}[t]{0.23\textwidth}
        \includegraphics[width=\linewidth]{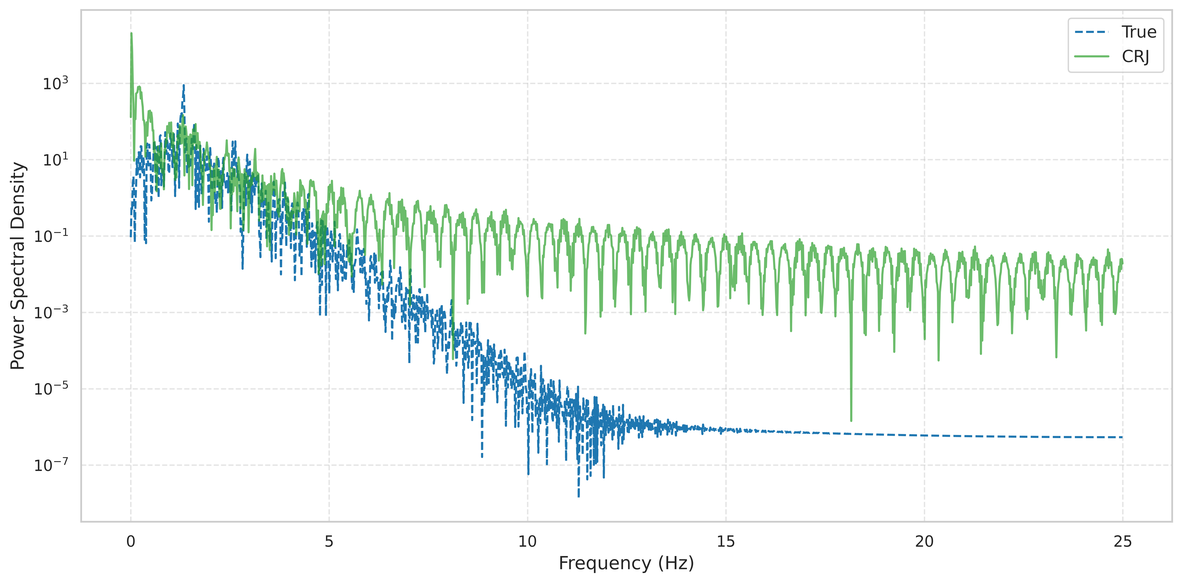}
        \caption{CRJ}
    \end{subfigure}
    \begin{subfigure}[t]{0.23\textwidth}
        \includegraphics[width=\linewidth]{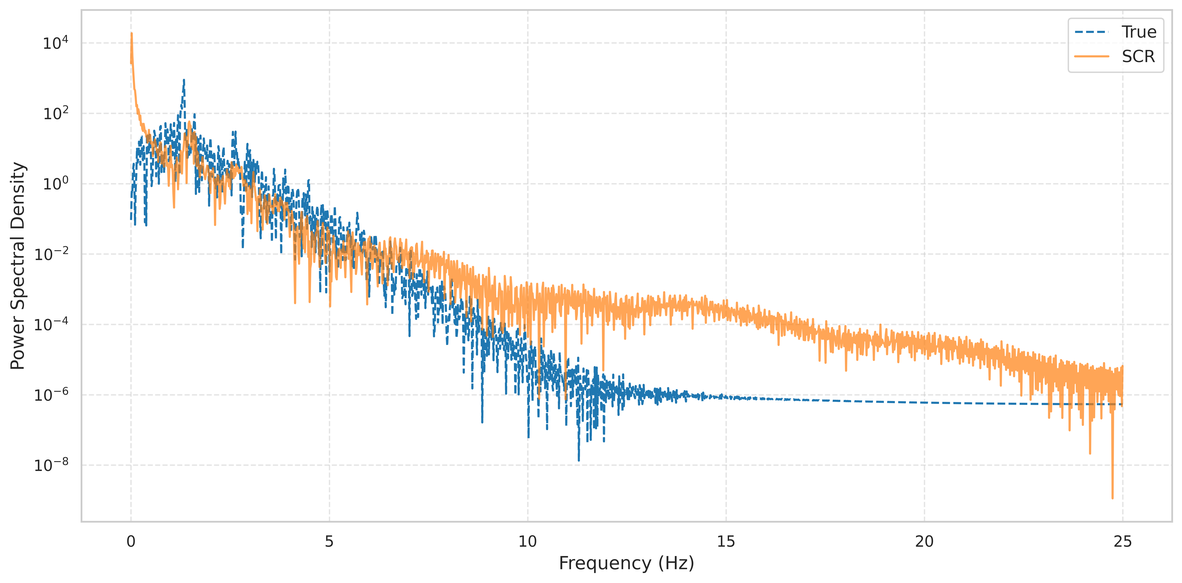}
        \caption{SCR}
    \end{subfigure}
    \begin{subfigure}[t]{0.23\textwidth}
        \includegraphics[width=\linewidth]{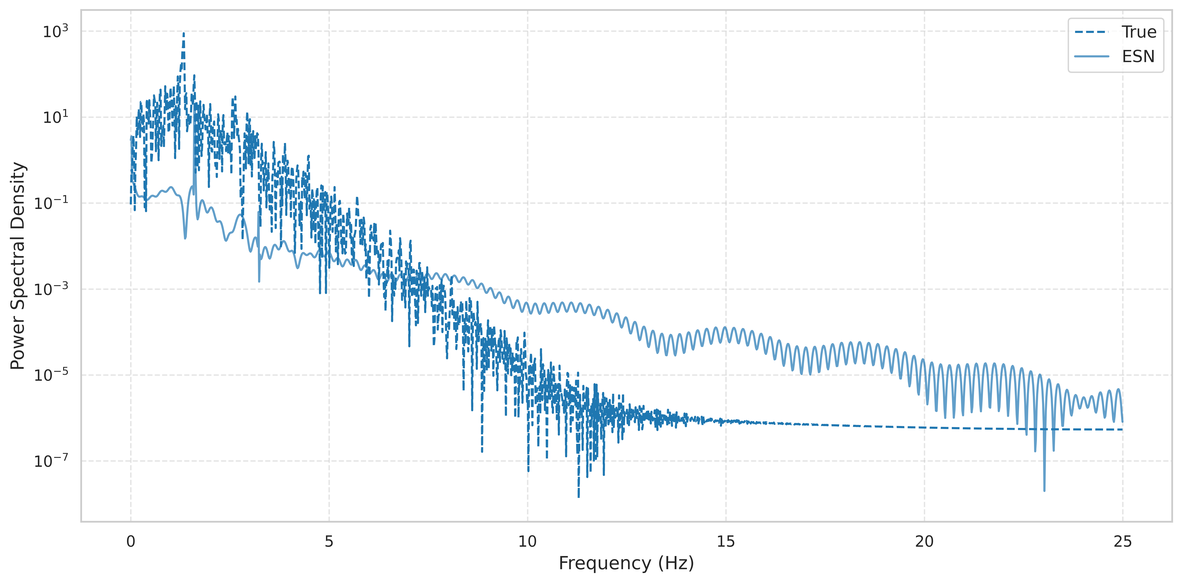}
        \caption{Vanilla ESN}
    \end{subfigure}

    \caption{Analysis of the $z$-component PSD of the Lorenz system with different reservoir architectures.}
    \label{fig:lorenz-psd_cl}
\end{figure}

\begin{figure}[htbp]
    \centering

    \begin{subfigure}[t]{0.23\textwidth}
        \includegraphics[width=\linewidth]{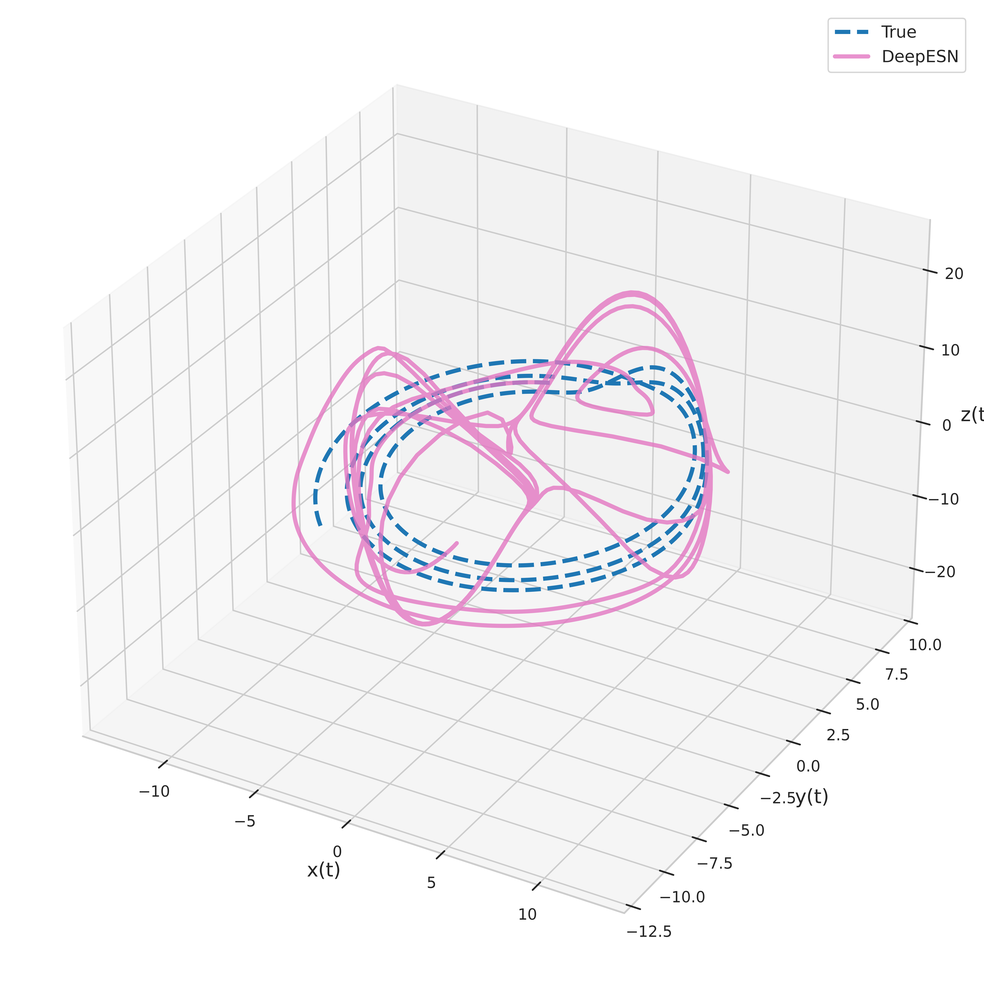}
        \caption{DeepESN}
    \end{subfigure}
    \begin{subfigure}[t]{0.23\textwidth}
        \includegraphics[width=\linewidth]{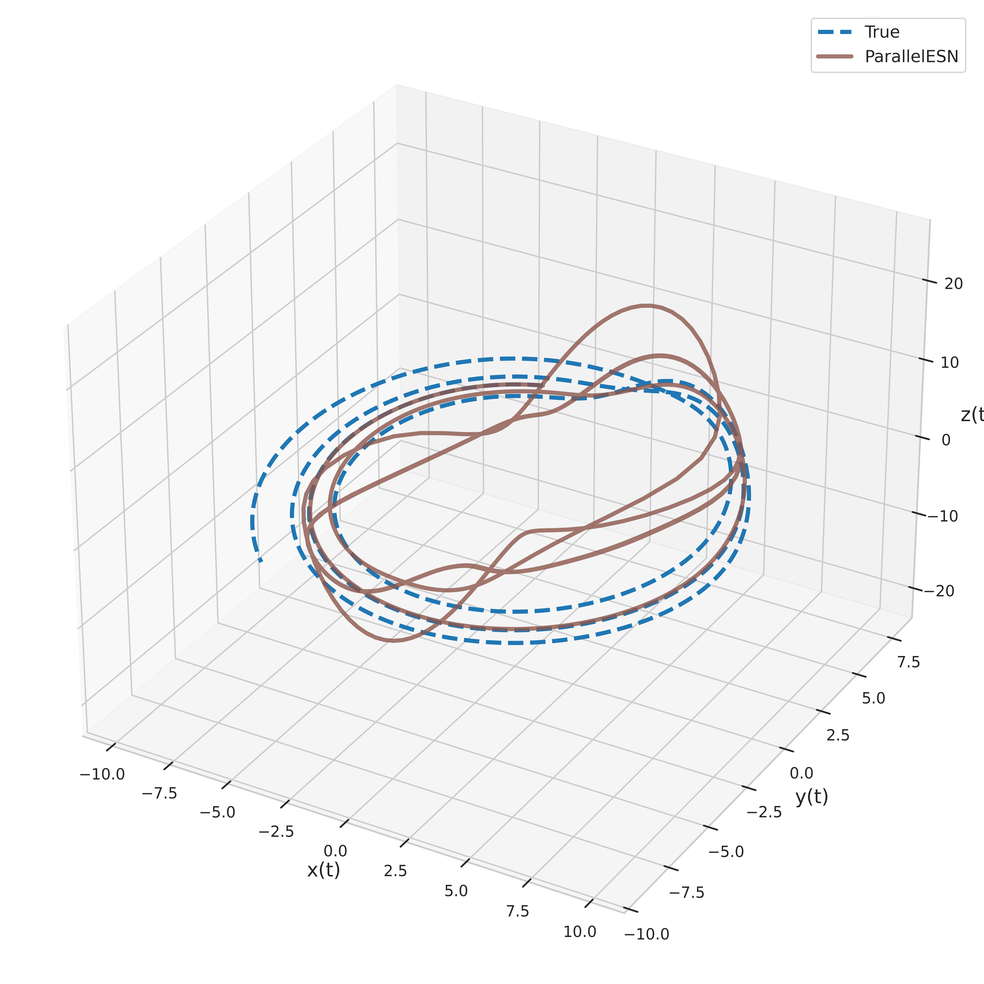}
        \caption{Parallel ESN}
    \end{subfigure}
    \begin{subfigure}[t]{0.23\textwidth}
        \includegraphics[width=\linewidth]{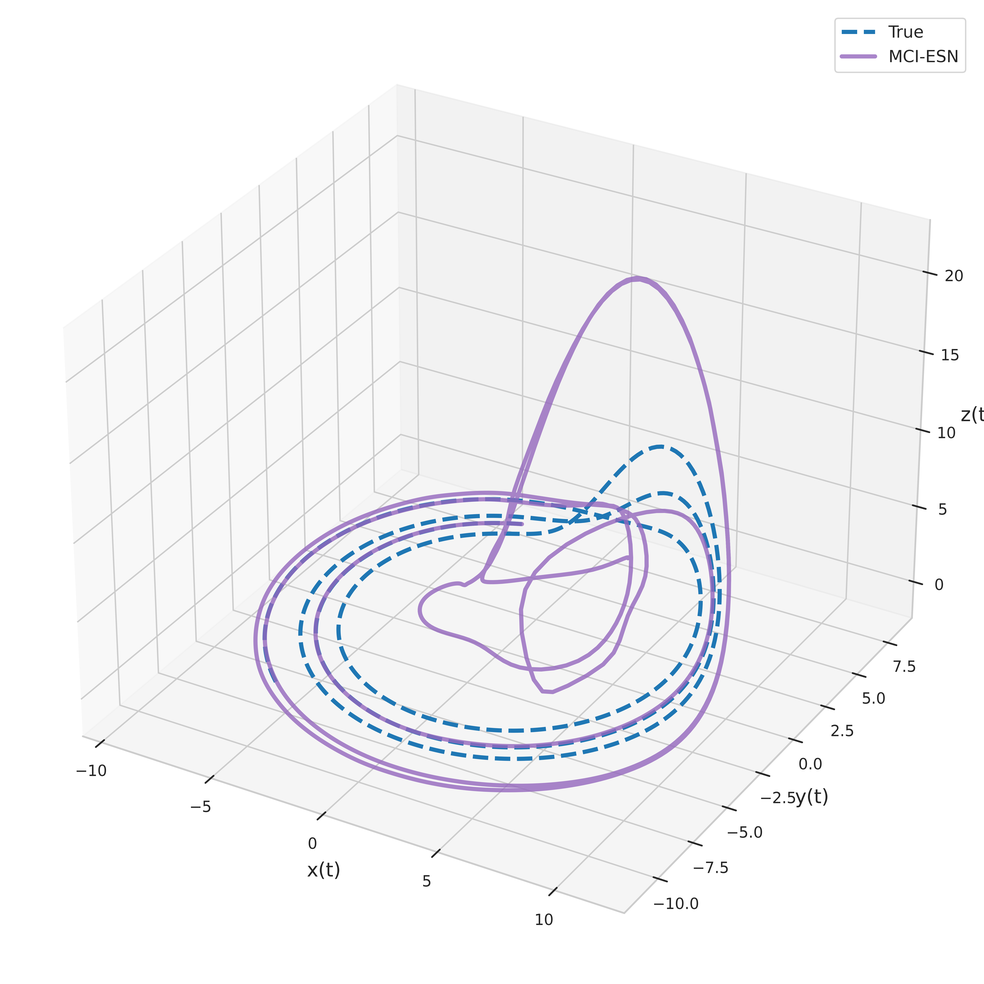}
        \caption{MCI-ESN}
    \end{subfigure}
    \begin{subfigure}[t]{0.23\textwidth}
        \includegraphics[width=\linewidth]{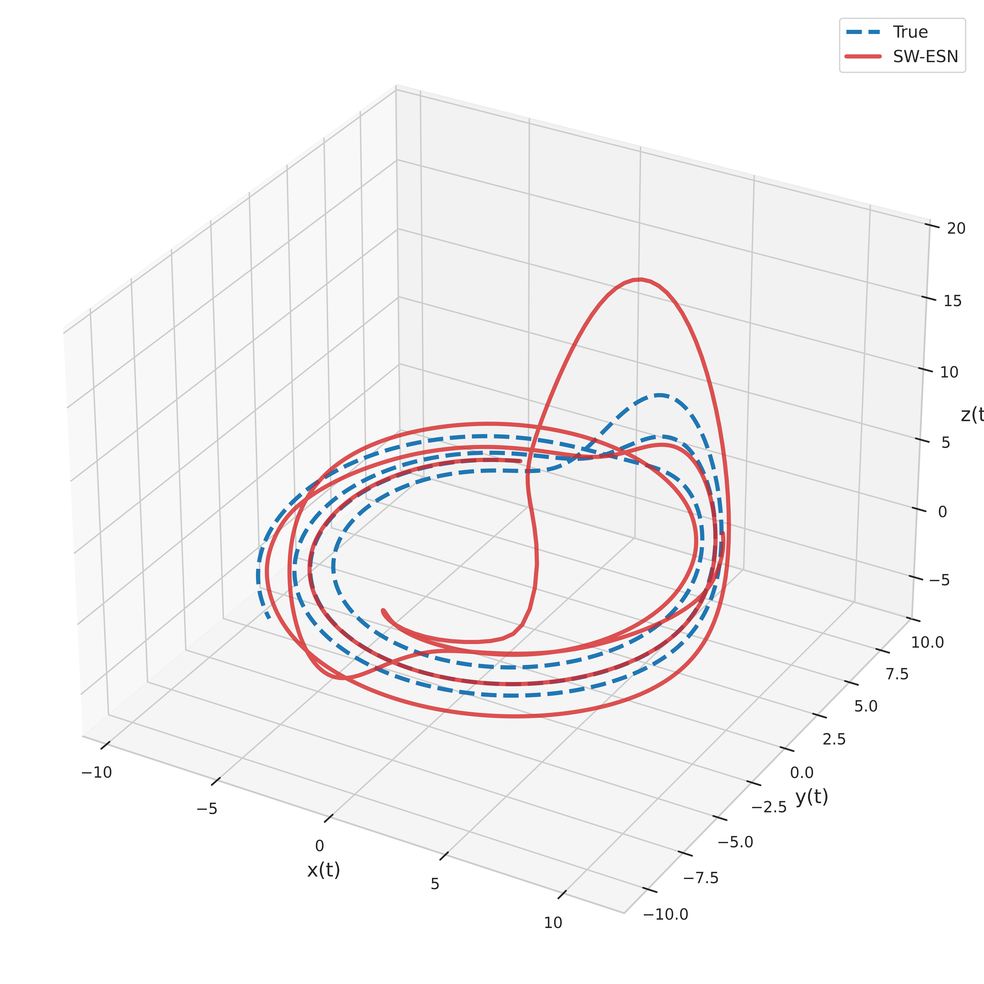}
        \caption{SW-ESN}
    \end{subfigure}

    \vspace{0.4cm}

    \begin{subfigure}[t]{0.23\textwidth}
        \includegraphics[width=\linewidth]{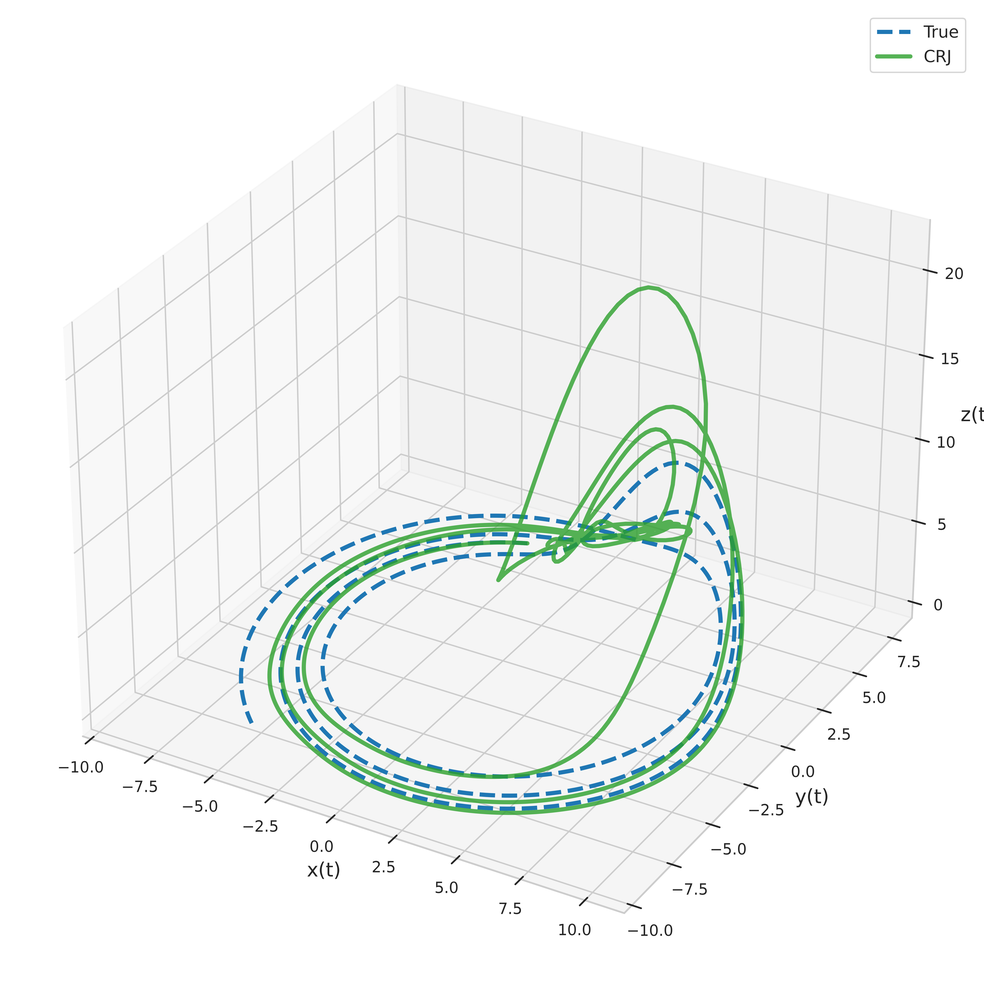}
        \caption{CRJ}
    \end{subfigure}
    \begin{subfigure}[t]{0.23\textwidth}
        \includegraphics[width=\linewidth]{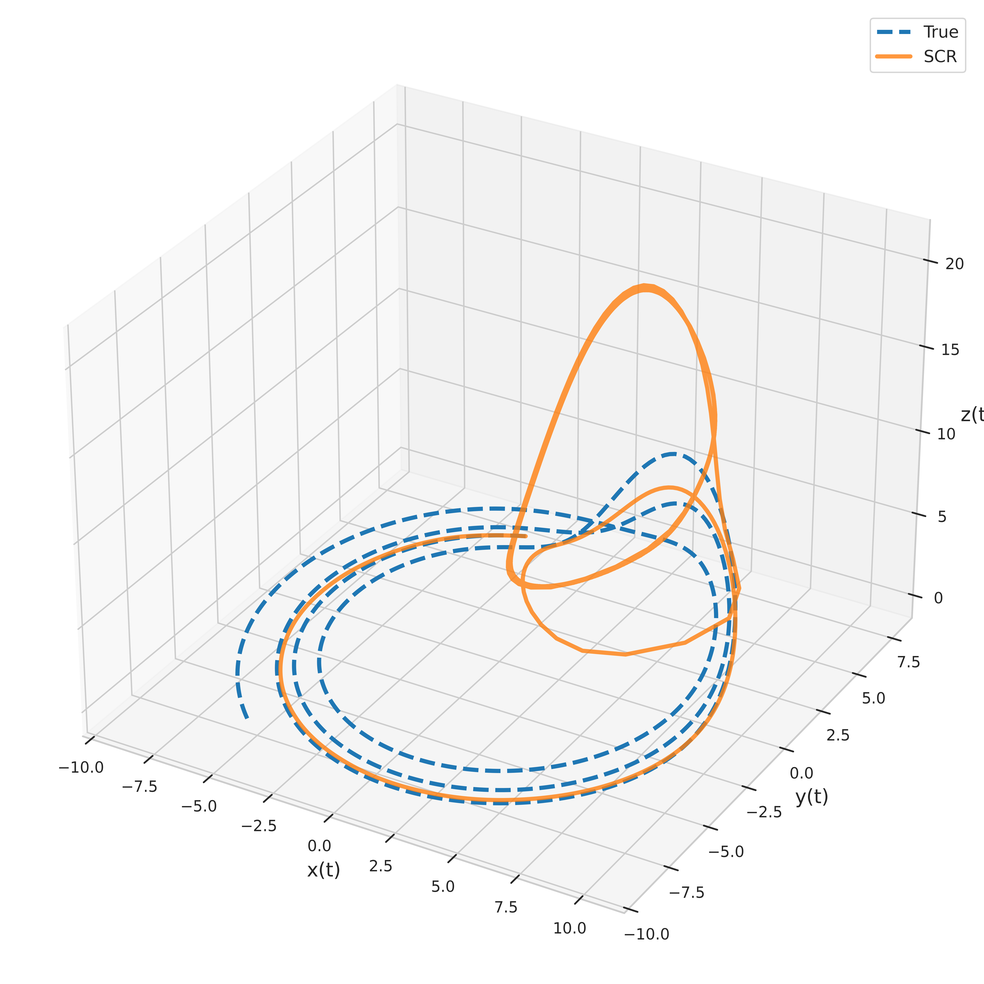}
        \caption{SCR}
    \end{subfigure}
    \begin{subfigure}[t]{0.23\textwidth}
        \includegraphics[width=\linewidth]{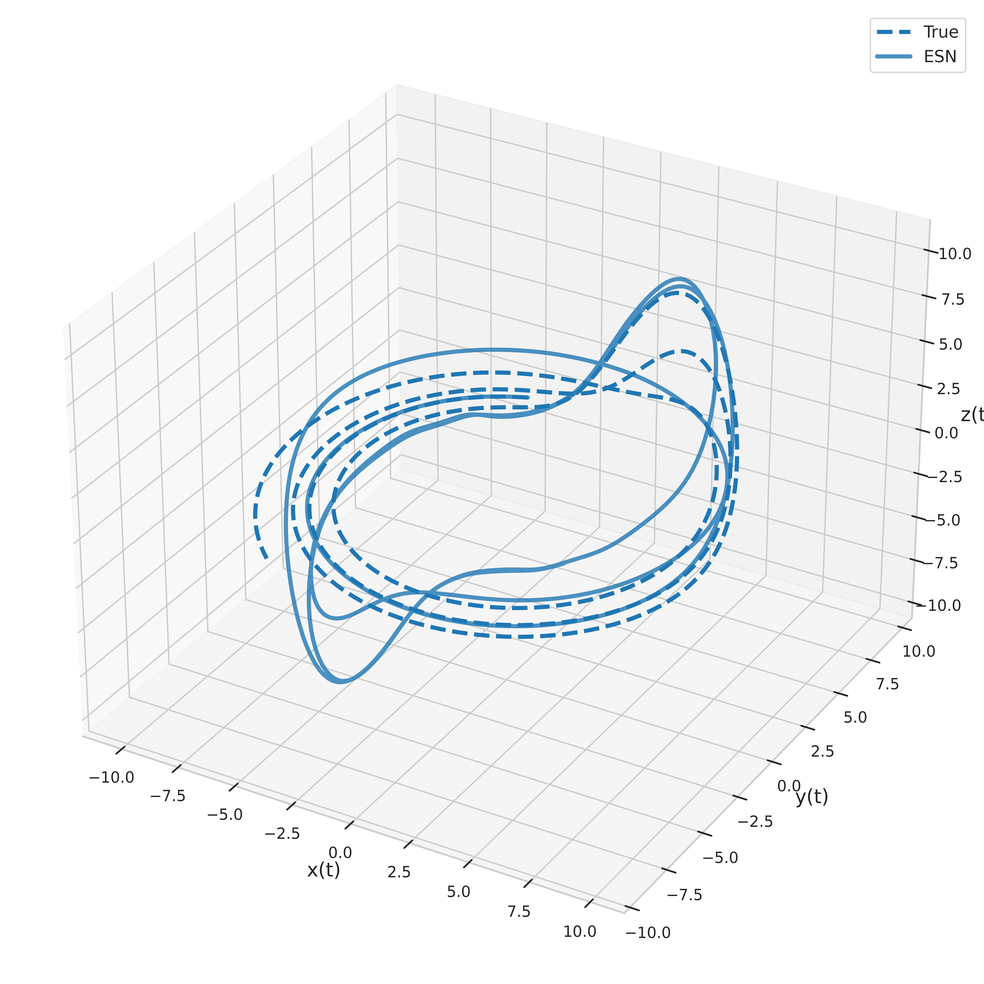}
        \caption{Vanilla ESN}
    \end{subfigure}

    \caption{3D Phase portraits for Rössler system predicted by different reservoir architectures.}
    \label{fig:rossler-phase_cl}
\end{figure}

\begin{figure}[htbp]
    \centering

    \begin{subfigure}[t]{0.23\textwidth}
        \includegraphics[width=\linewidth]{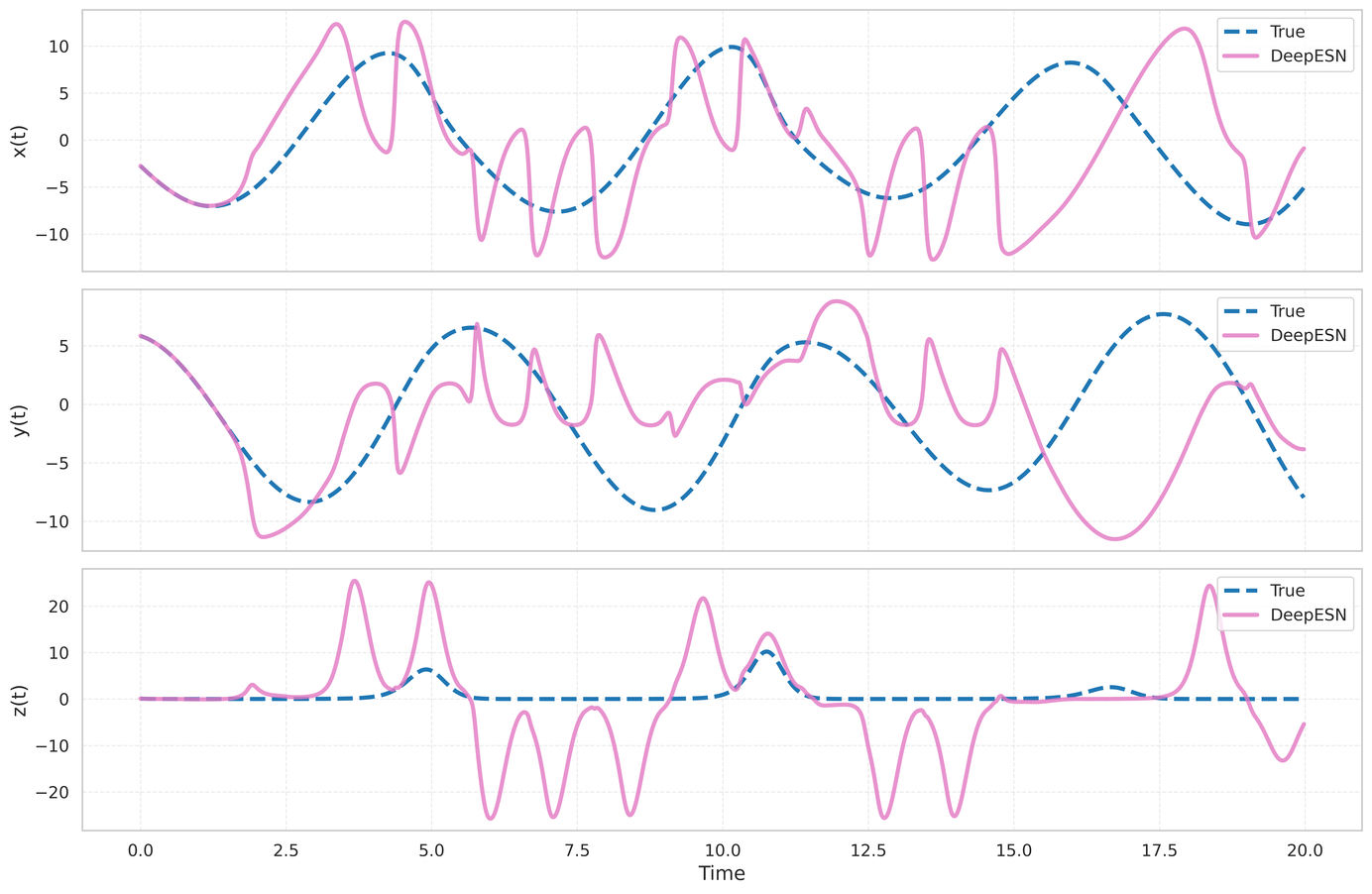}
        \caption{DeepESN}
    \end{subfigure}
    \begin{subfigure}[t]{0.23\textwidth}
        \includegraphics[width=\linewidth]{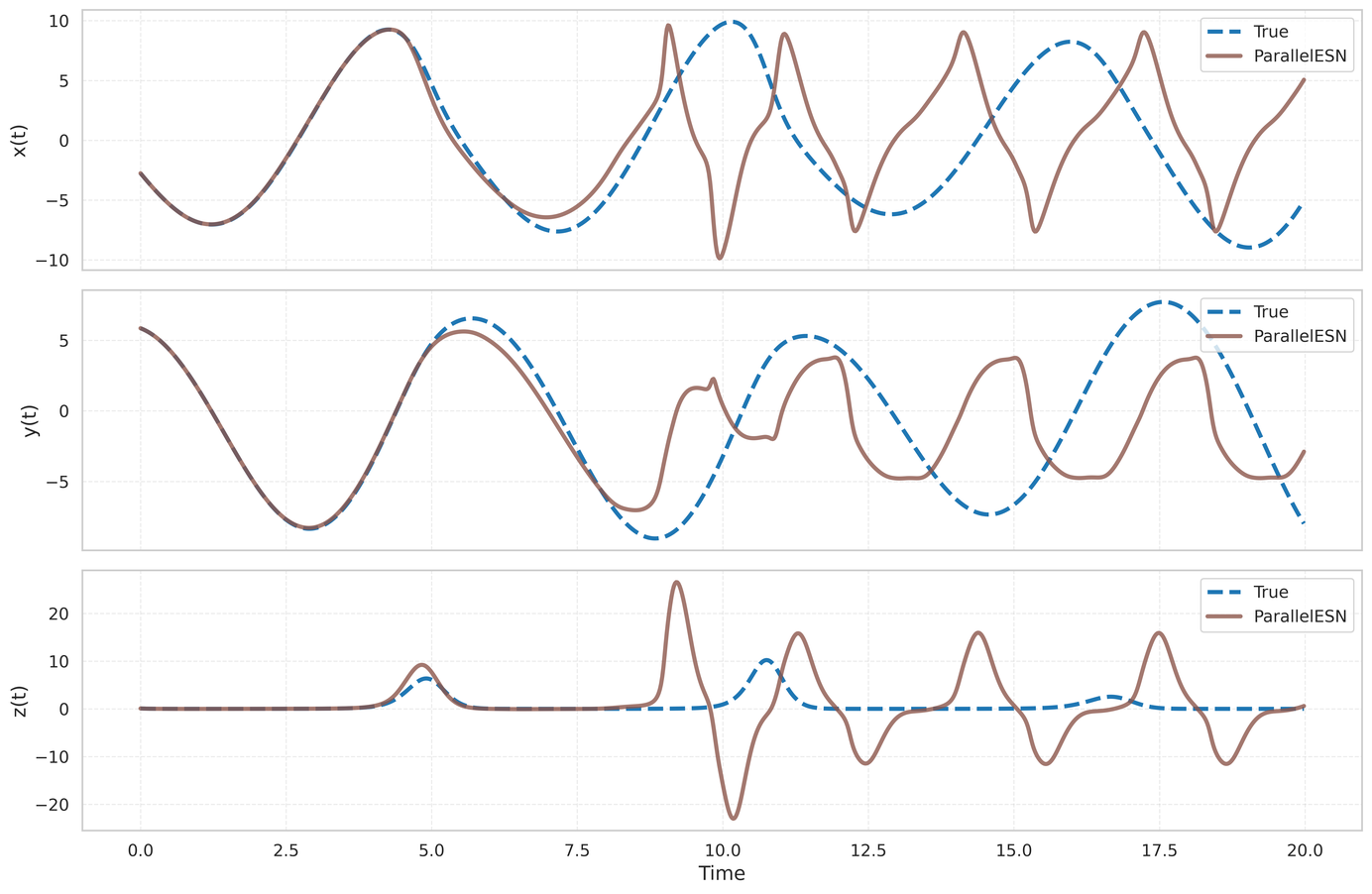}
        \caption{Parallel ESN}
    \end{subfigure}
    \begin{subfigure}[t]{0.23\textwidth}
        \includegraphics[width=\linewidth]{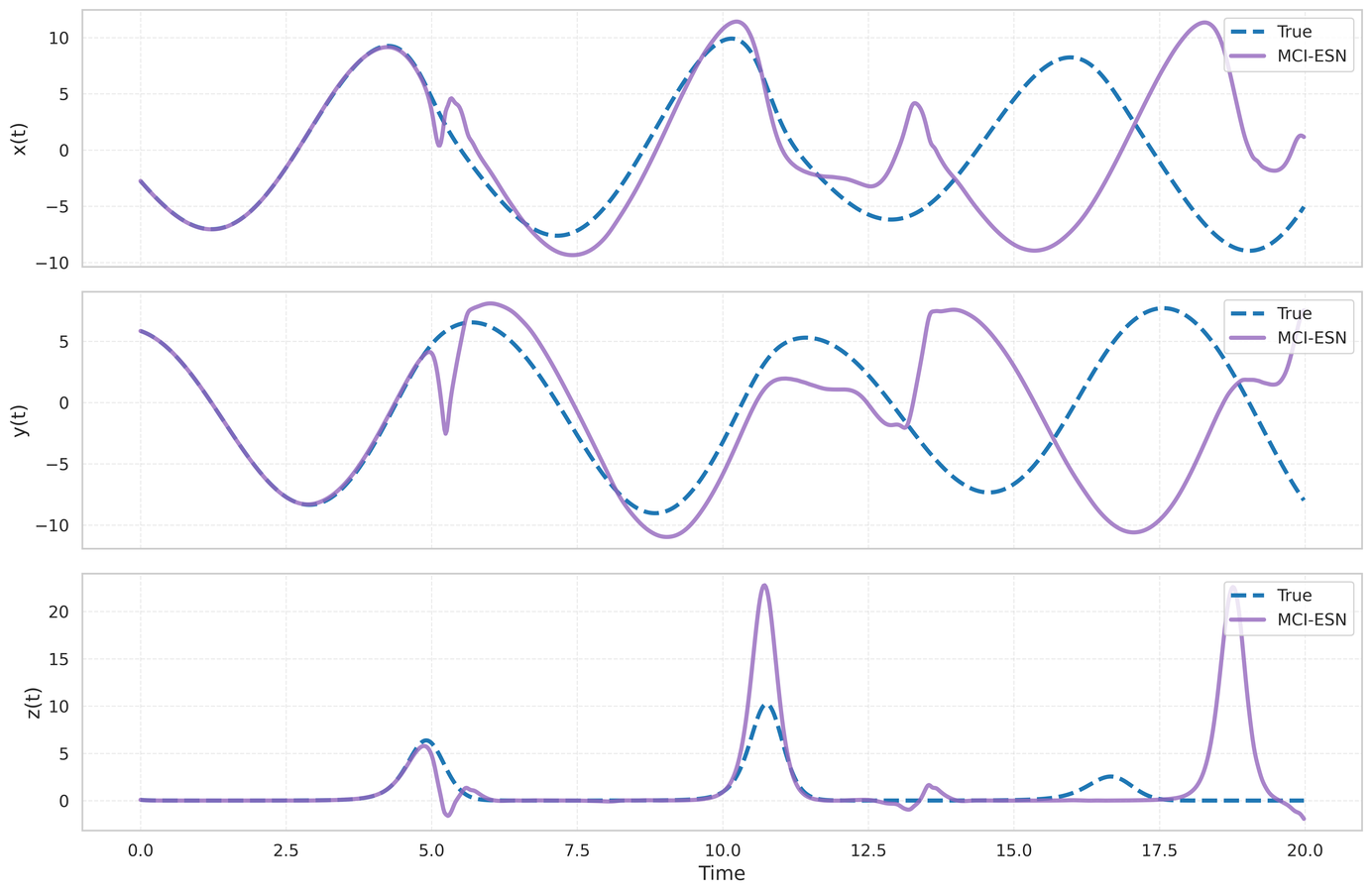}
        \caption{MCI-ESN}
    \end{subfigure}
    \begin{subfigure}[t]{0.23\textwidth}
        \includegraphics[width=\linewidth]{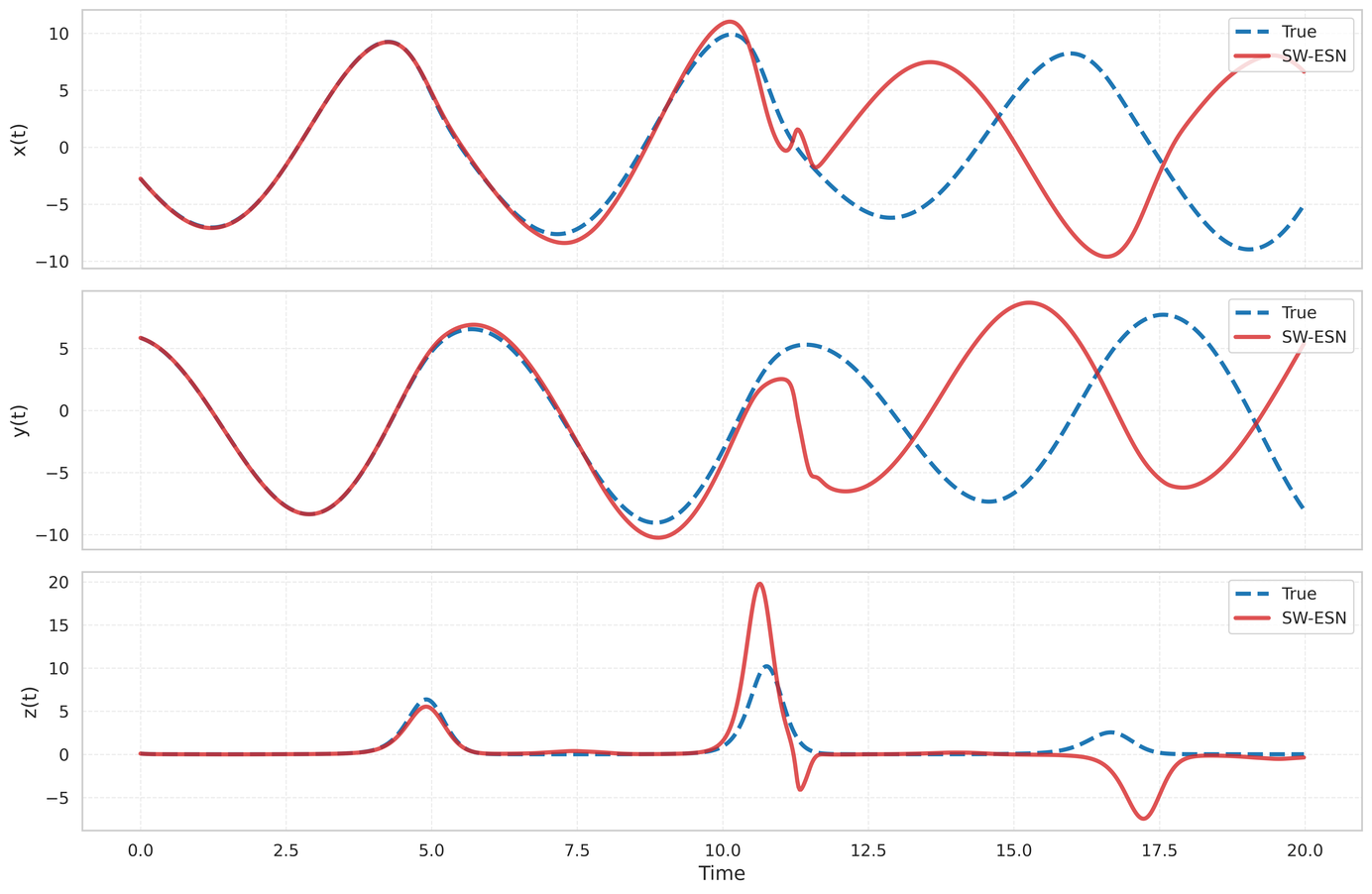}
        \caption{SW-ESN}
    \end{subfigure}

    \vspace{0.4cm}

    \begin{subfigure}[t]{0.23\textwidth}
        \includegraphics[width=\linewidth]{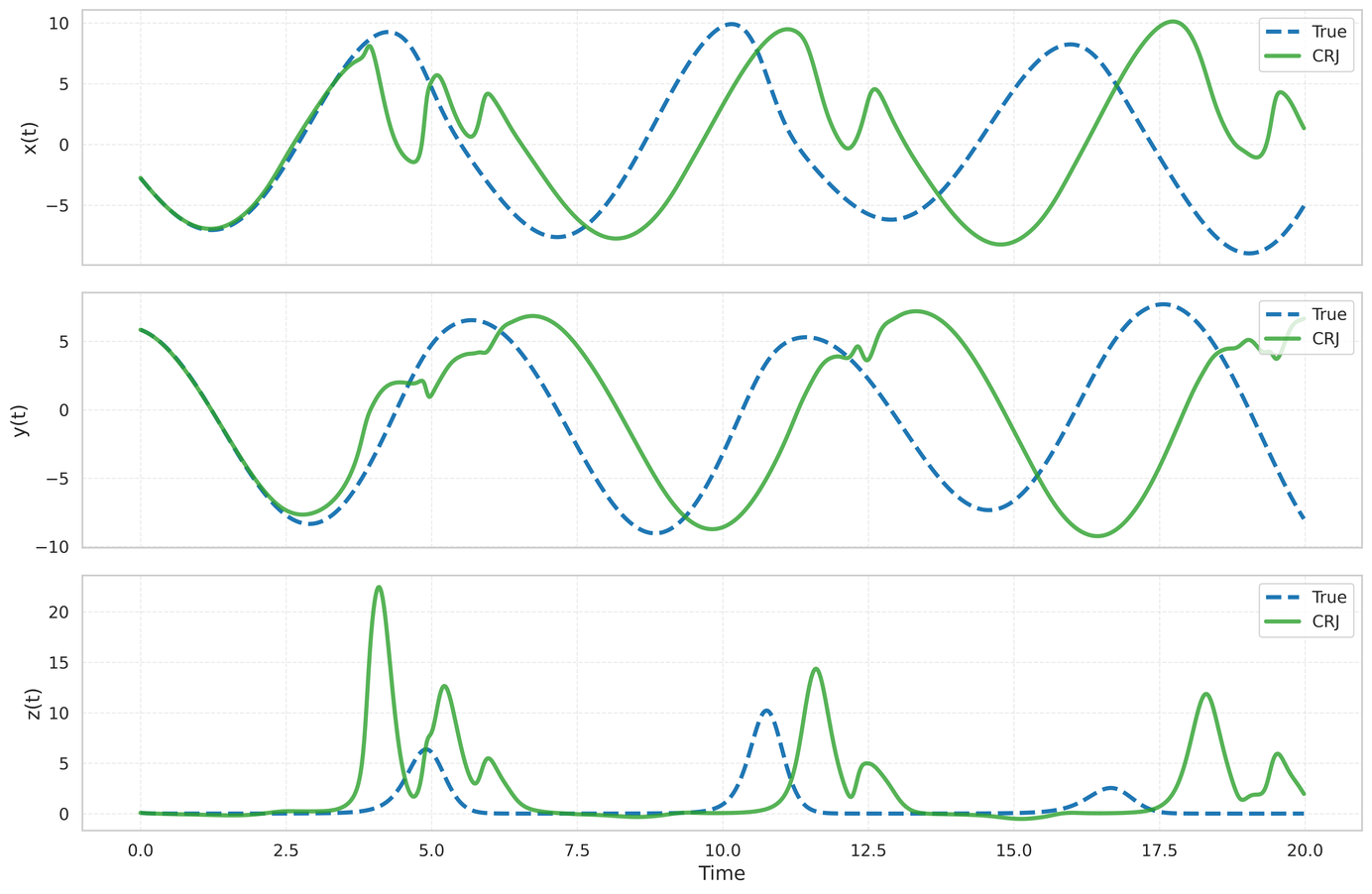}
        \caption{CRJ}
    \end{subfigure}
    \begin{subfigure}[t]{0.23\textwidth}
        \includegraphics[width=\linewidth]{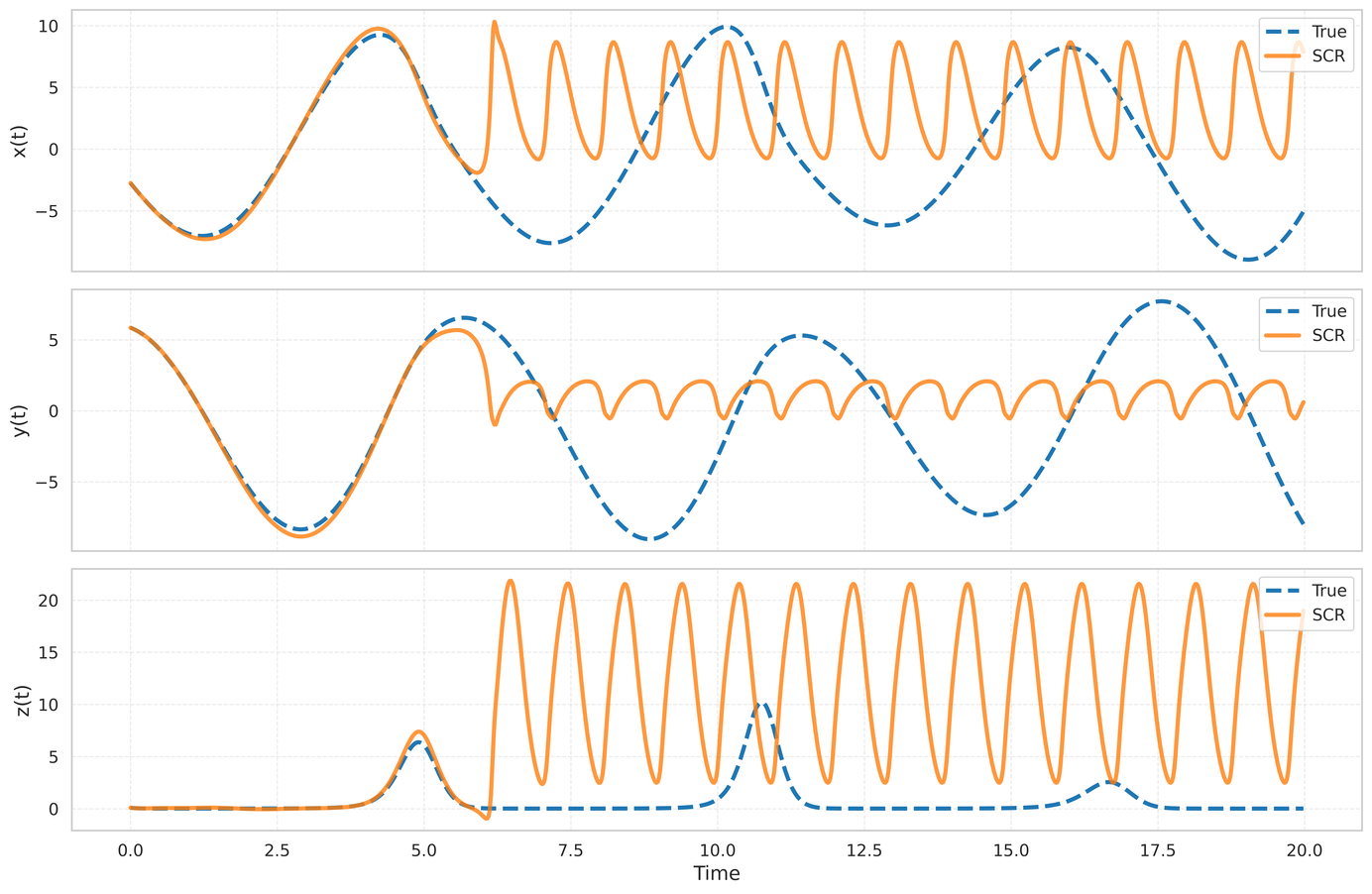}
        \caption{SCR}
    \end{subfigure}
    \begin{subfigure}[t]{0.23\textwidth}
        \includegraphics[width=\linewidth]{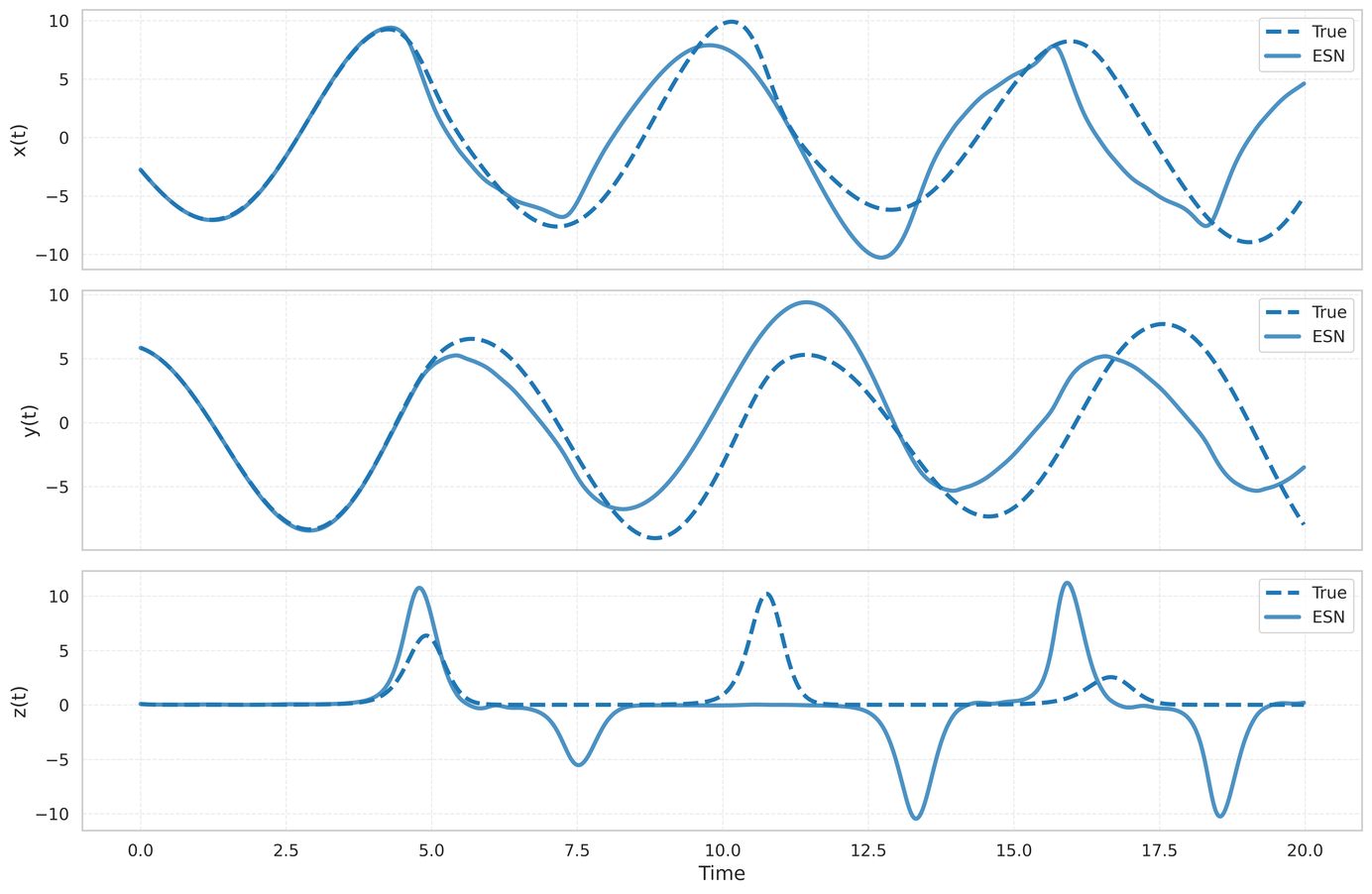}
        \caption{Vanilla ESN}
    \end{subfigure}

    \caption{Comparison of true and predicted trajectories for the Rössler system across various reservoir models.}
    \label{fig:rossler-pred_cl}
\end{figure}

\begin{figure}[htbp]
    \centering

    \begin{subfigure}[t]{0.23\textwidth}
        \includegraphics[width=\linewidth]{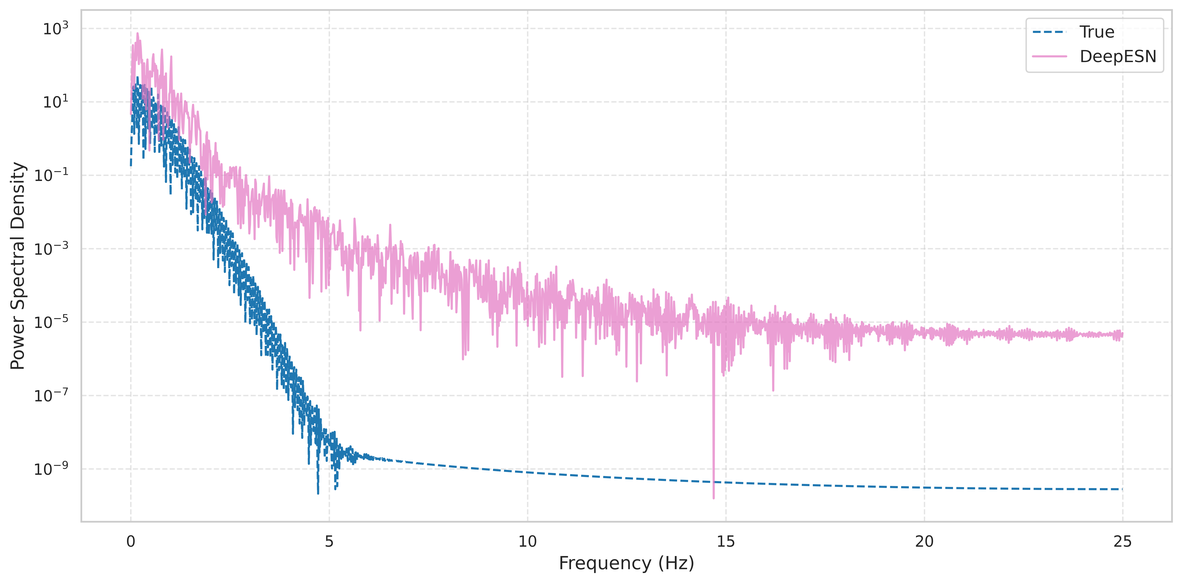}
        \caption{DeepESN}
    \end{subfigure}
    \begin{subfigure}[t]{0.23\textwidth}
        \includegraphics[width=\linewidth]{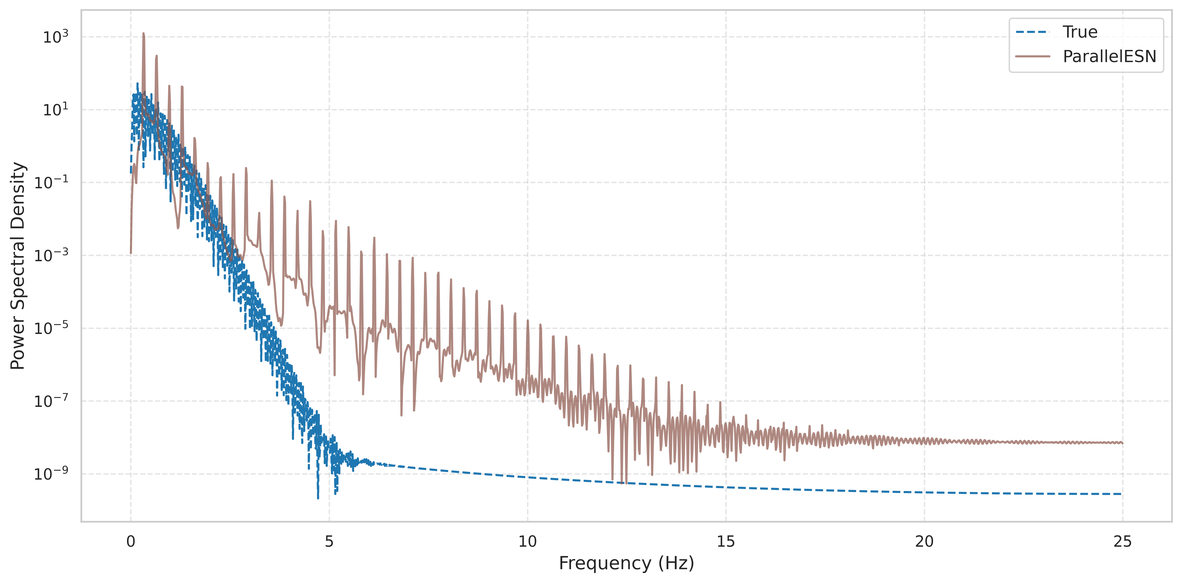}
        \caption{Parallel ESN}
    \end{subfigure}
    \begin{subfigure}[t]{0.23\textwidth}
        \includegraphics[width=\linewidth]{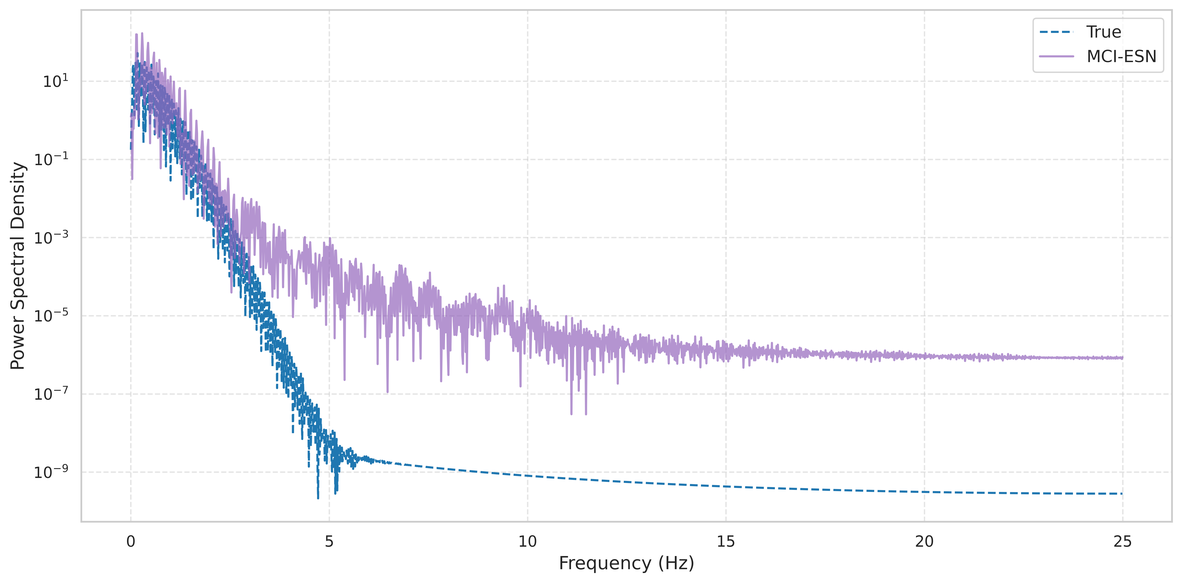}
        \caption{MCI-ESN}
    \end{subfigure}
    \begin{subfigure}[t]{0.23\textwidth}
        \includegraphics[width=\linewidth]{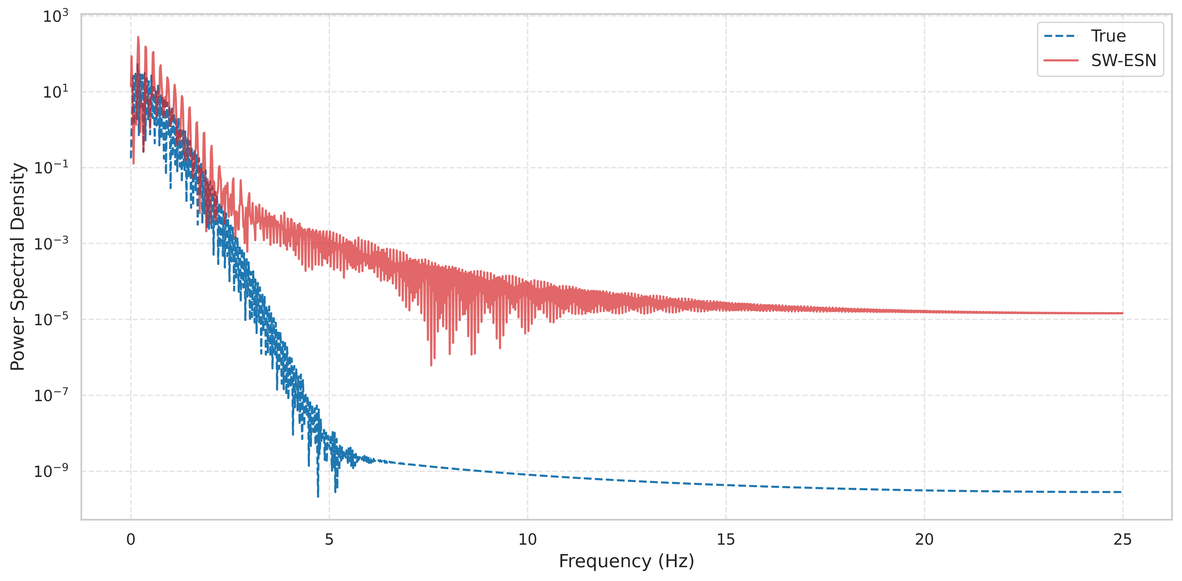}
        \caption{SW-ESN}
    \end{subfigure}

    \vspace{0.4cm}

    \begin{subfigure}[t]{0.23\textwidth}
        \includegraphics[width=\linewidth]{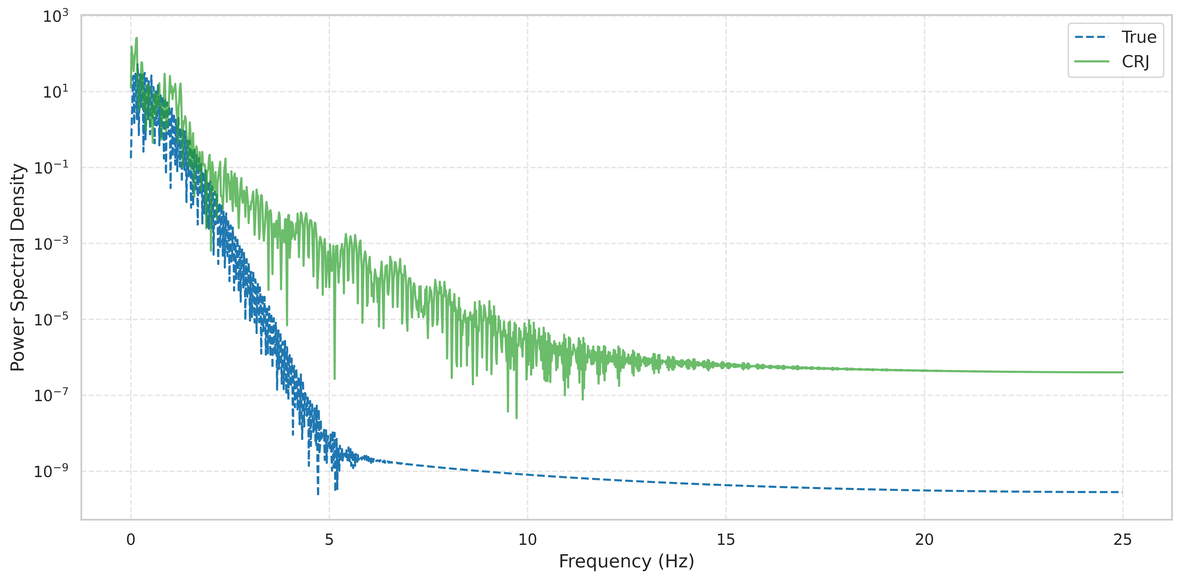}
        \caption{CRJ}
    \end{subfigure}
    \begin{subfigure}[t]{0.23\textwidth}
        \includegraphics[width=\linewidth]{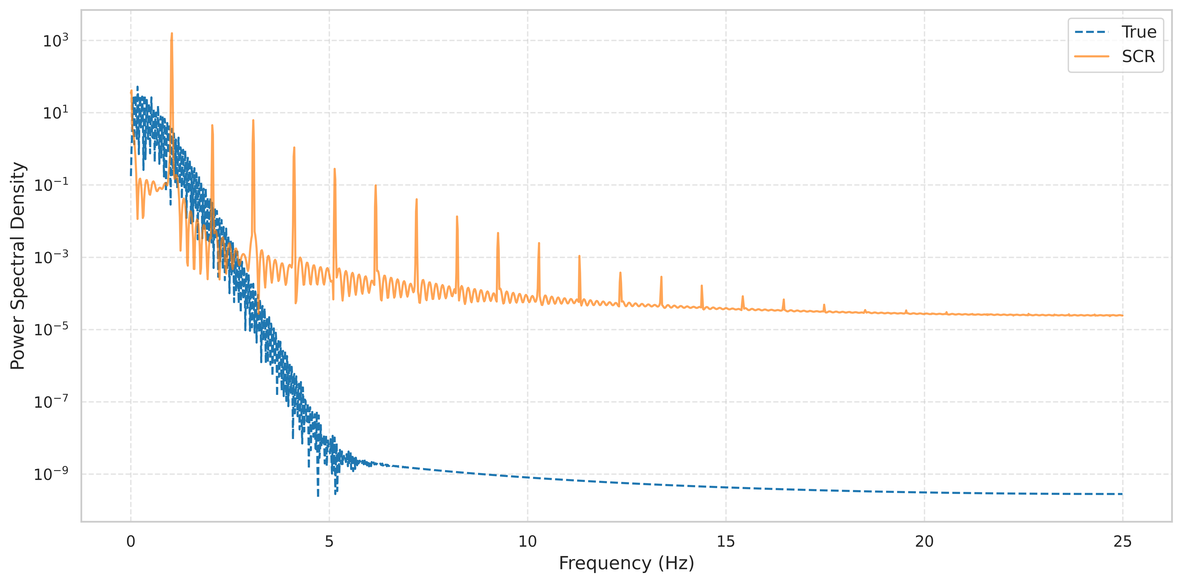}
        \caption{SCR}
    \end{subfigure}
    \begin{subfigure}[t]{0.23\textwidth}
        \includegraphics[width=\linewidth]{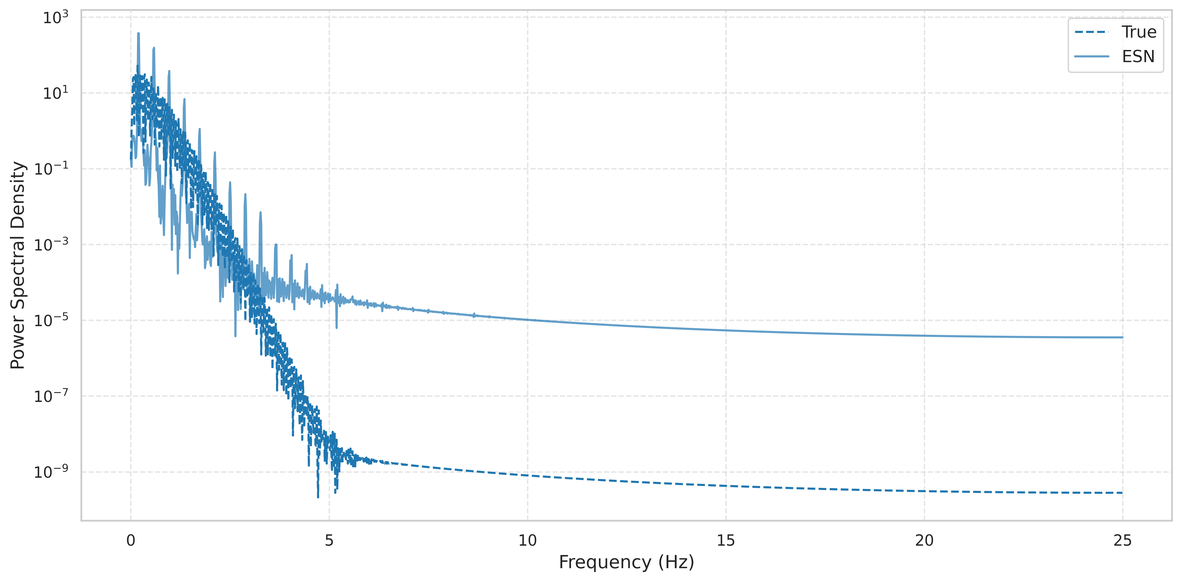}
        \caption{Vanilla ESN}
    \end{subfigure}

    \caption{Analysis of the $z$-component PSD of the Rössler system with different reservoir architectures.}
    \label{fig:rossler-psd_cl}
\end{figure}

\begin{table}[!ht]
\centering
\caption{Maximum Lyapunov Exponent \((\lambda_{\max})\) and Lyapunov Time for each dynamical system.}
\label{tab:LLE_LT_datasets}
\resizebox{0.5\textwidth}{!}{%
\begin{tabular}{l c c l}
\toprule
\textbf{Dataset} & \(\lambda_{\max}\) & \textbf{Lyapunov Time} \\
\midrule 
Mackey-Glass \cite{farmer1982chaotic}  & 0.006100 & 163.934426 \\
Lorenz      \cite{wolf1985determining}  & 0.905600 & 1.104972   \\
R\"{o}ssler \cite{sano1985measurement}  & 0.071400 & 14.005602  \\
Chen        \cite{chen1999yet}  & 0.829600 & 1.205148   \\
Chua        \cite{bilotta2008gallery}  & 0.428400 & 2.334267   \\
\bottomrule
\end{tabular}
}
\end{table}


Table \ref{tab:hyperparams_exp} presents the specific set of hyperparameters utilized for each reservoir model during the experimental evaluations, conducted in both open-loop and closed-loop configurations. In addition to the reservoir-specific settings listed in the table, all variants share a common set of core hyperparameters: the spectral radius $\rho$ is fixed at $0.95$, the leaking rate $\alpha$ is set to $0.8$, the input scaling factor is maintained at $0.2$, and the regularization coefficient $\beta$ is chosen as $10^{-4}$. These values were selected based on preliminary tuning to ensure stability and consistent performance across different reservoir architectures.

\begin{table}[!ht]
\centering
\caption{Experimental configurations for each reservoir variant.}
\label{tab:hyperparams_exp}
\begin{tabular}{lccccccc}
\toprule
\textbf{Model}& \textbf{Hyperparameter}& \textbf{Value}\\
\midrule
\multirow{2}{*}{Vanilla ESN}& Reservoir Size& 300\\
 & Connectivity Ratio&0.05\\
\midrule
 \multirow{2}{*}{SCR}& Reservoir Size&300\\
 & Edge Weight&0.8\\
 \midrule
 \multirow{3}{*}{CRJ}& Reservoir Size&300\\
 & Edge Weight&0.8\\
 & Jump Size&15\\
 \midrule
 \multirow{3}{*}{SW-ESN}& Reservoir Size&300\\
 & Node Degree $E$&6\\
 & Rewiring Probability $p$&0.1\\
 \midrule
 \multirow{4}{*}{MCI-ESN}& Sub-reservoir Size&300\\
 & Edge Weight $\mu$&0.8\\
 &  Inter-reservoir Connection Weight $\eta$&0.8\\
 & Weight Coefficient $\theta$&0.5\\
 \midrule
 \multirow{2}{*}{Parallel Res}
 & Number of Reservoirs &3\\
 & Reservoir Sizes & 300, 300, 300\\
 \midrule
 \multirow{2}{*}{DeepESN}
 & Number of Layers & 3\\
 & Reservoir Sizes&300, 300, 300\\
 \bottomrule
\end{tabular}
\end{table}

We begin our analysis by evaluating the performance of various reservoirs under open-loop forecasting (cf. \S~\ref{sec:open_loop}), followed by closed-loop forecasting (cf. \S~\ref{sec:closed_loop}). Finally, in \S~\ref{sec:ablation}, we conduct a series of ablation studies to understand the contribution of different architectural components and hyperparameters to the overall performance of the model.

\subsubsection{Performance in Open-Loop Forecasting}\label{sec:open_loop}
To rigorously evaluate model performance within this forecasting framework, we assess each variant of the reservoir model across a suite of benchmark dynamical systems. These evaluations are conducted using a comprehensive set of quantitative metrics, as outlined in (cf. \S~\ref{sec:metrics}), to ensure a well-rounded analysis of each model's capabilities. 

In particular, we focus on NRMSE as a primary performance metric, reporting results across multiple prediction horizons. This evaluation strategy enables us to systematically investigate how model accuracy evolves as the forecasting window extends further into the future. NRMSE trends give us a clear picture of how predictive accuracy degrades over time and how different reservoir structures cope with increasing temporal distance from the input. Specifically, we assess each model's predictive performance at fixed forecasting horizons of $10, 100, 500,$ and $1000$ time steps. These intervals are chosen to capture both short-term dynamics (e.g., horizon 10) and long-term behavior (e.g., horizon 1000), thereby offering insights into the immediate responsiveness as well as the temporal generalization capabilities of the models. A summary of these results is presented in Table \ref{tab:nrmse_horizon_ol}, providing a comparative view of how each reservoir structure handles the increasing challenge of open-loop prediction at longer forecasting intervals.

\begin{table}[!ht]
\centering
\caption{NRMSE $\downarrow$ under open-loop forecasting (Teacher-Forced) over multiple horizons.}
\label{tab:nrmse_horizon_ol}
\resizebox{\textwidth}{!}{%
\begin{tabular}{l l c c c c l}
\toprule
\multirow{2}{*}{\textbf{Dataset}} & \multirow{2}{*}{\textbf{Reservoir}} & \multicolumn{4}{c}{\textbf{Horizon}} \\
\cmidrule(lr){3-6} 
 & & 10-step & 100-step & 500-step & 1000-step & \\
\midrule

\multirow{7}{*}{Lorenz} 
  & Vanilla ESN     & 0.002437$\pm$0.001046 & 0.001036$\pm$0.000266 & 0.000893$\pm$0.000154 & 0.000736$\pm$0.000140 \\
  & SCR             & 0.001887$\pm$0.000556 & 0.000696$\pm$0.000149 & 0.000733$\pm$0.000084 & 0.000625$\pm$0.000059 \\
  & CRJ             & 0.002064$\pm$0.000495 & 0.000807$\pm$0.000083 & 0.000838$\pm$0.000072 & 0.000669$\pm$0.000051 \\
  & SWRes           & 0.002133$\pm$0.000321 & 0.000912$\pm$0.000311 & 0.000759$\pm$0.000185 & 0.000628$\pm$0.000131 \\
  & MCI-ESN         & 0.069455$\pm$0.014249 & 0.026075$\pm$0.004153 & 0.018802$\pm$0.001715 & 0.016333$\pm$0.001308 \\
  & Parallel ESN    & 0.000463$\pm$0.000157 & 0.000165$\pm$0.000036 & 0.000211$\pm$0.000050 & 0.000161$\pm$0.000033 \\
  & DeepESN         & 0.000391$\pm$0.000061 & 0.000167$\pm$0.000021 & 0.000270$\pm$0.000037 & 0.000204$\pm$0.000021 \\

\midrule

\multirow{7}{*}{R\"{o}ssler} 
  & Vanilla ESN     & 0.009348$\pm$0.002206 & 0.008922$\pm$0.000445 & 0.000220$\pm$0.000024 & 0.000522$\pm$0.000093 \\
  & SCR             & 0.011283$\pm$0.002984 & 0.012705$\pm$0.002132 & 0.000278$\pm$0.000028 & 0.000674$\pm$0.000131 \\
  & CRJ             & 0.007034$\pm$0.001826 & 0.008010$\pm$0.000734 & 0.000187$\pm$0.000010 & 0.000394$\pm$0.000089 \\
  & SWRes           & 0.007385$\pm$0.002695 & 0.008504$\pm$0.001891 & 0.000214$\pm$0.000032 & 0.000437$\pm$0.000090 \\
  & MCI-ESN         & 0.033053$\pm$0.010549 & 0.028297$\pm$0.001585 & 0.000877$\pm$0.000048 & 0.002748$\pm$0.000607 \\
  & Parallel ESN    & 0.001718$\pm$0.000664 & 0.001895$\pm$0.000207 & 0.000057$\pm$0.000007 & 0.000162$\pm$0.000022 \\
  & DeepESN         & 0.004170$\pm$0.000317 & 0.003282$\pm$0.000324 & 0.000169$\pm$0.000014 & 0.000956$\pm$0.000179 \\

\midrule

\multirow{7}{*}{Chen} 
  & Vanilla ESN     & 0.004935$\pm$0.000914 & 0.006758$\pm$0.001511 & 0.011530$\pm$0.000344 & 0.017099$\pm$0.001885 \\
  & SCR             & 0.002600$\pm$0.000547 & 0.005036$\pm$0.000850 & 0.010850$\pm$0.002102 & 0.017532$\pm$0.002441 \\
  & CRJ             & 0.003271$\pm$0.000671 & 0.003986$\pm$0.000886 & 0.009799$\pm$0.002753 & 0.014434$\pm$0.004866 \\
  & SWRes           & 0.003494$\pm$0.001171 & 0.005001$\pm$0.000840 & 0.009873$\pm$0.001941 & 0.014741$\pm$0.002460 \\
  & MCI-ESN         & 0.046200$\pm$0.006837 & 0.064432$\pm$0.007459 & 0.096352$\pm$0.008595 & 0.129847$\pm$0.013815 \\
  & Parallel ESN    & 0.000372$\pm$0.000111 & 0.001490$\pm$0.000316 & 0.004206$\pm$0.001221 & 0.006440$\pm$0.000906 \\
  & DeepESN         & 0.000604$\pm$0.000177 & 0.001600$\pm$0.000240 & 0.006135$\pm$0.001514 & 0.010021$\pm$0.002114 \\

\midrule

\multirow{7}{*}{Chua} 
  & Vanilla ESN             & 0.002845$\pm$0.001656 & 0.001154$\pm$0.000056 & 0.001003$\pm$0.000045 & 0.000889$\pm$0.000043 \\
  & SCR             & 0.002627$\pm$0.000978 & 0.001357$\pm$0.000180 & 0.001126$\pm$0.000112 & 0.000969$\pm$0.000081 \\
  & CRJ             & 0.004620$\pm$0.001728 & 0.001215$\pm$0.000031 & 0.001065$\pm$0.000025 & 0.000947$\pm$0.000021 \\
  & SW-ESN          & 0.002347$\pm$0.000527 & 0.001173$\pm$0.000022 & 0.001048$\pm$0.000020 & 0.000942$\pm$0.000022 \\
  & MCI-ESN         & 0.006204$\pm$0.001590 & 0.001593$\pm$0.000158 & 0.001351$\pm$0.000106 & 0.001173$\pm$0.000083 \\
  & Parallel ESN    & 0.001799$\pm$0.000687 & 0.000898$\pm$0.000029 & 0.000784$\pm$0.000020 & 0.000694$\pm$0.000015 \\
  & DeepESN         & 0.014859$\pm$0.005930 & 0.005713$\pm$0.004236 & 0.003605$\pm$0.002518 & 0.002839$\pm$0.001957 \\

\bottomrule
\end{tabular}
}
\end{table}
Across all datasets and horizons, Parallel ESN and DeepESN consistently outperform other architectures, achieving the lowest NRMSE, especially at longer horizons, demonstrating superior temporal generalization and robustness to compounding prediction errors. In contrast, MCI-ESN shows relatively poor performance, suggesting limited long-range predictive capabilities in the open-loop setting. Notably, while simpler reservoirs like Vanilla ESN, SCR, and CRJ perform competitively at short horizons, their performance degrades more rapidly at longer horizons. This highlights the benefits of architectural innovations, such as parallel and deep structures in preserving predictive accuracy over extended timeframes.

\begin{figure}[!ht]
    \centering
    \includegraphics[width=\linewidth]{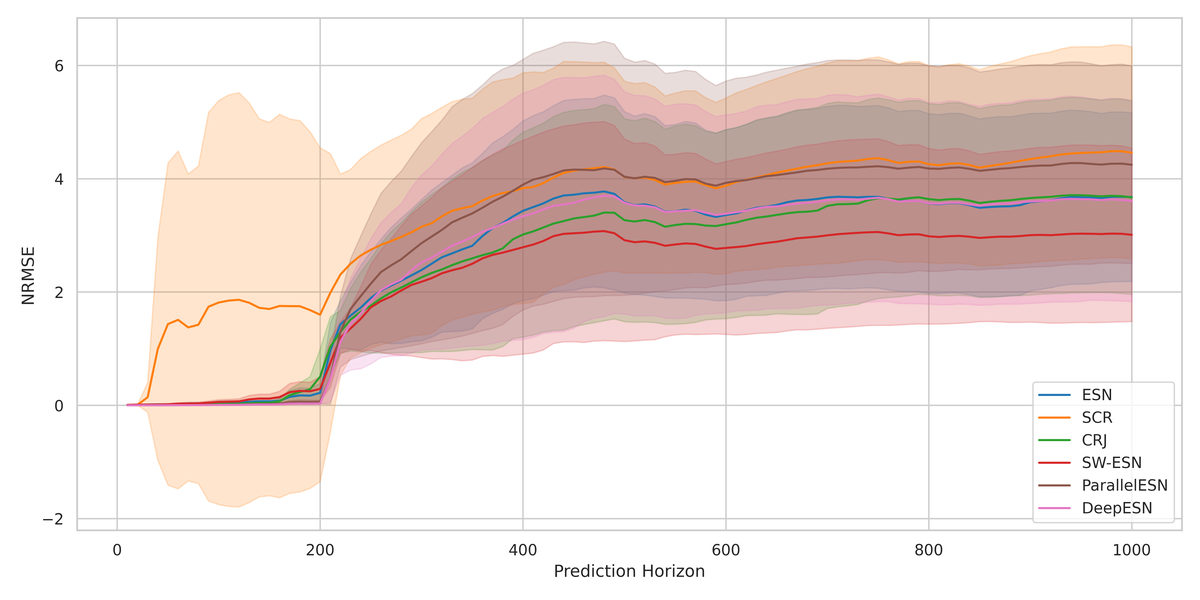}
    \caption{NRMSE of the Lorenz system across increasing prediction horizons under closed-loop settings.}
    \label{fig:Lorenz_NRMSE_cl}
\end{figure}

\begin{figure}[!ht]
    \centering
    \includegraphics[width=\linewidth]{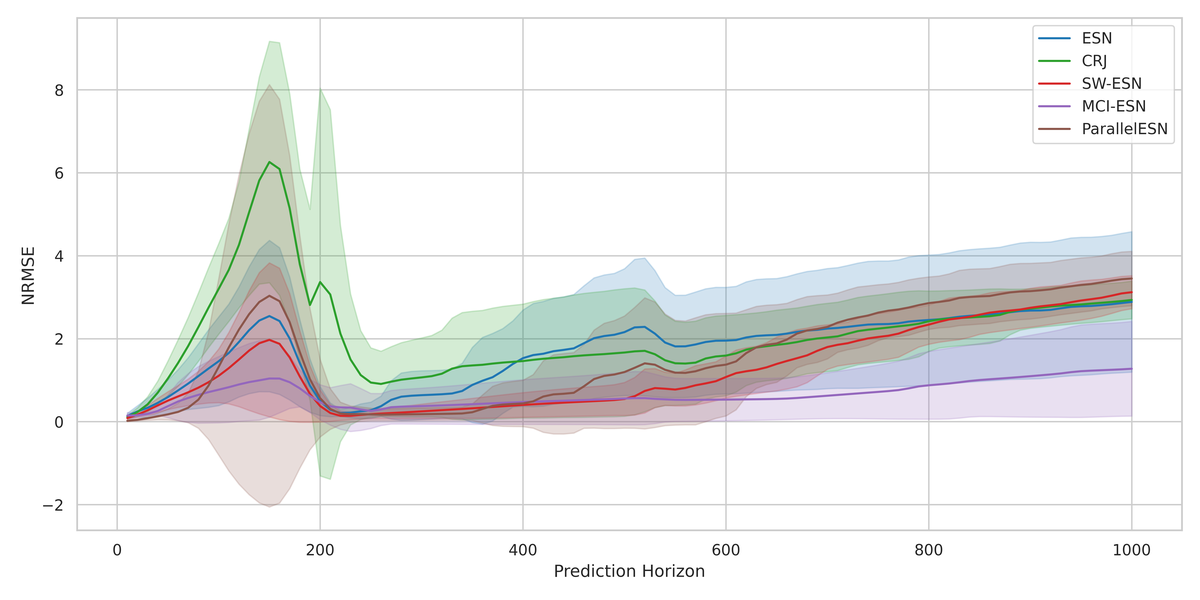}
    \caption{NRMSE of the Rössler system across increasing prediction horizons under closed-loop settings.}
    \label{fig:Rossler_NRMSE_cl}
\end{figure}

\begin{figure}[!ht]
    \centering
    \includegraphics[width=\linewidth]{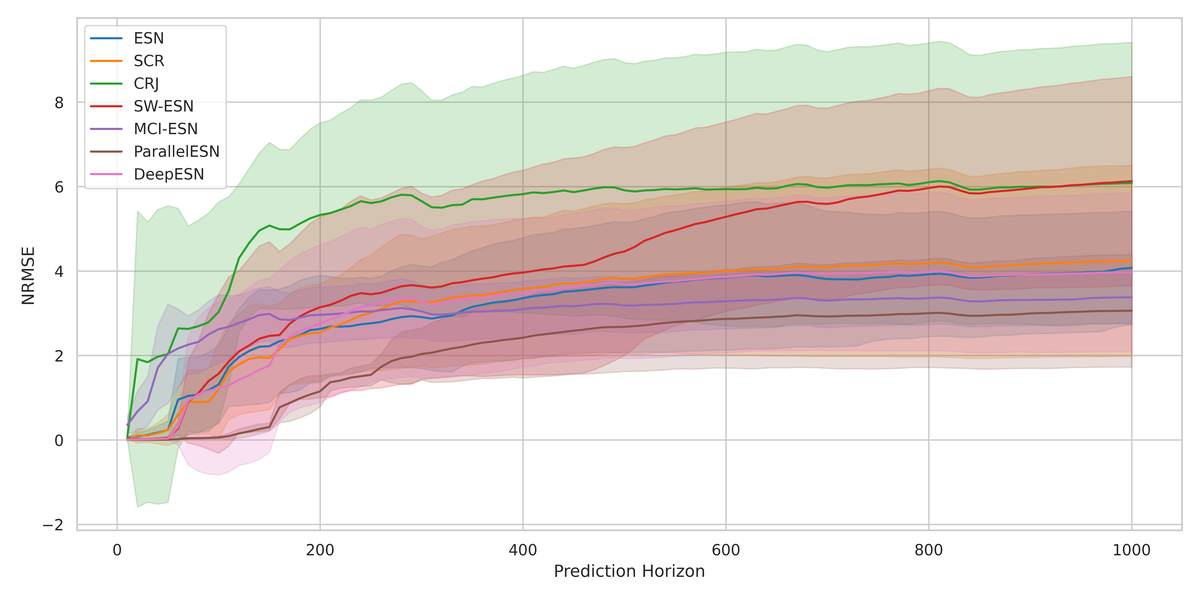}
    \caption{NRMSE of the Chen system across increasing prediction horizons under closed-loop settings.}
    \label{fig:Chen_NRMSE_cl}
\end{figure}

\begin{figure}[!ht]
    \centering
    \includegraphics[width=\linewidth]{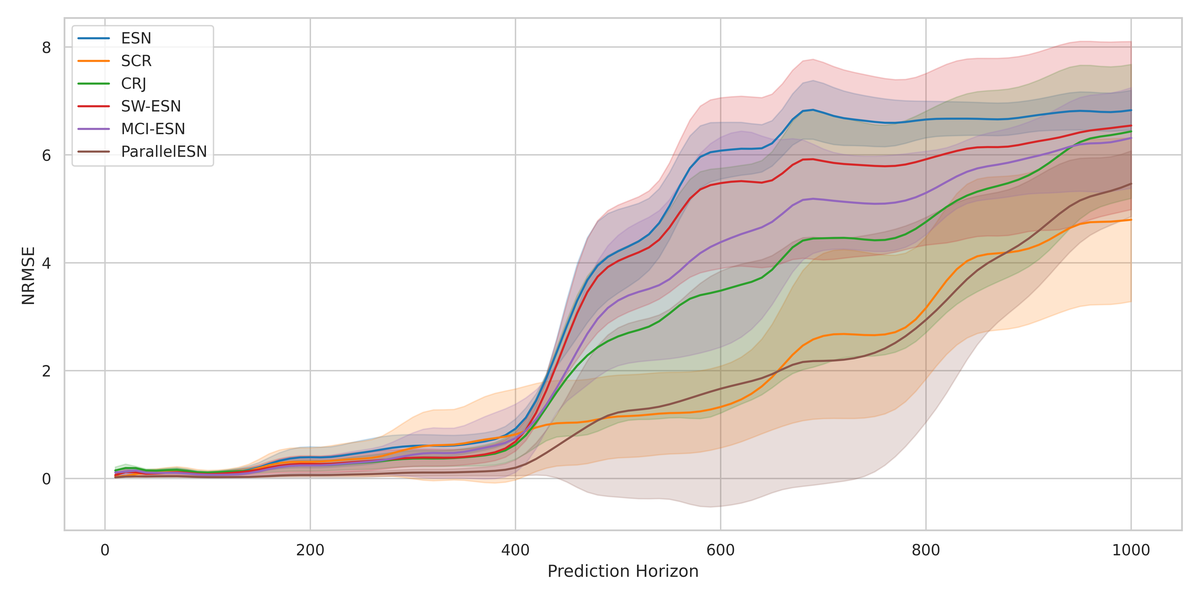}
    \caption{NRMSE of the Chua system across increasing prediction horizons under closed-loop settings.}
    \label{fig:Chua_NRMSE_cl}
\end{figure}

Table \ref{tab:lle_ol} presents the largest Lyapunov exponent $\lambda_{max}$ computed from the predicted trajectories of each model in the open-loop forecasting setting, along with the corresponding deviation $\Delta\lambda$ from the true system's Lyapunov exponent. By comparing $\lambda_{max}$ of the predicted trajectory with that of the ground truth, $\Delta\lambda$ quantifies how faithfully a model captures the intrinsic instability and chaotic nature of the underlying system. Smaller values of $\Delta\lambda$ indicate better preservation of the system’s long-term dynamical structure, whereas larger divergences suggest that the model fails to replicate the true chaotic behavior over extended horizons. This metric complements NRMSE-based evaluations by focusing on qualitative, structural fidelity in addition to pointwise accuracy.  Across all four chaotic systems, we observe that DeepESN and Parallel ESN consistently exhibit the least $\Delta\lambda$, indicating superior capability in preserving the long-term chaotic structure of the dynamics. Notably, Parallel ESN achieves near-perfect alignment with the true Lyapunov spectrum in all cases, often with the lowest divergence. In contrast, models like MCI-ESN, while showing modest to poor NRMSE performance at short horizons (as seen in Table \ref{tab:nrmse_horizon_ol}), also display higher Lyapunov divergence, particularly in the Lorenz and Chen systems. This confirms that their predictions diverge both numerically and structurally over time.


\begin{table}[!ht]
\centering
\caption{Lyapunov Exponent Divergence for 3000-step horizon under open-loop forecastings.}
\label{tab:lle_ol}
\begin{tabular}{l l c c}
\toprule
Dataset & Reservoir  & $\lambda_{max}$ & $\Delta\lambda$ \\
\midrule
\multirow{7}{*}{Lorenz} 
  & ESN             & 0.100327$\pm$0.000153 & 0.805273$\pm$0.000153 \\
  & SCR             & 0.100477$\pm$0.000175 & 0.805123$\pm$0.000175 \\
  & CRJ             & 0.100417$\pm$0.000206 & 0.805183$\pm$0.000206 \\
  & SW-ESN          & 0.100490$\pm$0.000145 & 0.805110$\pm$0.000145 \\
  & MCI-ESN         & 0.098996$\pm$0.000699 & 0.806604$\pm$0.000699 \\
  & Parallel ESN    & 0.100496$\pm$0.000087 & 0.805104$\pm$0.000087 \\
  & DeepESN         & 0.100474$\pm$0.000020 & 0.805126$\pm$0.000020 \\

\midrule

\multirow{7}{*}{R\"{o}ssler} 
  & Vanilla ESN             & 0.130915$\pm$0.001960 & 0.059515$\pm$0.001960 \\
  & SCR             & 0.133025$\pm$0.002089 & 0.061625$\pm$0.002089 \\
  & CRJ             & 0.131812$\pm$0.001965 & 0.060412$\pm$0.001965 \\
  & SW-ESN          & 0.130884$\pm$0.001999 & 0.059484$\pm$0.001999 \\
  & MCI-ESN         & 0.136856$\pm$0.003348 & 0.065456$\pm$0.003348 \\
  & Parallel ESN    & 0.129957$\pm$0.001497 & 0.058557$\pm$0.001497 \\
  & DeepESN         & 0.129414$\pm$0.002070 & 0.058014$\pm$0.002070 \\

\midrule

\multirow{7}{*}{Chen} 
  & Vanilla ESN             & 0.114143$\pm$0.000160 & 0.715457$\pm$0.000160 \\
  & SCR             & 0.114119$\pm$0.000171 & 0.715481$\pm$0.000171 \\
  & CRJ             & 0.113941$\pm$0.000099 & 0.715659$\pm$0.000099 \\
  & SW-ESN          & 0.114110$\pm$0.000168 & 0.715490$\pm$0.000168 \\
  & MCI-ESN         & 0.114347$\pm$0.001142 & 0.715253$\pm$0.001142 \\
  & Parallel ESN    & 0.114075$\pm$0.000077 & 0.715525$\pm$0.000077 \\
  & DeepESN         & 0.114123$\pm$0.000207 & 0.715477$\pm$0.000207 \\

\midrule
\multirow{7}{*}{Chua} 
  & Vanilla ESN             & 0.144518$\pm$0.000133 & 0.283882$\pm$0.000133 \\
  & SCR             & 0.144601$\pm$0.000232 & 0.283799$\pm$0.000232 \\
  & CRJ             & 0.144627$\pm$0.000166 & 0.283773$\pm$0.000166 \\
  & SW-ESN          & 0.144444$\pm$0.000144 & 0.283956$\pm$0.000144 \\
  & MCI-ESN         & 0.144540$\pm$0.000199 & 0.283860$\pm$0.000199 \\
  & Parallel ESN    & 0.144500$\pm$0.000107 & 0.283900$\pm$0.000107 \\
  & DeepESN         & 0.145040$\pm$0.000371 & 0.283360$\pm$0.000371 \\

\bottomrule

\end{tabular}
\end{table}

\subsubsection{Performance in Closed-Loop Forecasting}\label{sec:closed_loop}

Closed-loop forecasting, also known as \textit{regressive prediction}, presents a significantly more challenging evaluation scenario compared to the open-loop (teacher-forced) setting. In this regime, the model generates future predictions by recursively feeding its own previous outputs as inputs, rather than relying on ground truth data during inference (cf.  \S  \ref{sec:forecasting_setups}). This setup provides a rigorous test of the model's ability to autonomously sustain accurate and coherent dynamics over extended time horizons. As such, it serves as a more realistic and stringent measure of a model’s long-term generalization capabilities. 

Apart from the commonly used NRMSE and Lyapunov Divergence, we also incorporate two additional metrics— VPT and ADev to gain a more comprehensive understanding of each model's forecasting behavior. While NRMSE provides a general measure of point-wise prediction accuracy and Lyapunov Divergence quantifies the exponential error growth characteristic of chaotic systems, VPT and ADev offer complementary insights. The corresponding results are summarized in Tables~\ref{tab:nrmse_horizon_cl}--\ref{tab:vpt_ADev_cl}. Table~\ref{tab:nrmse_horizon_cl} shows NRMSE values, Table~\ref{tab:lle_cl} reports Lyapunov Divergence, and Table~\ref{tab:vpt_ADev_cl} presents VPT and ADev, together offering a well-rounded view of each model's forecasting performance.

\begin{table}[!ht]
\centering
\caption{NRMSE $\downarrow$ under closed-loop forecasting (Autoregressive) over multiple horizons.}
\label{tab:nrmse_horizon_cl}
\resizebox{\textwidth}{!}{%
\begin{tabular}{l l c c c c l}
\toprule
\multirow{2}{*}{\textbf{Dataset}} & \multirow{2}{*}{\textbf{Reservoir}} & \multicolumn{4}{c}{\textbf{Horizon}}   \\
\cmidrule(lr){3-6} 
 & & 10-step & 100-step & 500-step  & 1000-step\\
\midrule

\multirow{7}{*}{Lorenz} 
  & Vanilla ESN             & 0.007614$\pm$0.005111 & 0.047484$\pm$0.029881 & 3.584183$\pm$1.602574 & 3.677976$\pm$1.489045 \\
  & SCR             & 0.007028$\pm$0.003446 & 1.815099$\pm$3.566290 & 4.030702$\pm$1.691541 & 4.459001$\pm$1.870905 \\
  & CRJ             & 0.007260$\pm$0.003186 & 0.022971$\pm$0.012299 & 3.268853$\pm$1.791016 & 3.670972$\pm$1.708620 \\
  & SW-ESN          & 0.008793$\pm$0.003788 & 0.060875$\pm$0.039085 & 2.915492$\pm$1.788156 & 3.011428$\pm$1.532399 \\
  & MCI-ESN         & 1.230876$\pm$0.838207 & 5.617800$\pm$2.763230 & 4.298552$\pm$0.575811 & 3.943699$\pm$0.483962 \\
  & Parallel ESN    & 0.001887$\pm$0.000858 & 0.009435$\pm$0.004681 & 4.035640$\pm$2.092327 & 4.247591$\pm$1.743061 \\
  & DeepESN         & 0.001315$\pm$0.000578 & 0.008068$\pm$0.003934 & 3.573995$\pm$1.962448 & 3.606994$\pm$1.774597 \\

\midrule
\multirow{7}{*}{R\"{o}ssler} 
  & Vanilla ESN             & 0.139767$\pm$0.063956 & 1.457615$\pm$1.068153 & 2.160282$\pm$1.528565 & 2.888498$\pm$1.692990 \\
  & SCR             & 0.196739$\pm$0.106687 & 109.940464$\pm$135.397594 & 4.142599$\pm$2.144532 & 3.837662$\pm$1.509674 \\
  & CRJ             & 0.110214$\pm$0.036627 & 3.201257$\pm$1.091736 & 1.668403$\pm$1.537698 & 2.935533$\pm$0.453440 \\
  & SW-ESN          & 0.086966$\pm$0.038720 & 1.106773$\pm$0.646281 & 0.548803$\pm$0.397856 & 3.122546$\pm$0.397420 \\
  & MCI-ESN         & 0.148857$\pm$0.074668 & 0.772386$\pm$0.799236 & 0.559700$\pm$0.631791 & 1.277491$\pm$1.140839 \\
  & Parallel ESN    & 0.022148$\pm$0.013351 & 1.320780$\pm$2.124518 & 1.197415$\pm$1.352267 & 3.452498$\pm$0.655188 \\
  & DeepESN         & 0.046696$\pm$0.018655 & 15.828215$\pm$18.602838 & 2.418219$\pm$1.935758 & 3.775366$\pm$0.689312 \\

\midrule
\multirow{7}{*}{Chen} 
  & Vanilla ESN             & 0.048106$\pm$0.028900 & 1.316188$\pm$0.924288 & 3.615808$\pm$1.540769 & 4.074841$\pm$1.342420 \\
  & SCR             & 0.017162$\pm$0.005381 & 1.193524$\pm$1.146527 & 3.819821$\pm$1.775425 & 4.240608$\pm$2.261324 \\
  & CRJ             & 0.077772$\pm$0.075951 & 3.033906$\pm$2.605282 & 5.913305$\pm$3.024390 & 6.082467$\pm$3.334916 \\
  & SW-ESN          & 0.017158$\pm$0.008025 & 1.570141$\pm$1.885309 & 4.464775$\pm$2.462887 & 6.129073$\pm$2.477179 \\
  & MCI-ESN         & 0.358289$\pm$0.197905 & 2.625182$\pm$0.815653 & 3.184659$\pm$0.650357 & 3.378137$\pm$0.618295 \\
  & Parallel ESN    & 0.003498$\pm$0.002863 & 0.060526$\pm$0.031325 & 2.680420$\pm$1.057335 & 3.062002$\pm$1.334566 \\
  & DeepESN         & 0.007324$\pm$0.004353 & 1.224323$\pm$2.049013 & 3.682612$\pm$1.755248 & 3.972897$\pm$1.877817 \\

\midrule
\multirow{7}{*}{Chua} 
  & Vanilla ESN             & 0.059206$\pm$0.043820 & 0.082288$\pm$0.045194 & 4.214087$\pm$0.774263 & 6.829367$\pm$0.368721 \\
  & SCR             & 0.033995$\pm$0.020053 & 0.098248$\pm$0.054107 & 1.150105$\pm$0.766824 & 4.797387$\pm$1.516078 \\
  & CRJ             & 0.149077$\pm$0.064165 & 0.110253$\pm$0.023748 & 2.634248$\pm$1.599648 & 6.436059$\pm$1.241802 \\
  & SW-ESN          & 0.063730$\pm$0.044489 & 0.079487$\pm$0.034560 & 4.033542$\pm$1.037445 & 6.544182$\pm$1.558120 \\
  & MCI-ESN         & 0.113083$\pm$0.044609 & 0.068140$\pm$0.011962 & 3.300845$\pm$1.209103 & 6.314190$\pm$0.933128 \\
  & Parallel ESN    & 0.021298$\pm$0.001878 & 0.025304$\pm$0.009829 & 1.224578$\pm$1.569606 & 5.465975$\pm$0.608724 \\
  & DeepESN         & 0.065898$\pm$0.026901 & 3.233513$\pm$3.217908 & 7.121215$\pm$1.262932 & 8.471579$\pm$0.856082 \\

\bottomrule
\end{tabular}
}
\end{table}
In the closed-loop (autoregressive) forecasting results, we observe that methods like Parallel ESN and DeepESN consistently achieve the lowest NRMSE across multiple horizons and datasets. Specifically, in the Lorenz and Chen systems, DeepESN and Parallel ESN deliver excellent performance, especially at long horizons ($1000$ steps), with DeepESN achieving the lowest errors in many cases. In the R\"{o}ssler system, MCI-ESN and SW-ESN emerge as strong contenders with relatively low NRMSE at medium and long horizons, although DeepESN exhibits instability at shorter horizons (e.g., very high variance at $100$-step). For the Chua system, Parallel ESN stands out again with superior long-term forecasting ability, while DeepESN struggles with high errors as the forecast horizon increases.

When compared with the open-loop forecasting results based on Lyapunov divergence, a few key patterns emerge. Despite all models estimating the largest Lyapunov exponent $\lambda_{max}$ with reasonable accuracy, the divergence metric $\Delta\lambda$—which reflects the stability of long-term predictions—varies significantly. Parallel ESN and DeepESN consistently exhibit minimal divergence across all datasets, suggesting that their internal dynamics are better aligned with the true system and more robust to accumulating errors in open-loop rollouts. In contrast, MCI-ESN, while delivering low NRMSE in the closed-loop setting, tends to show higher Lyapunov divergence, indicating less stability during long-term forecasting.

\begin{table}[!ht]
\centering
\caption{VPT, Normalized VPT (w.r.t. Lyapunov Time), ADev for $1000$-step horizon under closed-loop forecasting.}
\label{tab:vpt_ADev_cl}
\begin{tabular}{l l c c c}
\toprule
Dataset & Reservoir & VPT $\uparrow$ & Normalized VPT $\uparrow$ & ADev  $\downarrow$ \\
\midrule

\multirow{7}{*}{Lorenz} 
  & Vanilla ESN             & 4.056 & 3.673 & 99\\
  & SCR             & 3.396 & 3.075 & 97\\
  & CRJ             & 3.960 & 3.586 & 84\\
  & SW-ESN          & 4.060 & 3.677 & 31\\
  & MCI-ESN         & 0.564 & 0.511 & 94\\
  & Parallel ESN     & 4.164 & 3.771 & 118\\
  & DeepESN         & 4.160 & 3.767 & 72\\

\midrule

\multirow{7}{*}{R\"{o}ssler} 
  & Vanilla ESN             & 8.760 & 0.625 & 80\\
  & SCR             & 4.528 & 0.323 & 73\\
  & CRJ             & 7.732 & 0.552 & 86\\
  & SW-ESN          & 10.328 & 0.737 & 74\\
  & MCI-ESN         & 13.992 & 0.999 & 22\\
  & Parallel ESN     & 10.084 & 0.720 & 67\\
  & DeepESN         & 6.212 & 0.444 & 87\\

\midrule
\multirow{7}{*}{Chen} 
  & Vanilla ESN             & 1.532 & 1.271 & 182\\
  & SCR             & 1.484 & 1.231 & 126\\
  & CRJ             & 1.104 & 0.916 & 147\\
  & SW-ESN          & 1.392 & 1.155 & 137\\
  & MCI-ESN         & 0.636 & 0.528 & 98\\
  & Parallel ESN     & 2.844 & 2.359 & 51\\
  & DeepESN         & 2.080 & 1.726 & 111\\

\midrule

\multirow{7}{*}{Chua} 
  & Vanilla ESN             & 7.748 & 3.319 & 5\\
  & SCR             & 9.480 & 4.061 & 4\\
  & CRJ             & 7.952 & 3.407 & 7\\
  & SW-ESN          & 7.840 & 3.359 & 6\\
  & MCI-ESN         & 7.524 & 3.223 & 4\\
  & Parallel ESN     & 12.160 & 5.209 & 7\\
  & DeepESN         & 1.716 & 0.735 & 30\\


\bottomrule
\end{tabular}
\end{table}

In closed-loop forecasting, Parallel ESN consistently delivers the highest normalized VPT and low ADev across all datasets, making it the most robust overall. SW-ESN performs competitively, especially on Lorenz (low ADev) and Chua. MCI-ESN excels on R\"{o}ssler with the highest VPT and lowest ADev, while DeepESN performs well on Lorenz and Chen but struggles on Chua. SCR and CRJ show moderate stability, with good results on Chua but weaker on R\"{o}ssler and Chen. 

\begin{table}[!ht]
\centering
\caption{Lyapunov Exponent Divergence over a 3000-step horizon under closed-loop forecasting.}
\label{tab:lle_cl}
\begin{tabular}{l l c c}
\toprule
Dataset & Reservoir  & $\lambda_{max}$ & $\Delta\lambda$ \\
\midrule
\multirow{7}{*}{Lorenz} 
  & Vanilla ESN             & 0.053008$\pm$0.039281 & 0.852592$\pm$0.039281 \\
  & SCR             & 0.073503$\pm$0.033156 & 0.832097$\pm$0.033156 \\
  & CRJ             & 0.063013$\pm$0.013082 & 0.842587$\pm$0.013082 \\
  & SW-ESN          & 0.067597$\pm$0.051911 & 0.838003$\pm$0.051911 \\
  & MCI-ESN         & 0.104069$\pm$0.016799 & 0.801531$\pm$0.016799 \\
  & Parallel ESN    & 0.095701$\pm$0.011368 & 0.809899$\pm$0.011368 \\
  & DeepESN         & 0.100776$\pm$0.013394 & 0.804824$\pm$0.013394 \\

\midrule

\multirow{7}{*}{R\"{o}ssler} 
  & Vanilla ESN             & 0.088085$\pm$0.043284 & 0.035010$\pm$0.030434 \\
  & SCR             & 0.026417$\pm$0.081027 & 0.081103$\pm$0.044845 \\
  & CRJ             & 0.116377$\pm$0.053139 & 0.061857$\pm$0.031944 \\
  & SW-ESN          & 0.101633$\pm$0.050607 & 0.046085$\pm$0.036759 \\
  & MCI-ESN         & 0.142384$\pm$0.008242 & 0.070984$\pm$0.008242 \\
  & Parallel ESN    & 0.041751$\pm$0.011123 & 0.029649$\pm$0.011123 \\
  & DeepESN         & 0.095733$\pm$0.039326 & 0.036045$\pm$0.028972 \\

\midrule

\multirow{7}{*}{Chen} 
  & Vanilla ESN             & 0.121329$\pm$0.006253 & 0.708271$\pm$0.006253 \\
  & SCR             & 0.098816$\pm$0.034521 & 0.730784$\pm$0.034521 \\
  & CRJ             & 0.072199$\pm$0.106809 & 0.757401$\pm$0.106809 \\
  & SW-ESN          & 0.111848$\pm$0.174197 & 0.717752$\pm$0.174197 \\
  & MCI-ESN         & 0.106998$\pm$0.008240 & 0.722602$\pm$0.008240 \\
  & Parallel ESN    & 0.116106$\pm$0.004690 & 0.713494$\pm$0.004690 \\
  & DeepESN         & 0.114695$\pm$0.002894 & 0.714905$\pm$0.002894 \\

\midrule
\multirow{7}{*}{Chua} 
  & Vanilla ESN             & 0.154737$\pm$0.005862 & 0.273663$\pm$0.005862 \\
  & SCR             & 0.157002$\pm$0.003224 & 0.271398$\pm$0.003224 \\
  & CRJ             & 0.134255$\pm$0.039613 & 0.294145$\pm$0.039613 \\
  & SW-ESN          & 0.142032$\pm$0.007079 & 0.286368$\pm$0.007079 \\
  & MCI-ESN         & 0.155856$\pm$0.006072 & 0.272544$\pm$0.006072 \\
  & Parallel ESN    & 0.149935$\pm$0.011521 & 0.278465$\pm$0.011521 \\
  & DeepESN         & 0.144333$\pm$0.028098 & 0.284067$\pm$0.028098 \\

\bottomrule

\end{tabular}
\end{table}

\subsubsection{Ablation Study}\label{sec:ablation}

In this section, we undertake a comprehensive exploration of how different reservoir topologies influence the performance of time-series prediction tasks. Our evaluation framework systematically probes the behavior of each reservoir architecture under a wide and diverse set of hyperparameter configurations, as outlined in Table \ref{tab:hyperparam_space}. By varying these configurations, we assess not only the average predictive performance but also the sensitivity of each topology to hyperparameter tuning—thereby revealing its robustness and adaptability in different scenarios. This in-depth analysis enables us to identify key architectural traits and design principles that drive improved accuracy, temporal generalization, and stability in reservoir computing models. 

To ensure a consistent and fair comparison across all models—including Vanilla ESN, SCR, CRJ, SW-ESN, MCI-ESN, Parallel ESN, and DeepESN—the same set of core hyperparameters was tuned. These include the spectral radius $\rho$, input scaling, leaking rate $\alpha$, reservoir size N, and regularization factor $\beta$. The spectral radius and leaking rate were systematically varied from $0.1$ to $1.4$ and $0.1$ to $1.0$ respectively, using grid search with a step size of $0.1$. Reservoir sizes were explored in discrete increments from $50$ to $600$, also using grid search. The input scaling parameter was sampled randomly within the range $[0.1, 2.0]$ to introduce variability and better capture performance sensitivity. Regularization was tuned over the range $[10^{-8}, 10^{-2}]$ on a logarithmic scale due to its potential to influence model behavior across several orders of magnitude. 

\begin{table}[!ht]
\centering
\caption{Hyperparameter search space for each reservoir variant.}
\label{tab:hyperparam_space}
\resizebox{\textwidth}{!}{%
\begin{tabular}{llcccc}
\toprule
\multicolumn{2}{l}{\hspace{31mm}\textbf{Hyperparameter}} & \textbf{Range / Values} & \textbf{Log Scale?} & \textbf{Grid/Random} & \textbf{Step Size} \\
\midrule

\multirow{5}{*}{\textbf{Vanilla ESN}} 
 & Spectral Radius $\rho$& \{0.1, 0.2, ..., 1.4\} & No & Grid & 0.1 \\
 & Input Scaling & \{0.1, 0.2, ..., 2.0\} & No & Random & n/a \\
 & Leaking Rate $\alpha$& \{0.1, 0.2, ..., 1.0\} & No & Grid & 0.1 \\
 & Reservoir Size N& \{50, 100, ..., 600\} & No & Grid & n/a \\
 & Regularization $\beta$& [1e-8, 1e-2] & Yes & Grid & n/a \\
\cmidrule(lr){1-6}

\multirow{4}{*}{\textbf{SCR}} 
 & Spectral Radius $\rho$& \{0.1, 0.2, ..., 1.4\} & No & Grid & 0.1 \\
 & Input Scaling & \{0.1, 0.2, ..., 2.0\} & No & Random & n/a \\
 & Leaking Rate $\alpha$& \{0.1, 0.2, ..., 1.0\} & No & Grid & 0.1 \\
 & Reservoir Size N& \{50, 100, ..., 600\} & No & Grid & n/a \\
 & Regularization $\beta$& [1e-8, 1e-2] & Yes & Grid & n/a \\
\cmidrule(lr){1-6}

\multirow{5}{*}{\textbf{CRJ}} 
 & Spectral Radius $\rho$& \{0.1, 0.2, ..., 1.4\} & No & Grid & 0.1 \\
 & Input Scaling & \{0.1, 0.2, ..., 2.0\} & No & Random & n/a \\
 & Leaking Rate $\alpha$& \{0.1, 0.2, ..., 1.0\} & No & Grid & 0.1 \\
 & Reservoir Size N& \{50, 100, ..., 600\} & No & Grid & n/a \\
 & Regularization $\beta$& [1e-8, 1e-2] & Yes & Grid & n/a \\
\cmidrule(lr){1-6}

\multirow{4}{*}{\textbf{SW-ESN}}
 & Spectral Radius $\rho$& \{0.1, 0.2, ..., 1.4\} & No & Grid & 0.1 \\
 & Input Scaling & \{0.1, 0.2, ..., 2.0\} & No & Random & n/a \\
 & Leaking Rate $\alpha$& \{0.1, 0.2, ..., 1.0\} & No & Grid & 0.1 \\
 & Reservoir Size N& \{50, 100, ..., 600\} & No & Grid & n/a \\
 & Regularization $\beta$& [1e-8, 1e-2] & Yes & Grid & n/a \\
\cmidrule(lr){1-6}

\multirow{4}{*}{\textbf{MCI-ESN}}
 & Spectral Radius $\rho$& \{0.1, 0.2, ..., 1.4\} & No & Grid & 0.1 \\
 & Input Scaling & \{0.1, 0.2, ..., 2.0\} & No & Random & n/a \\
 & Leaking Rate $\alpha$& \{0.1, 0.2, ..., 1.0\} & No & Grid & 0.1 \\
 & Reservoir Size N& \{50, 100, ..., 600\} & No & Grid & n/a \\
 & Regularization $\beta$& [1e-8, 1e-2] & Yes & Grid & n/a \\
\cmidrule(lr){1-6}

\multirow{4}{*}{\textbf{Parallel ESN}}& Spectral Radius $\rho$& \{0.1, 0.2, ..., 1.4\} & No & Grid & 0.1 \\
 & Input Scaling & \{0.1, 0.2, ..., 2.0\} & No & Random & n/a \\
 & Leaking Rate $\alpha$& \{0.1, 0.2, ..., 1.0\} & No & Grid & 0.1 \\
 & Reservoir Size N& \{50, 100, ..., 600\} & No & Grid & n/a \\
 & Regularization $\beta$& [1e-8, 1e-2] & Yes & Grid & n/a \\
\cmidrule(lr){1-6}

\multirow{4}{*}{\textbf{DeepESN}}
 & Spectral Radius $\rho$& \{0.1, 0.2, ..., 1.4\} & No & Grid & 0.1 \\
 & Input Scaling & \{0.1, 0.2, ..., 2.0\} & No & Random & n/a \\
 & Leaking Rate $\alpha$& \{0.1, 0.2, ..., 1.0\} & No & Grid & 0.1 \\
 & Reservoir Size N& \{50, 100, ..., 600\} & No & Grid & n/a \\
 & Regularization $\beta$& [1e-8, 1e-2] & Yes & Grid & n/a \\
\bottomrule
\end{tabular}
}
\end{table}

We performed an extensive ablation study to evaluate the impact of key hyperparameters—reservoir size, spectral radius, input scaling and leak rate (see Tables\ref{tab:ablation_nrmse_reservoir_size}, \ref{tab:ablation_vpt_reservoir_size}, \ref{tab:ablation_nrmse_spectral_radius}, \ref{tab:ablation_vpt_spectral_radius}, \ref{tab:ablation_nrmse_input_scales},   \ref{tab:vpt_input_scale}  \ref{tab:ablation_nrmse_leaking_rate}, \ref{tab:ablation_vpt_leaking_rate})—on the performance of various reservoir computing architectures, as measured by NRMSE and VPT. For reservoir size, increasing the number of neurons generally led to a marked reduction in NRMSE across most models. For instance, CRJ achieved its best performance at a reservoir size of $500$, with an NRMSE of $0.0052$, while Parallel ESN and DeepESN achieved excellent predictive accuracy even at smaller sizes like $100$ and $200$. VPT trends generally increased with reservoir size, plateauing around a size of $300-500$ for most models.

For spectral radius, the optimal value varied across architectures. Most models displayed U-shaped performance curves, with minimal NRMSE at moderate values of spectral radius (typically around $0.4$ to $0.6$). DeepESN and CRJ achieved their lowest errors around this range, while MCI-ESN suffered from large errors as spectral radius increased beyond $0.5$, indicating sensitivity to reservoir stability. Correspondingly, VPT peaked near these optimal radii but dropped sharply for MCI-ESN at higher values due to instability.

The impact of input scaling followed a similar pattern. Models such as CRJ and DeepESN benefited from moderate input scales $(0.5-1.0)$, whereas very small or large scaling factors degraded performance. Interestingly, Parallel ESN showed robust performance over a wide range of input scalings, with NRMSE values remaining low even at higher scales. In contrast, MCI-ESN was again highly sensitive, showing erratic behavior and substantial performance degradation as the input scale increased, especially beyond $1.0$. These results highlight the importance of carefully tuning hyperparameters per model to ensure stable and accurate long-term prediction in reservoir computing systems.

\begin{table}[ht]
\centering
\caption{Impact of reservoir size on NRMSE in $100$-step closed-loop predictions.}
\label{tab:ablation_nrmse_reservoir_size}
\begin{tabular}{cccccccc}
\toprule
\textbf{Reservoir Size} & \textbf{ESN} & \textbf{SCR} & \textbf{CRJ} & \textbf{SW-ESN} & \textbf{MCI-ESN} & \textbf{Parallel ESN} & \textbf{DeepESN} \\
\midrule
50  & 3.2486 & 11.7379 & 0.1932 & 4.4505 & 6.5992 & 0.5897 & 6.6939 \\
100 & 0.9316 & 0.1473  & 0.1448 & 0.5000 & 5.3697 & 0.0346 & 0.0584 \\
200 & 0.0078 & 0.0140  & 0.0225 & 0.0370 & 4.1580 & 0.0074 & 0.0060 \\
300 & 0.0185 & 0.0312  & 0.0138 & 0.0418 & 5.6504 & 0.0180 & 0.0107 \\
400 & 0.0070 & 0.0133  & 0.0073 & 0.0110 & 5.6343 & 0.0103 & 0.0022 \\
500 & 0.0137 & 0.0057  & 0.0052 & 0.0151 & 4.6811 & 0.0041 & 0.0041 \\
600 & 0.0098 & 0.0050  & 0.0102 & 0.0028 & 4.8415 & 0.0053 & 0.0064 \\
\bottomrule
\end{tabular}
\end{table}

\begin{table}[ht]
\centering
\caption{Impact of reservoir size on VPT.}
\label{tab:ablation_vpt_reservoir_size}
\begin{tabular}{cccccccc}
\toprule
\textbf{Reservoir Size} & \textbf{ESN} & \textbf{SCR} & \textbf{CRJ} & \textbf{SW-ESN} & \textbf{MCI-ESN} & \textbf{Parallel ESN} & \textbf{DeepESN} \\
\midrule
50  & 1.10 & 0.44 & 3.14 & 0.48 & 0.12 & 2.50 & 1.66 \\
100 & 3.08 & 3.12 & 3.80 & 3.24 & 0.14 & 4.10 & 4.02 \\
200 & 4.14 & 4.06 & 4.04 & 4.12 & 0.42 & 4.28 & 4.20 \\
300 & 4.14 & 4.10 & 4.12 & 4.10 & 0.24 & 4.08 & 4.12 \\
400 & 4.10 & 4.16 & 4.14 & 4.22 & 0.24 & 4.18 & 4.22 \\
500 & 4.08 & 4.12 & 4.14 & 4.16 & 0.40 & 4.14 & 4.30 \\
600 & 4.18 & 4.12 & 4.24 & 6.26 & 0.20 & 4.18 & 4.20 \\
\bottomrule
\end{tabular}
\end{table}

\begin{table}[ht]
\centering
\caption{Impact of spectral radii on NRMSE in $100$-step closed-loop predictions.}
\label{tab:ablation_nrmse_spectral_radius}
\begin{tabular}{cccccccc}
\toprule
\textbf{Spectral Radius} & \textbf{ESN} & \textbf{SCR} & \textbf{CRJ} & \textbf{SW-ESN} & \textbf{MCI-ESN} & \textbf{Parallel ESN} & \textbf{DeepESN} \\
\midrule
0.1 & 0.0363 & 0.0066 & 0.0092 & 0.0277 & 0.0016 & 0.0057 & 0.0270 \\
0.2 & 0.0107 & 0.0246 & 0.0230 & 0.0295 & 0.0036 & 0.0094 & 0.0676 \\
0.3 & 0.0342 & 0.0190 & 0.0299 & 0.0211 & 0.0043 & 0.0121 & 0.0258 \\
0.4 & 0.0318 & 0.0275 & 0.0342 & 0.0053 & 0.0038 & 0.0112 & 0.0045 \\
0.5 & 0.0117 & 0.0324 & 0.0393 & 0.0305 & 0.0374 & 0.0083 & 0.0049 \\
0.6 & 0.0065 & 0.0198 & 0.0359 & 0.0234 & 2.4935 & 0.0085 & 0.0013 \\
0.7 & 0.0171 & 0.0058 & 0.0286 & 0.0557 & 4.8491 & 0.0111 & 0.0036 \\
0.8 & 0.0203 & 0.0155 & 0.0218 & 0.0589 & 4.5545 & 0.0131 & 0.0064 \\
0.9 & 0.0145 & 0.0314 & 0.0153 & 0.0476 & 4.8243 & 0.0160 & 0.0124 \\
1.0 & 0.0243 & 0.0318 & 0.0149 & 0.0372 & 4.6047 & 0.0186 & 0.0081 \\
1.1 & 0.0225 & 0.0244 & 0.0276 & 0.0183 & 2.6993 & 0.0122 & 0.0065 \\
1.2 & 0.0345 & 0.0229 & 0.0413 & 0.0065 & 3.1728 & 0.0038 & 0.0012 \\
1.3 & 0.0310 & 0.0162 & 0.0547 & 0.0085 & 4.4709 & 0.0032 & 0.0026 \\
1.4 & 0.0805 & 0.0040 & 0.0365 & 0.0123 & 4.0900 & 0.0112 & 0.0145 \\
\bottomrule
\end{tabular}
\end{table}

\begin{table}[ht]
\centering
\caption{Impact of spectral radii on VPT.}
\label{tab:ablation_vpt_spectral_radius}
\begin{tabular}{cccccccc}
\toprule
\textbf{Spectral Radius} & \textbf{ESN} & \textbf{SCR} & \textbf{CRJ} & \textbf{SW-ESN} & \textbf{MCI-ESN} & \textbf{Parallel ESN} & \textbf{DeepESN} \\
\midrule
0.1  & 4.10 & 4.16 & 4.12 & 4.26 & 9.08 & 9.06 & 4.06 \\
0.2  & 4.22 & 4.16 & 4.24 & 4.14 & 4.28 & 4.34 & 4.06 \\
0.3  & 4.06 & 4.12 & 4.22 & 4.16 & 4.24 & 4.24 & 4.10 \\
0.4  & 3.96 & 4.12 & 4.08 & 4.14 & 4.28 & 4.22 & 4.08 \\
0.5  & 4.10 & 4.10 & 4.04 & 3.88 & 4.16 & 4.22 & 4.14 \\
0.6  & 4.08 & 4.14 & 4.04 & 3.68 & 1.02 & 4.14 & 4.20 \\
0.7  & 4.12 & 4.16 & 4.06 & 3.98 & 0.22 & 4.12 & 4.16 \\
0.8  & 4.16 & 4.28 & 4.08 & 4.02 & 0.44 & 4.10 & 4.14 \\
0.9  & 4.16 & 4.10 & 4.12 & 4.06 & 0.26 & 4.08 & 4.10 \\
1.0  & 4.14 & 4.12 & 4.10 & 4.12 & 0.30 & 4.08 & 4.10 \\
1.1  & 4.18 & 4.16 & 4.04 & 4.18 & 0.36 & 4.12 & 4.16 \\
1.2  & 4.12 & 4.26 & 4.02 & 4.14 & 0.38 & 4.16 & 4.08 \\
1.3  & 4.08 & 4.24 & 4.00 & 4.12 & 0.28 & 4.18 & 4.14 \\
1.4  & 3.94 & 4.12 & 4.00 & 4.16 & 0.34 & 4.12 & 4.20 \\
\bottomrule
\end{tabular}
\end{table}

\begin{table}[ht]
\centering
\caption{Impact of input-scaling on NRMSE in $100$-step closed-loop predictions.}
\label{tab:ablation_nrmse_input_scales}
\begin{tabular}{cccccccc}
\toprule
\textbf{Input Scale} & \textbf{ESN} & \textbf{SCR} & \textbf{CRJ} & \textbf{SW-ESN} & \textbf{MCI-ESN} & \textbf{Parallel ESN} & \textbf{DeepESN} \\
\midrule
0.1 & 0.0363 & 0.0066 & 0.0092 & 0.0277 & 0.0016 & 0.0057 & 0.0270 \\
0.2 & 0.0107 & 0.0246 & 0.0230 & 0.0295 & 0.0036 & 0.0094 & 0.0676 \\
0.3 & 0.0342 & 0.0190 & 0.0299 & 0.0211 & 0.0043 & 0.0121 & 0.0258 \\
0.4 & 0.0318 & 0.0275 & 0.0342 & 0.0053 & 0.0038 & 0.0112 & 0.0045 \\
0.5 & 0.0117 & 0.0324 & 0.0393 & 0.0305 & 0.0374 & 0.0083 & 0.0049 \\
0.6 & 0.0065 & 0.0198 & 0.0359 & 0.0234 & 2.4935 & 0.0085 & 0.0013 \\
0.7 & 0.0171 & 0.0058 & 0.0286 & 0.0557 & 4.8491 & 0.0111 & 0.0036 \\
0.8 & 0.0203 & 0.0155 & 0.0218 & 0.0589 & 4.5545 & 0.0131 & 0.0064 \\
0.9 & 0.0145 & 0.0314 & 0.0153 & 0.0476 & 4.8243 & 0.0160 & 0.0124 \\
1.0 & 0.0243 & 0.0318 & 0.0149 & 0.0372 & 4.6047 & 0.0186 & 0.0081 \\
1.1 & 0.0225 & 0.0244 & 0.0276 & 0.0183 & 2.6993 & 0.0122 & 0.0065 \\
1.2 & 0.0345 & 0.0229 & 0.0413 & 0.0065 & 3.1728 & 0.0038 & 0.0012 \\
1.3 & 0.0310 & 0.0162 & 0.0547 & 0.0085 & 4.4709 & 0.0032 & 0.0026 \\
1.4 & 0.0805 & 0.0040 & 0.0365 & 0.0123 & 4.0900 & 0.0112 & 0.0145 \\
1.5 & 0.0082 & 0.0032 & 0.0043 & 0.0193 & 0.0297 & 0.0054 & 0.0038 \\
1.6 & 0.0185 & 0.0312 & 0.0138 & 0.0418 & 5.6504 & 0.0180 & 0.0107 \\
1.7 & 0.0228 & 4.4907 & 0.0295 & 0.1464 & 7.1289 & 0.0102 & 0.0072 \\
1.8 & 0.3334 & 0.7872 & 0.1533 & 0.6245 & 7.4037 & 0.0132 & 0.0037 \\
1.9 & 0.3988 & 0.2476 & 0.3114 & 13.4986 & 10.1076 & 0.0058 & 0.0426 \\
2.0 & 0.7972 & 0.5078 & 0.2136 & 2.0327 & 9.3896 & 0.0465 & 0.1193 \\
\bottomrule
\end{tabular}
\end{table}

\begin{table}[ht]
\centering
\caption{Impact of input scaling on VPT.}
\label{tab:vpt_input_scale}
\begin{tabular}{cccccccc}
\toprule
\textbf{Input Scale} & \textbf{ESN} & \textbf{SCR} & \textbf{CRJ} & \textbf{SW-ESN} & \textbf{MCI-ESN} & \textbf{Parallel ESN} & \textbf{DeepESN} \\
\midrule
0.1  & 4.10 & 4.16 & 4.12 & 4.26 & 9.08 & 9.06 & 4.06 \\
0.2  & 4.22 & 4.16 & 4.24 & 4.14 & 4.28 & 4.34 & 4.06 \\
0.3  & 4.06 & 4.12 & 4.22 & 4.16 & 4.24 & 4.24 & 4.10 \\
0.4  & 3.96 & 4.12 & 4.08 & 4.14 & 4.28 & 4.22 & 4.08 \\
0.5  & 4.10 & 4.10 & 4.04 & 3.88 & 4.16 & 4.22 & 4.14 \\
0.6  & 4.08 & 4.14 & 4.04 & 3.68 & 1.02 & 4.14 & 4.20 \\
0.7  & 4.12 & 4.16 & 4.06 & 3.98 & 0.22 & 4.12 & 4.16 \\
0.8  & 4.16 & 4.28 & 4.08 & 4.02 & 0.44 & 4.10 & 4.14 \\
0.9  & 4.16 & 4.10 & 4.12 & 4.06 & 0.26 & 4.08 & 4.10 \\
1.0  & 4.14 & 4.12 & 4.10 & 4.12 & 0.30 & 4.08 & 4.10 \\
1.1  & 4.18 & 4.16 & 4.04 & 4.18 & 0.36 & 4.12 & 4.16 \\
1.2  & 4.12 & 4.26 & 4.02 & 4.14 & 0.38 & 4.16 & 4.08 \\
1.3  & 4.08 & 4.24 & 4.00 & 4.12 & 0.28 & 4.18 & 4.14 \\
1.4  & 3.94 & 4.12 & 4.00 & 4.16 & 0.34 & 4.12 & 4.20 \\
1.5  & 4.20 & 4.16 & 4.24 & 4.18 & 4.14 & 4.18 & 4.16 \\
1.6  & 4.14 & 4.10 & 4.12 & 4.10 & 0.24 & 4.08 & 4.12 \\
1.7  & 4.18 & 0.72 & 3.90 & 3.26 & 0.18 & 4.18 & 4.22 \\
1.8  & 3.82 & 2.36 & 4.10 & 2.58 & 0.08 & 4.16 & 4.18 \\
1.9  & 3.04 & 3.16 & 4.00 & 1.14 & 0.02 & 4.24 & 4.06 \\
2.0  & 3.76 & 2.98 & 3.28 & 1.26 & 0.06 & 4.16 & 3.12 \\
\bottomrule
\end{tabular}
\end{table}

\begin{table}[ht]
\centering
\caption{Impact of  leaking rate on NRMSE in $100$-step closed-loop predictions.}
\label{tab:ablation_nrmse_leaking_rate}
\begin{tabular}{cccccccc}
\toprule
\textbf{Leaking Rate} & \textbf{ESN} & \textbf{SCR} & \textbf{CRJ} & \textbf{SW-ESN} & \textbf{MCI-ESN} & \textbf{Parallel ESN} & \textbf{DeepESN} \\
\midrule
0.1 & 0.0363 & 0.0066 & 0.0092 & 0.0277 & 0.0016 & 0.0057 & 0.0270 \\
0.2 & 0.0107 & 0.0246 & 0.0230 & 0.0295 & 0.0036 & 0.0094 & 0.0676 \\
0.3 & 0.0342 & 0.0190 & 0.0299 & 0.0211 & 0.0043 & 0.0121 & 0.0258 \\
0.4 & 0.0318 & 0.0275 & 0.0342 & 0.0053 & 0.0038 & 0.0112 & 0.0045 \\
0.5 & 0.0117 & 0.0324 & 0.0393 & 0.0305 & 0.0374 & 0.0083 & 0.0049 \\
0.6 & 0.0065 & 0.0198 & 0.0359 & 0.0234 & 2.4935 & 0.0085 & 0.0013 \\
0.7 & 0.0171 & 0.0058 & 0.0286 & 0.0557 & 4.8491 & 0.0111 & 0.0036 \\
0.8 & 0.0203 & 0.0155 & 0.0218 & 0.0589 & 4.5545 & 0.0131 & 0.0064 \\
0.9 & 0.0145 & 0.0314 & 0.0153 & 0.0476 & 4.8243 & 0.0160 & 0.0124 \\
1.0 & 0.0243 & 0.0318 & 0.0149 & 0.0372 & 4.6047 & 0.0186 & 0.0081 \\
\bottomrule
\end{tabular}
\end{table}

\begin{table}[ht]
\centering
\caption{Impact of leaking rate on VPT.}
\label{tab:ablation_vpt_leaking_rate}
\begin{tabular}{cccccccc}
\toprule
\textbf{Leaking Rate} & \textbf{ESN} & \textbf{SCR} & \textbf{CRJ} & \textbf{SW-ESN} & \textbf{MCI-ESN} & \textbf{Parallel ESN}& \textbf{DeepESN} \\
\midrule
0.1 & 4.10 & 4.16 & 4.12 & 4.26 & 9.08 & 9.06 & 4.06 \\
0.2 & 4.22 & 4.16 & 4.24 & 4.14 & 4.28 & 4.34 & 4.06 \\
0.3 & 4.06 & 4.12 & 4.22 & 4.16 & 4.24 & 4.24 & 4.10 \\
0.4 & 3.96 & 4.12 & 4.08 & 4.14 & 4.28 & 4.22 & 4.08 \\
0.5 & 4.10 & 4.10 & 4.04 & 3.88 & 4.16 & 4.22 & 4.14 \\
0.6 & 4.08 & 4.14 & 4.04 & 3.68 & 1.02 & 4.14 & 4.20 \\
0.7 & 4.12 & 4.16 & 4.06 & 3.98 & 0.22 & 4.12 & 4.16 \\
0.8 & 4.16 & 4.28 & 4.08 & 4.02 & 0.44 & 4.10 & 4.14 \\
0.9 & 4.16 & 4.10 & 4.12 & 4.06 & 0.26 & 4.08 & 4.10 \\
1.0 & 4.14 & 4.12 & 4.10 & 4.12 & 0.30 & 4.08 & 4.10 \\
\bottomrule
\end{tabular}
\end{table}

\section{Applications, Challenges and Open Questions}\label{sec:applications}
Reservoir computing methods, particularly ESNs, have proven effective across a wide range of application domains involving time-series data and complex temporal dependencies. This section highlights some of the major application areas, discussing both classical results and more recent advances.

\subsection{Applications in Time-Series Prediction and  Modeling}
ESNs have been applied to a wide range of time-series prediction problems – from synthetic chaotic sequences to real-world temporal data – often achieving excellent performance with minimal training. We highlight a few notable application domains and benchmarks:
\begin{enumerate}

 \item \textit{Chaotic dynamical system forecasting:} One of the earliest and most prominent uses of reservoir computing is in time series forecasting \cite{Jaurigue2022}, including applications in economics, climate science, and energy systems. Given an input sequence \(\{u(t)\}\) and a desired prediction horizon \(k\), one trains a readout to produce
\begin{equation}
    \hat{y}(t+k) = \mathbf{W}_\mathrm{out} \, \mathbf{x}(t),  
\end{equation}
where \(\mathbf{x}(t)\) is the current reservoir state \cite{Jaeger2002,Lukosevicius2012}.  ESNs have shown a remarkable ability to learn and predict chaotic time series, which are highly sensitive to initial conditions and have complex nonlinear patterns. For example, ESNs can be trained on the Lorenz attractor time series and predict short-term future evolution accurately until the chaos-induced divergence kicks in. A landmark result by Pathak et al. \cite{pathak2018model} used ESNs to predict spatiotemporal chaos in the Kuramoto–Sivashinsky equation (a PDE with rich chaotic behavior). By using a large reservoir and parallel scheme, they could forecast the system’s state many Lyapunov times ahead, essentially until the inherent unpredictability of chaos dominates. This was a breakthrough because it demonstrated data-driven forecasting of a high-dimensional chaotic system without an explicit model, outperforming some traditional numerical approaches. Since then, ESNs have been applied to chaotic laser intensity signals, atmospheric turbulence data, and other chaotic systems. A common benchmark is the Mackey-Glass time series (a time-delay differential equation producing chaos) \cite{mackey1977oscillation} – ESNs quickly became known for outperforming other methods on short-term Mackey-Glass predictions. Another standard test is the Lorenz-96 weather model \cite{lorenz1996predictability}; ESN-based models have achieved state-of-the-art results in predicting this system and even incorporating uncertainty estimation via ensemble ESNs.

\item \textit{Synthetic memory benchmarks (NARMA):} The Nonlinear Auto-Regressive Moving Average (NARMA) benchmark \cite{atiya2000new} is a family of tasks frequently used to evaluate RC memory and nonlinear processing. In NARMA-$n$, the target output at time $t$ is a nonlinear function of the previous $n$ inputs and outputs. For instance, NARMA-10 is defined by a degree-10 polynomial relationship involving past inputs and outputs. These tasks are challenging because they require the network to maintain information about the last $n$ steps and compute a nonlinear combination. ESNs have been widely tested on NARMA tasks and consistently perform well, typically requiring a reservoir of size on the order of the task memory length. In fact, NARMA sequences are considered one of the “most important and widely used benchmark datasets” in ESN literature \cite{Rodan2011}. Studies show ESNs easily outperform linear models on NARMA and serve as a baseline to compare new RNN training algorithms. For example, Jaeger’s early ESN work showed excellent performance on NARMA-10 and NARMA-30 tasks, whereas fully trained RNNs struggled without specialized architectures \cite{Rodan2011}. The NARMA test also illustrates how spectral radius tuning affects memory: an ESN with spectral radius near 1 and sufficient units can achieve high performance on NARMA-$n$, while a smaller radius yields a rapid drop in performance for larger $n$ (since memory fades too fast).

 \item  \textit{Signal processing and communications:} ESNs have been applied to tasks such as channel equalization, radar signal prediction, and speech signal processing. Jaeger \& Haas  famously applied an ESN to a mobile communication channel equalization problem (mitigating multipath fading distortions) and achieved superior performance with far less training than conventional adaptive filters \cite{jaeger2004harnessing}. Another signal benchmark is the Santa Fe laser time series (a chaotic laser intensity dataset from a 1990s competition) \cite{weigend1994time} – ESNs were able to learn the laser’s dynamics and predict future intensity better than many classical methods. In audio processing, while LSMs are more common (due to spiking input), ESNs have been used for tasks like phoneme recognition, where the temporal pattern of speech audio must be classified. An ESN can serve as a front-end that maps the raw audio feature sequence into a high-dimensional trajectory, followed by a classifier. Results on spoken digit recognition and certain phoneme classification tasks showed ESNs performing comparably to early deep learning models, though modern end-to-end trained networks now surpass them. Reservoir computing is well-suited to classification and pattern recognition tasks involving sequential or streaming data. By treating the reservoir state as a \emph{feature space}, one can classify temporal patterns with a simple linear separator \cite{Verstraeten2007}.
ESNs and LSMs have been applied to phoneme or word classification from raw or pre-processed audio \cite{Triefenbach2013}. The reservoir dynamically encodes the time-varying amplitude and frequency characteristics of speech signals, while the linear readout assigns class labels.
Biomedical signals such as electroencephalogram (EEG) or electromyogram (EMG) often exhibit stochastic yet structured temporal patterns. ESNs can detect abnormal patterns, analyze mental states, or classify muscle movements with relatively few trainable parameters \cite{schrauwen2007overview}.
In adaptive filter scenarios, one aims to estimate an unknown or time-varying filter \(\mathbf{h}(t)\) from observation signals \(\mathbf{z}(t)\). While classical adaptive filters like LMS or RLS rely on linear updates, reservoir-based schemes can capture nonlinear relationships in \(\mathbf{z}(t)\) \cite{jaeger2004harnessing}.
 
 \item \textit{Financial and economic time series:} The rapid training of ESNs makes them suitable for financial forecasting where models may need frequent retraining as new data arrives. ESNs have been used for stock price prediction, volatility forecasting, and macro-economic indicator prediction. They often beat standard ARIMA and even some machine learning benchmarks on short-horizon forecasting, although like all methods, they struggle with the inherent noise in financial data. The advantage is that an ESN can be retrained on each day’s new data within milliseconds, enabling adaptive updates. Some studies combined ESNs with evolutionary algorithms to continually evolve the reservoir or output weights as market conditions changed.
Stock prices, exchange rates, and volatility indices, where ESNs have competed favorably against traditional models like ARIMA and even other neural networks \cite{Makarenkov2019}.
Real-time electricity load and renewable energy production forecasts, where short-term predictions benefit from RC’s capability to handle diverse, non-stationary inputs \cite{Gauthier2018}.

\item \textit{Control Systems and Robotics:}
Reservoir computing has been leveraged in both model-based and model-free control setups. The core idea is that the reservoir can emulate complex dynamical control laws without explicitly solving nonlinear optimization problems at run-time \cite{Antonelo2008}. 
In robotic control, one often needs a \emph{forward model} \( \hat{\mathbf{x}}(t+1) = f\bigl(\mathbf{x}(t), \mathbf{a}(t)\bigr)\), where \(\mathbf{x}(t)\) is the state and \(\mathbf{a}(t)\) is the control action. ESNs can learn \(f(\cdot)\) from sensorimotor data. Similarly, for \emph{inverse models}, given the desired next state \(\mathbf{x}_\mathrm{d}(t+1)\), an ESN can predict the required action \(\mathbf{a}(t)\).  
Combining ESNs with reinforcement learning algorithms enables policy learning for nonlinear, continuous control problems \cite{Sussillo2009}. The reservoir states provide a rich representation of the robot’s sensor history, helping the policy readout map sensor data to actions.
Adaptive reservoirs with intrinsic plasticity or partially trainable recurrent layers can handle non-stationary environments—e.g., wear-and-tear changes in robot joints, or time-varying external disturbances \cite{Lukosevicius2012}.

 \item \textit{Neuroscience and brain signal modeling:} Closing the loop back to neuroscience, reservoir computing has been used both to analyze neural data and to model neural circuits. On the data analysis side, ESNs have been applied to EEG signals for tasks like seizure prediction, cognitive state classification, or brain-computer interfaces \cite{yang2024brain}. For instance, an ESN can be fed raw EEG channels and trained to detect patterns corresponding to different mental states or events. Their fading memory is well-suited to capture EEG rhythms and transient patterns. Some work has shown ESNs achieving high accuracy in EEG-based emotion recognition and even predicting certain diseases from EEG time-series \cite{bouazizi2024novel,yang2024brain}. The “modular ESN” approach – splitting the reservoir into sub-reservoirs, possibly corresponding to different brain regions or frequency bands – has been explored to reflect the distributed processing of the brain and improve interpretability. On the modeling side, LSMs (spiking reservoirs) have been used as models of cortical microcircuits for computational neuroscience experiments. For example, researchers have modeled the prefrontal cortex working memory using an LSM where the fading activity of spikes in the reservoir holds transient information, analogous to neural persistent activity \cite{yan2024emerging}. By training appropriate readouts, these models can perform tasks like decision-making or sequence generation, offering insight into how random recurrent connectivity in the cortex combined with learning at synapses downstream could support cognitive functions. Another line of work attempts to use reservoir networks to understand neural data transformations – e.g. sensory processing streams. If an ESN can mimic the input-output behavior of a neural circuit (like the retina or cochlea), it suggests that the circuit may be exploiting a reservoir-like strategy. This has led to hypotheses that certain brain regions don’t need to fine-tune all recurrent connections; instead, plasticity might focus on output synapses (as RC suggests), which aligns with the idea of downstream readouts in biology (e.g. muscle control learning from fixed central pattern generators). Some research frameworks use reservoirs to simulate aspects of working memory and decision-making in neural circuits. By maintaining temporal information in high-dimensional neural states, reservoirs can model the time-dependent computations believed to underlie certain cognitive tasks \cite{Hinaut2014}.
Reservoir computing methods have also been applied to BCI tasks, where non-stationary EEG signals must be decoded in real time. The temporal flexibility and built-in noise robustness of ESNs can be advantageous for stable classification of motor imagery or cognitive states \cite{schrauwen2007overview}. While these ideas are still being tested, they highlight the fruitful exchange between RC and neuroscience: \textit{RC provides a functional framework to interpret neural dynamics}, and neuroscience provides inspiration to improve artificial reservoirs. 
 
 \item  \textit{Benchmark datasets and competitions:} In the RC community, a few benchmark tasks frequently appear to compare methods:
 \begin{enumerate}
     
\item 
Mackey-Glass chaotic time series \cite{mackey1977oscillation} (as mentioned earlier) – predict $x(t+\Delta)$ from past samples.
\item NARMA (orders 10, 20, 30) – evaluate memory capacity \cite{atiya2000new}.
\item Waveform generation – e.g. generate a Sine wave or a more complex oscillation from a cue (tests generative capability).
\item Isolated spoken digits \cite{Verstraeten2007} (for LSMs) – an early benchmark by Maass where an LSM had to classify spoken digit audio \cite{li2018biologically}.
\item Weather or power load time series – e.g. sunspot numbers (a classic nonlinear series of solar activity) \cite{Rodan2011}
 or wind power production data \cite{yan2024emerging}, to test multi-step forecasting.
 \end{enumerate}
\end{enumerate}

In many of these, ESNs have been top performers. For instance, on the sunspot number prediction (an irregular, quasi-periodic series), simple ESNs were able to outperform ARIMA and even some neural network approaches, capturing the cycles with lower error \cite{Rodan2011}. On IPNX radar signal tracking and Henon Map sequence prediction (another chaotic map), ESNs again have matched or beaten more complex models in several studies \cite{Rodan2011}. These benchmarks serve as a testing ground for new RC algorithms. It’s worth noting that for each success, there is usually a careful setting of ESN hyperparameters and often use of tricks like washout periods, state normalization, and regularization. As one 2023 systematic study put it, “while not as flexible as other RNN methods, RC has properties that make it a method of choice for these kinds of tasks”, citing the ability to set macro-scale network properties and the quick training as key advantages \cite{platt2022systematic}. Moreover, ESNs have been shown to not only make good predictions but sometimes model the system’s dynamics – for example, if a reservoir is trained to mimic a chaotic system, one can perturb the reservoir state and see a corresponding chaotic response, akin to the real system, indicating the ESN learned an underlying dynamical model \cite{platt2022systematic}. This property is crucial in scientific applications like weather forecasting, where one wants both a forecast and an understanding of stability/uncertainty. In neuroscience modeling, a fascinating application is using RC to decode neural signals. For example, an ESN can take inputs from hundreds of neurons recorded in motor cortex and predict an animal’s hand trajectory or muscle activity. The ESN’s reservoir essentially models the recurrent dynamics of the motor cortex (which we assume are present but we only record a subset), and the readout learns the mapping to behavior. Studies have found ESNs can achieve high decoding accuracy, sometimes better than linear decoders or even LSTMs given limited training data \cite{yang2024brain}. This is attributed to the ESN’s ability to model the temporal integration and coordination in neural firing. Such uses of RC in brain-machine interfaces are promising because they offer fast training (important when calibrating to a new session) and can potentially adapt online.

\subsection{Advantages and Practical Strengths of ESNs}

ESNs rose to prominence because they offer several compelling advantages over traditional RNN training:
\begin{enumerate}
    \item 

\textit{Fast and straightforward training:} Since only the output layer is trained (usually with linear regression), learning in an ESN is extremely fast and robust. There is no need for backpropagation through time, which can be slow and get stuck in poor local minima. Training an ESN often reduces to solving a single linear least-squares problem, for which efficient algorithms exist. This makes ESNs appealing for online learning scenarios or cases where retraining must be done frequently. The training complexity scales as $O(N_x^2 T)$ (for $T$ time steps of data and $N_x$ reservoir units, if solved via normal equations) or better with incremental methods – far cheaper than training an equivalent-sized RNN with gradient descent. As one study put it, reservoir computing enables RNNs to be “trained and deployed quickly and easily” \cite{platt2022systematic}. This efficiency also allows experimenting with very large reservoirs (hundreds or thousands of neurons) that would be infeasible to train with backpropagation.

   \item 
\textit{Low risk of overfitting and good generalization: }Because the vast majority of weights are fixed random values, the ESN effectively acts as a form of random feature generator. The only trained parameters are the linear output weights, which have a much smaller number of degrees of freedom than a full recurrent weight matrix. This reduction in trainable parameters can help generalization, especially in regimes of limited training data. With appropriate regularization (e.g. ridge regression), ESNs often avoid overfitting even when $N_x$ (number of features) is large relative to the amount of data. The random features, while not learned, tend to be diverse and, if the reservoir is well-designed, provide a basis that can approximate many functions of the input. This is analogous to random kitchen-sink or extreme learning approaches in feedforward networks.

   \item 
\textit{Rich dynamics and memory:} A well-tuned reservoir can exhibit a spectrum of time constants and nonlinear responses, endowing the ESN with a form of dynamic memory. The recurrent loops allow information to echo for several time steps, enabling the network to integrate information over time. While a single recurrent neuron is just a leaky integrator, a large network with recurrent connections can create much more complex, high-order memory of inputs. ESNs have been shown to attain high memory capacity, meaning they can reconstruct or utilize input signals from dozens of time steps in the past \cite{jaeger2004harnessing, Rodan2011}. This memory is “soft” (fading over time) but, combined with nonlinear mixing, is sufficient for a wide range of temporal tasks (from short-term signal processing to moderately long-horizon predictions). The user can also control memory vs. nonlinearity by adjusting parameters like the spectral radius and leaking rate.

   \item 
\textit{Ease of implementation and use:} ESNs  require only matrix multiplications and elementwise nonlinearities – operations that are easy to implement and fast on modern hardware. There is no need for complex gradient computations. Many open-source libraries (in Python, MATLAB, etc.) and tutorials exist, reflecting the relatively shallow learning curve to get started. As noted by Lukoševičius \cite{Lukosevicius2012}, ESNs are “conceptually simple and computationally inexpensive”, making them accessible to a wide audience. This ease has made ESNs a popular baseline for time-series tasks and a didactic tool for understanding recurrent networks.

   \item 
\textit{Interpretable linear readouts:} Since the only trained part is linear, the mapping from reservoir states to output can be analyzed more directly than in a fully nonlinear network. One can inspect the learned output weights to see which reservoir units (or which temporal features) are most influential for the task. This doesn’t fully solve interpretability (as the reservoir features themselves are complicated functions of the input), but it does simplify the last layer. In some applications, researchers treat the reservoir as a fixed nonlinear filter bank and focus on interpreting how the readout combines those filters to make decisions.

   \item 
\textit{Hardware friendliness and neuromorphic advantages:} The RC paradigm lends itself to non-von Neumann hardware implementations because only the fixed reservoir needs to be realized in analog or unconventional hardware, and training can happen off-line on a separate machine. This has led to reservoirs implemented with optical lasers, photonic circuits, spintronic devices, memristor crossbars, FPGAs, and other substrates. The fixed random weights can be embedded in material properties, and only the readout computation (a weighted sum) needs to be performed, which is trivial in hardware. This is attractive for low-power or fast computing: for example, a photonic ESN can operate at GHz speeds, processing time-series faster than electronic solutions. Similarly, neuromorphic chips (like Intel’s Loihi or analog CMOS circuits) can implement spiking reservoirs (LSMs) efficiently, leveraging RC to avoid costly synaptic weight training on chip. Ergo, ESNs can capitalize on “naturally available computational dynamics” of physical systems, using them as reservoirs and thereby offloading computation from digital processors. This is seen as a promising path for energy-efficient AI and is one reason RC has garnered interest in the neuromorphic computing community \cite{yan2024emerging}.

   \item 
\textit{Strong performance on certain tasks:} When properly tuned, ESNs have achieved state-of-the-art or highly competitive results on a variety of benchmarks, especially those involving chaotic time series and small training sets. For example, Jaeger and Haas \cite{jaeger2004harnessing} showed that ESNs could learn to generate complex periodic patterns and even outperformed carefully trained traditional RNNs on some benchmarks of nonlinear system identification \cite{Lukosevicius2012}. More recently, ESNs have been used to predict chaotic dynamical systems with accuracy matching or exceeding deep learning models, while using orders-of-magnitude less training time \cite{platt2022systematic}. While modern deep learning (e.g. LSTMs, Transformers) often surpass ESNs on very large datasets, ESNs remain highly competitive on problems where data is limited or fast training is required.

\end{enumerate}
Considering these advantages, it’s clear why ESNs sparked excitement in the 2000s as a way to “have your cake and eat it too” – harness the sequence-processing power of RNNs without the pain of difficult training. However, ESNs also come with a set of limitations and challenges, discussed next.

\subsection{Limitations and Challenges}

No free lunch comes without caveats. Several limitations of ESNs have been observed in practice:
\begin{enumerate}

 \item \textit{Need for careful hyperparameter tuning:} Perhaps the most commonly cited challenge is that ESN performance is highly sensitive to the choice of reservoir parameters (spectral radius, input scaling, bias distribution, leaking rate, etc.) and even the random initialization seed. A poorly tuned reservoir can lead to suboptimal performance – e.g. if the spectral radius is too low, the reservoir may forget inputs too quickly; if too high, the network may become unstable or output oscillatory predictions. Unlike feedforward networks where weight initialization is often “washed out” by training, in ESNs the random weights remain in place, so a bad draw can hurt. This means practitioners often must try multiple reservoirs and tune parameters via grid search or more sophisticated optimization such as Bayesian optimization. While tuning an ESN is far easier than training an RNN from scratch, it still requires expertise – Lukoševičius  warns that the apparent simplicity of ESNs “can be deceptive” and one must set a handful of global parameters “judiciously” for good results \cite{Lukosevicius2012}. Recently, automated hyperparameter search and theories linking reservoir spectra to task requirements have improved this situation, but it remains a user-side burden.

 \item 
\textit{Limited flexibility/adaptability of the reservoir:} Because the reservoir weights are fixed randomly, ESNs may not optimally represent the features needed for a given task, especially if the task has very specific dynamics. In a fully trained RNN (e.g. an LSTM or a backprop-trained network), the recurrent weights are tuned by gradient descent to create task-specific internal representations \cite{Jaeger2007,Steil2007}. ESNs forego this, so if the random reservoir isn’t well-aligned with the task, performance can suffer. As an example, if a task requires recognizing a very particular pattern over 50 time steps, a random reservoir might need to be quite large (many neurons) to have some subset of neurons that are sensitive to that pattern. A trained RNN might solve it with far fewer units by adjusting weights to “search” for that pattern \cite{platt2022systematic}. In essence, ESNs trade representational flexibility for training simplicity. Empirically, ESNs are sometimes outperformed by modern RNNs (e.g. gated architectures like LSTMs/GRUs) on complex benchmarks where large training data is available. ESNs also typically cannot do very long-term memory tasks (like remembering an event hundreds of time steps ago with no cue) as well as specialized architectures or fully-trained networks designed for that purpose.

 \item 
\textit{Linearity of readout:} The flip side of only training a linear output is that the final output is a linear combination of reservoir signals. If the reservoir hasn’t already transformed inputs in a way that makes the output linear separable/linearly related to the target, a linear readout may not suffice. In theory, a sufficiently large and nonlinear reservoir produces a rich set of basis functions so that linear combinations can approximate most functions (a bit like a random Fourier or polynomial basis). But in practice, there could be cases where a linear combination of the reservoir state is a limiting assumption. Some extensions address this by using a nonlinear readout (e.g. a shallow MLP or even a second-layer reservoir), at the cost of reintroducing a nonlinear training problem. It’s worth noting that including the input $u(n)$ in the readout (as often done) allows bypassing the reservoir for instantaneous effects, and including a bias allows arbitrary affine combinations, which often is enough. Still, one must ensure the reservoir is expressive enough that a linear mapper can extract the desired output.

\item
\textit{Stability and initialization issues:} If the reservoir weights are not chosen well, the ESN can suffer from instability (diverging activations) or from pathological dynamics (e.g. all neurons saturating at their maximum/minimum, providing almost no useful variation to the readout). Ensuring the echo state property by scaling the spectral radius is a primary safeguard, but when output feedback is used, or when one pushes the envelope (e.g. spectral radius very close to 1 or slight instability to get more memory), the ESN may become sensitive to noise or perturbation. In scenarios with streaming data, if input statistics shift (distributional change), a fixed ESN might become mis-calibrated – e.g. if inputs suddenly drive it into regions of state space it rarely saw during training, the linear readout might extrapolate poorly. Traditional RNNs could in principle adapt their weights to new regimes, but a fixed reservoir cannot, short of completely re-training the output or resetting internal states. There are techniques like adaptive bias scaling or inserting slow adaptation (via intrinsic plasticity) to mitigate this, but they add complexity.

 \item 
\textit{Efficiency at runtime:} While training is fast, using an ESN for inference still requires simulating potentially a large recurrent network. For example, an ESN with 1000 neurons will do on the order of 1000 multiplications and $\tanh$ evaluations per time-step, which is not a big deal on modern hardware (and trivially parallelizable), but a highly optimized deep network might do fewer operations especially if sequence length is short. In practice this is rarely a serious issue, and hardware implementations can drastically accelerate it, but it’s a consideration if comparing to simpler models.

 \item 
\textit{Lack of explicit long-term memory or modularity:} Tasks that require storing information for arbitrarily long times (like remembering a bit indefinitely until a reset signal) are not natural for standard ESNs, because the reservoir state will eventually forget due to fading memory. In fully trainable RNNs, one can insert memory cells or explicit loops that learn to maintain information (as in an LSTM’s gating mechanisms). ESNs by default rely on the reservoir’s eigen-spectrum for memory – to lengthen memory, one must push spectral radius to 1 (approaching instability) or make the network very large. Additionally, ESNs do not inherently separate timescales or subtasks; it’s a monolithic random network. If a problem is naturally modular (e.g. separate processes happening on different timescales), a random reservoir might not capture that structure efficiently. Recent research addresses this by building multiple reservoirs or hierarchical reservoirs (discussed in the next section) but the vanilla ESN does not have that capability. While RC systems can be robust to moderate changes in input statistics, significantly non-stationary environments can degrade performance. How best to adapt or reconfigure the reservoir to maintain the echo state property and adequate capacity for newly emerging patterns is still not fully resolved \cite{Jaeger2007,ehlers2025improving}.

 \item 
\textit{Current state of the art:} It should be noted that despite many successes, ESNs (and RC in general) have not achieved the same breadth of adoption as deep learning methods for large-scale tasks. The  tasks like language modeling, large-scale speech recognition, etc., are dominated by fully-trained models. RC has found more of a niche in smaller scale or specialized applications. A 2024 perspective noted that “large-scale industry-wide adoption of RC or a broadly convincing killer application beyond lab experiments are still not available”, despite its potential \cite{yan2024emerging}. This suggests that while RC is powerful, there remain gaps in scalability and integration that research is still addressing.

\end{enumerate}

\subsection{Future Research Directions}

As ESNs and reservoir computing move forward, several challenges and open research directions remain:
\begin{enumerate}
   
 \item 
\textit{Scalability and Large-Scale Reservoirs:} While ESNs can handle reservoirs of a few thousand neurons easily in software, scaling to millions of neurons (approaching brain-scale or very high-dimensional problems) is non-trivial. Memory and computation become an issue, and the random connectivity might need structure to be efficient. Future research may explore sparse reservoir representations, streaming implementations, or harnessing GPUs/TPUs better for reservoir simulation. On the hardware side, photonic and neuromorphic reservoirs offer a way to scale up without a huge energy cost, but connecting them to standard machine learning pipelines is a challenge. The community is interested in “reservoir chips” that could plug into computers as accelerators for temporal data processing.

 \item 
\textit{Improving Adaptability:} One limitation discussed is that fixed reservoirs don’t adapt to changing conditions or new tasks. A future direction is developing adaptive reservoirs that can evolve over time with simple rules, without losing the fast training advantage. There is already work on gating reservoirs or using feedback signals to slowly adjust reservoir weights in a unsupervised way (a bit like how the brain might adapt synapses while largely maintaining overall function). Another idea is meta-learning the reservoir: use optimization to find a good initial reservoir (or distribution of reservoirs) such that for a class of tasks, minimal output training yields good results. This could be done by evolutionary algorithms or differentiable methods (treating the reservoir as meta-parameters). Some recent studies have treated the random seed of an ESN as part of the hyperparameter search, effectively learning the graph topology or weight distribution that works best for a problem. Much of the research has focused on standard ESNs with simple \(\tanh\) or ReLU activations. Investigations into alternative, possibly domain-specific activation functions or alternative connectivity patterns (beyond sparse random) are promising yet still in progress \cite{grigoryeva2018echo}. RC’s reliance on a fixed, high-dimensional mapping of input sequences parallels the concept of random feature mappings in kernel methods \cite{Rahimi2008}. Hybrid approaches that combine reservoir dynamics with kernel-based learning may unify insights from both fields \cite{grigoryeva2018echo}. Though the reservoir is typically fixed or partially adapted, new methods for constrained reservoir optimization (e.g., spectral norm constraints, Jacobian-based regularization) can systematically balance stability (echo state property) and expressivity (high nonlinearity) \cite{Manjunath2013}. Future work can adapt interpretability techniques from feedforward networks (e.g., saliency maps, layer-wise relevance propagation) to recurrent reservoir structures. Such progress may illuminate how internal states encode temporal information and how these states evolve \cite{Tani2021,Solms2019}.

 \item 
\textit{Theoretical understanding:} Despite two decades of study, there is still more to learn theoretically about reservoir computing. For instance, a deeper understanding of the trade-off between memory and nonlinearity in reservoirs (often characterized by measures like memory capacity and Lyapunov exponents) could guide better design. The role of spectral radius vs. actual system Lyapunov spectrum needs more clarity – e.g., recent work examines reservoirs at the edge of chaos where spectral radius is ~1 but still stable in some driven sense. Also, the effect of various activation functions or neuron models on computational capacity is an open area; most ESNs use $\tanh$, but what if one uses ReLUs or other modern activations? Preliminary studies suggest ESNs with ReLU can work but may require different initialization (since ReLU is unbounded). Another theoretical frontier is understanding multi-layer (deep) reservoirs – can we derive echo state property conditions and capacity theorems for them? Some progress has been made, but a full theory analogous to deep neural network expressivity is not there yet for RC. Moreover, there’s interest in linking RC to other areas like control theory and filter theory: seeing reservoir computers as nonlinear filters provides a different mathematical lens (e.g. Volterra series expansions) to quantify what they compute.

 \item 
\textit{Benchmarking and standardization:} As reservoir computing diversifies, establishing standard benchmarks and evaluation methods is important. A recent review by Yilmaz et al. (2023) aimed to “review and critique the evaluation methods used in RC” and proposed a taxonomy of benchmark tasks. They emphasize the need for consistent metrics and tasks to compare RC models to deep learning models. Future work might involve community challenges or datasets specifically designed to push RC to its limits (similar to how image competitions drove CNN development). Especially with the rise of deep learning, it’s important to show clearly where RC excels (e.g. low data regime, low compute regime, certain signal types) and where it needs improvement, in order to justify further research and adoption.
\item
\textit{Integration with other learning paradigms:} One can imagine hybrid systems where an ESN is one component among many. For example, using an ESN as a front-end for a Reinforcement Learning (RL) agent dealing with partial observability: the ESN can maintain a hidden state embedding of the observation history, and a reinforcement learner (DQN, policy gradient, etc.) works on the ESN’s state. This can give the RL agent memory without having to use an LSTM (some works have tried this and found ESN front-ends both faster to train and sometimes more stable). Another integration is with probabilistic models: treating the ESN state as latent variables in a state-space model and performing probabilistic filtering. There’s growing interest in Bayesian echo state networks or ESNs that output uncertainty estimates, which could be valuable for critical applications like weather forecasting
 or medical trend prediction.

  \item 
\textit{Applications in new domains:} RC might find fertile ground in domains not yet fully explored. One such area is robotics – using ESNs as low-level controllers or sensor fusion modules for robots. Since ESNs are lightweight, they could run on robot hardware in real-time, and their ability to handle temporal patterns makes them good for things like predictive motor control or processing event-based camera streams. Indeed, some works show ESNs controlling robot arms or legged locomotion by learning from demonstration, with good generalization. Another domain is biology and medicine: apart from EEG, RC can be used for ECG analysis, detecting anomalies in heartbeat rhythms, or processing DNA sequences (treating them as sequences with certain patterns). The Internet of Things (IoT) is also promising: many IoT devices have to make sense of sensor time-series (temperature, vibration, etc.) with very limited computing resources – an ESN could be embedded in firmware to provide local analytics with minimal memory and power.

 \item 
\textit{Neuromorphic hardware and analog computing:} As technology moves beyond Moore’s law, RC is well-poised to leverage new hardware. The future might see optical reservoir processors as standalone devices for ultrafast signal processing (some prototypes exist, using delay-coupled lasers as reservoirs). Likewise, quantum reservoir computing has been proposed, where a quantum system’s exponentially large state space acts as a reservoir; small experiments have been done (like nuclear spin systems as reservoirs). While quantum RC is in infancy, it suggests even fundamental physics could be harnessed for computing in this way. Neuromorphic chips that emulate spiking networks (e.g. with memristive synapses) could implement massive LSMs with very low power – future research will likely focus on algorithms to program/training such systems effectively (perhaps adjusting input scaling or neuron biases to tune them). An emerging challenge is how to co-design hardware and reservoir: rather than treating hardware non-idealities as a quirk, design the reservoir shape (topology, weights) to fit well with the hardware’s strengths and weaknesses. This requires collaboration between materials scientists, physicists, and ML researchers. Biologically inspired neurons and spike-based computation are natural for implementing Liquid State Machines in neuromorphic hardware. Quantum reservoir computing, wherein the reservoir is realized by a quantum system, may enable exponential state-space expansions under certain conditions \cite{Fujii2017}. These directions could open up new application domains requiring extreme parallelism or quantum-based speedups.

 \item 
\textit{Interdisciplinary applications:} Reservoir computing has started to intersect with fields like fluid dynamics (for identifying flow regimes from sensor streams), power systems (grid load forecasting and anomaly detection), and even art (e.g. generating music sequences or interactive installations reacting to live inputs). Each new domain might pose unique challenges – e.g. in fluid dynamics, inputs might be high-dimensional fields requiring specialized input layering; in creative AI, the integration of human feedback might require adaptive readouts. The versatility of ESNs to serve as a component in these systems will be tested. If successful, it could broaden the impact of RC significantly.

\end{enumerate}

\section{Closing Remarks}\label{sec:conclusion}

Reservoir computing has emerged as a powerful framework for processing temporal and dynamical data, offering:
a \emph{computationally lightweight} training regime (linear or ridge regression on readout weights),
 \emph{rich dynamical capabilities} through large-scale recurrent networks that need minimal (or no) training internally, and
\emph{robustness} to moderate noise and input perturbations, thanks to the echo state property.

Its applications—ranging from time series forecasting and signal processing to control and neuroscience—demonstrate how the reservoir’s high-dimensional state space can embed complex temporal features so that a simple linear output layer can yield accurate predictions or classifications. Nevertheless, challenges remain in reservoir design, scalability to very large problem sizes, interpretability of internal states, and adaptation to non-stationary settings. Addressing these gaps requires continued research at the intersection of dynamical systems theory, stochastic optimization, and neuromorphic hardware design. 
Ergo, reservoir computing stands as a compelling paradigm that unites simplicity in training with the power of recurrent dynamics. Its relevance seems poised to grow, particularly as data-driven methods continue to converge with insights from neuroscience, quantum physics, and broader machine learning. 
The future holds exciting possibilities: we may see ESNs operating at the heart of neuromorphic chips, collaborating with deep networks in hybrid AI systems, and shedding light on how natural brains compute. Addressing current challenges of scalability, adaptability, and integration will be key. With ongoing research bridging theory and hardware, ESNs are likely to remain a relevant and evolving paradigm in sequential data processing. The reservoir of ideas is far from running dry – in fact, it is echoing into new territories of science.

\vspace{0.5em}

\begin{center}
\begin{tikzpicture}
\draw[gray, thick] (0,0)--(2,0);
\draw[gray, thick] (3,0) circle (0.1);
\draw[gray, thick] (4,0)--(6,0);
\end{tikzpicture}
\end{center}

\vspace{0.5em}

\section*{Acknowledgements} We express our deep appreciation to \textsc{Dr.\ Anima Nagar} (Department of Mathematics, IIT Delhi) for her sustained and stimulating discussions that have profoundly influenced the symbolic and dynamical perspectives of this work. In particular, her pioneering contributions to the theory of hyperspace dynamics \cite{nagar2011combined,nagar2014reflections,nagar2017dynamics,nagar2016variations,nagar2022revisiting} have served as an intellectual scaffold for examining the chaotic evolution and structural intricacies central to several of our other neural-network frameworks. Her pedagogical depth has enriched not only the technical content of our study but also its conceptual orientation toward structurally invariant characterizations of complex dynamics in recurrent architectures.

We are also indebted to \textsc{Dr.\ Joseph Auslander} (Department of Mathematics, University of Maryland) for his generous insights into the foundational constructs of topological dynamics. His seminal work on minimal flows, enveloping semigroups, and proximality \cite{auslander1988minimal,akin2003almost} has provided a  theoretical bedrock upon which many of our formal results on chaotic recurrence and topological invariants have been developed. The influence of his mathematical vision—particularly his treatment of distal and equicontinuous dynamics—resonates with the structural properties of the attractors we analyze, as well as with the methodological frameworks we introduce to unify dynamical-systems theory and neural computation in forthcoming investigations.

\bibliographystyle{unsrt}
\bibliography{ref}

\appendix
\section{Appendix}

\paragraph{Spectral Theory for the Reservoir Matrix.}
A central object in reservoir computing is the \emph{reservoir weight matrix} \(\mathbf{W}_{\mathrm{res}}\in \mathbb{R}^{N\times N}\), sometimes termed the “adjacency matrix” when viewed as a weighted graph.  The \emph{spectrum} of \(\mathbf{W}_{\mathrm{res}}\) refers to the set of eigenvalues \(\{\lambda_i\}\subset \mathbb{C}\).  We denote the \emph{spectral radius} by
\begin{equation}
\rho(\mathbf{W}_{\mathrm{res}}) 
\;=\; \max_{\lambda\in\sigma(\mathbf{W}_{\mathrm{res}})} |\lambda|,
\end{equation}
where \(\sigma(\mathbf{W}_{\mathrm{res}})\) denotes the spectrum of eigenvalues.  A fundamental result from linear operator theory (see Theorem 5.6.9 in \cite{HornJohnson2013} and Chapter IV in \cite{Kato1995}) states that
\begin{equation}
\rho(\mathbf{W}_{\mathrm{res}}) 
\;=\;
\lim_{k\to \infty} \;\bigl\|\mathbf{W}_{\mathrm{res}}^k\bigr\|^{1/k}
\end{equation}
for any submultiplicative matrix norm \(\|\cdot\|\).  This asymptotic characterization, known as the \emph{Gelfand formula}, is especially relevant to reservoir computing when investigating conditions for stability or ESP.  In practice, one computes \(\rho(\mathbf{W}_{\mathrm{res}})\) either by direct eigen-decomposition (via standard numerical linear algebra methods such as QR-based algorithms) or by approximate methods like the power iteration if \(N\) is very large \cite{TrefethenBau1997}.  When \(\rho(\mathbf{W}_{\mathrm{res}})<1\), we expect the associated dynamical system to exhibit fading memory, provided the nonlinearity satisfies a Lipschitz criterion, as discussed in \S~\ref{thm:esp}.

\paragraph{Matrix Norms and Submultiplicativity.}
Many of the key Lipschitz and contractive properties hinge on choosing a suitable matrix norm.  A matrix norm \(\|\mathbf{A}\|\) on \(\mathbb{R}^{N\times N}\) is \emph{submultiplicative} if 
\(
\|\mathbf{A}\,\mathbf{B}\|\le \|\mathbf{A}\|\;\|\mathbf{B}\|\)
for all \(\mathbf{A},\mathbf{B}\).  Common examples include:
\begin{equation}
\|\mathbf{A}\|_2 
\;=\;
\sqrt{\rho\!\bigl(\mathbf{A}^T \mathbf{A}\bigr)},
\quad
\|\mathbf{A}\|_1 
\;=\; 
\max_{1\le j \le N}\sum_{i=1}^N |a_{ij}|,
\quad
\|\mathbf{A}\|_\infty 
\;=\; 
\max_{1\le i \le N}\sum_{j=1}^N |a_{ij}|.
\end{equation}
Each provides a different but consistent way to measure operator size \cite{HornJohnson2013}.  In reservoir computing, one frequently picks a norm aligned with the activation function’s domain.  For example, if \(\mathbf{x}\in \ell^\infty\), the matrix \(\ell^\infty\)-induced norm \(\|\mathbf{W}_{\mathrm{res}}\|_\infty\) conveniently bounds \(\|\mathbf{x}\|\infty\)-type quantities.

\paragraph{Perturbation Bounds and Stability.}
To quantify how small changes in \(\mathbf{W}_{\mathrm{res}}\) or the nonlinearity affect the spectral radius, standard results in perturbation theory can be invoked \cite{Kato1995}.  For instance, if \(\mathbf{E}\) is a small perturbation, one can often bound \(\rho(\mathbf{W}_{\mathrm{res}}+\mathbf{E})\) relative to \(\rho(\mathbf{W}_{\mathrm{res}})\) via inequalities that involve \(\|\mathbf{E}\|\).  This is pivotal in analyzing the structural stability of reservoir networks under parametric changes, including random initialization seeds and scaling factors.

\paragraph{Lipschitz Continuity in Nonlinear Mappings.}
For a componentwise nonlinear activation \(F:\mathbb{R}^N\to\mathbb{R}^N\), Lipschitz continuity means that for all \(\mathbf{z}_1,\mathbf{z}_2 \in \mathbb{R}^N\),
\begin{equation}
\|F(\mathbf{z}_1)\,-\,F(\mathbf{z}_2)\|
\;\le\; 
L\,\|\mathbf{z}_1 - \mathbf{z}_2\|.
\end{equation}
The smallest such \(L\) is the \emph{global Lipschitz constant}.  If \(F\) is differentiable, one often estimates \(L\) via the supremum of the operator norm of \(D F(\mathbf{z})\) over the relevant domain.  In echo state networks, combining Lipschitz continuity of \(F\) with the condition \(\rho(\mathbf{W}_{\mathrm{res}})<1/L\) under a chosen matrix norm guarantees a global contraction on state updates, thereby proving the ESP \cite{jaeger2001echo, Lukosevicius2012}.

\paragraph{Random Matrix Theory for Reservoir Initialization.}
Random matrices arise naturally when the reservoir weights \(\mathbf{W}_{\mathrm{res}}\) and \(\mathbf{W}_{\mathrm{in}}\) are drawn from a probability distribution with mean zero and specified variance.  Large deviation principles or asymptotic eigenvalue distributions (e.g., the Wigner semicircle law for certain symmetric ensembles) sometimes give insight into typical spectral properties \cite{Tao2012}.  In practice, one often rescales a random Gaussian or uniform matrix so that \(\rho(\mathbf{W}_{\mathrm{res}})\approx \alpha\), a desired spectral radius near unity.  This random matrix viewpoint explains why, even without training internal weights, the reservoir can exhibit a rich repertoire of high-dimensional dynamics \cite{Rodan2011}.  

\paragraph{Deterministic vs.\ Stochastic Dynamical Systems.}
From a dynamical systems perspective, echo state networks and liquid state machines can be analyzed using results from both deterministic mapping theory (e.g.\ discrete-time iterations in \(\mathbb{R}^N\)) and the theory of stochastic processes if noise is present \cite{Arnold1998}.  When inputs include random components, the reservoir states form a Markov chain in \(\mathbb{R}^N\).  Ergodic theorems then ensure that time averages of certain state functionals converge almost surely to ensemble averages.  This viewpoint is sometimes exploited to argue the existence of well-defined long-term statistical behaviors in the reservoir, particularly when investigating properties like fading memory or mixing.

\paragraph{Contraction Mappings and Banach Fixed-Point Theorem.}
The proofs of the Echo State Property and Fading Memory Property frequently rely on contraction mapping principles \cite{OrtegaRheinboldt1970}.  If an operator \(\mathcal{T}\) mapping states (or entire input streams) to new states is shown to be a strict contraction in an appropriate metric, then a unique fixed point must exist and be globally attractive.  This is one rigorous lens through which to view why repeated application of \(\mathbf{W}_{\mathrm{res}}\)- and \(F\)-based updates eventually ``forgets'' initial conditions, leading to a unique solution consistent with the input sequence.   For the Banach fixed-point theorem in the context of neural networks, see the discussion in \cite{OrtegaRheinboldt1970}.

\paragraph{State-Space and Phase-Space Geometry.}
Certain analyses of echo state networks adopt the geometric language of dynamical systems, with the \emph{reservoir state} \(\mathbf{x}\in \mathbb{R}^N\) viewed as a point in phase space.  Understanding \(\mathbf{W}_{\mathrm{res}}\) as defining a set of attractors or limit cycles in the absence of input, one can then analyze how external inputs drive trajectories through (or near) those attractors \cite{strogatz2018nonlinear}.  Of particular interest are stable manifolds (if the autonomous system were stable) and how the input modifies local dynamics.  Phase-space geometry also helps interpret advanced metrics (like attractor deviation in chaotic tasks) by revealing topological properties of orbits that the reservoir aims to reproduce.

\paragraph{Regularized Readout Training and Pseudoinverse Solutions.}
A typical echo state network trains only the output weights \(\mathbf{W}_{\mathrm{out}}\in\mathbb{R}^{d\times (N+1)}\) (including bias) to map reservoir states \(\mathbf{x}(t)\in \mathbb{R}^N\) to desired outputs \(\mathbf{y}(t)\in \mathbb{R}^d\).  One canonical method is \emph{ridge regression} (also referred to as Tikhonov regularization), which seeks to minimize
\begin{equation}
\mathcal{L}(\mathbf{W}_{\mathrm{out}}) 
\;=\;
\sum_{t=1}^T \bigl\|\mathbf{y}(t) - \mathbf{W}_{\mathrm{out}}\,[\,1;\,\mathbf{x}(t)\bigr\|^2 
\;+\;\beta\,\|\mathbf{W}_{\mathrm{out}}\|^2_F,
\end{equation}
where \(\beta>0\) is the regularization parameter and the notation \([\,1;\,\mathbf{x}(t)\bigr]\) indicates that the bias is included as an extra dimension in the state vector.  Assuming the reservoir states have been collected into a design matrix \(\mathbf{X}\in\mathbb{R}^{(N+1)\times T}\) and target outputs form \(\mathbf{Y}\in\mathbb{R}^{d\times T}\), the problem becomes
\begin{equation}
\min_{\mathbf{W}_{\mathrm{out}}} \;\bigl\|\mathbf{Y} - \mathbf{W}_{\mathrm{out}}\;\mathbf{X}\bigr\|_F^2 
\;+\;
\beta\,\|\mathbf{W}_{\mathrm{out}}\|_F^2.
\end{equation}
Standard linear algebra (see Chapter 5 of \cite{GolubVanLoan2013}) shows that the unique solution is given by
\begin{equation}
\mathbf{W}_{\mathrm{out}}
\;=\;
\mathbf{Y}\,\mathbf{X}^T\bigl(\mathbf{X}\,\mathbf{X}^T + \beta\,\mathbf{I}\bigr)^{-1},
\end{equation}
provided \(\mathbf{X}\,\mathbf{X}^T + \beta\,\mathbf{I}\) is invertible.  When \(\beta=0\), this degenerates to the Moore–Penrose pseudoinverse solution \(\mathbf{W}_{\mathrm{out}} = \mathbf{Y}\,\mathbf{X}^+\).  For typical echo state network sizes, one has \(T\gg N\), so \(\mathbf{X}\in\mathbb{R}^{(N+1)\times T}\) is “tall” in the time dimension, facilitating stable least-squares solutions.  In smaller data regimes or if \(\rho(\mathbf{W}_{\mathrm{res}})\approx 1\), ridge regularization remains essential to avoid ill-conditioning.  The parameter \(\beta\) is often chosen via cross-validation, balancing fidelity to the training data against the risk of overfitting.  One can also employ more advanced regularizers (like LASSO or elastic net) if one wishes to prune readout components and induce sparsity \cite{HastieTibshiraniWainwright2022}; however, these lead to non-closed-form solutions that require iterative solvers.

\paragraph{Recursive Least Squares (RLS) and Online Adaptation.}
Beyond the offline solution, echo state networks can be trained incrementally via \emph{recursive least squares} (RLS).  Given a stream of pairs \((\mathbf{x}(t), \mathbf{y}(t))\), one maintains an evolving estimate of \(\mathbf{W}_{\mathrm{out}}(t)\) and an inverse correlation matrix \(\mathbf{P}(t)\).  The well-known RLS recursions (cf.\ Section 9.2 of \cite{Haykin2002}) are
\begin{equation}
\mathbf{z}(t)
\;=\;
\mathbf{P}(t-1)\,\mathbf{x}(t),
\quad
g(t)
\;=\;
\frac{\mathbf{z}(t)}{1 + \mathbf{x}(t)^T\,\mathbf{z}(t)},
\quad
\mathbf{W}_{\mathrm{out}}(t)
\;=\;
\mathbf{W}_{\mathrm{out}}(t-1)
\;+\;
\bigl[\mathbf{y}(t) - \mathbf{W}_{\mathrm{out}}(t-1)\,\mathbf{x}(t)\bigr]\;g(t)^T,
\end{equation}
\begin{equation}
\mathbf{P}(t)
\;=\;
\mathbf{P}(t-1)
\;-\;
g(t)\,\mathbf{x}(t)^T\,\mathbf{P}(t-1).
\end{equation}
Such a scheme updates the readout weights and covariance matrix at each time step with \(\mathcal{O}(N^2)\) complexity, making it feasible for online learning.  One can include forgetting factors or regularization in the recursion, matching real-time or nonstationary data scenarios.  The RLS approach preserves the interpretability and convexity of a linear readout but accommodates streaming data, extending echo state networks to scenarios requiring ongoing adaptation (e.g.\ time-varying systems or continual learning).

\paragraph{Transient Washout and State Initialization.}
A subtle but important step in reservoir computing training is the \emph{washout} or \emph{initial transient} period.  Because different initial states \(\mathbf{x}(0)\) can linger, it is standard to discard the initial \(\tau_{\text{wash}}\) time steps of reservoir states when forming \(\mathbf{X}\) and \(\mathbf{Y}\) for regression \cite{jaeger2001echo, Lukosevicius2012}.  The choice of \(\tau_{\text{wash}}\) depends on the magnitude of \(\rho(\mathbf{W}_{\mathrm{res}})\) and input scaling; typically, one wants \(\tau_{\text{wash}}\) large enough so that any dependence on \(\mathbf{x}(0)\) is negligible.  Formally, for a contractive reservoir (Lipschitz activation, \(\rho(\mathbf{W}_{\mathrm{res}})<1/L\)), the difference between two states driven by the same input but different initial conditions decays exponentially, as established in the ESP proof.  One ensures that the reservoir trajectory after \(\tau_{\text{wash}}\) closely approximates the unique solution consistent with the input history, thereby enabling consistent regression across all training samples.  In real-time applications, one can either fix a standard washout window at the beginning of a new input sequence or forcibly reset the reservoir to a known state after each trial.

\paragraph{Stationarity Assumptions and Input Preprocessing.}
In many reservoir computing analyses, it is tacitly assumed that the input sequence \(\mathbf{u}(t)\) is \emph{stationary}, or at least piecewise stationary.  Stationarity implies stable statistical properties such as constant mean and covariance over time, aligning with the requirement that the reservoir’s contractive dynamic plus a bounded forcing yields a stable orbit in state space.  If \(\mathbf{u}(t)\) exhibits significant nonstationarity, such as abrupt changes in amplitude or distribution, the reservoir might need additional normalization or an adaptive strategy.  Preprocessing steps often include subtracting the empirical mean from \(\mathbf{u}(t)\) or scaling it to a certain range (e.g.\ \([-1,1]\) for \(\tanh\) activation).  Such transformations help keep the reservoir state within the regime where the Lipschitz assumption is not severely violated.  A rigorous approach sometimes uses piecewise analysis of the input distribution, bounding how changes affect \(\|\mathbf{W}_{\mathrm{in}}\,\mathbf{u}(t)\|\) and ensuring the matrix spectral constraint continues to hold in each piecewise segment.

\paragraph{Generalization Bounds and Capacity Metrics.}
Although reservoir computing typically emphasizes fast training and expressive dynamics over formal learning theory, a growing body of work seeks to establish \emph{generalization guarantees}.  One might draw on results from Rademacher complexity or covering numbers for linear readouts on top of random features, extending analyses akin to random kitchen-sink expansions in kernel methods \cite{Rahimi2008}.  The core idea is that the reservoir states \(\mathbf{x}(t)\) act like randomly projected features \(\phi(\mathbf{u}(t))\), and linear function classes on such features often have well-studied generalization properties \cite{BartlettMendelson2002}.  While these proofs can become intricate in the presence of temporal dependence, they still convey that for a reservoir of size \(N\), the capacity to fit sequences scales with \(\mathcal{O}(N)\) up to logarithmic factors, provided the input distribution and the random matrix distribution align with the assumptions of random feature theory.  In practice, these theorems underscore the importance of choosing \(N\) large enough to ensure with high probability that the random reservoir forms a sufficiently rich basis for the target time-series mapping.

\paragraph{Functional Approximation and Universal Approximators.}
Recent work has shown that echo state networks can approximate a broad class of time-invariant fading memory filters \cite{grigoryeva2018echo}.  This involves heavier functional-analytic machinery: one constructs a reservoir operator \(\Phi\) in a function space of inputs and shows density arguments in a relevant topology.  In simpler terms, polynomials or random features inside the reservoir, combined with a linear readout, can approximate many Volterra- or Wiener-type functional expansions \cite{Boyd1985}.  These mathematical results unify reservoir computing with classical system identification methods, clarifying why ESNs can capture a large range of input–output dynamics without fully trained recurrent weights.

\paragraph{Memory Capacity and Orthogonal Projections.}
A key theoretical concept in reservoir computing is the \emph{memory capacity} (MC), which quantifies how effectively a reservoir can reconstruct (or recall) past inputs at its readout \cite{yildiz2012re, White2004}.  In one popular formulation, one considers a scalar input signal \(\{u(t)\}\) and a reservoir of dimension \(N\).  The linear readout is tasked with approximating delayed versions \(u(t-\tau)\) using the current reservoir state \(\mathbf{x}(t)\).  Formally, one sets up a linear regression to fit:
\begin{equation}
\hat{u}_{\tau}(t) 
\;=\;
\mathbf{w}_{\tau}^{T}\,\mathbf{x}(t),
\end{equation}
where \(\mathbf{w}_{\tau}\) is chosen to minimize the mean-squared error between \(\hat{u}_{\tau}(t)\) and the true delayed signal \(u(t-\tau)\).  If one denotes the resulting minimum MSE by \(E_{\tau}\), the memory capacity is given by
\begin{equation}
\mathrm{MC} 
\;=\;
\sum_{\tau=1}^{\infty} \Bigl(1 - \tfrac{E_{\tau}}{\mathrm{Var}[u(t)]}\Bigr),
\end{equation}
though in practice one truncates at a finite delay \(\tau_{\max}\).  By analyzing the orthogonal projection of each delayed signal onto the subspace spanned by reservoir states, one can show that \(\mathrm{MC}\le N\) in typical cases if the reservoir is purely linear \cite{jaeger2004harnessing}.  Nonlinear activation can raise or lower effective MC depending on the interplay between the spectral radius and saturations.  This projection-based definition underpins why “near-critical” reservoirs (spectral radius close to 1) often exhibit higher MC: they preserve input perturbations over many timesteps, allowing the linear readout to glean signals from multiple delays.  In advanced settings, one considers multi-dimensional inputs \(\mathbf{u}(t)\) and generalized forms of MC that measure spatiotemporal recall.  

\paragraph{Time-Lag Embedding and Takens' Theorem for Recurrent States.}
When dealing with chaotic signals or high-dimensional data, time-lag embedding (commonly used in nonlinear time-series analysis \cite{takens1981detecting, Sauer1991}) becomes relevant.  Echo state networks conceptually generate a high-dimensional “embedding” of the recent input history, albeit via random recurrent connections instead of uniform time delays.  The classical Takens’ Theorem states that, under generic conditions, one can embed the attractor of a continuous-time dynamical system in a reconstructed phase space of dimension \(d\) using delayed coordinates \(\{y(t),\,y(t-\Delta),\dots,y(t-(d-1)\Delta)\}\).  In the reservoir setting, the state \(\mathbf{x}(t)\in \mathbb{R}^N\) often serves a role akin to such a delay embedding, but with a stochastic or random linear transform in between.  Although the discrete-time update in an ESN does not strictly match the classic differential-equation hypotheses of Takens, many arguments carry over because the reservoir’s large dimensionality and fading memory effectively track the input’s orbit over a sufficiently wide range of lags.  For complex signals, the synergy of random recurrent mixing and random linear projections can approximate a near-embedding of the input’s generating dynamics, as demonstrated by universal approximation theorems in \cite{grigoryeva2018echo}.  This viewpoint clarifies why an adequately sized reservoir can “learn” chaotic attractors via readouts: the underlying manifold of states is richly sampled by the internal recursion.

\paragraph{Non-Autonomous Dynamical Systems and Pullback Attractors.}
From the perspective of \emph{non-autonomous} dynamical systems \cite{KloedenRasmussen2011}, one views the echo state network as a time-dependent iteration
\begin{equation}
\mathbf{x}(t+1) 
\;=\;
F\bigl(\mathbf{W}_{\mathrm{res}}\,\mathbf{x}(t) + \mathbf{W}_{\mathrm{in}}\,\mathbf{u}(t+1) + \mathbf{b}\bigr),
\end{equation}
where the input \(\mathbf{u}(t)\) is treated as a forcing term.  In contrast to the autonomous case \(\mathbf{u}(t)\equiv 0\), the concept of a \emph{pullback attractor} arises: for each time \(t\), one considers solutions driven by the history of inputs \(\{\mathbf{u}(s): s \le t\}\).  If the system is sufficiently contractive, these pullback solutions converge to a unique \(\mathbf{x}^{*}(t)\), capturing the idea that the state is determined by the recent input sequence rather than by initial conditions \cite{Arnold1998}.  Proving uniqueness of this pullback attractor under a spectral radius condition \(\rho(\mathbf{W}_{\mathrm{res}})<1/L\) formalizes the ESP.  It also sets the stage for studying how changes in the forcing (the input) shift the attractor geometry over time, ensuring a stable “tracking” of the non-autonomous dynamics.  Such analysis clarifies the relationship between reservoir stability and input design, highlighting the synergy of contractive feedback loops and controlled forcing in shaping robust echoes of the input history.

\paragraph{Lyapunov Exponents in Discrete Maps and Chaos Indicators.}
Lyapunov exponents are a standard tool to quantify sensitivity to initial conditions and chaotic divergence in discrete maps \cite{ott2002chaos}.  For an iterated map \(\mathbf{x}(t+1)=G(\mathbf{x}(t))\), one tracks the evolution of small perturbations via the Jacobian \(D G(\mathbf{x}(t))\).  The maximal Lyapunov exponent is given by
\begin{equation}
\lambda_{\max} 
\;=\; 
\lim_{T\to\infty} \;\frac{1}{T}\;\sum_{t=0}^{T-1} 
\ln \bigl\|\,
D G\bigl(\mathbf{x}(t)\bigr) \,\mathbf{v}_t
\,\bigr\|
\end{equation}
for some tangent vector \(\mathbf{v}_t\).  When \(\lambda_{\max}>0\), small differences grow exponentially, indicative of chaos.  In reservoir computing, if one views the internal state recursion as \(G(\mathbf{x}(t),\mathbf{u}(t))\), it may or may not be chaotic depending on \(\rho(\mathbf{W}_{\mathrm{res}})\) and the nature of \(F\).  However, when using an ESN to predict a \emph{chaotic} external system, comparing the Lyapunov exponents of the target trajectory with those induced by the autonomous (or closed-loop) reservoir provides a strong measure of dynamical fidelity, as elaborated in \S~\ref{sec:metrics} and works such as \cite{pathak2018model}.  

\paragraph{Continuous vs.\ Discrete Reservoirs and Hybrid Formulations.}
Although the majority of echo state networks are built upon discrete-time updates, some variants or related reservoir models (e.g.\ continuous-time recurrent neural networks, Liquid State Machines) employ differential equation formalisms.  One might define
\begin{equation}
\frac{d\mathbf{x}(t)}{dt}
\;=\;
-\alpha\,\mathbf{x}(t) \;+\; F\bigl(\mathbf{W}_{\mathrm{res}}\,\mathbf{x}(t) + \mathbf{W}_{\mathrm{in}}\,\mathbf{u}(t)\bigr),
\end{equation}
leading to an ODE whose stability can be analyzed via the eigenvalues of \(\mathbf{W}_{\mathrm{res}}\) and Lipschitz bounds of \(F\).  The fundamental spectral requirements, such as \(\rho(\mathbf{W}_{\mathrm{res}})<\frac{1}{L}\), still emerge but under a continuous-time contraction argument \cite{maass2002real}.  Hybrid approaches exist as well, where the continuous reservoir is sampled at discrete intervals for readout or further processing, especially in neuromorphic and spiking implementations.

\paragraph{Rosenstein’s Algorithm for the Largest Lyapunov Exponent.}
A commonly used practical technique for estimating the largest Lyapunov exponent from time series data is due to Rosenstein, Collins, and De Luca \cite{Rosenstein1993}. The algorithm begins by reconstructing the phase space of a scalar time series \(\{y_n\}_{n=1}^T\) via time-delay embedding: one chooses an embedding dimension \(m\) and time delay \(\tau\), then forms vectors
\begin{equation}
\mathbf{X}_n 
\;=\;
\bigl(y_n,\; y_{n+\tau},\; \dots,\; y_{n+(m-1)\tau}\bigr)\in \mathbb{R}^m.
\end{equation}
Under mild conditions on the original dynamical system (related to Takens’ theorem), these vectors \(\mathbf{X}_n\) preserve the essential geometry of the system’s attractor. For each embedded point \(\mathbf{X}_n\), one locates its \emph{nearest neighbor} \(\mathbf{X}_{n'}\) such that \(n'\neq n\) but \(\|\mathbf{X}_n - \mathbf{X}_{n'}\|\) is minimized over the dataset. Denoting the small initial separation between these neighboring trajectories by
\begin{equation}
d_0(n) 
\;=\;
\bigl\|\mathbf{X}_n - \mathbf{X}_{n'}\bigr\|,
\end{equation}
Rosenstein’s method tracks the temporal evolution of that separation over integer time steps \(k\). If the system’s largest Lyapunov exponent is \(\lambda_{\max}\), then small perturbations grow on average according to
\begin{equation}
d_k(n) 
\;=\; 
\bigl\|\mathbf{X}_{n + k} - \mathbf{X}_{n' + k}\bigr\|
\;\approx\;
d_0(n)\,\exp\bigl(\lambda_{\max}\,k\,\Delta t\bigr),
\end{equation}
where \(\Delta t\) is the sampling time step. One takes a \emph{logarithm} and then averages over all valid pairs \(n,\,n'\) to obtain
\begin{equation}
\left\langle \ln d_k \right\rangle 
\;\approx\;
\ln d_0 \;+\; \lambda_{\max}\,k\,\Delta t.
\end{equation}
The slope of the linear region of \(\bigl\langle \ln d_k \bigr\rangle\) versus \(k\) yields an empirical estimate of \(\lambda_{\max}\). In practice, one must choose an embedding dimension \(m\) sufficiently large to unfold the attractor, but not so large as to exhaust the available data. One also restricts the time horizon \(k\) to remain within the interval where the growth of \(\ln d_k\) is nearly linear before nonlinear saturation or trajectory folding occur. The method’s chief advantage is its relative simplicity and robustness to modest amounts of data, while still capturing the exponential divergence rate of nearby trajectories.

\paragraph{Wolf’s Algorithm for Phase-Space Neighbor Tracking.}
An earlier and more elaborate procedure devised by Wolf, Swift, Swinney, and Vastano \cite{Wolf1985} focuses on tracking local divergence in a phase-space embedding through a step-by-step re-initialization scheme. As in Rosenstein’s method, one begins with an embedded time series \(\{\mathbf{X}_n\}\). To estimate the largest Lyapunov exponent, one chooses an initial reference point \(\mathbf{X}_{n_0}\) and locates its nearest neighbor \(\mathbf{X}_{n_1}\). Denote their separation by 
\begin{equation}
d_0 \;=\; \|\mathbf{X}_{n_0} - \mathbf{X}_{n_1}\|.
\end{equation}
The algorithm then evolves both trajectories forward in time. If \(\mathbf{X}_{n_0+k}\) and \(\mathbf{X}_{n_1+k}\) are the respective successors, one measures \(d_k = \|\mathbf{X}_{n_0 + k} - \mathbf{X}_{n_1 + k}\|\).  Provided \(d_k\) remains sufficiently small (staying within the locally linear regime of the attractor), its growth rate approximates
\begin{equation}
d_k \;\approx\; d_0 \,\exp\!\bigl(\lambda_{\max}\,k\,\Delta t\bigr).
\end{equation}
Once \(d_k\) grows beyond some threshold (or reaches a set time window), Wolf’s algorithm \emph{renormalizes} the separation back to a small distance by picking a new neighbor in the vicinity of \(\mathbf{X}_{n_0 + k}\).  Doing so repeatedly over a long trajectory yields a piecewise measurement of exponential divergence, from which an average is taken.  Formally, if the re-initializations happen at times \(k_1,k_2,\dots,k_r\), one obtains segment-wise expansions 
\begin{equation}
d_{k_j} \;=\; d_0 \,\exp\!\bigl(\lambda_{\max}\,(k_j - k_{j-1})\,\Delta t\bigr),
\end{equation}
and the final Lyapunov exponent estimate is an average of the slopes \(\frac{1}{k_r - k_0}\sum_j \ln(d_{k_j}/d_0)\).  This re-initialization step helps to follow diverging trajectories while keeping the local linear approximation valid, but it can require careful parameter tuning (thresholds for “too large” separation) and a well-sampled attractor.

\paragraph{Phase-Space Reconstruction Details and Practical Considerations.}
Both Wolf’s and Rosenstein’s algorithms rely on the general procedure of \emph{phase-space reconstruction} or \emph{delay embedding}.  Given a one-dimensional raw signal \(\{y(t)\}\), one constructs \(\mathbf{X}_n=(\,y_n,\dots,y_{n+(m-1)\tau}\,)\) in \(\mathbb{R}^m\).  One typically selects \(\tau\) based on partial mutual information or autocorrelation drop-off, and picks \(m\) until the attractor “unfolds” (for instance, via false nearest neighbor tests).  After reconstruction, the data \(\{\mathbf{X}_n\}_{n=1}^T\) is presumed to lie on or near a manifold of dimension \(D\).  When the true system dimension \(D\) is unknown, the practical approach is to increment \(m\) until the largest Lyapunov exponent estimate stabilizes \cite{Kantz2004}.  

Because these methods estimate exponent(s) from finite data, they can be sensitive to measurement noise, short data length, or choice of parameters (embedding dimension, time delays, neighbor search radii).  Nonetheless, if implemented carefully, Rosenstein’s and Wolf’s algorithms remain cornerstones for extracting \(\lambda_{\max}\) from experimental or numerical time series exhibiting chaos.  In the context of reservoir computing, such methods are valuable for validating whether an ESN has learned a chaotic attractor accurately, by comparing the exponent(s) of the reservoir-driven trajectory to that of the original system.  A near match in \(\lambda_{\max}\) implies the reservoir replicates the chaotic divergence rate—a strong indicator of dynamical fidelity \cite{pathak2017using}.

\paragraph{Intrinsic Plasticity and Homeostatic Adaptation.}
A notable extension to classical echo state networks is allowing the reservoir neurons to undergo mild, unsupervised parameter updates that maintain an optimal dynamic range.  In particular, \emph{intrinsic plasticity} (IP) modifies each neuron’s bias or slope to achieve a target firing-rate distribution, typically chosen to be exponential or Gaussian \cite{Triesch2005}.  Denoting the reservoir activation function by \(f(\cdot)\) and letting each neuron \(i\) have parameters \(\alpha_i,\;\beta_i\), one writes
\begin{equation}
x_i(t+1) 
\;=\;
f\bigl(\alpha_i\;\mathbf{w}_i \cdot \mathbf{x}(t) + \beta_i\bigr),
\end{equation}
and adjusts \(\alpha_i,\beta_i\) via gradient-based or entropy-based rules to match a specified activity distribution \(\rho(x)\).  One canonical update is:
\begin{equation}
\Delta \beta_i \;\propto\; \mu\,(1 - 2\,\langle x_i\rangle),  
\quad
\Delta \alpha_i \;\propto\; \mu\,\bigl(\tfrac{\langle x_i\rangle}{\sigma_i^2} - \tfrac{1}{x_i}\bigr),
\end{equation}
averaged over time, where \(\mu\) is a small learning rate \cite{schrauwen2007overview}.  This procedure locally reshapes the neuron’s input-output curve, effectively spreading the reservoir states across a broader portion of phase space.  From a mathematical standpoint, one may interpret IP as maintaining the reservoir’s state distribution within a high-entropy regime, akin to the “edge-of-chaos” ideal in random dynamical systems.  While IP does not alter the fixed recurrent weights \(\mathbf{W}_{\mathrm{res}}\) themselves, it systematically prevents saturations or dead neurons, thereby preserving expressivity and memory capacity over extended training periods.

\paragraph{Excitatory–Inhibitory Balance and Criticality.}
In neurobiologically inspired reservoir models, such as liquid state machines with spiking neurons, the notion of \emph{balanced} E/I connectivity (a ratio of excitatory to inhibitory synapses) plays a crucial role in achieving high computational capacity \cite{Boerlin2013}.  A network is said to be excitatory–inhibitory balanced if for each neuron, the net incoming synaptic input oscillates around zero—positive (excitatory) and negative (inhibitory) inputs effectively cancel out on average.  One can model this in a rate-based ESN by imposing that each row of \(\mathbf{W}_{\mathrm{res}}\) has a constrained fraction of negative entries and a constraint on the sum
\begin{equation}
\sum_{j} w_{ij} \;\approx\; 0
\quad\text{for each neuron }i,
\end{equation}
potentially with an additional scaling to keep \(\rho(\mathbf{W}_{\mathrm{res}})\) near \(1\).  Recent mathematical studies show that such balanced random networks often operate near a critical point between quiescent and chaotic regimes, akin to analyses of random neural circuits in \cite{Sompolinsky1988}.  Formally, one can derive conditions on the variance of positive and negative weights that yield a leading eigenvalue close to the imaginary axis, marking the onset of chaotic dynamics.  This balancing is highly relevant to large-scale memory properties: near-critical balanced reservoirs may exhibit extended correlation times, increasing their memory capacity and expressive power while still remaining stable enough to satisfy echo state criteria in practical settings.

\paragraph{Bifurcation Analysis of the Reservoir Dynamics.}
Another rigorous tool for studying stability and transitions in echo state networks is \emph{bifurcation analysis} of the map
\begin{equation}
\mathbf{x}(t+1) 
\;=\;
F\!\Big(\mathbf{W}_{\mathrm{res}}\,\mathbf{x}(t) + \mathbf{b}\Big),
\end{equation}
considered in the absence of input (\(\mathbf{u}(t)\equiv 0\)).  As parameters (e.g.\ spectral radius, biases, or leak rates) vary, one can detect critical thresholds at which the map transitions from a stable fixed point to oscillatory or chaotic regimes \cite{Savi2006}.  In principle, if \(\rho(\mathbf{W}_{\mathrm{res}})\) crosses certain values, the Jacobian \(D F(\mathbf{x}^*)\) at a fixed point \(\mathbf{x}^*\) may acquire eigenvalues with magnitude exceeding 1, indicating a loss of stability.  One can track changes in the number or nature of attractors by analyzing the map’s derivative structure and computing eigenvalues of \(\mathrm{Id} - \mathbf{W}_{\mathrm{res}} \cdot Df(\mathbf{x})\).  Although such a purely autonomous analysis does not fully capture input-driven dynamics, it still reveals the reservoir’s “intrinsic” dynamical tendencies and helps explain how operating near these bifurcation points can yield rich transients and high-dimensional state trajectories under mild input forcing.

\paragraph{Spiking Extensions: Integrate-and-Fire Reservoirs.}
While the original ESN uses continuous activations \(f(\cdot)\) such as \(\tanh\), a growing line of research employs \emph{spiking} neurons in the reservoir, yielding a class of models known as \emph{liquid state machines} (LSMs) \cite{maass2002real}.  Each neuron \(i\) in the reservoir obeys dynamics akin to
\begin{equation}
\dot{v}_i(t) 
\;=\;
-f_{\text{reset}}\bigl(v_i(t)\bigr)
\;+\;
\sum_{j} w_{ij}\,S_j(t),
\end{equation}
where \(v_i\) is the membrane potential, \(S_j(t)\) indicates spike train inputs from neuron \(j\), and \(f_{\text{reset}}\) enforces resets and refractory behavior upon spiking.  The matrix \(\mathbf{W}_{\mathrm{res}}\) connects these neurons with random excitatory and inhibitory synapses.  Formally, one still obtains a high-dimensional dynamical system, but now in the space of spike events rather than continuous rates.  \emph{Computationally}, the output readout integrates or filters the reservoir’s spike trains over time.  A rigorous analysis of LSMs typically invokes measure-theoretic arguments for counting or bounding spike events, or uses a continuous-time version of fading memory results when the network’s “leak” or membrane decay ensures old inputs vanish exponentially \cite{Buesing2010}.  The same concept of \(\rho(\mathbf{W}_{\mathrm{res}})\) can be extended by linearizing about a mean firing rate, revealing that if the effective recurrent coupling is too large, spiking activity can become pathologically synchronous or unbounded, thereby destroying reliable echo states.  By properly scaling weights and ensuring a critical balance of excitatory and inhibitory currents, one obtains a stable yet computationally expressive spiking reservoir akin to the standard ESN’s “edge-of-chaos” principle.

\paragraph{Hyperparameter Optimization and Bayesian Methods.}
Because reservoir computing performance can be sensitive to hyperparameters such as spectral radius, input scaling, leak rate, connectivity density, and regularization strength \(\beta\), recent approaches automate this selection via \emph{Bayesian optimization} or evolutionary strategies \cite{Bergstra2012}.  In Bayesian optimization, one treats the validation error as an unknown function of the hyperparameters, say
\begin{equation}
f(\alpha,\beta,\dots) \;=\;\mathrm{ValidationError}(\alpha,\beta,\dots),
\end{equation}
and fits a Gaussian process or another surrogate model to sample evaluations.  Subsequent queries pick hyperparameter sets that balance exploration (sampling uncharted regions) and exploitation (refining near the current optimum) \cite{Shahriari2016}.  From a rigorous standpoint, each sample involves training and validating a reservoir, so the function \(f\) is black-box but can be optimized with limited evaluations, which is crucial when reservoir sizes are large or data sets are expensive.  While these methods do not yield closed-form solutions, they systematically reduce the guesswork in reservoir design, leading to more robust or problem-specific configurations.  The literature shows consistent improvement over naive grid-search or manual tuning, especially for tasks requiring fine control of memory depth or nonlinearity \cite{Feinberg2020}.

\paragraph{Observability, Controllability, and the Reservoir Framework.}
In classical control theory, a linear time-invariant (LTI) system \(\dot{\mathbf{x}}(t) = \mathbf{A}\,\mathbf{x}(t) + \mathbf{B}\,\mathbf{u}(t),\;\mathbf{y}(t)=\mathbf{C}\,\mathbf{x}(t)\) is said to be \emph{observable} if knowledge of \(\mathbf{y}(\cdot)\) over a time interval suffices to recover the internal state \(\mathbf{x}(t)\), and \emph{controllable} if one can drive \(\mathbf{x}(t)\) from any initial configuration to any target via \(\mathbf{u}(t)\).  Although reservoir computing uses random, fixed recurrent weights \(\mathbf{W}_{\mathrm{res}}\) and focuses on training output (readout) weights, the same notions can be adapted to partial analyses of whether different input signals can “move” the reservoir state through its high-dimensional manifold, and whether the readout can “see” enough of that manifold to produce accurate output.  In discrete time, writing 
\begin{equation}
\mathbf{x}(t+1) 
=
\mathbf{W}_{\mathrm{res}}\,\mathbf{x}(t) \;+\; \mathbf{W}_{\mathrm{in}}\,\mathbf{u}(t+1),
\quad
\mathbf{y}(t) = \mathbf{W}_{\mathrm{out}}\,\mathbf{x}(t),
\end{equation}
one can attempt to verify conditions reminiscent of observability and controllability for pairs \((\mathbf{W}_{\mathrm{res}}, \mathbf{W}_{\mathrm{in}})\) and \((\mathbf{W}_{\mathrm{res}}, \mathbf{W}_{\mathrm{out}})\).  For instance, if the linear part is viewed in isolation (assuming \(\mathbf{x}(t)\) remains in a linear regime), one can examine the \(\mathcal{O}\)-matrix \(\begin{bmatrix} \mathbf{W}_{\mathrm{out}} \\ \mathbf{W}_{\mathrm{out}} \,\mathbf{W}_{\mathrm{res}} \\ \mathbf{W}_{\mathrm{out}} \,\mathbf{W}_{\mathrm{res}}^2 \\ \vdots \end{bmatrix}\) for rank considerations, paralleling standard LTI observability criteria \cite{ZhouDoyleGlover1996}.  In the strongly nonlinear setting, these arguments extend to \emph{differential} or \emph{empirical} observability analyses.  Such viewpoint clarifies why large, randomly initialized reservoirs often succeed in capturing complex temporal transformations: with high probability, random connectivity confers a generic near-full-rank condition, thereby yielding approximate controllability from input and approximate observability through a linear readout.

\paragraph{Gershgorin’s Theorem and Bounding the Reservoir Spectrum.}
A frequently employed practical bound for \(\rho(\mathbf{W}_{\mathrm{res}})\) relies on Gershgorin’s circle theorem \cite{HornJohnson2013}.  Each eigenvalue \(\lambda\) of \(\mathbf{W}_{\mathrm{res}}\in\mathbb{R}^{N\times N}\) lies in at least one Gershgorin disc
\begin{equation}
D\ \!\Bigl(w_{ii},\; R_i\Bigr),
\end{equation}
where \(w_{ii}\) is the diagonal entry and \(R_i = \sum_{j\neq i} |w_{ij}|\).  Thus, 
\begin{equation}
|\lambda - w_{ii}| \;\le\; R_i, \quad \forall\,i=1,\dots,N.
\end{equation}
If each row-sum of absolute off-diagonal entries is small, the discs remain close to the real axis and near zero, implying \(\rho(\mathbf{W}_{\mathrm{res}})\) is small.  This is often used in random initialization heuristics, e.g.\ choosing i.i.d.\ entries of \(\mathbf{W}_{\mathrm{res}}\) with mean zero, ensuring that the average row-sum \(\mathbb{E}[R_i]\) is controlled.  While Gershgorin’s theorem can be loose relative to exact eigenvalue computation, it provides a straightforward and computationally cheap way to check whether \(\rho(\mathbf{W}_{\mathrm{res}})\) is likely under an upper bound.  In reservoir practice, one might then apply a uniform scaling factor so that \(\rho(\mathbf{W}_{\mathrm{res}})\approx \alpha<1\), ensuring stability-like conditions for the Echo State Property.

\paragraph{Extended and Bidirectional Reservoirs.}
While most echo state networks are forward-driven (\(\mathbf{u}\) influences \(\mathbf{x}\), which in turn influences \(\mathbf{y}\)), some tasks benefit from allowing feedback from the output back into the reservoir, or from constructing “bidirectional” updates.  A generalized model might read
\begin{equation}
\mathbf{x}(t+1)
=
F\!\bigl(\mathbf{W}_{\mathrm{res}}\;\mathbf{x}(t)\;+\;\mathbf{W}_{\mathrm{in}}\;\mathbf{u}(t+1)\;+\;\mathbf{W}_{\mathrm{fb}}\;\mathbf{y}(t)\bigr),
\end{equation}
where \(\mathbf{W}_{\mathrm{fb}}\) is the feedback matrix \cite{jaeger2001echo}.  Studying the spectral properties of
\begin{equation}
\widetilde{\mathbf{W}}_{\mathrm{res}}
\;=\;
\begin{bmatrix}
\mathbf{W}_{\mathrm{res}} & \mathbf{W}_{\mathrm{fb}} \\
\mathbf{0} & \mathbf{0}
\end{bmatrix}
\end{equation}
in a larger augmented state space helps clarify how output feedback might push the system toward or away from stable echo states, depending on \(\rho\bigl(\widetilde{\mathbf{W}}_{\mathrm{res}}\bigr)\).  A rigorous approach demands analyzing the closed-loop or \emph{autonomous} modes that arise once the readout is fed back.  Some advanced works consider “bidirectional” architectures where the input, state, and output updates are all interdependent, effectively forming a larger, partially random dynamical system that still retains a contraction property if appropriately scaled.  This underscores how the spectral radius argument extends to augmented matrices, and how feedback gains must remain within certain bounds to preserve fading memory.

\paragraph{Polynomial and Volterra Expansions as Reservoir Features.}
From a functional approximation perspective, an ESN can be seen as implementing a wide class of \emph{Volterra-like} expansions of the input \cite{Boyd1985}, where nonlinear and delayed versions of \(\mathbf{u}(t)\) are generated by the recurrent dynamics and then linearly combined by \(\mathbf{W}_{\mathrm{out}}\).  A direct analogy is:
\begin{equation}
y(t)
\;\approx\;
\sum_{\ell=0}^D \;\sum_{k_1,\ldots,k_\ell} 
c_{k_1,\ldots,k_\ell}\,
\prod_{r=1}^\ell u\bigl(t - k_r\bigr),
\end{equation}
which is a polynomial expansion in delayed inputs.  The reservoir effectively “folds” these polynomials into random combinations inside \(\mathbf{x}(t)\).  Some authors show that a large enough reservoir with certain connectivity approximates such expansions up to an error that can be made arbitrarily small, proving universality for any finite-memory filter \cite{grigoryeva2018echo}.  Formally, one might compare the reservoir state \(\mathbf{x}(t)\) to the partial sums of a Volterra series, revealing how each dimension can store a product of delayed inputs.  This viewpoint can be made rigorous by bounding the projection errors and ensuring that for random \(\mathbf{W}_{\mathrm{res}}\), one obtains a near-orthogonal basis spanning the relevant function space.

\paragraph{Neural Dynamical Systems and Conceptor Theory.}
A more recent development in reservoir computing is the notion of \emph{conceptors}, introduced by Jaeger \cite{Jaeger2014}, to interpolate or switch among multiple dynamical patterns stored in a single ESN.  At the linear algebraic level, a conceptor \(\mathbf{C}\in \mathbb{R}^{N\times N}\) is a matrix that acts on the reservoir state \(\mathbf{x}(t)\) to filter or “gate” the dominant attractor modes.  One typical choice is a \emph{ridge regression} definition of \(\mathbf{C}\), solved by minimizing
\begin{equation}
\|\mathbf{x}(t) - \mathbf{C}\,\mathbf{x}(t)\|^2
\;+\;
\alpha\,\|\mathbf{C}\|^2_F,
\end{equation}
over time and subject to constraints ensuring \(\mathbf{C}\) is symmetric idempotent.  This leads to
\begin{equation}
\mathbf{C}
=
\bigl(\mathbf{R} + \alpha\,\mathbf{I}\bigr)^{-1}\,\mathbf{R},
\quad
\mathbf{R} 
=
\frac{1}{T}\sum_{t=1}^T \mathbf{x}(t)\,\mathbf{x}(t)^T.
\end{equation}
Mathematically, \(\mathbf{C}\) becomes a projection operator in \(\mathbb{R}^N\) that “selects” the subspace characteristic of a particular pattern.  In advanced use cases, multiple conceptors \(\{\mathbf{C}_1,\dots,\mathbf{C}_k\}\) can be combined or sequentially applied to morph from one attractor to another.  The underlying linear-algebraic formalism, coupled with the reservoir’s nonlinear embeddings, offers a powerful mechanism for controlling or storing multiple time-series patterns in a single ESN, bridging ideas from principal component analysis, projectors in Hilbert spaces, and state-space dynamical control.

\end{document}